\documentclass[sigconf]{aamas} 

\usepackage{balance} 
\usepackage{natbib}
\usepackage{soul}
\usepackage{caption}
\usepackage{graphicx}
\usepackage{amsmath}
\usepackage{amsthm}
\usepackage{algorithm}
\usepackage{algorithmicx}
\usepackage{subfigure}
\usepackage{textcomp}
\usepackage{algpseudocode}
\usepackage{enumitem}
\usepackage{mathtools}
\usepackage{hyperref}
\usepackage{booktabs}
\usepackage{xcolor}
\usepackage{diagbox}
\usepackage{tikz}
\usepackage{lipsum}
\usepackage{epsfig,bm,dsfont}
\DeclareFixedFont{\myfont}{OT1}{ptm}{m}{n}{8pt}
\DeclareFixedFont{\myfontb}{OT1}{ptm}{bx}{n}{8pt}

\definecolor{DSgray}{cmyk}{0,1,0,0}

\newtheorem{myThm}{Theorem}

\DeclareMathOperator*{\argmin}{arg\,min}

\newcommand{\qw}{\mathbb{E}}

\newcommand{\citethm}[1]{Theorem \ref{#1}}

\newcommand{\citeeq}[1]{Equation \eqref{#1}}

\newtheorem{definition}{Definition}

\newtheorem{assumption}{Assumption}
\newtheorem*{lemma*}{Lemma}

\newtheorem*{theorem*}{Theorem}

\setcopyright{ifaamas}
\acmConference[AAMAS '21]{Proc.\@ of the 20th International Conference on Autonomous Agents and Multiagent Systems (AAMAS 2021)}{May 3--7, 2021}{Online}{U.~Endriss, A.~Now\'{e}, F.~Dignum, A.~Lomuscio (eds.)}
\copyrightyear{2021}
\acmYear{2021}
\acmDOI{}
\acmPrice{}
\acmISBN{}

\acmSubmissionID{14}

\title[AAMAS-2021 Formatting Instructions]{Imitation Learning from Pixel-Level Demonstrations by HashReward}

\author{Xin-Qiang Cai, Yao-Xiang Ding, Yuan Jiang, Zhi-Hua Zhou}
\affiliation{
  \department{National Key Laboratory for Novel Software Technology}
  \institution{Nanjing University}}
\email{{caixq, dingyx, jiangy, zhouzh}@lamda.nju.edu.cn}

\begin{abstract}
One of the key issues for imitation learning lies in making policy learned from limited samples to generalize well in the whole state-action space. This problem is much more severe in high-dimensional state environments, such as game playing with raw pixel inputs. Under this situation, even state-of-the-art adversary-based imitation learning algorithms fail. Through empirical studies, we find that the main cause lies in the failure of training a powerful discriminator to generate meaningful rewards in high-dimensional environments. Although it seems that dimensionality reduction can help, a straightforward application of off-the-shelf methods cannot achieve good performance. In this work, we show in theory that the balance between dimensionality reduction and discriminative training is essential for effective learning. To achieve this target, we propose HashReward, which utilizes the idea of supervised hashing to realize such an ideal balance. Experimental results show that HashReward could outperform state-of-the-art methods for a large gap under the challenging high-dimensional environments.
\end{abstract}


\newcommand{\BibTeX}{\rm B\kern-.05em{\sc i\kern-.025em b}\kern-.08em\TeX}

\begin{document}


\pagestyle{fancy}
\fancyhead{}


\maketitle 


\section{Introduction}
\label{sec:introduction}
In recent years, reinforcement learning (RL) has achieved a great breakthrough in many domains including robot controlling and game playing 
\citep{DBLP:journals/nature/MnihKSRVBGRFOPB15,DBLP:conf/icml/SchulmanLAJM15}. In spite of remarkable success, there are two main issues unsolved. First, in common situations, RL algorithms rely so much on the well-specified reward functions and exploration strategies, which require delicate designs in many complex problems. Second, the sample and computational complexities for practical RL algorithms are usually large, making it an unacceptable choice in solving many practical problems. On the other hand, imitation learning (IL), aiming at learning a good policy from demonstrations, enables the possibility of sample efficient policy learning without the need of designing rewarding and exploration strategies by hand. 

Nevertheless, since collecting expert demonstrations is costly, the number of expert trajectories for IL is usually limited in practice. This may significantly increase the risks of over-fitting. One of the key ideas to improve generalization performance is to learn a proper reward function from expert demonstrations. The learner can get high rewards only when it generates behaviors similar to the demonstrations.
Guided by such rewards, the learner is encouraged to mine out the expert's policy by RL algorithms, instead of directly performing behavior cloning \citep{DBLP:reference/ml/AbbeelN10,DBLP:conf/nips/HoE16}.
This idea motivates the emergence of adversary-based IL. Under this family of approaches, the policy and the adversarial discriminator are jointly trained during learning, in order to force the learner to minimize the discrepancy between the generated distribution and the demonstrated data. One representative approaches among them, generative adversarial imitation learning (GAIL) \citep{DBLP:conf/nips/HoE16}, achieves state-of-the-art performance in many tasks with relatively low-dimensional state spaces. However, it has been verified that the performance of such approaches, even deep neural network-based GAIL, degenerates seriously in tasks with high-dimensional state spaces~\citep{DBLP:journals/corr/abs-1810-10593,yu2020intrinsic}. On the other hand, it is well known that deep RL algorithms 
can achieve even superhuman performance under these domains \citep{DBLP:journals/nature/MnihKSRVBGRFOPB15}. 

In this work, we propose thorough studies towards understanding why adversary-based algorithms fail in high-dimensional IL environments. For obtaining a good policy, the discriminator plays an important role. The cause for performance degeneration in high-dimensional space can be interpreted by the hardness of properly dealing with the discrimination-rewarding trade-off in the discriminator learning process. Because the number of demonstration samples is limited, under high-dimensional state space, the learned discriminator may bias towards the discrimination side, providing meaningless rewarding signals for policy learning. Experimental observations and theoretical results are provided to support the above interpretations: They disclose that an ideal balance between performing dimensionality reduction (DR) and discriminative training is essential for learning.

Based on the above findings, we study a practical approach to enhance discriminator learning. It is non-trivial to guarantee that discriminative information in high-dimensional space can be preserved after performing DR. This suggests that directly utilizing off-the-shelf unsupervised DR algorithms, such as unsupervised autoencoder and hashing, is not the right choice because they make the processes of DR and discriminator learning totally isolated. As a result, essential information for discrimination can be heavily lost due to DR, making discriminator training biased towards discrimination or rewarding only. To address this issue, we propose a novel IL algorithm named HashReward, which utilizes a supervised hashing strategy to incorporate DR and discriminator training into a unified procedure. Experiments show that HashReward achieves significant improvement comparing to other state-of-the-art IL approaches under high-dimensional environments, which are broadly recognized to be challenging for IL algorithms.

The rest of the paper is organized as follows. Section~\ref{sec:related-work} discusses related work. Section~\ref{sec:problem-formulation} provides preliminaries. Section~\ref{sec:high-dim-imitation} demonstrates how to solve the underlying discrimination-rewarding trade-off problem for adversary-based IL methods, and then introduces HashReward algorithm. Section~\ref{sec:experiment-results} reports the experimental setup and results. Finally, Section~\ref{sec:conclusion} concludes the paper.

\section{Related Work}
\label{sec:related-work}
Recently, adversary-based methods have achieved great success in a wide range of IL scenarios, including robotic control~\citep{DBLP:conf/nips/HoE16}, game playing~\citep{DBLP:conf/aaai/HesterVPLSPHQSO18}, and simulating environments~\citep{DBLP:conf/aaai/Shi0DCZ19}.
 As the outputs from the discriminator are utilized as reward signals for the learner, they can be treated as a generalization of the family of IRL approaches. Among them, GAIL \citep{DBLP:conf/nips/HoE16} achieves state-of-the-art performance in handling low-dimensional IL problems, while its performance degenerates significantly in high-dimensional environments as discussed in Section~\ref{sec:introduction}. Though effective high-dimensional IL is challenging, there are a few works related to this topic, which assume additional signals besides provided demonstrations. ~\citep{DBLP:conf/aaai/HesterVPLSPHQSO18} and~\citep{DBLP:conf/nips/AytarPBP0F18} augmented RL to learn from both environment rewards and expert demonstrations, meanwhile~\citep{DBLP:conf/nips/ChristianoLBMLA17},~\citep{DBLP:conf/icml/BrownGNN19} and~\citep{DBLP:conf/nips/IbarzLPILA18} utilized human preference to enhance IL. In comparison, we consider pure IL without using additional environmental signals, which is more challenging. There are some related works proposed for this scenario, e.g., CNN-AIRL~\citep{DBLP:journals/corr/abs-1810-10593}, which uses an adversarial IRL method to play Atari game \textit{Enduro} with pixels inputs. But it utilizes autoencoder as the DR method, whose unsupervised information is not enough illustrated in the experiments. D-REX~\citep{brown2019drex} improves~\citep{DBLP:conf/icml/BrownGNN19} by learning a behavior cloning model to generate the ranked samples, and has achieved promising performance under the IL problems with suboptimal demonstrations. But in order to train the behavior cloning model efficiently, D-REX requires samples from quite different experts, which are not available in the settings of most IL tasks. GIRIL~\citep{yu2020intrinsic} is a non-GAIL method, and deals with the high-dimensional IL problem from a different angle than DR, which encodes action signals into VAE and utilizes the mechanism of curiosity to produce reward signals.
  But as will be discussed in subsequent sections, appropriate dimensionality reduction is inevitable for solving high-dimensional IL problems. VAIL~\citep{DBLP:conf/iclr/PengKTAL19} is the most closely related work to ours. It improves GAIL by employing information-theoretic regularization to learn a better feature representation as the input to the discriminator, and successfully used image feature in a continuous domain. Nevertheless, VAIL does not utilize explicit supervised loss in DR, which is crucial as shown in our experiments. In summary, the high-dimensional IL problem remains challenging for existing IL approaches.

Several recent works \citep{DBLP:journals/jcss/StrehlL08,DBLP:journals/corr/BlundellUPLRLRW16,DBLP:conf/nips/TangHFSCDSTA17,DBLP:conf/ijcai/YinCP18} have shown that the latent hashing features with unsupervised information learned by autoencoder can help address the exploration issue in challenging RL problems. Through hashing, a high-dimensional state is effectively discretized to make similar states mapping into the same hashing code~\citep{DBLP:conf/aaai/JiangL18}, leading to the convenience for utilizing counting-based exploration. Though unsupervised hashing seems to be promising for improving IL, we observe that the performance is not that satisfactory in the experiments. This phenomenon is reasonable since direct unsupervised hashing may lead to the risk of losing discriminative information in the original state space. This motivates us to propose HashReward to address this issue.

\section{Preliminaries}
\label{sec:problem-formulation}
In policy learning problems, a Markov Decision Process (MDP) can be represented by a tuple $\langle \mathcal{S}, \mathcal{A}, \mathcal{P}, \gamma, r, T \rangle$, in which $\mathcal{S}$ denotes the set of states, $\mathcal{A}$ denotes the set of actions, $\mathcal{P}: \mathcal{S} \times \mathcal{A} \times \mathcal{S} \to \mathbb{R}$ denotes the transition probability distributions of the state and action pairs, $\gamma \in (0, 1]$ denotes the discount factor, $r: \mathcal{S} \to \mathbb{R}$ denotes the reward function, $S_0: \mathcal{S} \to \mathbb{R}$ denotes the initial state distribution and $T$ denotes the horizon. The objective of RL is to learn a policy $\pi$ to maximize the expected total rewards $\mathbb{E}[\sum_{t = 0} ^ \infty \gamma^t r(s_t, a_t)]$ obtained by $\pi$.

Different from RL, in IL, the learner has no access to $r$. Instead, there are $m$ expert demonstrations $\{\tau_{E, 1}, \tau_{E, 2}, \dots, \tau_{E, m}\}$ available, where $\tau_{E, i}, i\in [m]$ is an expert trajectory (a series of state-action pairs) drawn independently from {\it expert's trajectory distribution} $\mu_{\pi_{\mathrm{E}}}$, induced by the expert's policy $\pi_{\mathrm{E}}$, initial state distribution $S_0$ and the transition probability distribution $\mathcal P$. The goal of the learner is to generate $\pi_{\mathrm{G}}$ such that the induced {\it learner's trajectory distribution} $\mu_{\pi_{\mathrm{G}}}$ matches $\mu_{\pi_{\mathrm{E}}}$.

Instead of directly minimizing the discrepancy between trajectory distributions to solve IL problem, existing adversary-based methods turn to the equivalent goal of minimizing the distance between the learner and expert occupancy measures $d(\rho_{\pi_{\mathrm{G}}}, \rho_{\pi_{\mathrm{E}}})$, where $\rho_\pi: \mathcal{S} \times \mathcal{A} \to \mathbb{R}$ is defined as $\rho_\pi(s, a) = \pi(a|s)\sum_{t = 0} ^ \infty \gamma ^ t Pr(s_t = s|\pi)$. Based on the idea of generative adversarial training, they perform policy learning by solving a min-max optimization problem, i.e. $\min_{\pi_{\mathrm{G}}} \max_D D(\rho_{\pi_{\mathrm{E}}}, \rho_{\pi_{\mathrm{G}}})$, in which $D$ is the discriminator. Among these approaches, GAIL~\citep{DBLP:conf/nips/HoE16} achieves state-of-the-art performance in many task environments. The objective of GAIL is 
\begin{equation}
  \min_{\pi_{\mathrm{G}}} \max_D \mathbb{E}_{\rho \sim \rho_{\pi_{\mathrm{E}}}} [\log D(\rho)] + \mathbb{E}_{\rho \sim \rho_{\pi_{\mathrm{G}}}} [\log(1 - D(\rho))],
\label{eq:gan_loss}
\end{equation}
where the discriminator $D : \mathcal{S} \times \mathcal{A} \to [0, 1]$ has the formulation of a classifier trying to discriminate state-action pairs generated by the learner and the expert. It is proved that by GAIL, the learned $\pi_{\mathrm{G}}$ can minimize the regularized version of Jensen-Shannon divergence, i.e.,
\begin{equation}
\label{eq:gail-purpose}
  \pi_{\mathrm{G}} = \argmin_{\pi \in \Pi} -\mathbb{H}(\pi) + d_{JS}(\rho_\pi, \rho_{\pi_{\mathrm{E}}}),
\end{equation}
where $\Pi$ denotes the policy set, $d_{JS}(\rho_\pi, \rho_{\pi_{\mathrm{E}}})$ is the Jensen-Shannon divergence between $\rho_\pi$ and $\rho_{\pi_{\mathrm{E}}}$, and $\mathbb{H(\pi)}$ is the causal entropy used as a policy regularizer. GAIL solves \citeeq{eq:gan_loss} by alternatively taking a gradient ascent step to train the discriminator $D$ and a minimization step to learn policy $\pi_{\mathrm{G}}$ based on off-the-shelf RL algorithm which utilizes $-\log D(s, a)$ as the pseudo reward function. 

\section{High-Dimensional Imitation Learning by HashReward}
\label{sec:high-dim-imitation}
As learning an effective discriminator plays an essential rule in adversary-based methods, in this section, first we analyze the reason why existing state-of-the-art adversary-based methods fail by proposing a generalization bound about learning a discriminator in IL scenarios. Inspired by the theoretical conclusion, then we propose HashReward to solve the high-dimensional IL problem.
\subsection{Balancing the Discrimination-Rewarding Trade-Off}
\label{subsec:why-gail-fails}
Intuitively, a well-trained discriminator should balance the following two capabilities. On the one hand, the discriminator should be powerful enough, in order to generate negative rewards on trajectories that are significantly different from the expert's demonstrations. On the other hand, the discriminator should not be too strong to discriminate trajectories that are sufficiently similar to the expert's demonstrations. This will ensure that positive rewards are generated on these good trajectories, and encourage the policy to improve in the right direction. 
We identify this problem as the {\it discrimination-rewarding trade-off}.
The reward curves on Atari game \textit{Qbert} in Figure~\ref{fig:curve-pseudo} evince that the deficiency of 
state-of-the-art adversary-based methods (i.e., GAIL) in high-dimensional environments is due to the tendency of learning too powerful discriminators. Intuitively, this is due to the curse of dimensionality: We usually cannot collect sufficient expert's demonstration data to meet the demand of high-dimensional learning, leading to overfitting of discriminator training.
In such a case, dimensionality reduction (DR) alone is not hard to think of, but the key to finding out a solution is how to do the DR. Thus we start from a theoretical analysis motivated by the generalization theory of GAN \citep{DBLP:conf/iclr/Zhang0ZX018} over the trajectory learning.
Let $\hat\mu_{\pi_{\mathrm{E}, m}}$ be expert's empirical trajectory distribution obtained from $m$ expert trajectories $\tau_{E,i}, i\in[m]$, over the trajectory space $\mathcal T$. Generally, we assume that a feature transformation $\phi(\tau_{\mathrm{E}})$, which is a bijective mapping from $\mathcal T$ to another trajectory space $\mathcal T'$ exists. 
We can see that $\phi$ plays a crucial role in the following discussions. For simplicity, we assume that the learner directly minimizes the neural distance \cite{DBLP:conf/icml/Arora0LMZ17} over trajectory distributions {\it under the mapped feature space}, i.e,
\begin{equation}
   \min_{\pi_{\mathrm{G}}\in\mathcal G}[d_{\mathcal D'}(\hat{\mu}_{\pi_{\mathrm{E}, m}}, \mu_{\pi_{\mathrm{G}}})],
\label{eq:theory_opt}
\end{equation}
where 
\begin{equation}
\begin{aligned}
    d_{\mathcal D'}(\hat{\mu}_{\pi_{\mathrm{E}, m}}, \mu_{\pi_{\mathrm{G}}}) = \sup_{D\in\mathcal D'}&\{\mathbb E_{\tau_{\mathrm{E}}\sim \hat\mu_{\pi_{\mathrm{E}, m}}}[D(\phi(\tau_{\mathrm{E}}))] \\
    & - \mathbb E_{\tau_{\mathrm{G}}\sim \mu_{\pi_{\mathrm{G}}}}[D(\phi(\tau_{\mathrm{G}})]\},
\end{aligned}
\label{eq:dropt}
\end{equation}
in which $\mathcal G$ is the policy hypothesis space, and $\mathcal D'$ is the neural network based discriminator hypothesis space under trajectory space $\mathcal T'$. To proceed the analysis, we utilize $\mathcal D$ to denote the hypothesis space under the original trajectory space $\mathcal T$, which is different from $\mathcal D'$ only in input dimension. By solving \citeeq{eq:theory_opt}, we expect to obtain a $\pi_{\mathrm{G}}$ that minimizes the expected neural distance in the original trajectory space, i.e., $d_{\mathcal D}(\mu_{\pi_{\mathrm{E}}}, \mu_{\pi_{\mathrm{G}}}) = \sup_{D\in\mathcal D}\{\mathbb E_{\tau_{\mathrm{E}}\sim \mu_{\pi_{\mathrm{E}}}}[D(\phi(\tau_{\mathrm{E}}))] - \mathbb E_{\tau_{\mathrm{G}}\sim \mu_{\pi_{\mathrm{G}}}}[D(\phi({\tau}_G)]\}$ by minimizing $d_{\mathcal D'}(\hat{\mu}_{\pi_{\mathrm{E}, m}}, \mu_{\pi_{\mathrm{G}}})$, which would guarantee a good $\pi_{\mathrm{G}}$ when $\mathcal D$ is rich enough. In practice, the optimization in \citeeq{eq:theory_opt} may not be exactly solved, thus we utilize $\hat d_{\mathcal D'}, \hat\pi_{\mathrm{G}}$ to denote the resulted neural distance and policy after training. We further introduce the following assumptions.
\begin{assumption}
\label{assumption:D}
$\mathcal D'$ is a class of neural networks, whose definition could be found in the supplementary material. Furthermore, $\mathcal{D'}$ is even, i.e., $D \in \mathcal{D'}$ implies $-D \in \mathcal{D'}$. Meanwhile $\forall D \in \mathcal{D'}$, $\left\Vert D \right\Vert_\infty \leq \Delta$, in which $\left\Vert D \right\Vert_\infty = \sup_{\tau \in\mathcal T'}|D(\tau)|$.
\end{assumption}
\begin{assumption}
\label{assumption:eta}
$\hat d_{\mathcal D'}(\hat\mu_{\pi_{\mathrm{E}, m}}, \mu_{\hat\pi_{\mathrm{G}}})\leq \eta.$ 
\end{assumption}
%
Assumption~\ref{assumption:D} is easily satisfied by general neural network models.
Furthermore, if $\hat \pi_{\mathrm{G}}$ is sufficiently trained w.r.t. $\hat d_{\mathcal D'}$, then $\eta$ in Assumption~\ref{assumption:eta} is also small. 
We then have the following sample complexity result, whose proof is included in the supplementary material.

\begin{myThm}
\label{thm:generalization-bound}
Let $\Delta_1 = |d_{\mathcal D}(\mu_{\pi_{\mathrm{E}}}, \mu_{\hat\pi_{\mathrm{G}}}) - d_{\mathcal D'}(\mu_{\pi_{\mathrm{E}}}, \mu_{\hat\pi_{\mathrm{G}}})|$, $\Delta_2 = |\hat d_{\mathcal D'}(\hat\mu_{\pi_{\mathrm{E}, m}}, \mu_{\hat\pi_{\mathrm{G}}}) - d_{\mathcal D'}(\hat\mu_{\pi_{\mathrm{E}, m}}, \mu_{\hat\pi_{\mathrm{G}}})|.$  
Given expert trajectory data $X$ which consists of $m$ trajectories $\tau_{\pi_{\mathrm{E}}} \in\mathcal T$, if $m \geq 3\|\phi(\mathbf{X})\|_{\mathrm{F}}\mathcal R$, then with probability at least $1 - \delta$, we have
\begin{equation}
\label{eq:DR-generalization-bound}
\begin{aligned}
d_{\mathcal D}(\mu_{\pi_{\mathrm{E}}}, \mu_{\hat \pi_{\mathrm{G}}}) &\leq \Delta_1 + \Delta_2 + 6\Delta\sqrt{\frac{\log(2/\delta)}{2m}} \\ &+ \frac{24\left\Vert \phi(\mathbf{X})\right\Vert_\mathrm{F}\mathcal{R}}{m}(1 + \log\frac{m}{3\left\Vert \phi(\mathbf{X})\right\Vert_\mathrm{F}\mathcal{R}}) + \eta, \\ 
\end{aligned}
\end{equation}
where $\mathcal R$ denotes the spectral normalized complexity of $\mathcal D'$ as defined in the supplementary material.
\end{myThm}
The above result reveals the key factors to learn a good policy. First, see the sample complexity terms involving $m$. Under a properly chosen $\mathcal D'$ which ensures small $\mathcal R$, the key for sharpening the bound is to control the Frobenius norm of $\phi(\mathbf{X})$, which can be achieved by learning a good DR version of $\phi$. Furthermore, $\Delta_1$ measures the gap of optimal discriminator between $\mathcal D$ and $\mathcal D'$. This shows that $\phi$ should also be distance-preserving. These two observations show the advantage of learning a hashing function as $\phi$: It is well-known that hashing mappings usually perform DR with sparsity as well as distance-preserving property~\citep{DBLP:conf/vldb/GionisIM99}, thus fit for our needs desirably. We should also pay attention to $\Delta_2$, which measures the quality of discriminator training in $\mathcal D'$. We find this should be stressed not only for learning the discriminator under $\mathcal D'$ mapped {\it after $\phi$}, but also for {\it learning $\phi$ itself}. These observations from Theorem~\ref{thm:generalization-bound} motivate our HashReward approach.


\subsection{HashReward} 
\label{sec:hashreward}
According to \citethm{thm:generalization-bound}, learning $\phi$ in a proper way is essential to our task: A good $\phi$ should reduce the input norm ($\left\Vert \phi(\mathbf{X}) \right\Vert_\mathrm{F}$) meanwhile preserving discrimination properties ($\Delta_1$ and $\Delta_2$). This particularly shows that directly applying off-the-shelf unsupervised DR, like autoencoder or unsupervised hashing, could lead to unsatisfying results for risks of losing discriminative information.
To achieve a good balance, we propose HashReward, which is a novel adversary-based IL approach utilizing supervised hashing for learning effective discriminators to achieve the balance between discrimination and rewarding.
 The key idea lies in using unsupervised reconstructive information to learn the hashing code (reducing $\left\Vert \phi(\mathbf{X}) \right\Vert_\mathrm{F}$ and $\Delta_1$) as well as leading supervised discriminative information into the whole DR part (reducing $\Delta_2$), so that the hashing code could obtain the ability to represent the original high-dimensional states.
  To generate such effective representation, the network structure utilized for the autoencoder and discriminator training is illustrated in Figure~\ref{fig:architecture}. We utilize autoencoder to train 
  the hashing code that maintains reconstructive information of the original pixels. Meanwhile, the action signal is concatenated to the hashing code to formulate the input of the discriminator training. By this way, the supervised discriminative information is directly propagated back to learn hashing codes. The loss function for discriminator training is divided into two parts, i.e, 
\begin{figure*}[t]
    \centering
    \includegraphics[width=0.7\textwidth]{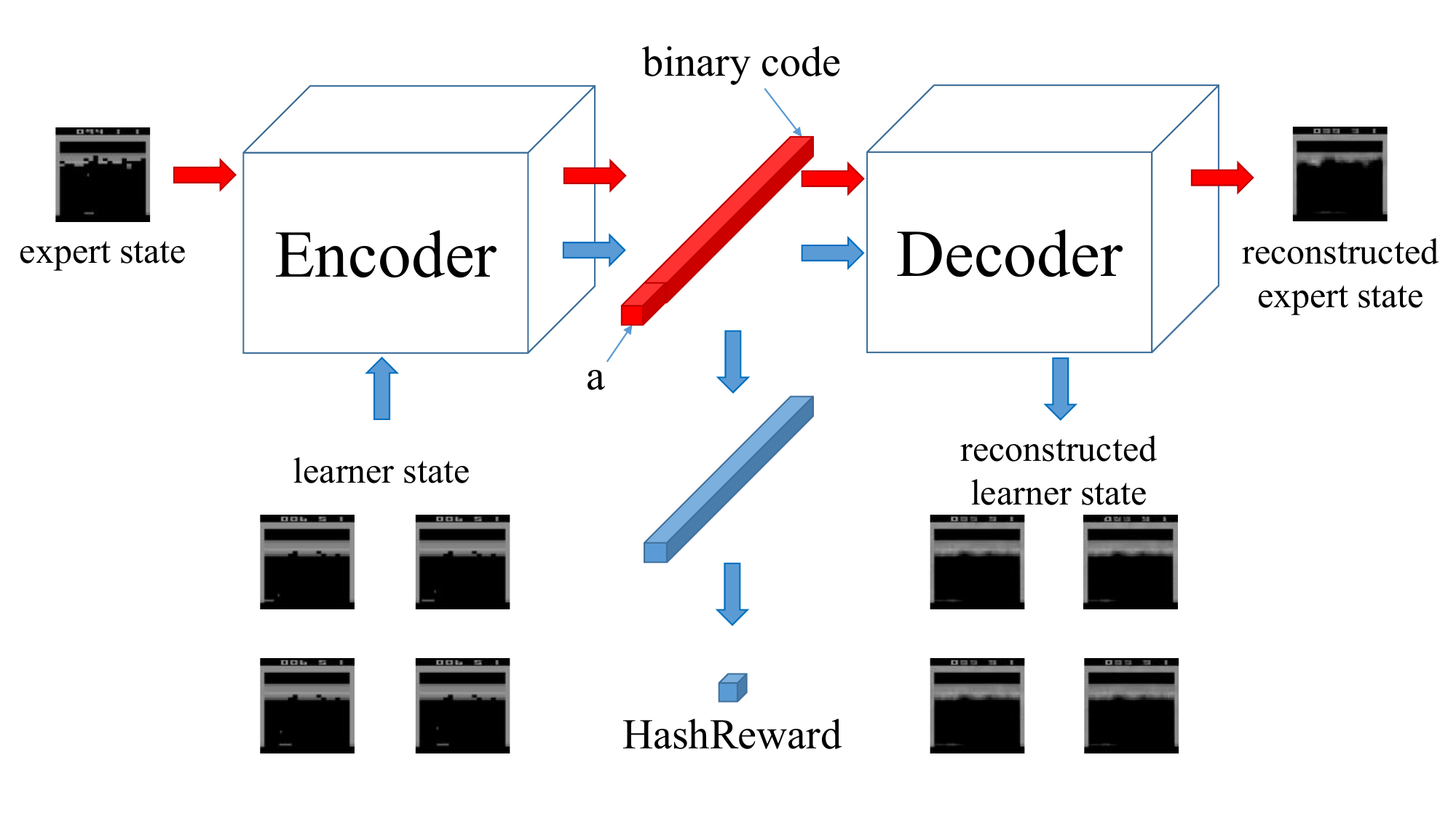}
    \caption{\small Illustration of the HashReward model architecture which contains two modules: the autoencoder module and the discriminator module. The red solid block represents the concatenation of the hashing code layer and the action signal $a$. The blue solid block represents hidden dense layers of the discriminator module.}
    \label{fig:architecture}
\end{figure*}
\begin{equation}
\label{eq:loss}
\begin{split}
  L = L_H + L_D,
\end{split}
\end{equation}
where $L_H$ denotes the hashing training loss, which propagates error for training the autoencoder and hashing code layer, and $L_D$ denotes the discriminator training loss, which is utilized for training the discriminator layers as well as enhancing the supervision of DR part. $L_D$ is similar to the inner maximization in \citeeq{eq:gan_loss}, except that the input state $s$ is replaced by the binary hashing code $b(s)$. Inspired by \citet{DBLP:conf/cvpr/Liu0SC16}, we define the hashing loss $L_H$ as 
\begin{algorithm}[!t]
    \caption{HashReward}
    \label{alg}
    \begin{algorithmic}[1]
        \Require Expert demonstrations $\tau_{\mathrm{E}} \sim \mu_{\pi_{\mathrm{E}}}$; Initialized learner's policy $\pi_{\mathrm{G},0}$.
        \State Pretrain autoencoder with samples from expert demonstrations and the random policy.
        \For {iteration $t = 1, 2, \cdots, T$}
          \State Utilize $\pi_{\mathrm{G},t-1}$ to generate learner's trajectories, i.e. $\tau_{\mathrm{G}} \sim \mu_{\pi_{\mathrm{G}}}$.
          \State Sample a mini-batch of state-action pairs $\{(s, a)\}_t$ from both $\tau_{\mathrm{G}}$ and $\tau_{\mathrm{E}}$.
          \State Update HashReward network by \citeeq{eq:loss} using $\{(s, a)\}_t$, then generate rewards $\hat r$ for all state-action pairs in $\{(s, a)\}_t$.
          \State $\pi_{\mathrm{G}, t-1} \rightarrow \pi_{\mathrm{G}, t}$ using $\hat r$ by RL update.
        \EndFor
    \end{algorithmic}
\end{algorithm}
\begin{equation}
\begin{aligned}
  L_H(  \{s_i  , y_i\}, \{s_j, y_j\}) 
   &= \left\Vert s_i-s'_i\right\Vert_2 ^ 2 + \left\Vert s_j-s'_j\right\Vert_2 ^ 2 \\
   &+  \lambda \big(\| 1 - \left\vert b(s_i)\right\vert \|_2^2 + \| 1 - \left\vert b(s_j)\right\vert \|_2^2\big) \\
   & + \frac{1}{2}\mathbb{I}(y_{ij}) \left\Vert  b(s_i) - b(s_j) \right\Vert_2 ^ 2
   \\ &+ \frac{1}{2}(1 - \mathbb{I}(y_{ij})) \max(2l - \left\Vert b(s_i) - b(s_j) \right\Vert_2 ^ 2, 0), 
\end{aligned}
\label{eq:loss_suphash}
\end{equation}
in which $l$ denotes the length of the hashing code and $\mathbb{I}(y_{ij})$ is the indicator function which takes 1 if $y_i$ equals to $y_j$, and 0 otherwise. In \citeeq{eq:loss_suphash}, $(s_i, y_i), (s_j, y_j)$ denote a pair of state-label instance. For one state-label instance $(s, y)$, we utilize $s$ to denote a state sampled from the learner's policy or a state from the trajectory generated by the expert. Furthermore, we utilize $y$ to indicate where $s$ is sampled, such that $y = 1$ if $s$ is sampled from the demonstration and $y = 0$ otherwise. The first two terms of $L_H$ represent the reconstruction error, making the reconstructed states $s'$ similar to the original states $s$. The next two terms (regularization terms weighted by $\lambda$) are used to enforce $b(s)$ to get close to binary values in $\{-1, 1\}$, where $b(s)$ is the logit output of the hashing layer. The last two terms in $L_H$ are essential for introducing the supervision into hashing code training. From these terms, the unbinarized hashing codes $b(s_i)$ and $b(s_j)$ of two states $s_i, s_j$ will get similar only when they have the same labels. By this way, the discriminative information is effectively propagated for the learning hashing representations. Overall, the output of HashReward network is $D(s, a) \in (0, 1)$, and we utilize $-logD(s, a)$ as the pseudo reward for the agent. This process in all experiments can be compared to the network architecture in the supplementary material.

\begin{table*}[!t]
\centering 
\footnotesize
\caption{The performance of each method on Atari after 10M timesteps. Boldface numbers indicate the best results. The state space is 84 $\times$ 84 $\times$ 4.}
\vspace{-2mm}
\label{tab:atari_results}
\resizebox{1\textwidth}{!}{
\begin{tabular}{c|c|ccccccccc}
\toprule
               & \textbf{Expert Reward} & \textbf{GAIL} & \textbf{VAIL} & \textbf{GIRIL}              & \textbf{GAIL-AE} & \textbf{GAIL-AE-Up} & \textbf{GAIL-UH} & \textbf{GAIL-UH-Up} & \textbf{HashReward-AE}         & \textbf{HashReward}           \\
\midrule
BeamRider      & 2139.20 $\pm$ 41.60                     & 854.47$\pm$220.63              & 615.63$\pm$258.67              & \textbf{2973.27$\pm$224.00} & 638.04$\pm$95.95                  & 2182.59$\pm$1424.69                  & 1412.91$\pm$230.91                & 2089.55$\pm$1232.42                  & 596.00$\pm$8.14                                 & 1613.68$\pm$203.83                             \\
Breakout       & 144.35 $\pm$ 29.27                      & 10.48$\pm$1.70                 & 24.32$\pm$2.60                 & 61.53$\pm$17.98                              & 17.32$\pm$6.83                    & 28.70$\pm$0.22                       & 1.04$\pm$0.52                     & 49.85$\pm$9.78                       & 32.10$\pm$2.84                                  & \textbf{67.73$\pm$13.77}      \\
Boxing         & 95.70 $\pm$ 2.63                        & 26.78$\pm$3.24                 & 2.47$\pm$1.55                  & -3.64$\pm$1.57                               & 26.05$\pm$19.83                   & -5.47$\pm$23.43                      & 0.59$\pm$1.17                     & -15.36$\pm$15.18                     & -0.72$\pm$1.15                                  & \textbf{84.71$\pm$2.13}       \\
BattleZone     & 23000.00 $\pm$ 2549.51                  & 11863.33$\pm$767.09            & 7566.67$\pm$1503.43            & 9070.00$\pm$2203.47                          & 11030.00$\pm$4734.64              & 9133.33$\pm$5255.45                  & 4670.00$\pm$1432.29               & 11400.00$\pm$3559.52                 & 7043.33$\pm$1728.01                             & \textbf{15623.33$\pm$278.61}  \\
ChopperCommand & 3135.00 $\pm$ 145.86                    & 1469.33$\pm$135.22             & 1190.00$\pm$37.50              & 604.67$\pm$52.93                             & 1151.67$\pm$17.46                 & 1148.00$\pm$74.69                    & 1144.33$\pm$446.19                & 995.33$\pm$362.83                    & 924.00$\pm$107.34                               & \textbf{1522.67$\pm$78.36}    \\
CrazyClimber   & 95245.00 $\pm$ 2477.39                  & 35451.33$\pm$1002.73           & 41170.67$\pm$5024.62           & 5020.00$\pm$599.03                           & 10321.00$\pm$1539.70              & 9590.00$\pm$5689.36                  & 4049.00$\pm$907.23                & 3159.33$\pm$2228.60                  & 44766.67$\pm$21591.68                           & \textbf{63076.00$\pm$1841.86} \\
Enduro         & 469.85 $\pm$ 18.21                      & 26.53$\pm$26.35                & 119.87$\pm$10.06               & 0.00$\pm$0.00                                & 70.02$\pm$70.66                   & 0.04$\pm$0.05                        & 0.00$\pm$0.00                     & 0.00$\pm$0.00                        & 209.07$\pm$19.05                                & \textbf{219.88$\pm$72.06}     \\
Kangaroo       & 4175.00 $\pm$ 94.21                     & 1695.67$\pm$87.96              & 1482.00$\pm$427.65             & 32.00$\pm$2.83                               & 935.00$\pm$63.90                  & 542.00$\pm$44.50                     & 80.00$\pm$94.36                   & 26.33$\pm$5.31                       & 1352.00$\pm$249.99                              & \textbf{1925.67$\pm$145.11}   \\
MsPacman       & 3163.00 $\pm$ 160.88                    & 298.00$\pm$22.56               & 1316.87$\pm$182.96             & 121.23$\pm$75.75                             & 711.43$\pm$9.19                   & 674.67$\pm$3.07                      & 655.00$\pm$41.36                  & 701.33$\pm$16.21                     & 726.93$\pm$42.11                                & \textbf{1463.07$\pm$52.68}    \\
Pong           & 21.00 $\pm$ 0.00                        & -18.40$\pm$0.29                & -18.78$\pm$0.18                & -20.14$\pm$0.07                              & -14.57$\pm$5.65                   & -20.90$\pm$0.15                      & -17.89$\pm$0.79                   & -17.59$\pm$0.26                      & -17.97$\pm$0.94                                 & \textbf{-5.25$\pm$4.79}       \\
Qbert          & 4750.00 $\pm$ 50.00                     & 3634.42$\pm$388.09             & 3260.62$\pm$209.12             & 1050.33$\pm$197.75                           & 1126.25$\pm$416.21                & 1303.62$\pm$33.38                    & 360.33$\pm$111.00                 & 287.17$\pm$67.09                     & 4000.50$\pm$327.93                              & \textbf{4553.58$\pm$155.14}   \\
Seaquest       & 1835.00 $\pm$ 23.56                     & 761.13$\pm$38.89               & 826.93$\pm$40.99               & 628.93$\pm$52.75                             & 1124.00$\pm$400.68                & 1151.33$\pm$628.84                   & 689.93$\pm$5.04                   & 664.80$\pm$10.47                     & 1339.80$\pm$188.92                              & \textbf{1400.73$\pm$67.41}    \\
SpaceInvaders  & 743.50 $\pm$ 26.03                      & 341.67$\pm$23.38               & 346.70$\pm$27.42               & \textbf{550.10$\pm$12.95}                             & 302.23$\pm$38.18                  & 318.75$\pm$36.39                     & 309.83$\pm$74.85                  & 512.32$\pm$49.89                     & 496.33$\pm$67.44                                & 546.48$\pm$42.09     \\
UpNDown        & 34200.50 $\pm$ 2083.35                  & 21086.17$\pm$848.61            & 31637.43$\pm$2301.67           & 24399.07$\pm$9663.36                         & 20252.53$\pm$10582.64             & 43972.33$\pm$1105.54                 & 4227.27$\pm$1264.39               & 8715.40$\pm$4807.56                  & \textbf{64648.43$\pm$17684.03} & 36989.70$\pm$12739.30                          \\
Zaxxon         & 11190.00 $\pm$ 490.82                   & 8130.33$\pm$626.00             & 8473.00$\pm$1507.70            & 2644.33$\pm$63.73                            & 8577.67$\pm$852.58                & 9062.00$\pm$1336.34                  & 1614.67$\pm$64.63                 & 1546.67$\pm$28.77                    & 9005.00$\pm$1404.70                             & \textbf{10068.33$\pm$410.69}                            \\
\bottomrule
\end{tabular}}
\end{table*}

\begin{table*}[!t]
\centering 
\footnotesize
\caption{The performance of each method on MuJoCo after 1M timesteps. The state space is 84 $\times$ 84 $\times$ 1.}
\vspace{-2mm}
\label{tab:mujoco_results}
\resizebox{1\textwidth}{!}{
\begin{tabular}{c|c|ccccccccc}
\toprule
                &  \textbf{Expert Reward}              &  \textbf{GAIL}                                        &  \textbf{VAIL}                     &  \textbf{GIRIL} &  \textbf{GAIL-AE}                   &  \textbf{GAIL-AE-Up}                                        &  \textbf{GAIL-UH}                 &  \textbf{GAIL-UH-Up}               &  \textbf{HashReward-AE}                                          &  \textbf{HashReward}              \\
\midrule
Humanoid        & 1029.15 $\pm$ 53.95    &    360.13$\pm$9.53    &  392.69$\pm$9.87      &  207.27$\pm$20.07     &  \textbf{435.95$\pm$7.52}      &  422.90$\pm$28.37     &  424.11$\pm$22.36     &  433.74$\pm$9.76      &  364.61$\pm$1.98      &  370.92$\pm$3.65      \\
HalfCheetah     & 1189.01 $\pm$ 140.28     &  -1198.28$\pm$344.95  &  -138.48$\pm$36.00    &  -280.27$\pm$159.74   &  -488.56$\pm$232.77   &  -706.58$\pm$339.33   &  -402.19$\pm$49.28    &  -420.65$\pm$54.38    &  -149.73$\pm$253.05   &  \textbf{-78.92$\pm$168.05}    \\
Hopper          & 2304.87 $\pm$ 287.45       &  496.31$\pm$294.86    &  2210.22$\pm$5.88     &  975.13$\pm$17.02     &  753.02$\pm$60.52     &  1149.54$\pm$309.80   &  968.16$\pm$16.20     &  943.70$\pm$149.26    &  2116.94$\pm$35.24    &  \textbf{2216.83$\pm$76.57}    \\
HumanoidStandup & 110722.99 $\pm$ 6360.51 &  67946.69$\pm$3690.98 &  77412.49$\pm$2948.07 &  81217.79$\pm$1013.14 &  79006.35$\pm$2857.10 &  75210.63$\pm$2109.55 &  68248.25$\pm$3586.59 &  70330.63$\pm$2901.27 &  82718.25$\pm$2706.36 &  \textbf{83423.17$\pm$1878.71} \\
Reacher         & -10.00 $\pm$ 2.83   &  -28.99$\pm$5.42      &  -74.98$\pm$13.82     &  \textbf{-13.53$\pm$1.31}      &  -18.79$\pm$1.01      &  -20.23$\pm$0.80      &  -15.54$\pm$0.36      &  -15.55$\pm$0.33      &  -100.10$\pm$25.82    &  -39.35$\pm$12.24                            \\
\bottomrule
\end{tabular}}
\end{table*}

\begin{figure*}[!t]
    \centering
    \subfigure[Breakout]{
        \includegraphics[ width=0.19\textwidth]{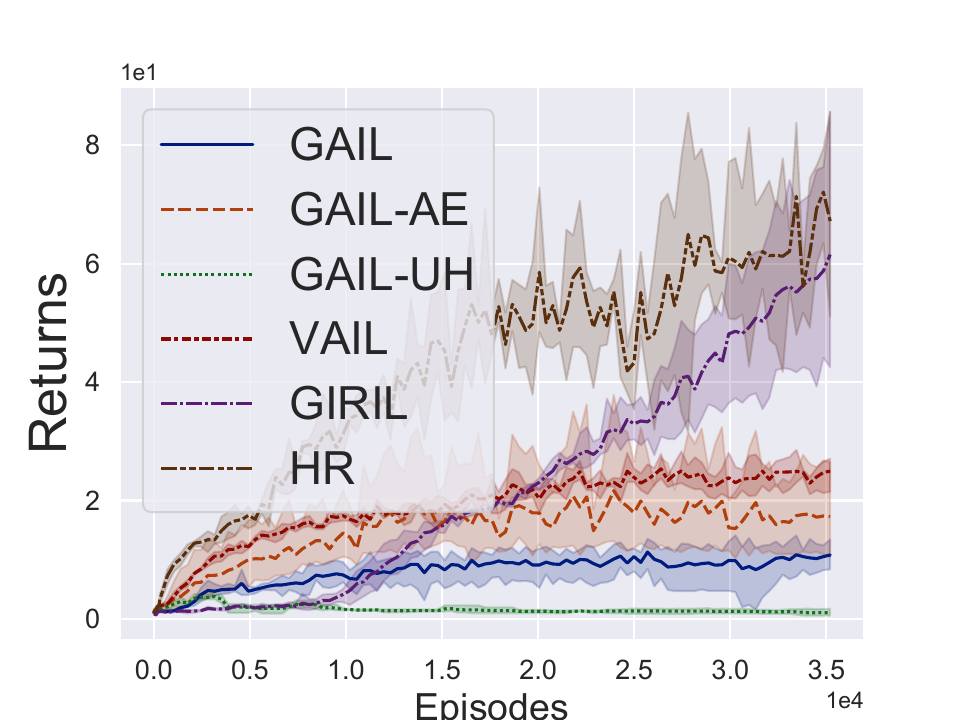}}
    \subfigure[BeamRider]{ 
        \includegraphics[ width=0.19\textwidth]{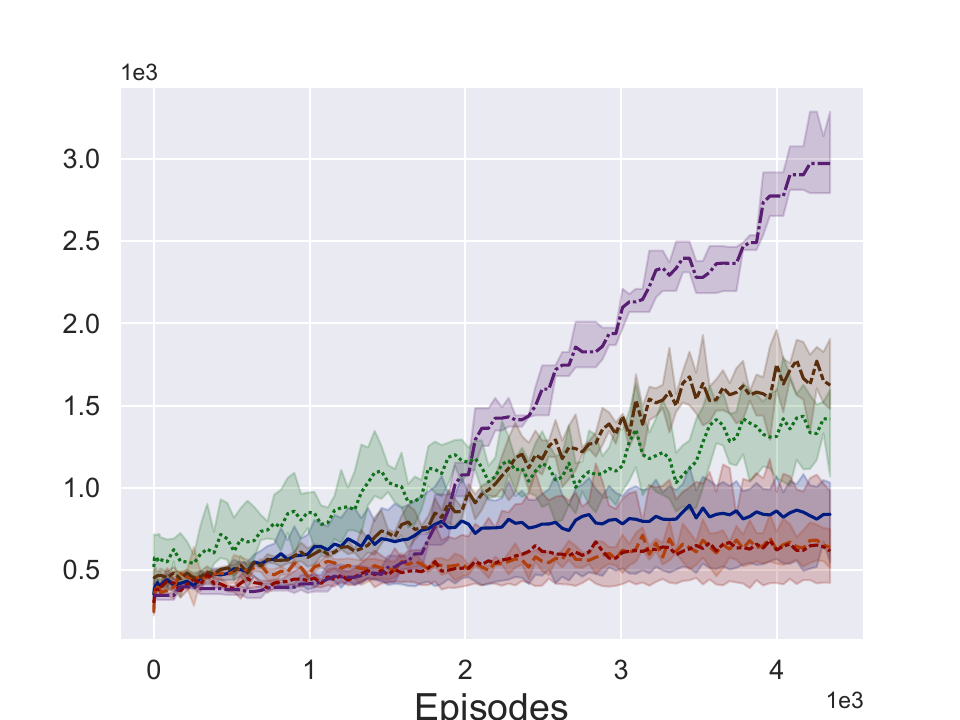}}
    \subfigure[Boxing]{ 
        \includegraphics[ width=0.19\textwidth]{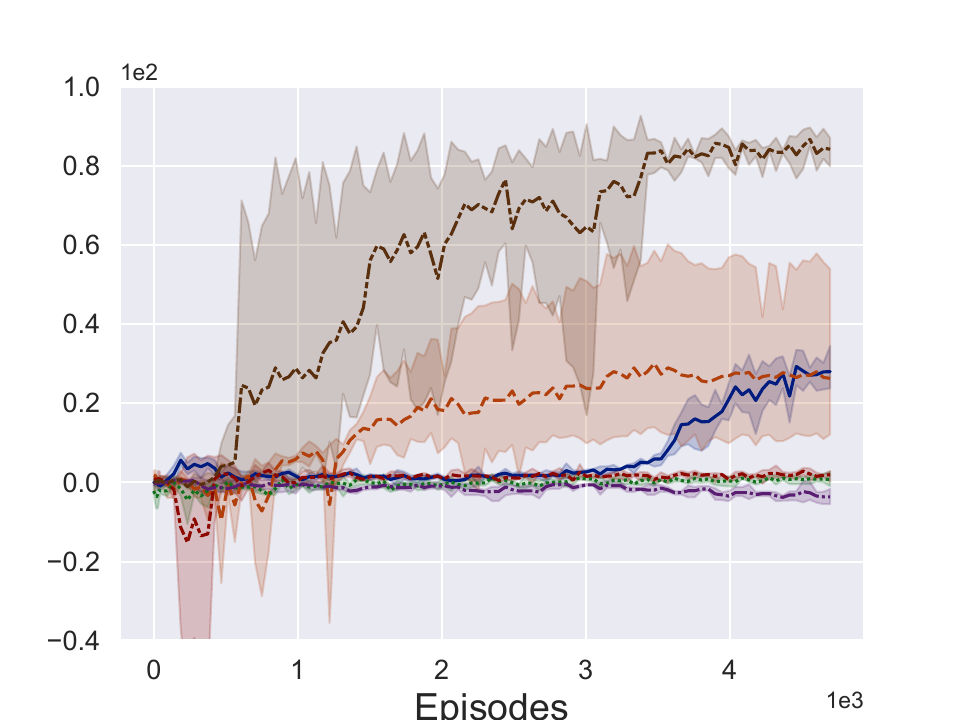}}
    \subfigure[BattleZone]{ 
        \includegraphics[ width=0.19\textwidth]{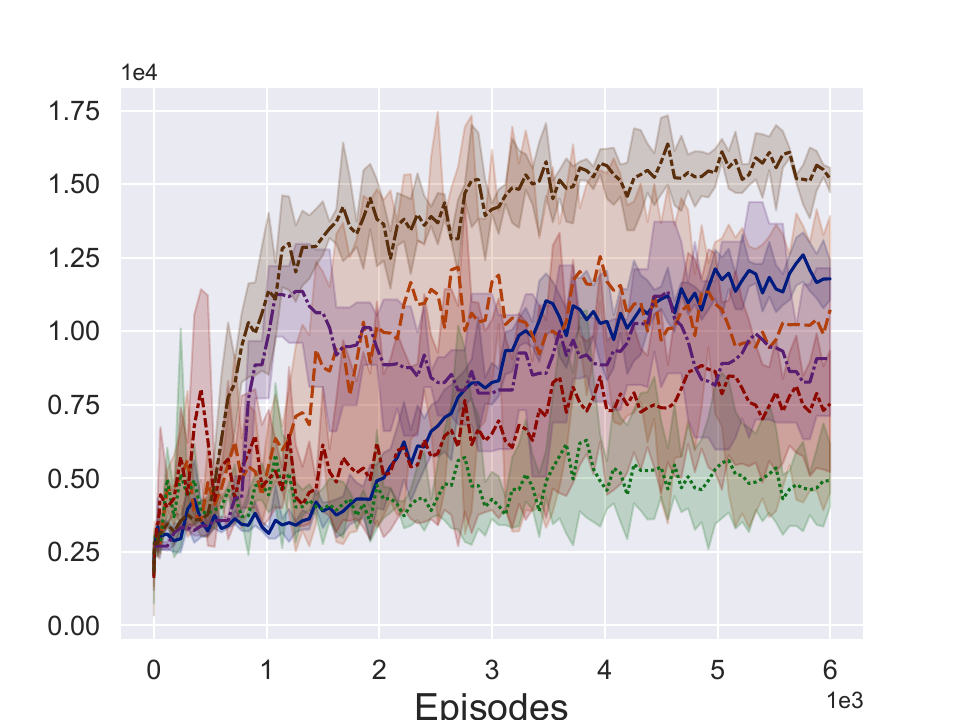}}
    \subfigure[ChopperCommand]{
        \includegraphics[ width=0.19\textwidth]{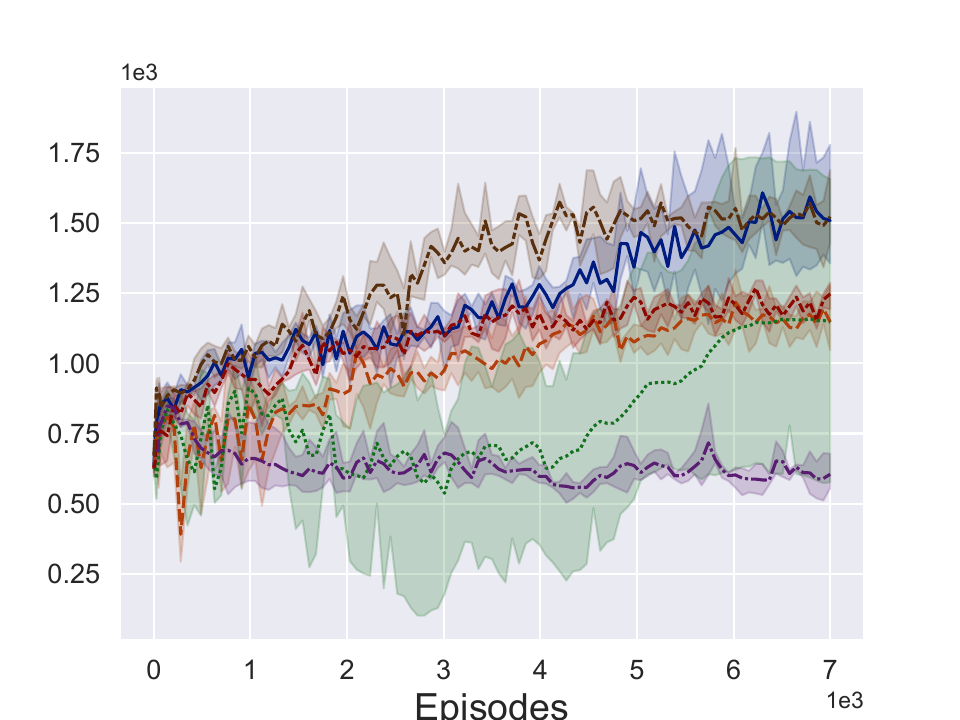}}
    \subfigure[CrazyClimber]{
        \includegraphics[ width=0.19\textwidth]{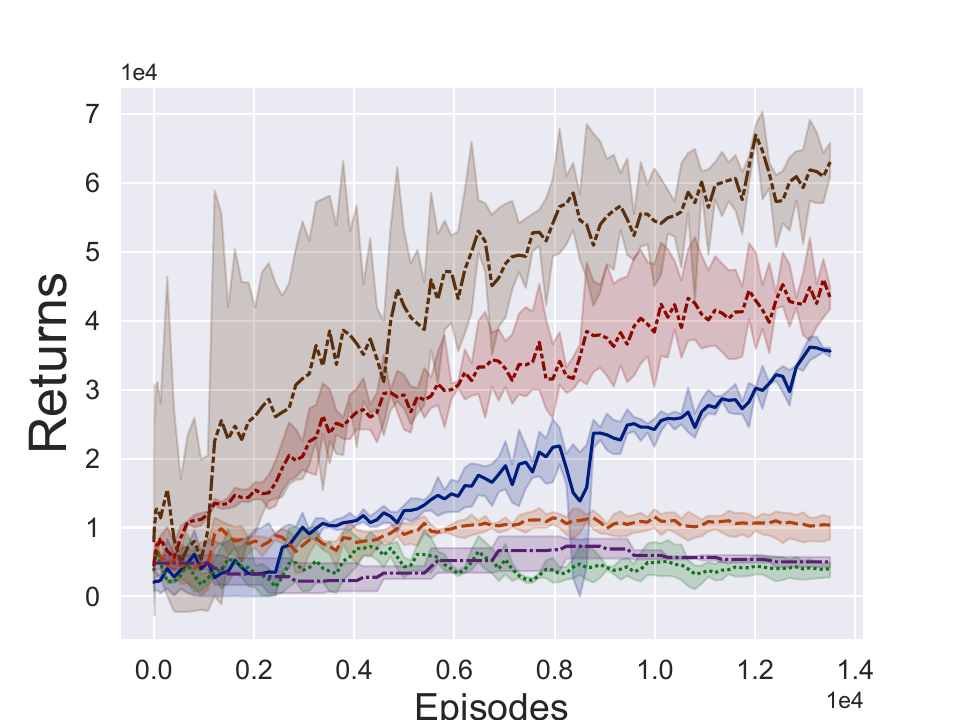}}
    \subfigure[Enduro]{
        \includegraphics[ width=0.19\textwidth]{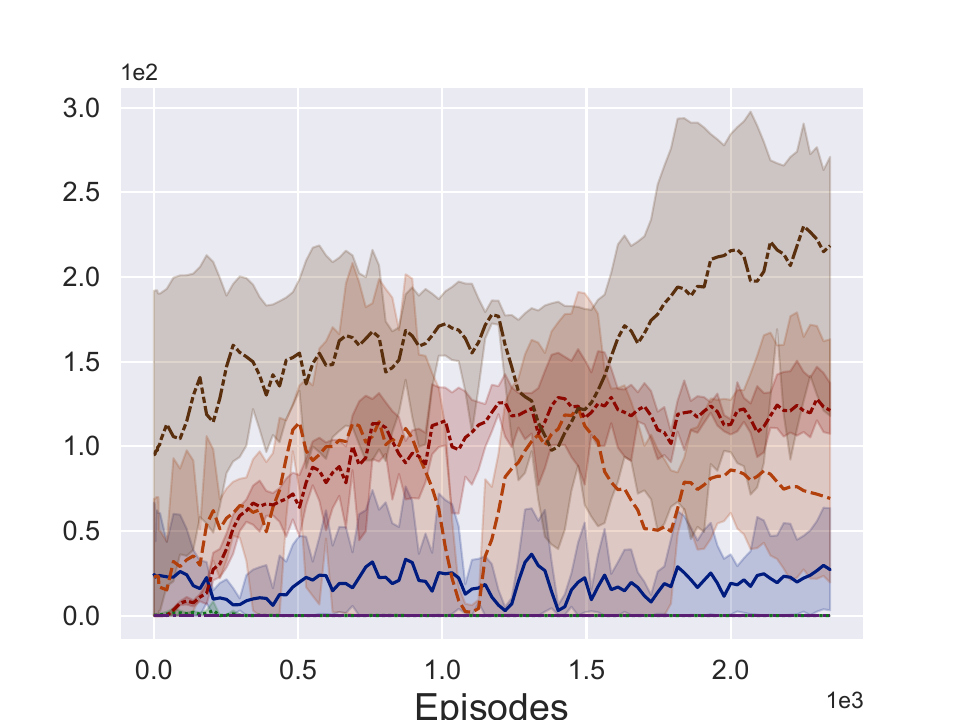}}
    \subfigure[Kangaroo]{
        \includegraphics[ width=0.19\textwidth]{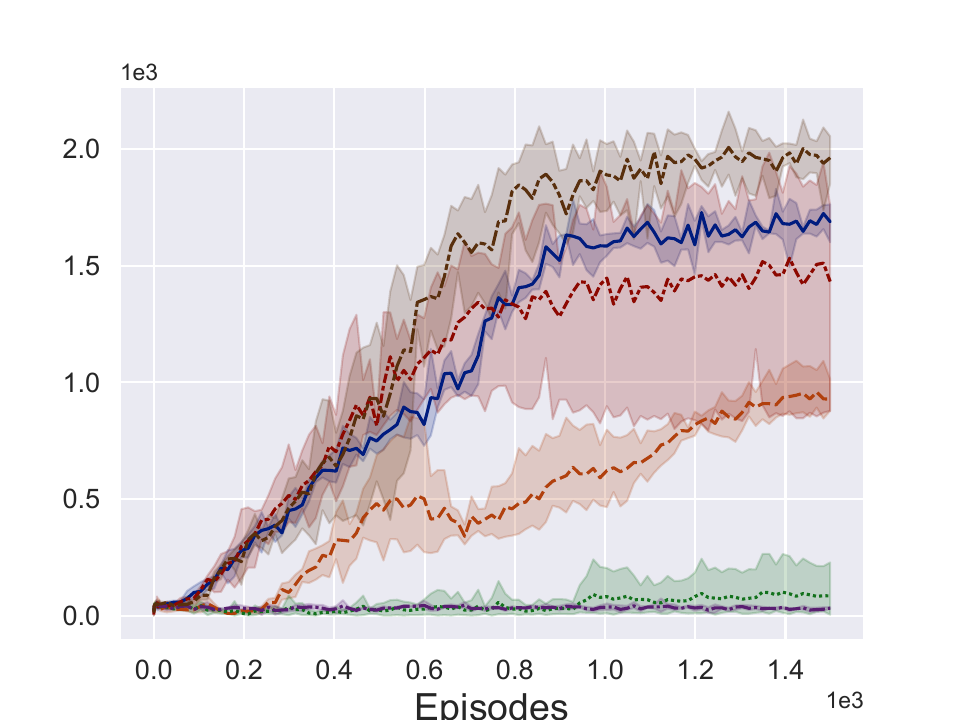}}
    \subfigure[MsPacman]{
        \includegraphics[ width=0.19\textwidth]{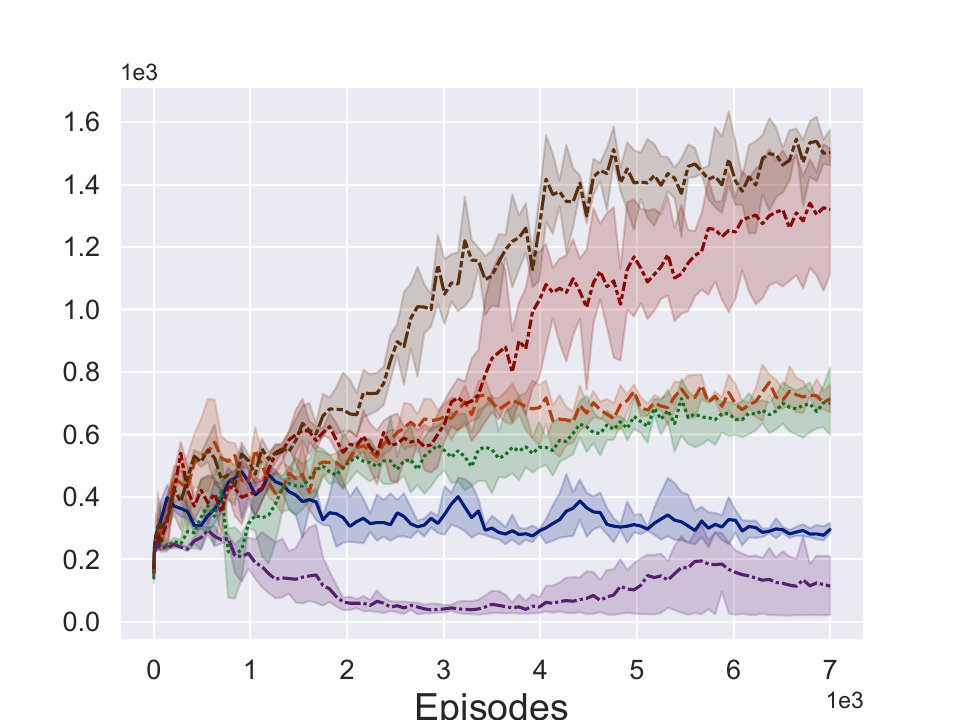}}
    \subfigure[Pong]{
        \includegraphics[ width=0.19\textwidth]{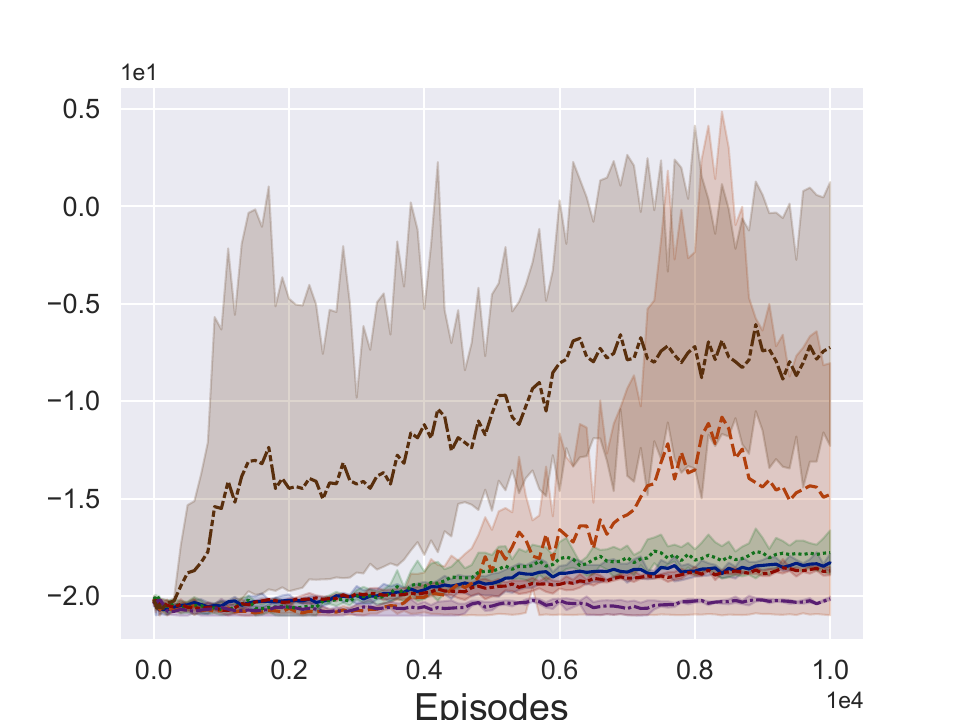}}
    \subfigure[Qbert]{
        \includegraphics[ width=0.19\textwidth]{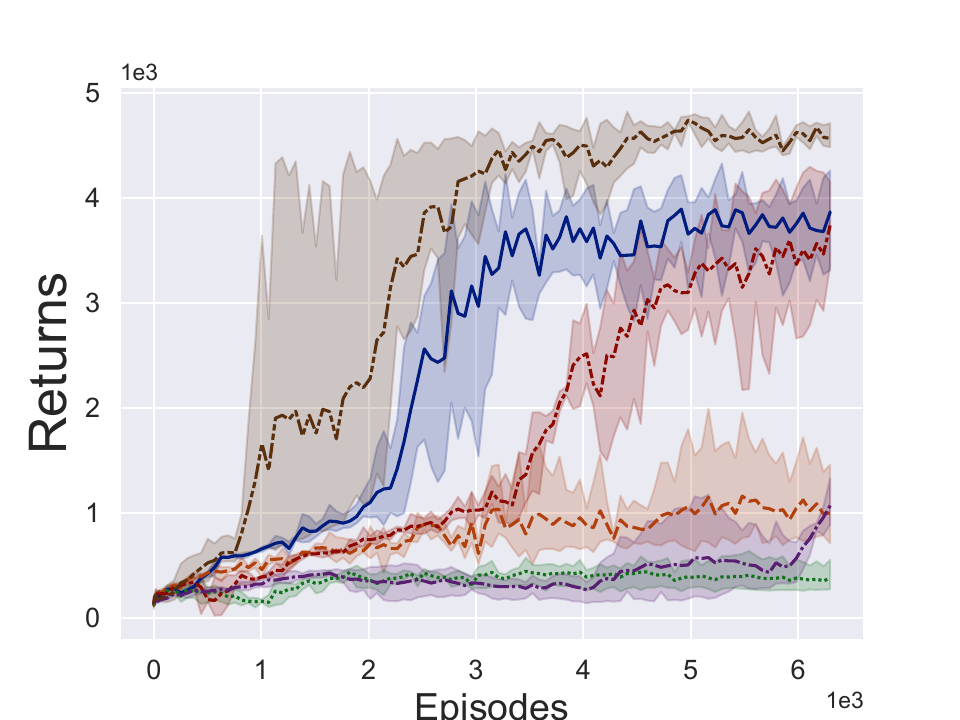}}
    \subfigure[Seaquest]{
        \includegraphics[ width=0.19\textwidth]{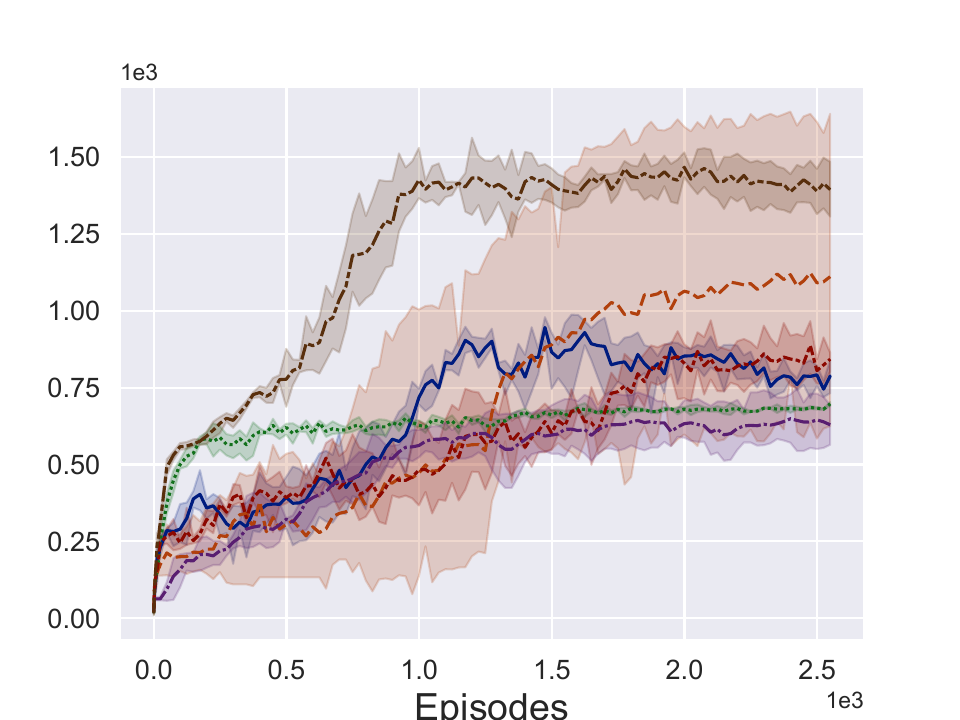}}
    \subfigure[SpaceInvaders]{
        \includegraphics[ width=0.19\textwidth]{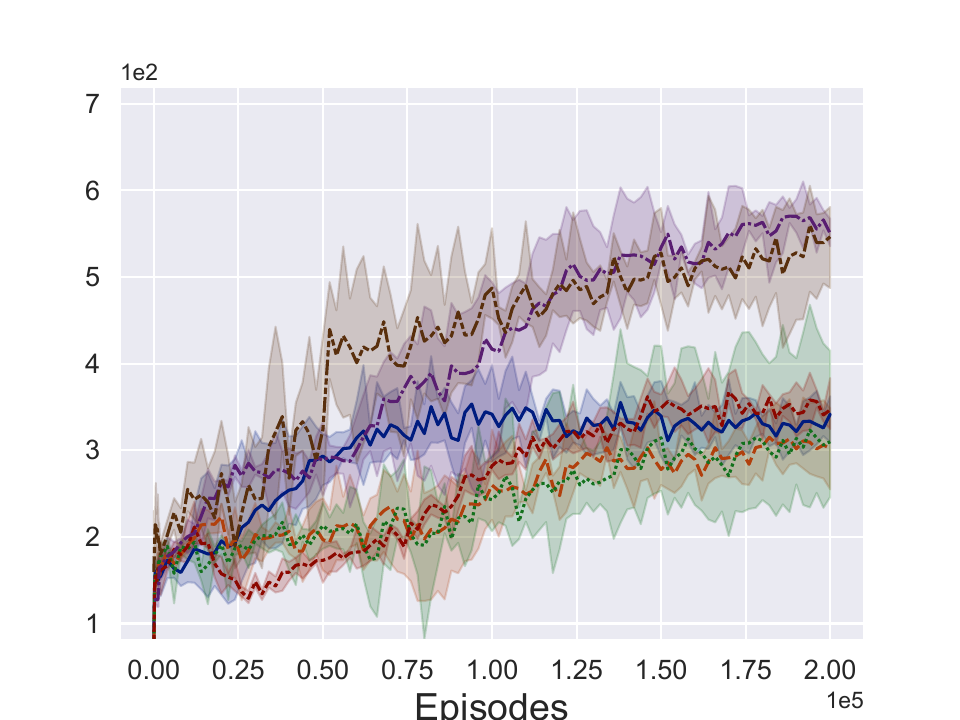}}
    \subfigure[UpNDown]{
        \includegraphics[ width=0.19\textwidth]{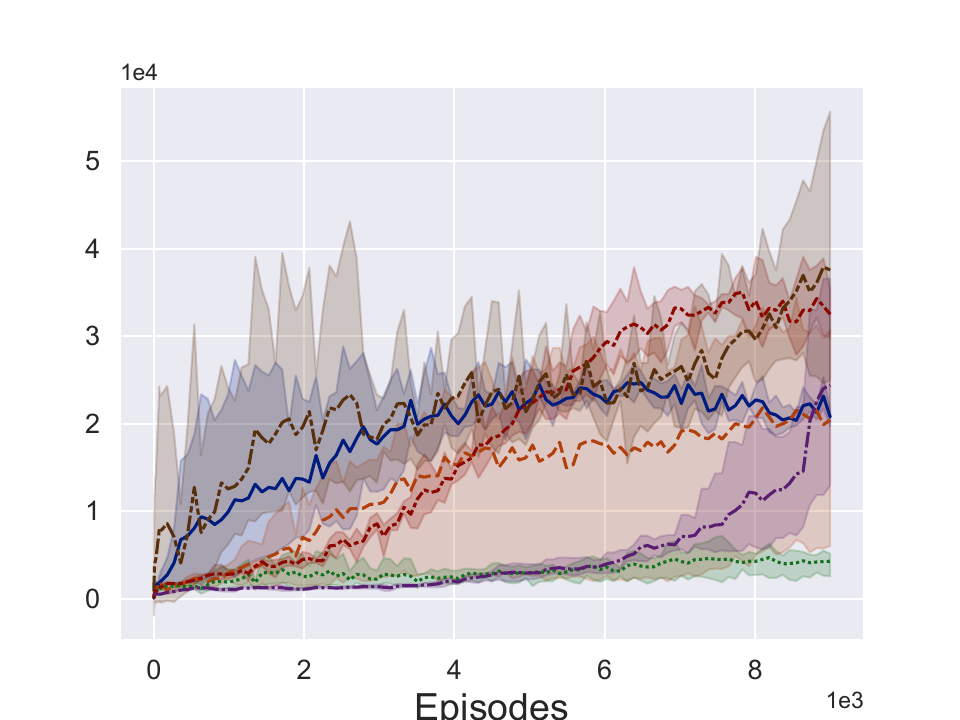}}
    \subfigure[Zaxxon]{
        \includegraphics[ width=0.19\textwidth]{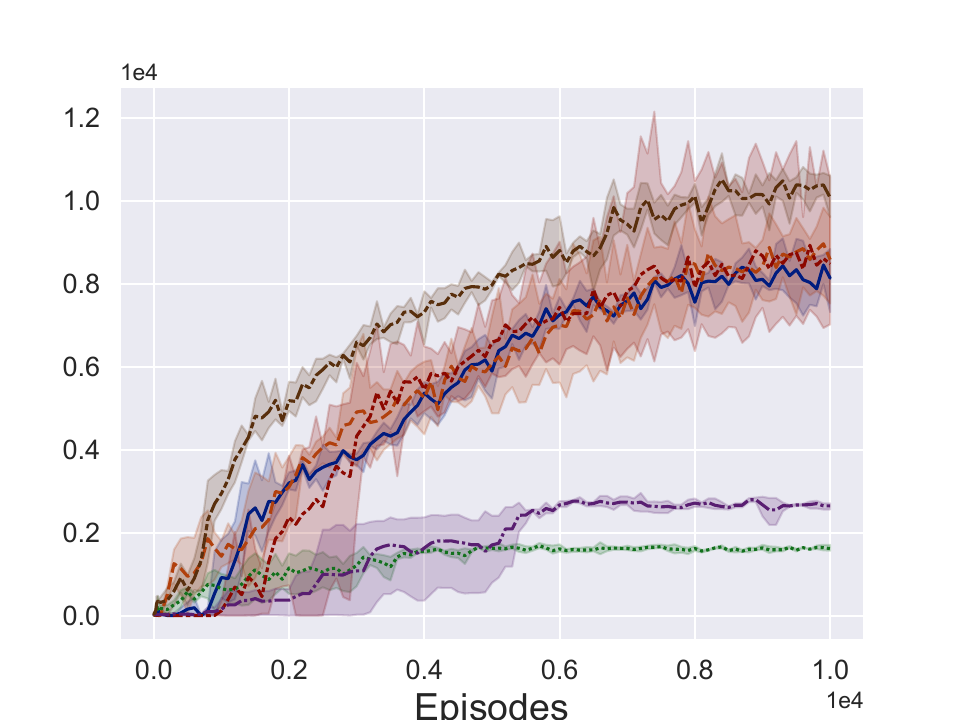}}
    \vspace{-2mm}
    \caption{Learning curves of on Atari after 10M timesteps, where shaded regions indicate the standard deviation. The state space is $84 \times 84 \times 4$.}
    \label{fig:atari-res}
\end{figure*}

\begin{figure*}[!t]
    \centering
    \subfigure[Humanoid]{ 
        \includegraphics[ width=0.19\textwidth]{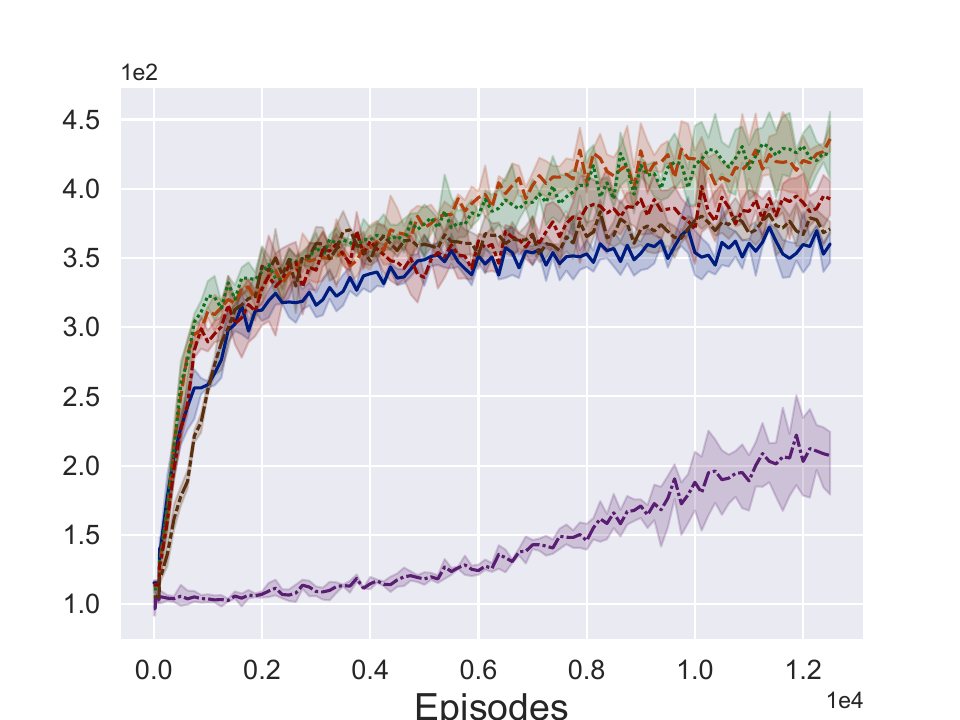}}
    \subfigure[HalfCheetah]{
        \includegraphics[ width=0.19\textwidth]{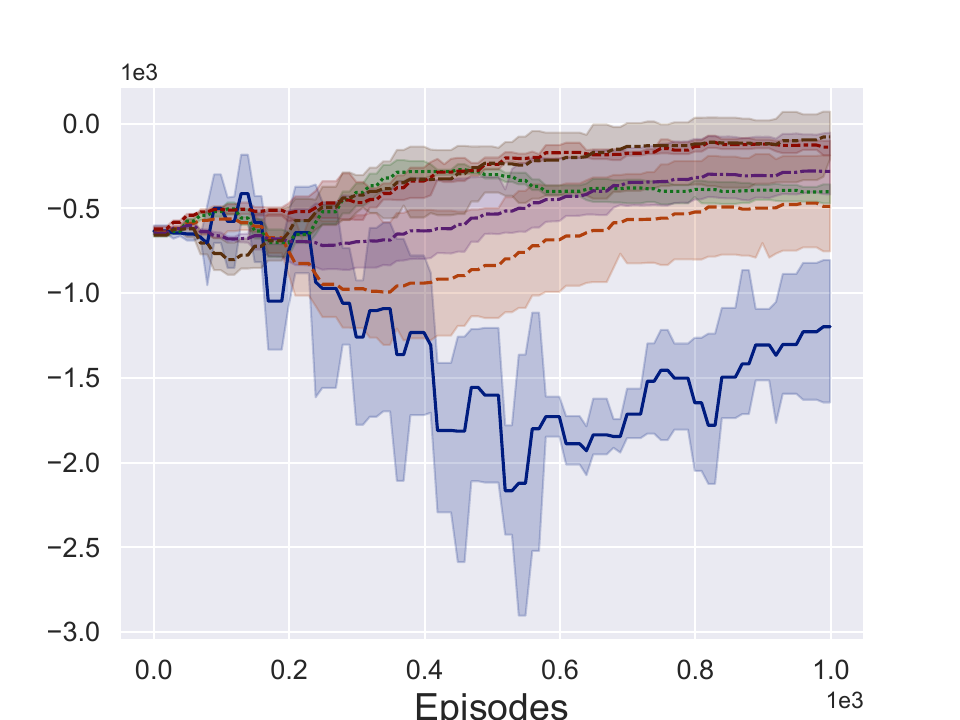}}
    \subfigure[Hopper]{ 
        \includegraphics[ width=0.19\textwidth]{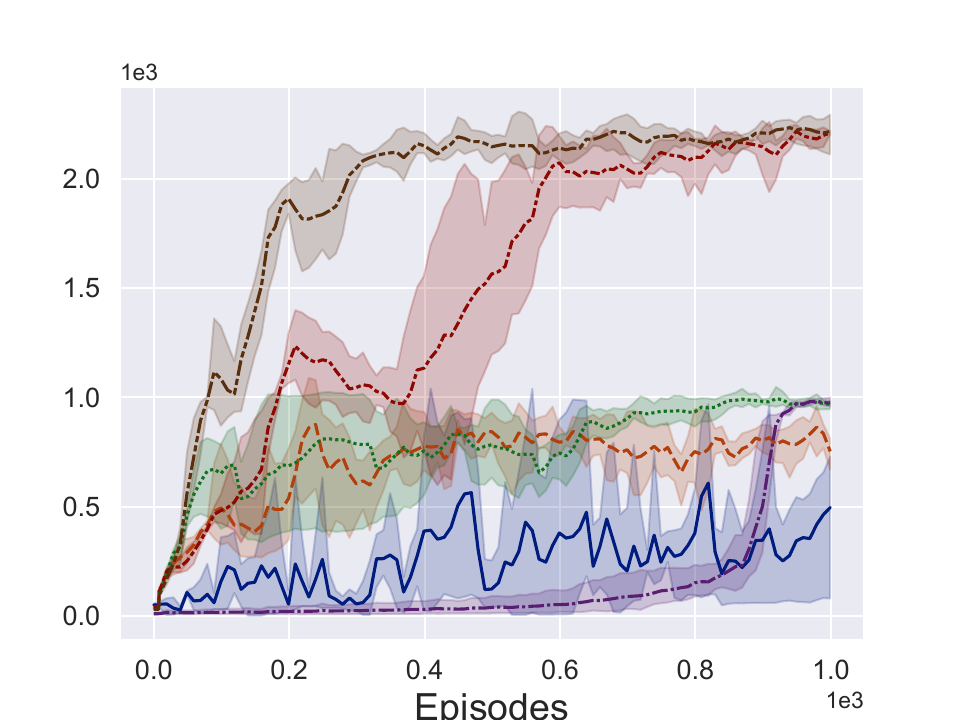}}
    \subfigure[HumanoidStandup]{ 
        \includegraphics[ width=0.19\textwidth]{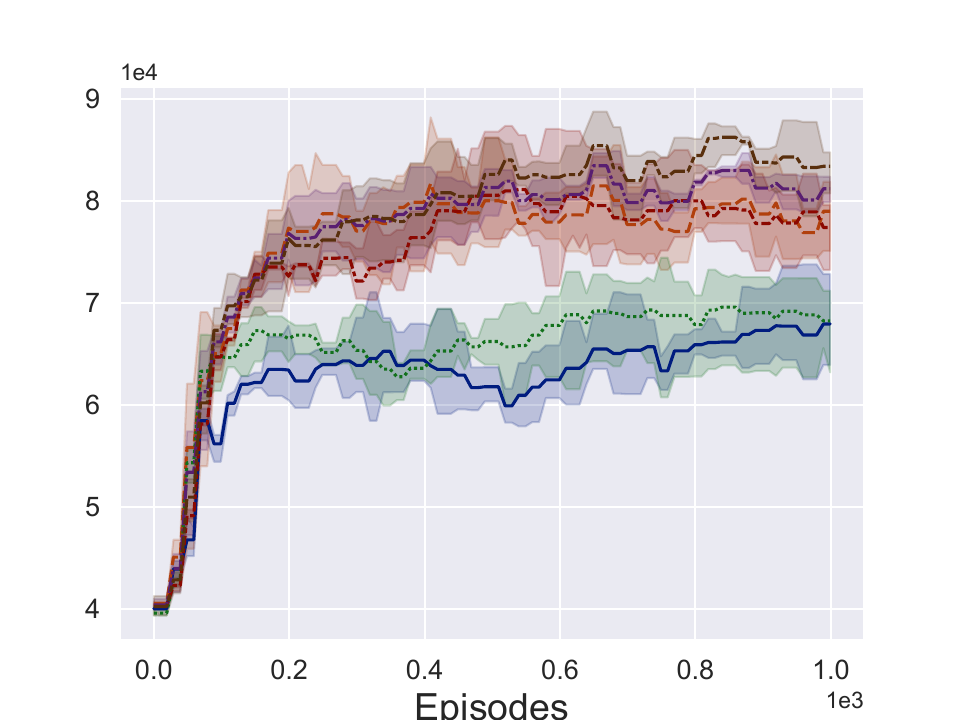}}
    \subfigure[Reacher]{
        \includegraphics[ width=0.19\textwidth]{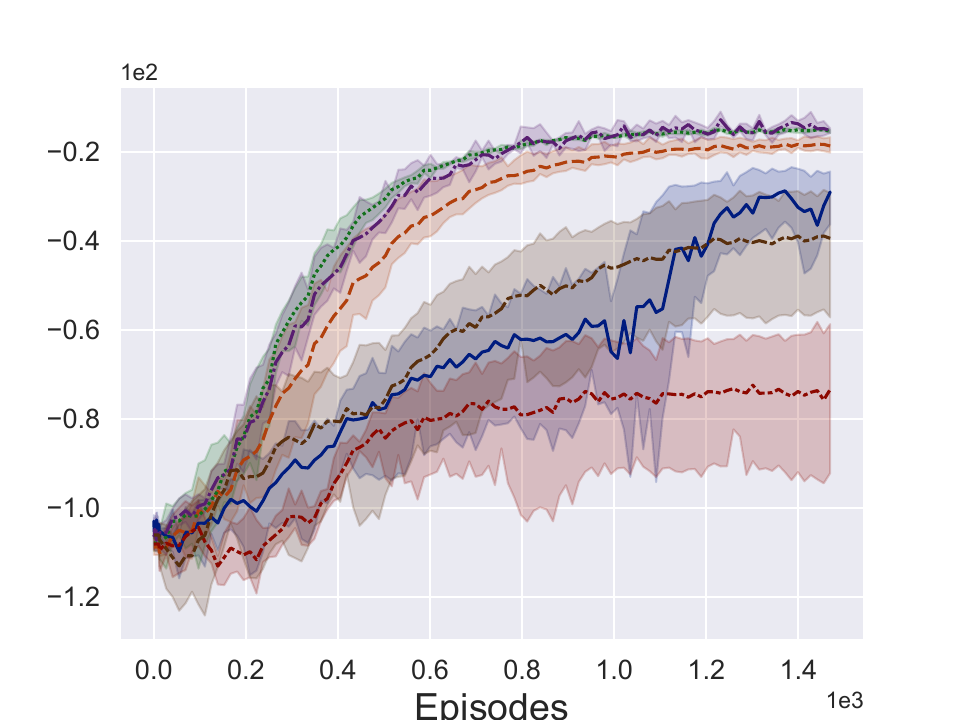}}
    \vspace{-2mm}
    \caption{Learning curves of on MuJoCo after 1M timesteps, where shaded regions indicate the standard deviation. The state space is $84 \times 84 \times 1$.}
    \vspace{-2mm}
    \label{fig:mujoco-res}
\end{figure*}

It can be seen that by utilizing \citeeq{eq:loss_suphash}, the separability of the hashing codes, which are the inputs to the discriminative layers, can be reliably preserved even when the dimension of hashing layer is much smaller than the original input dimension. 
Instead of fixing the parameters of the hashing code layer in the training process, 
the HashReward model and the policy are updated alternatively during learning. By this way, the training of the hashing code layer and the discriminator module are coupled together.
In order to achieve faster convergence, a pretraining stage with samples from expert demonstrations and the random policy for the autoencoder module is included in learning. The learning procedure of HashReward is illustrated in Algorithm~\ref{alg}. 

\section{Experiment}
\label{sec:experiment-results}
\subsection{Experimental Setup}
\label{subsec:experimental-setting}
\textbf{Environment.} We choose 15 games in Arcade Learning Environment \citep{DBLP:journals/jair/BellemareNVB13} and 5 simulators in MuJoCo~\cite{DBLP:conf/iros/TodorovET12}. The experiment is implemented in OpenAI Gym platform~\citep{DBLP:journals/corr/BrockmanCPSSTZ16}, which contains Atari 2600 video games with high-dimensional observation space (raw pixels), and pixels of each state in MuJoCo are obtained by a camera. The input states for the learner are set as low-dim continuous control inputs in MuJoCo.
We train converged DQN-based agents as experts in Atari, and DDPG-based~\cite{DBLP:journals/corr/LillicrapHPHETS15} agents in MuJoCo. 20 expert trajectories are collected for each game, also 3 trials with different random seeds are conducted for each environment. All experiments are conducted on server clusters with NVIDIA Tesla K80 GPUs.


\textbf{Contenders.} There are five basic contenders in the experiment, i.e., GAIL~\citep{DBLP:conf/nips/HoE16}, VAIL~\citep{DBLP:conf/iclr/PengKTAL19}, GIRIL~\citep{yu2020intrinsic}, GAIL with autoencoder (GAIL-AE) which utilizes only the first two autoencoder loss terms in \citeeq{eq:loss_suphash}, and GAIL with unsupervised hashing (GAIL-UH) which utilizes only the first four unsupervised hashing loss terms in \citeeq{eq:loss_suphash}. The codes of GAIL-AE and VAIL are real numbers, while those of GAIL-UH and HashReward belong to $\{-1, 1\}$. We initialize autoencoder pretraining for 40M frames of updates for GAIL-AE, GAIL-UH, GIRIL, and HashReward in Atari (1M in MuJoCo). To find out whether keeping autoencoder stable during training will increase the performance of GAIL-AE and GAIL-UH, we conduct experiments of GAIL-AE and GAIL-UH with updating autoencoder during training as GAIL-AE-Up and GAIL-UH-Up. Besides, to show the necessity of hashing for HashReward, we remove the third and fourth terms in Equation~\eqref{eq:loss_suphash} as HashReward-AE. The basic RL algorithm is PPO \citep{DBLP:journals/corr/SchulmanWDRK17}, and the reward signals of all methods are scaled into $[0, 1]$ to enhance the performance of RL part. We set all hyper-parameters and network architectures of the policy part the same to \citep{baselines}. Also, the hyper-parameters of DR and discriminator for all methods are the same: The DR and discriminator updates using Adam with a decayed learning rate of 3e-4; the batch size is 256; $\lambda$ is 0.01. We implement VAIL and GIRIL with the recommended hyper-parameters in their paper. The ratio of update frequency between the learner and discriminator is 3: 1.
\vspace{-2mm}

\begin{figure*}
    \centering
    \subfigure[Expert]{
    \begin{minipage}[t]{0.09\linewidth}
        \includegraphics[ width=1\textwidth]{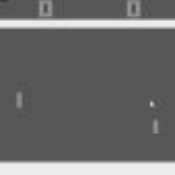}
        \includegraphics[ width=1\textwidth]{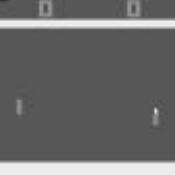}
        \includegraphics[ width=1\textwidth]{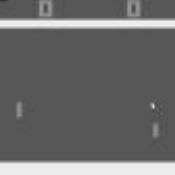}
        \includegraphics[ width=1\textwidth]{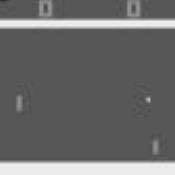}
    \end{minipage}}
    \subfigure[GAIL]{
    \begin{minipage}[t]{0.09\linewidth}
        \includegraphics[ width=1\textwidth]{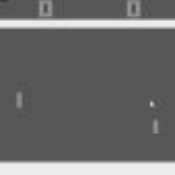}
        \includegraphics[ width=1\textwidth]{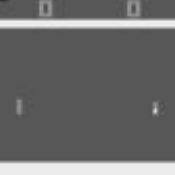}
        \includegraphics[ width=1\textwidth]{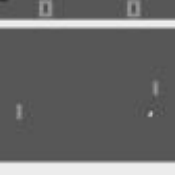}
        \includegraphics[ width=1\textwidth]{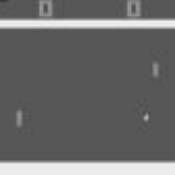}
    \end{minipage}}
    \subfigure[VAIL]{
    \begin{minipage}[t]{0.09\linewidth}
        \includegraphics[ width=1\textwidth]{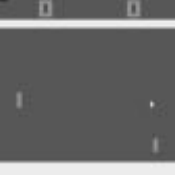}
        \includegraphics[ width=1\textwidth]{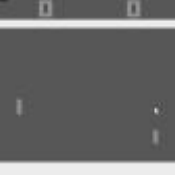}
        \includegraphics[ width=1\textwidth]{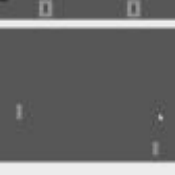}
        \includegraphics[ width=1\textwidth]{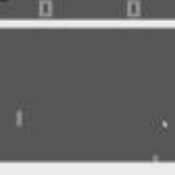}
    \end{minipage}}
    \subfigure[GIRIL]{
    \begin{minipage}[t]{0.09\linewidth}
        \includegraphics[ width=1\textwidth]{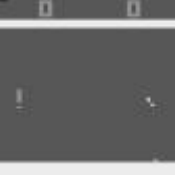}
        \includegraphics[ width=1\textwidth]{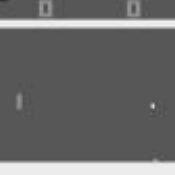}
        \includegraphics[ width=1\textwidth]{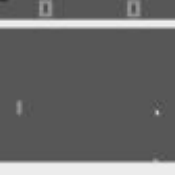}
        \includegraphics[ width=1\textwidth]{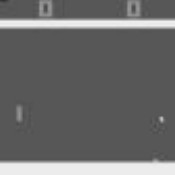}
    \end{minipage}}
    \subfigure[AE]{
    \begin{minipage}[t]{0.09\linewidth}
        \includegraphics[ width=1\textwidth]{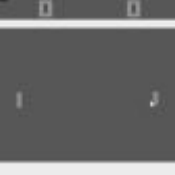}
        \includegraphics[ width=1\textwidth]{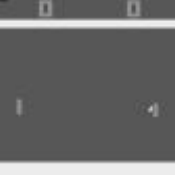}
        \includegraphics[ width=1\textwidth]{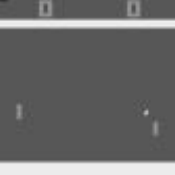}
        \includegraphics[ width=1\textwidth]{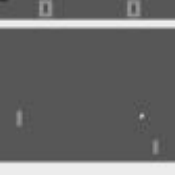}
    \end{minipage}}
    \subfigure[AE-Up]{
    \begin{minipage}[t]{0.09\linewidth}
        \includegraphics[ width=1\textwidth]{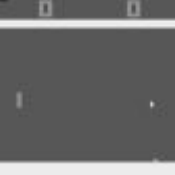}
        \includegraphics[ width=1\textwidth]{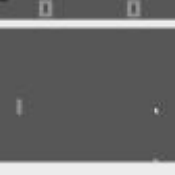}
        \includegraphics[ width=1\textwidth]{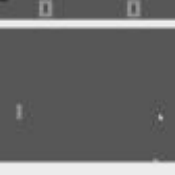}
        \includegraphics[ width=1\textwidth]{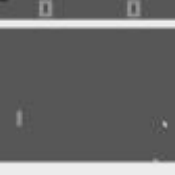}
    \end{minipage}}
    \subfigure[UH]{
    \begin{minipage}[t]{0.09\linewidth}
        \includegraphics[ width=1\textwidth]{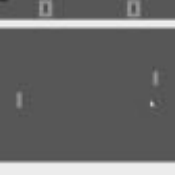}
        \includegraphics[ width=1\textwidth]{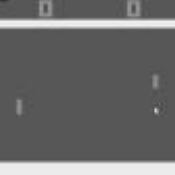}
        \includegraphics[ width=1\textwidth]{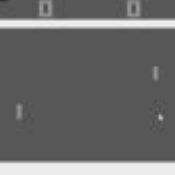}
        \includegraphics[ width=1\textwidth]{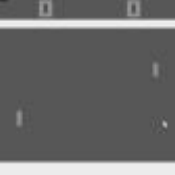}
    \end{minipage}}
    \subfigure[UH-Up]{
    \begin{minipage}[t]{0.09\linewidth}
        \includegraphics[ width=1\textwidth]{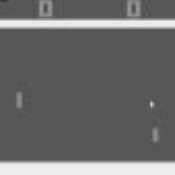}
        \includegraphics[ width=1\textwidth]{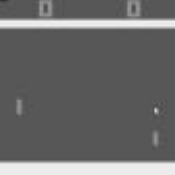}
        \includegraphics[ width=1\textwidth]{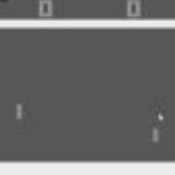}
        \includegraphics[ width=1\textwidth]{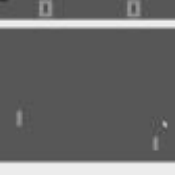}
    \end{minipage}}
    \subfigure[HR-AE]{
    \begin{minipage}[t]{0.09\linewidth}
        \includegraphics[ width=1\textwidth]{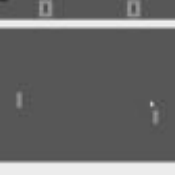}
        \includegraphics[ width=1\textwidth]{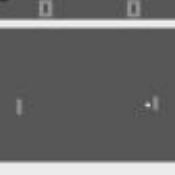}
        \includegraphics[ width=1\textwidth]{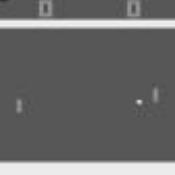}
        \includegraphics[ width=1\textwidth]{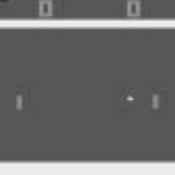}
    \end{minipage}}
    \subfigure[HR]{
    \begin{minipage}[t]{0.09\linewidth}
        \includegraphics[ width=1\textwidth]{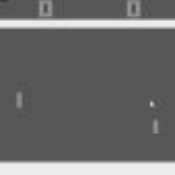}
        \includegraphics[ width=1\textwidth]{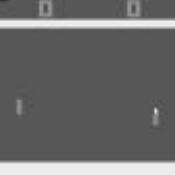}
        \includegraphics[ width=1\textwidth]{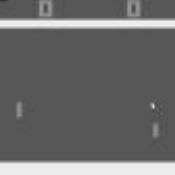}
        \includegraphics[ width=1\textwidth]{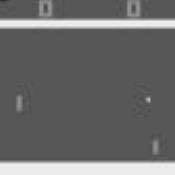}
    \end{minipage}}
    \caption{The sequence generated by each approach on \textit{Pong}, with the comparison of expert demonstrations (the first column), where `AE', `AE-Up', `UH', `UH-Up', `HR-AE' and `HR' indicate GAIL-AE, GAIL-AE-Up, GAIL-UH, GAIL-UH-Up, HashReward-AE and HashReward respectively. The timestamp for each row of images and seed for each environment are the same.}
    \label{fig:pong-raw-images}
\end{figure*}

\begin{figure*}[!t]
    \centering
    \subfigure{ 
        \includegraphics[ width=0.19\textwidth]{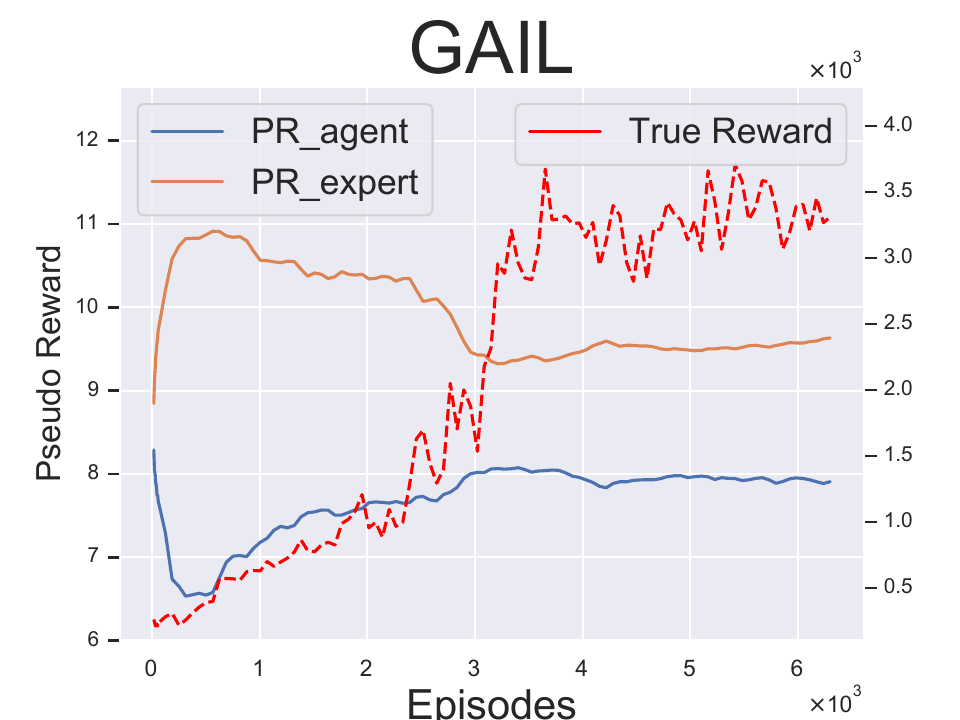}}
    \subfigure{
        \includegraphics[ width=0.19\textwidth]{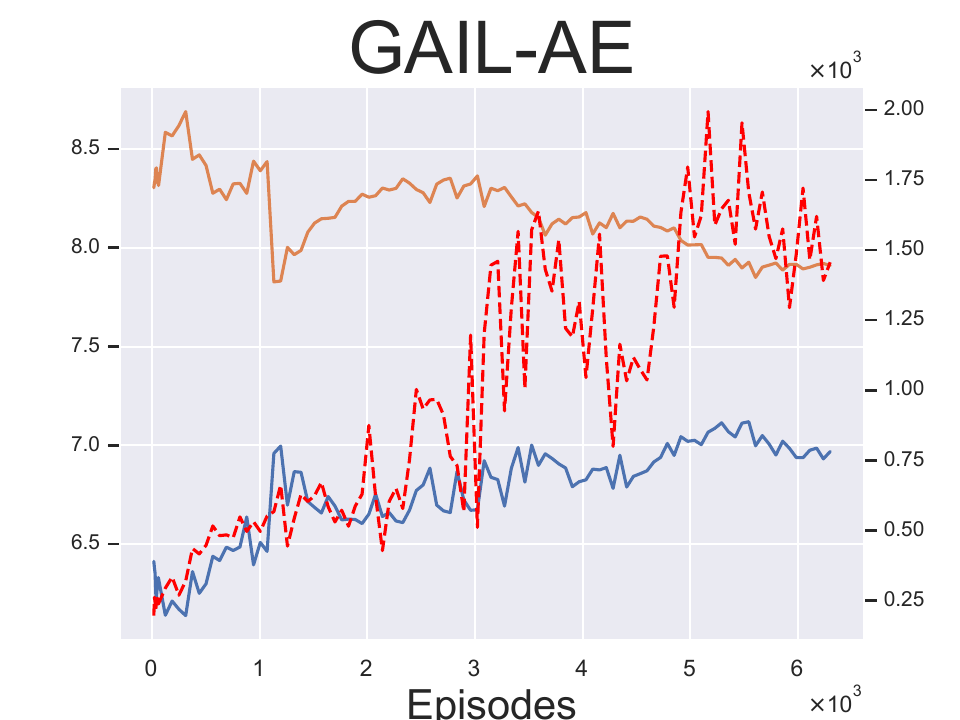}}
    \subfigure{ 
        \includegraphics[ width=0.19\textwidth]{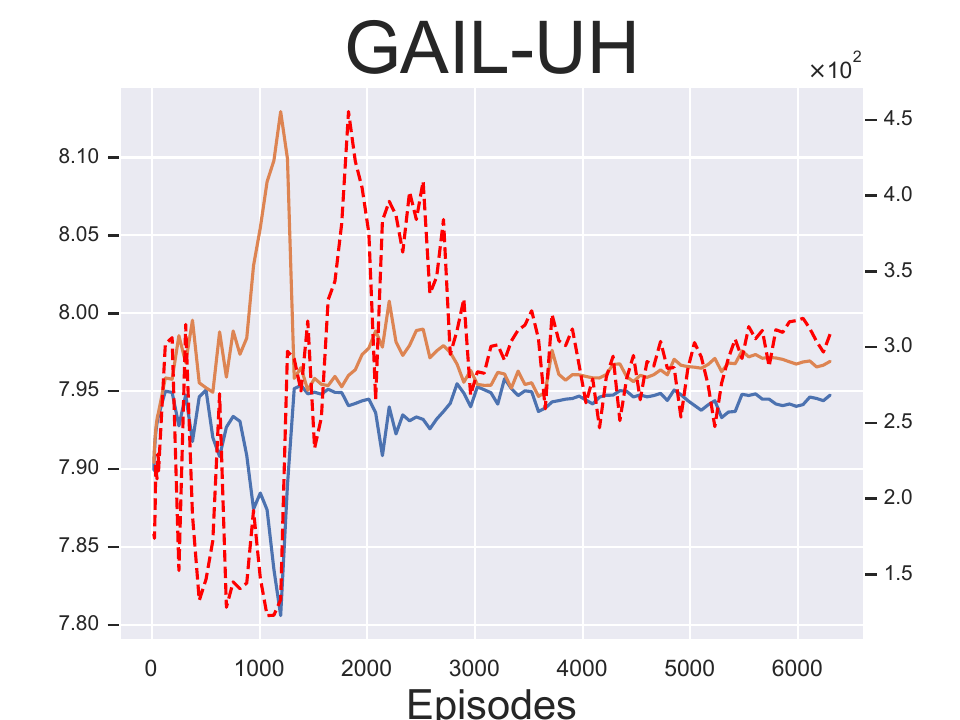}}
    \subfigure{
        \includegraphics[ width=0.19\textwidth]{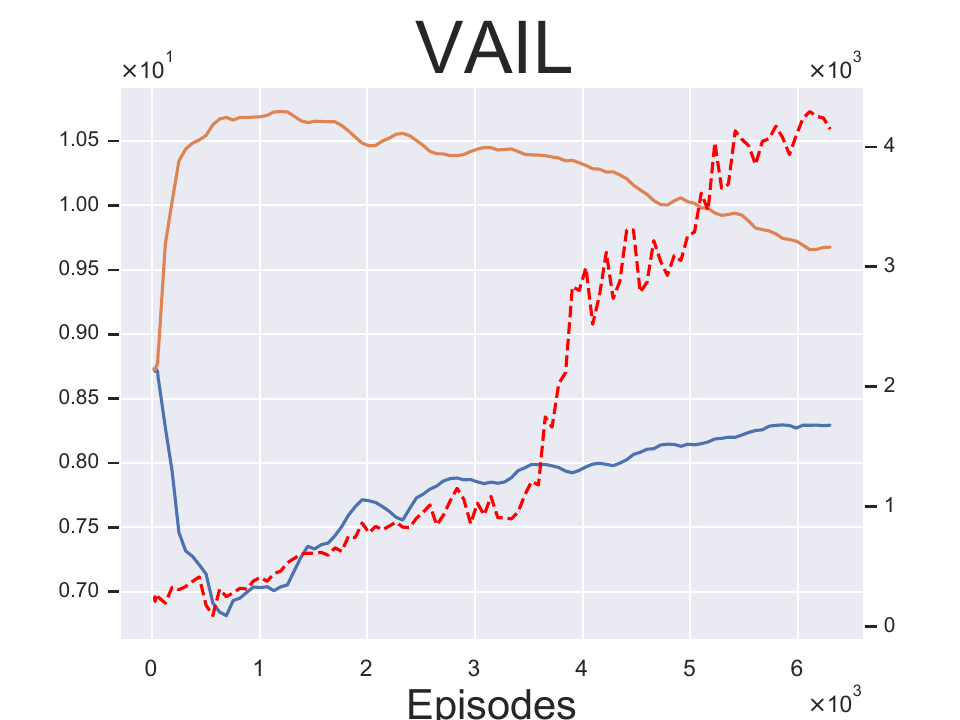}}
    \subfigure{
        \includegraphics[ width=0.19\textwidth]{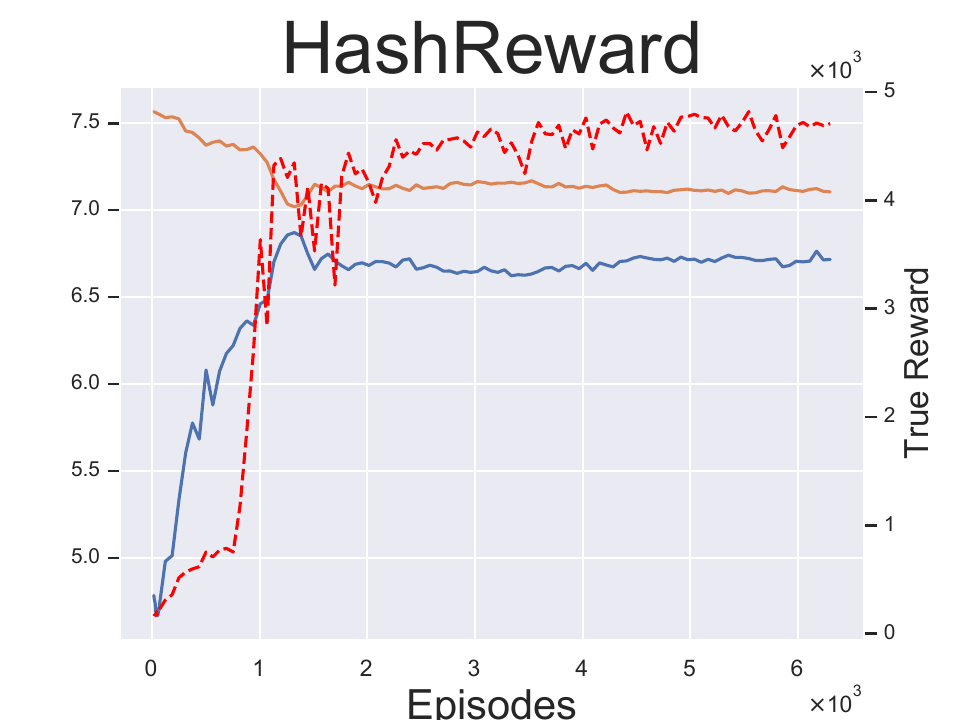}}
    \caption{Pseudo reward generated by discriminator and true reward curves of five basic approaches on \textit{Qbert}. `PR\_agent' and `PR\_expert' indicate pseudo reward for the agent and expert samples respectively. The blue curve denotes pseudo reward provided by the discriminator, and the red one denotes the true reward.}
    \label{fig:curve-pseudo}
\end{figure*}

\begin{figure*}[!t]
    \centering
    \subfigure[0 Step]{ 
        \includegraphics[ width=0.23\textwidth]{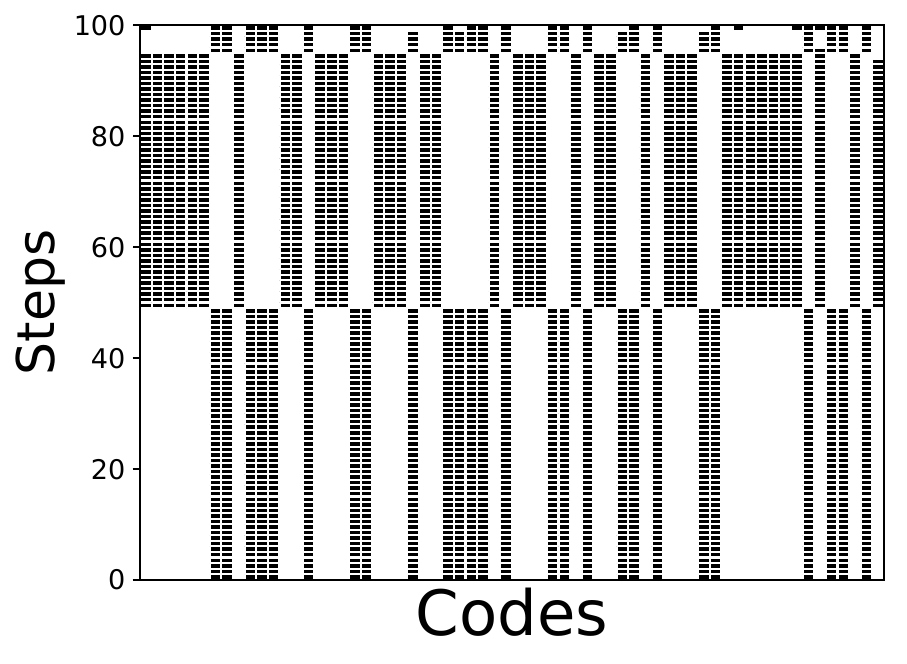}}
    \subfigure[1e6 Steps]{
        \includegraphics[ width=0.23\textwidth]{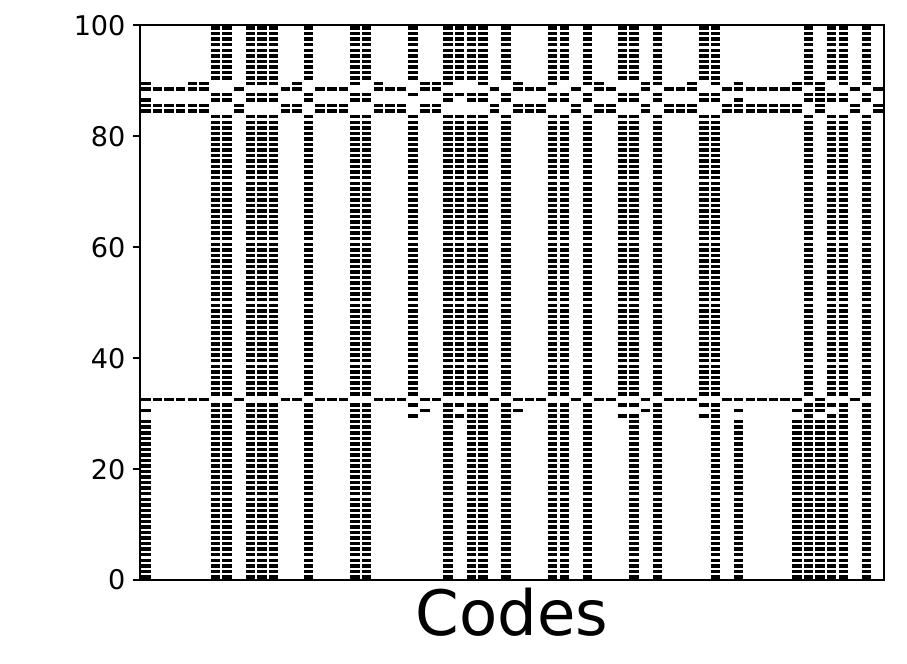}}
    \subfigure[1e7 Steps]{
        \includegraphics[ width=0.23\textwidth]{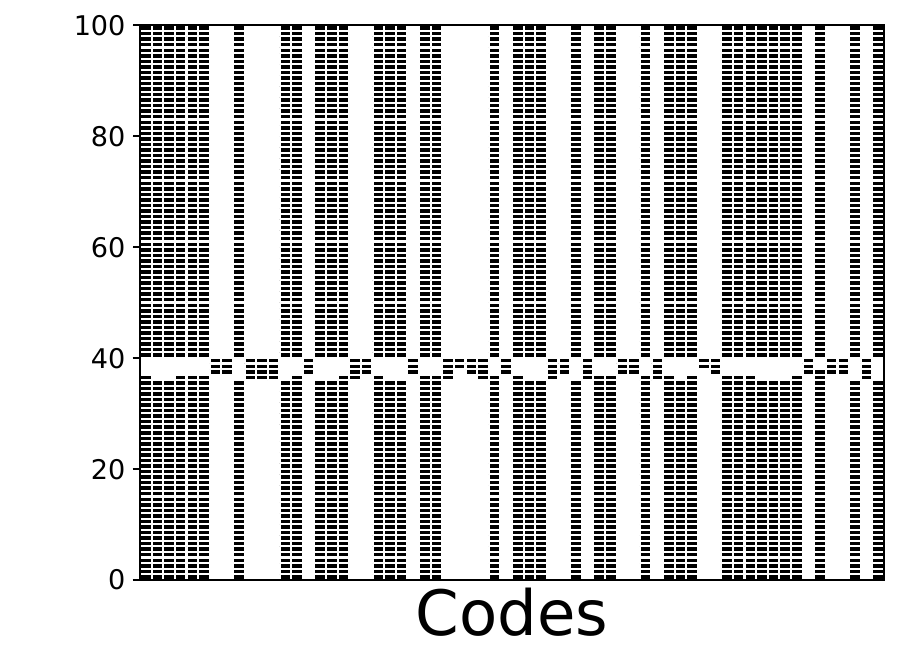}}
    \subfigure[Expert]{ 
        \includegraphics[ width=0.23\textwidth]{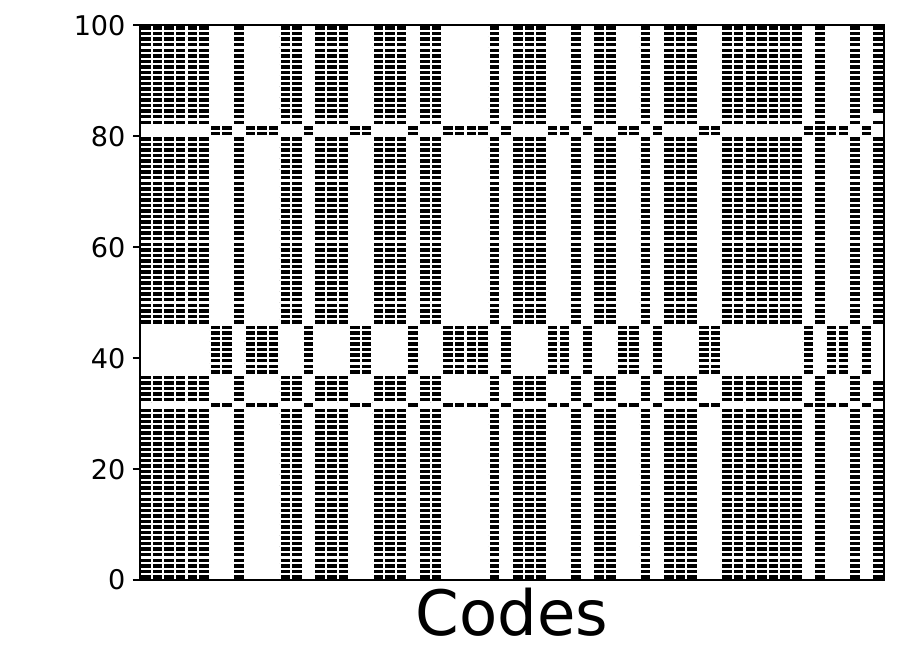}}
    \caption{The 64-bit HashReward codes for samples from different policy steps as well as expert on \textit{Qbert} game. The samples are gathered from the same continual period within an episode with a fixed random seed.}
    \vspace{-2mm}
    \label{fig:embeddings}
\end{figure*}

\begin{figure*}[t]
    \centering
    \includegraphics[width=0.78\textwidth]{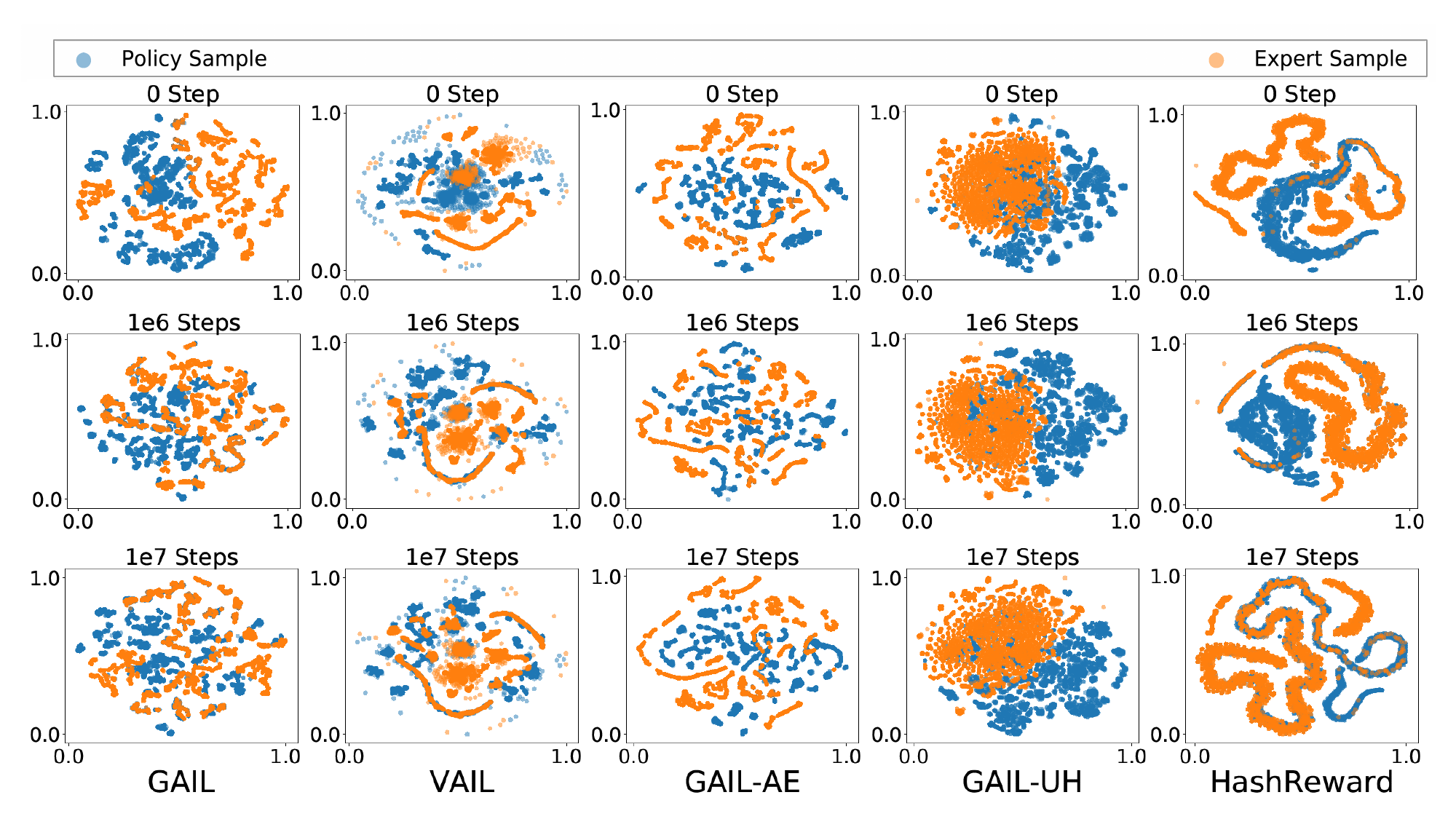}
    \caption{t-SNE Visualizations of policy samples (blue points) and expert samples (orange points) on {\it Qbert}.}
    \label{fig:qbert-tsne}
\end{figure*}
\subsection{Results}
\label{subsec:empirical-results}
Experimental results in Atari are reported in Table~\ref{tab:atari_results} and Figure~\ref{fig:atari-res}, and that of the major methods in MuJoCo are shown in Table~\ref{tab:mujoco_results} and Figure~\ref{fig:mujoco-res}. More results including the performances of other contenders are reported in the supplementary material. 

We can observe that VAIL outperforms GAIL on most environments, and achieves near-expert performance on \textit{UpNDown} and \textit{Hopper}, but still fails compared with HashReward. GIRIL achieves the best performance in {\it BeamRider}, {\it SpaceInvaders} and {\it Reacher}, but its behavior is also unsatisfactory in some environments. For variants of GAIL, GAIL-AE outperform GAIL on several environments, but fail on most compared with VAIL and HashReward. Meanwhile, GAIL-UH only achieve satisfactory performance on \textit{BeamRider}, \textit{Humanoid} and \textit{Reacher}. 
It also evinces that without supervision, the change of feature could confuse the discriminator training. HashReward-AE outperforms GAIL on most games, while remains a gap compared with HashReward. This demonstrates that hashing is necessary for HashReward even with supervision. As expected, HashReward outperform its contenders significantly and achieve the best performance on most of environments (12/15 in Atari, 3/5 in MuJoCo), meanwhile gains expert-level reward on \textit{Boxing}, \textit{Qbert}, \textit{SpaceInvaders}, \textit{UpNDown}, \textit{Zaxxon} and \textit{Hopper}.

It can be concluded that GAIL could not be improved with simple unsupervised DR; also the usage of supervision only from discriminator (regular term of loss in VAIL) to enhance DR part is not enough; besides, another potential reason for not that satisfactory results of VAIL on Atari is that its discriminator does not utilize action signal, which is of a great importance signal for these games~\citep{DBLP:journals/corr/abs-1810-10593}; meanwhile, although GIRIL achieves decent performance on some environments, it does not solve DR problems directly, which may limit its performance; and HashReward provides a powerful approach for tackling high-dimensional IL problems, which can encourage the learner to generate expert comparable policies in challenging IL tasks.

\subsection{Is HashReward Meaningful?}
\label{subsec:is-hashreward-meaningful}
\textbf{Comparisons on \textit{Pong}.}
To find out whether HashReward has dug out the expert policy closely, we report the comparisons of a sequence of expert demonstrations and the sequence generated by each method during the same period of time on \textit{Pong}, shown in Figure~\ref{fig:pong-raw-images}. We can observe that the expert hits the ball with the short side of the bar, which is the `kill-shot' and hard to imitate. For contenders, they try to learn the `kill-shot' but hit the ball by the long side of the bar instead of the short side, except for GAIL-AE-Up, GAIL-UH, GAIL-UH-Up, GIRIL, and VAIL which miss the ball. While the sequence of HashReward is the same as that of demonstrations, which shows that HashReward has successfully learned the `kill-shot'.

\textbf{Pseudo Reward Curves.} In order to understand why HashReward outperforms its contenders and whether HashReward tackles the discrimination-rewarding trade-off, first we analyze the true reward and pseudo reward (generated by reward function) curves of five basic methods for a single training process on \textit{Qbert}, illustrated in Figure~\ref{fig:curve-pseudo}.
For GAIL and GAIL-AE, the discriminator seems to excessively focus on discriminating between the learner's and expert's samples, such that the pseudo reward does not rise even when the true reward has increased, which verifies our perspective on why GAIL (and with unsupervised DR) fails in such tasks.
For GAIL-UH, the discriminator fails to discriminate between learner's and expert's samples, as the pseudo reward curve for learner almost overlaps with that for the expert in the whole training process. This indicates that including supervised information in hashing code training is indeed essential. For VAIL, the change of the pseudo reward for agent and expert is similar to that of GAIL, which reveals that the supervision of VAIL in DR is not enough. 
 For HashReward, we can observe that the change of HashReward ideally reflects that of the true reward. Furthermore, the over-discrimination of pseudo reward is successfully avoided by HashReward. More results from the rest environments and other methods can be found in the supplementary material, which reflects similar phenomena. 


\textbf{64-bit Embeddings.} Besides, to verify whether HashReward generates meaningful codes, we train a 64-bit embedding version of HashReward model on \textit{Qbert}. Afterward we input samples generated from policies of different training stages (after 0 steps, 1e6 steps, and 1e7 steps) to the learned HashReward model to compare their codes. To collect samples, we run each policy for an episode with a fixed random seed and collect 100 samples within the same interval of time-steps. 
The codes are demonstrated in Figure~\ref{fig:embeddings}.
At 0 step and 1e6 step, the codes are clearly different from that of the expert, showing that the policies are not sufficiently trained.
On the contrary, at 1e7 step, as the learner has dug out the latent expert policy, their codes are close to each other. This shows that HashReward learns a meaningful embedding space to reflect the quality of imitation.

\textbf{t-SNE for Embeddings.} Furthermore, Figure~\ref{fig:qbert-tsne} illustrates t-SNE~\citep{vanDerMaaten2008} visualization of the embedding layer outputs from five of the comparison methods under \textit{Qbert}, and each figure contains 5000 policy and expert samples respectively. We show how embeddings of expert demonstrations, as well as policy samples, evolve for the same learning stages as above.
 We can observe that for GAIL and GAIL-AE, the embeddings for individual samples remain isolated throughout training, thus are easy to be fully discriminated. By using information bottleneck loss, VAIL partially solves this problem by generating some local clustered structures, which however still easy to be distinguished by the discriminator. For GAIL-UH, even though the embeddings seem to be globally clustered, we can observe that the policy and expert sample embeddings keep overlapping during the whole training process. This means that it fails to keep high-dimensional discriminative information when performing DR. For HashReward, we can observe that both expert and policy embeddings are globally clustered. Meanwhile, the overlaps between the two embedding sets increase along with the learning process, revealing the improvement of policy learning. This verifies HashReward could desirably dig out the hidden manifold structure of input space, which could lead to successful learning.


The above results reveal the close relationship between making the discriminator provide ground-truth consistent reward signals and properly dealing with the discrimination-rewarding trade-off. Meanwhile, they also indicate that effective supervision is essential for learning a proper discriminator, which is the key leading to the superior performance of HashReward in high-dim IL problems. 
\section{Conclusions}
\label{sec:conclusion}
In this paper, we tackle the challenging problem of IL in high-dimensional environments, under which even state-of-the-art IL algorithms fail. Based on theoretical and empirical studies, we identify that such failure results from their improper treatment of tackling the discrimination-rewarding trade-off. Through this finding, we propose a novel high-dimensional IL method named HashReward, which utilizes supervised hashing to learn an effective discriminator, encouraging learners to dig out latent expert policies from demonstrations by providing efficient and stable reward signals. Experiments under both Atari and MuJoCo environments verify the effectiveness of HashReward, which outperforms state-of-the-art contenders with significant gaps. We expect HashReward can also provide inspirations in designing other hashing strategies containing effective semantic information, for solving other challenging IL problems, e.g., exploration-demanding games like \textit{MontezumaRevenge}. We will explore more on these possibilities. 

\textbf{Acknowledgment: }This research was supported by the NSFC (61673201, 61921006).  The authors would like to thank Yang Yu, Jieping Ye, and the anonymous reviewers for their insightful comments and suggestions.



\bibliographystyle{ACM-Reference-Format} 
\balance
\bibliography{hash}

\newpage
\appendix

\section{Details of Theoretical Results}
\label{sec:generalization-bound}
We provide the full version of Theorem 1 below. First, the notion of spectral normalized complexity is given in the following definition.
\begin{definition}[\citep{DBLP:conf/nips/BartlettFT17}]
  Let fixed activation functions $(\sigma_1, \dots, \sigma_N)$ and reference matrices $(M_1, \dots, M_N)$ be given, where $\sigma_i$ is $\rho_i$-Lipschitz and $\sigma_i(0) = 0$. Let spectral norm bounding parameters $(s_1, \dots, s_N)$ and matrix $(2,1)$-norm bounding parameters $(b_1, \dots, b_N)$ be given. Let 
 \begin{equation}
  \label{eq:reward-function}
    D_\mathcal{W}(x) = \sigma_N(W_N\sigma_{N - 1}(W_{N - 1} \dots \sigma_1(W_1 x) \dots))
  \end{equation}
be the neural network associated with weight matrices $(W_1, W_2, \dots, W_N)$. Let $\mathcal D$ denote the discriminator set consisting of all choices of neural network $D_\mathcal{W}$:
\begin{equation}
\label{eq:reward-function-set}
\begin{aligned}
  \mathcal{D}_\mathcal{W} \coloneqq &\{D_\mathcal{W}: \mathcal{W} = (W_1, \cdots, W_N) | \left\Vert W_i\right\Vert_\sigma \leq s_i, \\ &\left\Vert W_i ^ T - M_i ^ T \right\Vert_{2,1} \leq b_i\},\quad i\in[N],
\end{aligned}
\end{equation}
where $\left\Vert W\right\Vert_\sigma$ and $\left\Vert W\right\Vert_{2,1}$ are the matrix spectral norm and $(2,1)$-norm of $W$. 
Moreover, assume that each matrix in $(W_1, \dots, W_N)$ has dimension at most $M$ along each axis, then the {\it spectral normalized complexity} $\mathcal R$ is defined as
\begin{equation}
\label{eq:spectral-normalized-complexity}
\mathcal{R} = \sqrt{\log(2M ^ 2)}\prod_{i = 1} ^ N s_i\rho_i(\sum_{i = 1} ^ N (\frac{b_i}{s_i}) ^ {2/3}) ^ {3/2}.
\end{equation}
  \label{def:sn_complexity}
\end{definition}
We restate Theorem 1 below.
\begin{myThm}
Assume that assumptions in the main paper hold.
Let $\Delta_1 = |d_{\mathcal D}(\mu_{\pi_E}, \mu_{\hat\pi_G}) - d_{\mathcal D'}(\mu_{\pi_E}, \mu_{\hat\pi_G})|$, $\Delta_2 = |\hat d_{\mathcal D'}(\hat\mu_{\pi_E, m}, \mu_{\hat\pi_G}) - d_{\mathcal D'}(\hat\mu_{\pi_E, m}, \mu_{\hat\pi_G})|.$  
Given expert trajectory data $X$ which consists of $m$ trajectories $\tau_{\pi_E} \in\mathcal T$, if $m \geq 3\|\phi(X)\|_F\mathcal R$, then with probability at least $1 - \delta$, we have
\begin{equation}
\begin{aligned}
d_{\mathcal D}(\mu_{\pi_E}, \mu_{\hat\pi_G}) &\leq \Delta_1 + \Delta_2 + 6\Delta\sqrt{\frac{\log(2/\delta)}{2m}} \\ &+ \frac{24\left\Vert \phi(X)\right\Vert_F\mathcal{R}}{m}(1 + \log\frac{m}{3\left\Vert \phi(X)\right\Vert_F\mathcal{R}}) + \eta. \\ 
\end{aligned}
\end{equation}
\end{myThm}
\begin{proof}
\label{proof:thm1}
We have 
\begin{equation}
\begin{aligned}
 d_{\mathcal D}(\mu_{\pi_E}, \mu_{\hat\pi_G}) \leq & d_{\mathcal D}(\mu_{\pi_E}, \mu_{\hat\pi_G}) - d_{\mathcal D'}(\mu_{\pi_E}, \mu_{\hat\pi_G}) \\
 & + d_{\mathcal D'}(\mu_{\pi_E}, \mu_{\hat\pi_G}) - d_{\mathcal D'}(\hat\mu_{\pi_E, m}, \mu_{\hat\pi_G}) \\
& + d_{\mathcal D'}(\hat\mu_{\pi_E, m}, \mu_{\hat\pi_G}) 
- \hat d_{\mathcal D'}(\hat\mu_{\pi_E, m}, \mu_{\hat\pi_G}) +  \eta \\ 
\leq & d_{\mathcal D'}(\mu_{\pi_E}, \mu_{\hat\pi_G}) - d_{\mathcal D'}(\hat\mu_{\pi_E, m}, \mu_{\hat\pi_G}) + \Delta_1 + \Delta_2 + \eta.
\end{aligned}
\label{eq:proof1}
\end{equation}
According to \cite{DBLP:conf/nips/BartlettFT17}, the key advantage of introducing spectral normalized complexity $\mathcal R$ is that the empirical Rademacher complexity of $\mathcal D'$ can be bounded by
\begin{equation}
  \label{eq:empirical-Rademacher}
  \hat{\mathcal{R}}_m(\mathcal{D'}) \leq \frac{24\left\Vert \phi(X) \right\Vert_F\mathcal{R}}{m}(1 + \log\frac{m}{3\left\Vert \phi(X) \right\Vert_F\mathcal{R}}).
\end{equation}
Thus Rademacher complexity based generalization bounds for general GANs \cite{DBLP:conf/iclr/Zhang0ZX018} can be utilized to characterize the gap between $d_{\mathcal D'}(\hat\mu_{\pi_E, m}, \mu_{\hat\pi_G})$ and $d_{\mathcal D'}(\mu_{\pi_E}, \mu_{\hat\pi_G})$. We have 
\begin{equation}
\begin{aligned}
  & d_{\mathcal D'}(\hat\mu_{\pi_E, m}, \mu_{\hat\pi_G}) - d_{\mathcal D'}(\mu_{\pi_E}, \mu_{\hat\pi_G}) \\
  & \leq \sup_{D\in\mathcal D'} \big(\qw_{\tau_E\sim \hat\mu_{\pi_{E}, m}}[D(\phi(\tau_E))]-\qw_{\tau_E\sim \mu_{\pi_{E}}}[D(\phi(\tau_E))]\big)  
  \label{eq:proof2}
\end{aligned}
\end{equation}
based on the fact that $\mathcal D'$ is even and the subadditivity of the superior function. Since $D\in\mathcal D'$ takes value in $[-\Delta, \Delta]$, then changing any trajectory in $X$ will change the l.h.s. of \citeeq{eq:proof2} by no more than $2\Delta/m$. By McDiamid's inequality and the definition of the empirical Rademacher complexity, then with probability at least $1-\delta$, we have
\begin{equation}
\begin{aligned}
&\sup_{D\in\mathcal D'} \big(\qw_{\tau_E\sim \hat\mu_{\pi_{E}, m}}[D(\phi(\tau_E))]-\qw_{\tau_E\sim \mu_{\pi_{E}}}[D(\phi(\tau_E))]\big) \\
& \leq \hat{\mathcal R}_m(\mathcal D') + 6\Delta\sqrt{\log(2/\delta)/2m}.
\label{eq:proof3}
\end{aligned}
\end{equation}
Combining Equation \ref{eq:proof1}, \ref{eq:empirical-Rademacher}, \ref{eq:proof2} and \ref{eq:proof3}, we arrive at the final result. 
\end{proof}
%
%
%
\section{Network Architectures}
\begin{figure*}[!t]
    \centering
    \includegraphics[width=0.8\textwidth]{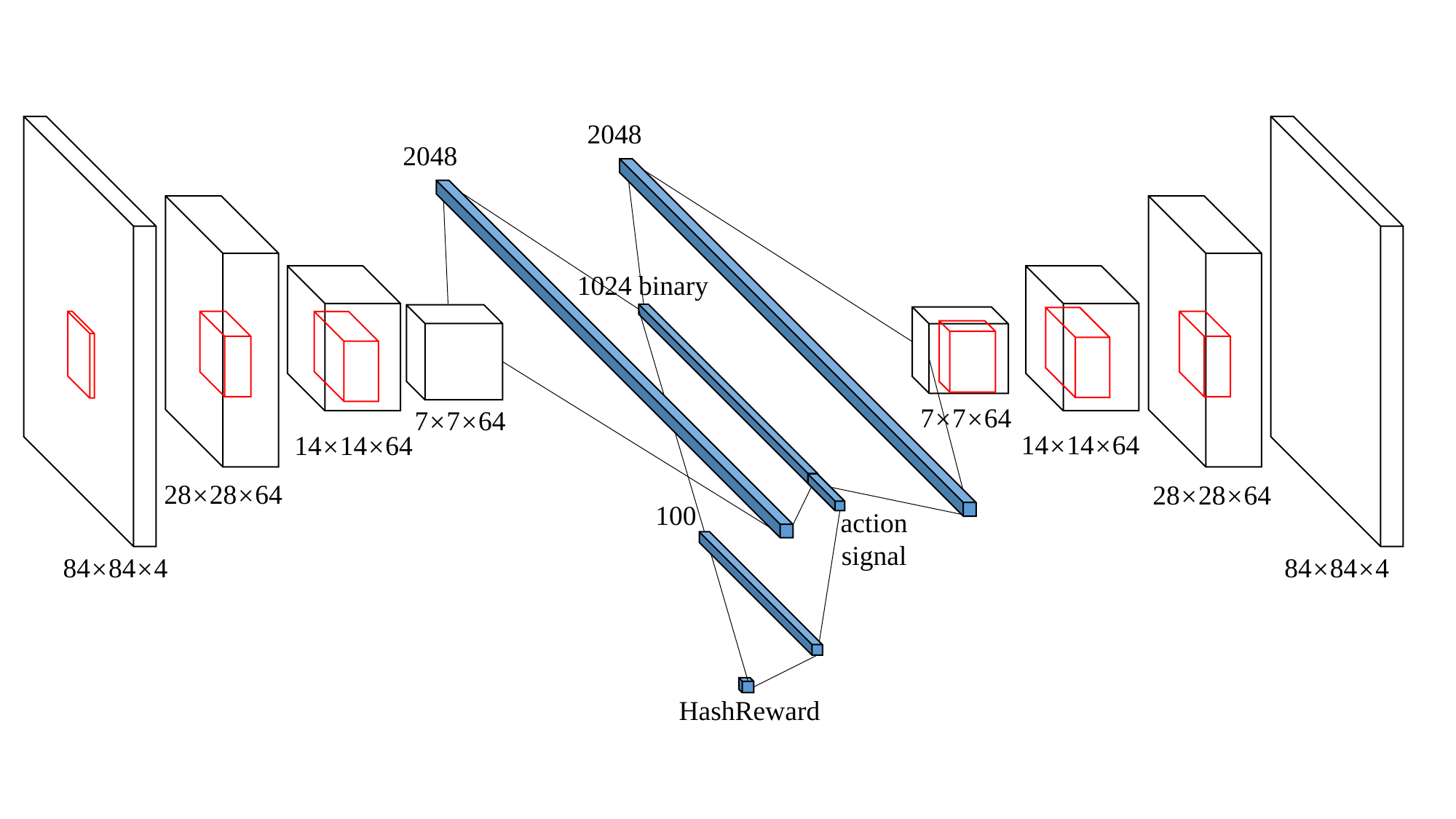}
    \caption{Illustration of the architecture for HashReward in the Atari experiment. The solid blocks represent the dense layers.}
    \label{fig:net-architecture}
\end{figure*}

\begin{figure*}[!t]
    \centering
    \subfigure{ \label{subfig:break_rest3} 
        \includegraphics[ width=0.185\textwidth]{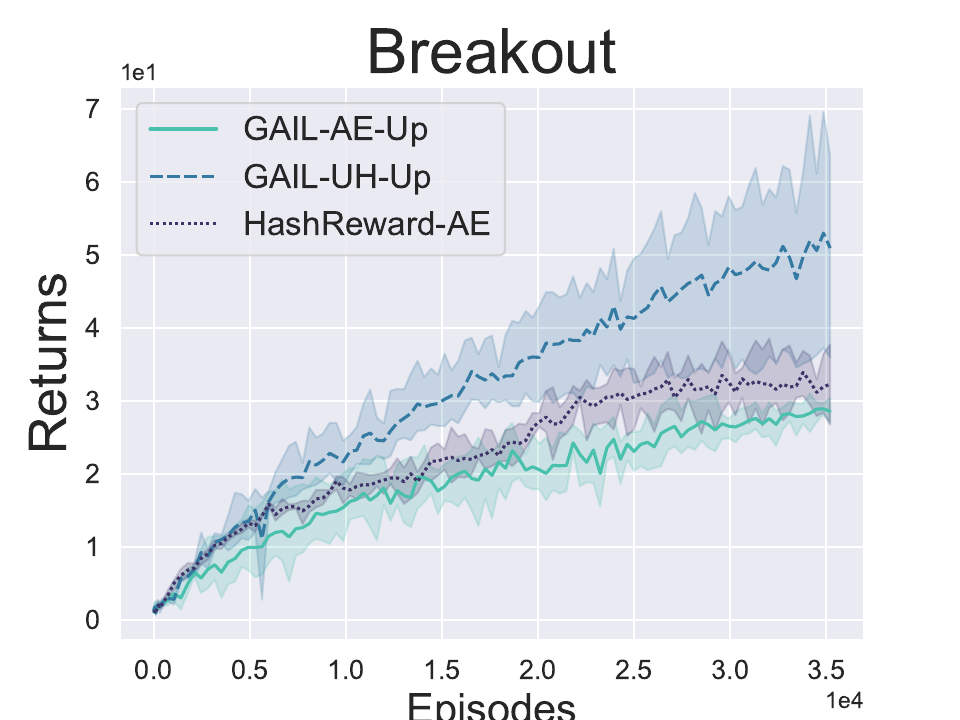}}
    \subfigure{ \label{subfig:beam_rest3} 
        \includegraphics[ width=0.185\textwidth]{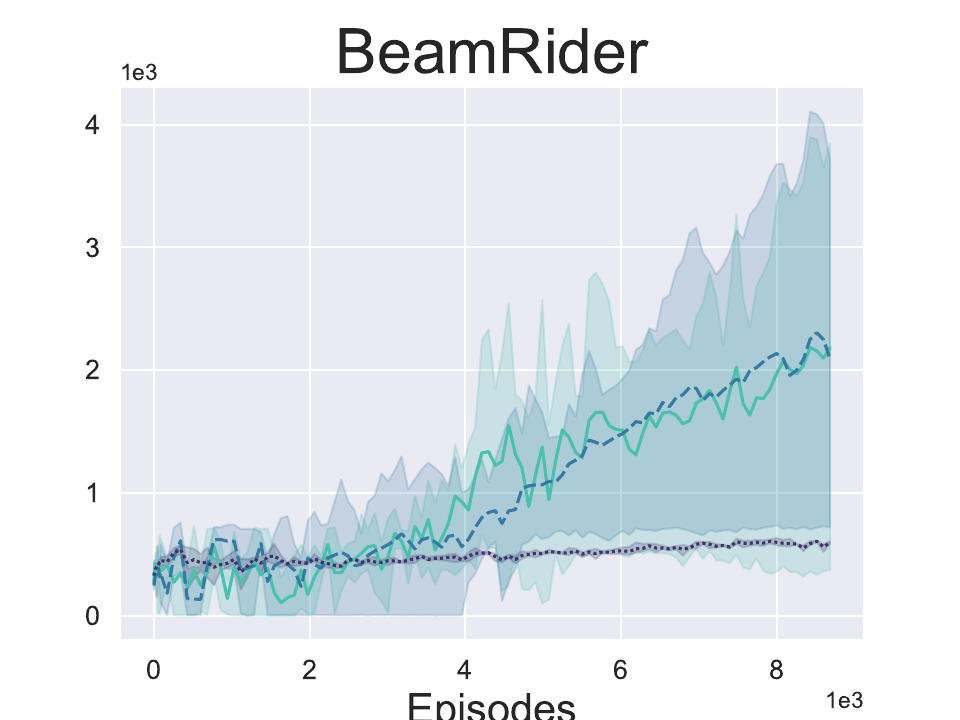}}
    \subfigure{ \label{subfig:boxing_rest3} 
        \includegraphics[ width=0.185\textwidth]{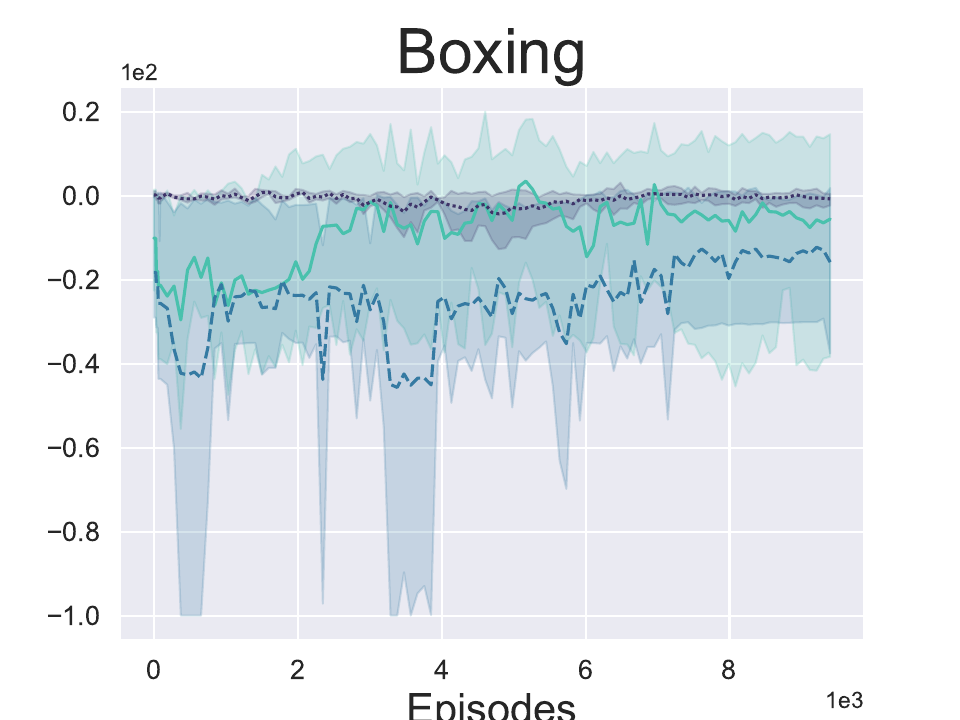}}
    \subfigure{ \label{subfig:battle_rest3} 
        \includegraphics[ width=0.185\textwidth]{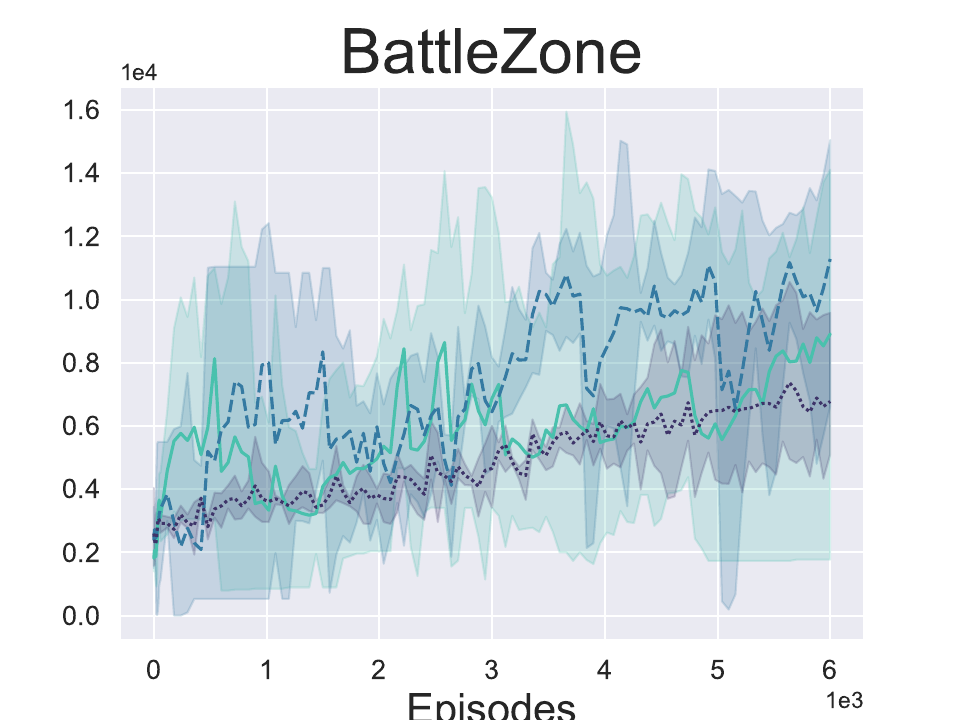}}
    \subfigure{ \label{subfig:ch_rest3} 
        \includegraphics[ width=0.185\textwidth]{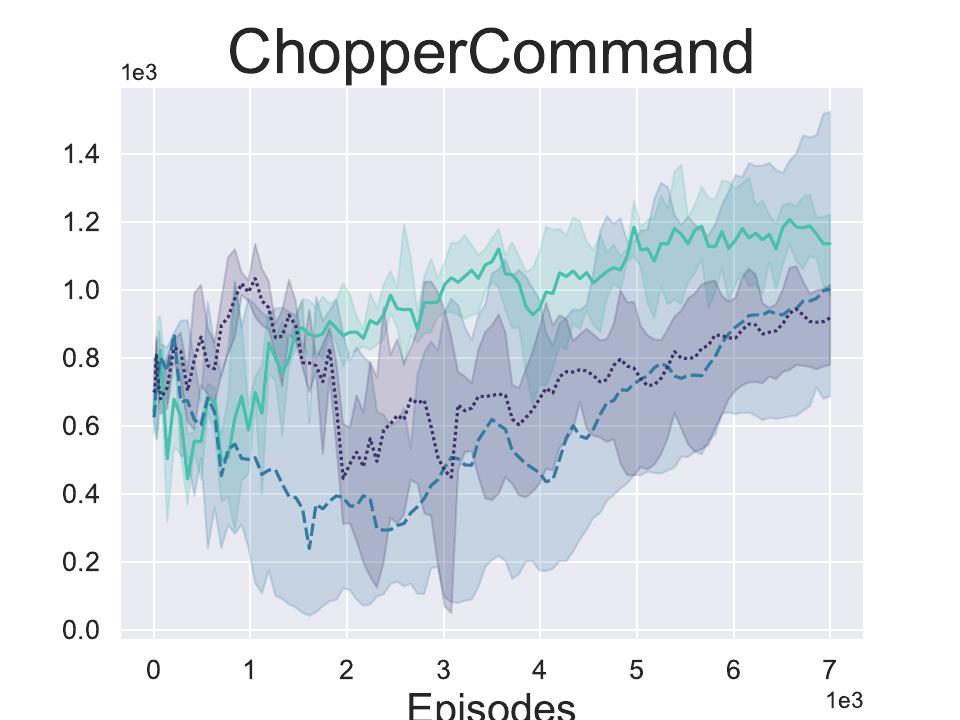}}
    \subfigure{ \label{subfig:crazy_rest3} 
        \includegraphics[ width=0.185\textwidth]{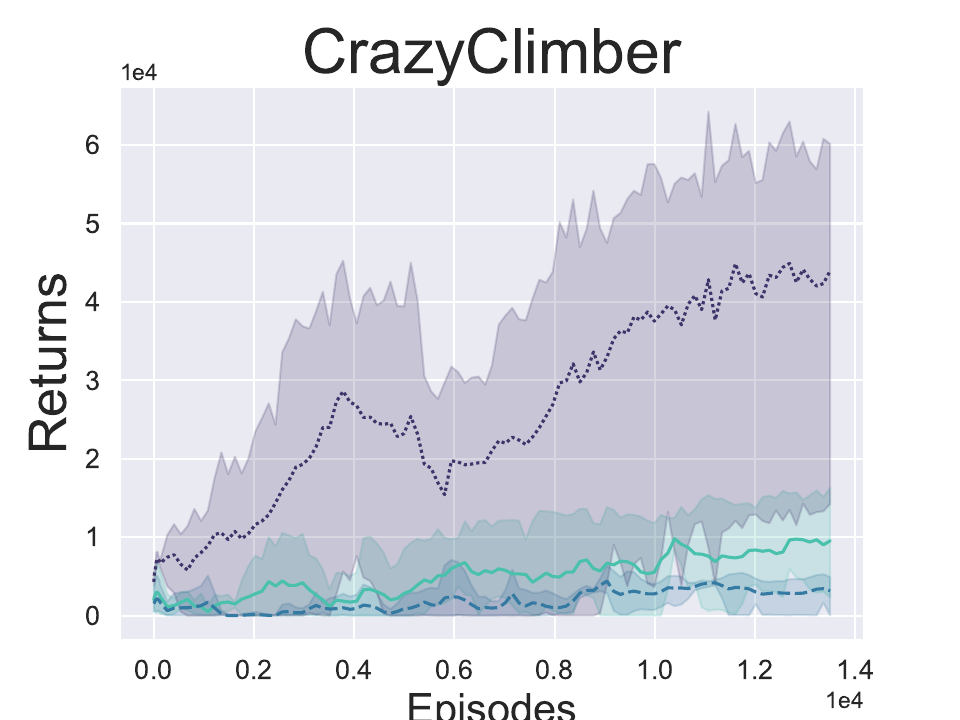}}
    \subfigure{ \label{subfig:enduro_rest3} 
        \includegraphics[ width=0.185\textwidth]{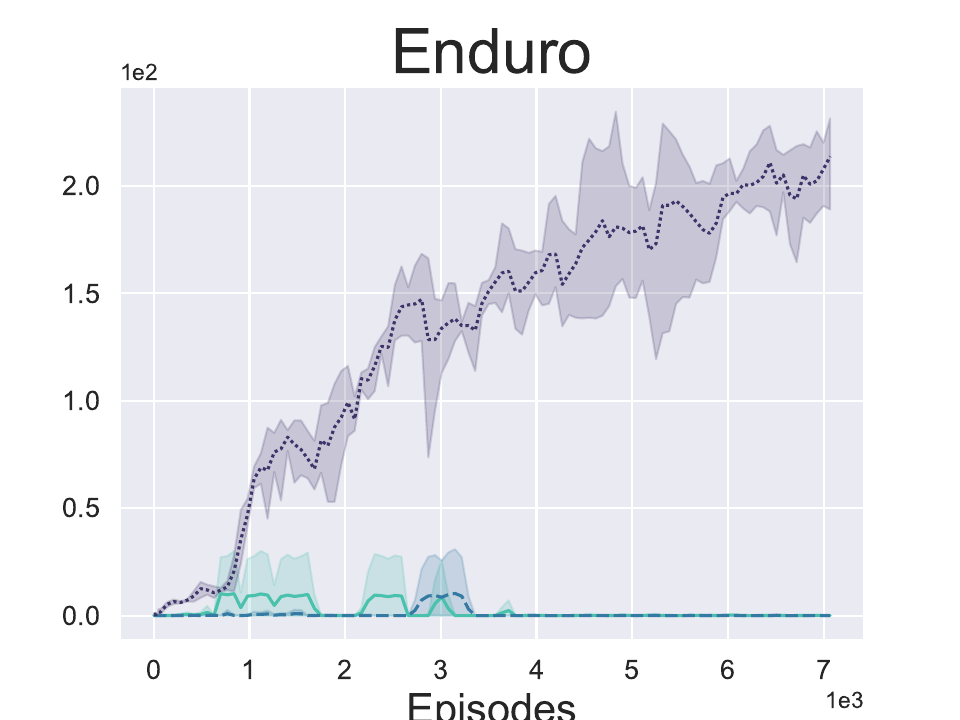}}
    \subfigure{ \label{subfig:kang_rest3} 
        \includegraphics[ width=0.185\textwidth]{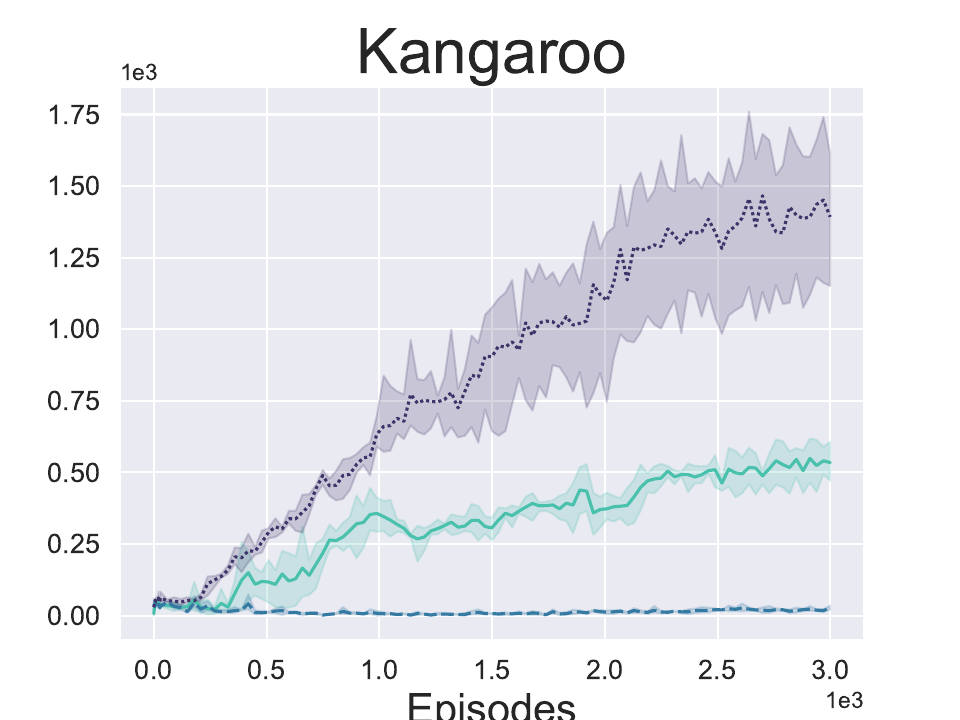}}
    \subfigure{ \label{subfig:mspacman_rest3} 
        \includegraphics[ width=0.185\textwidth]{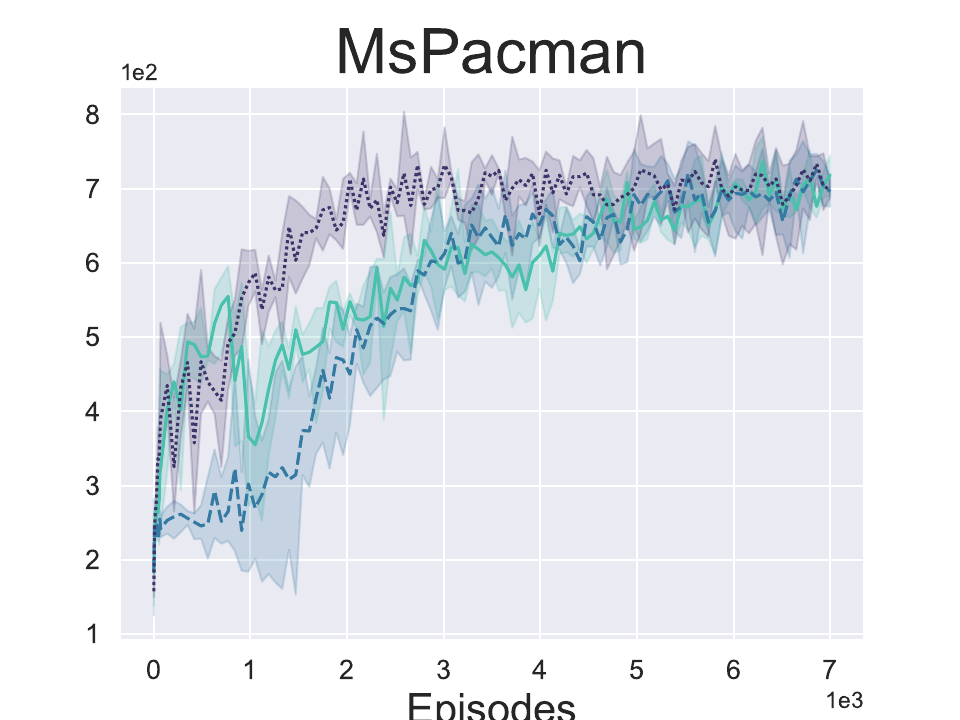}}
    \subfigure{ \label{subfig:pong_rest3} 
        \includegraphics[ width=0.185\textwidth]{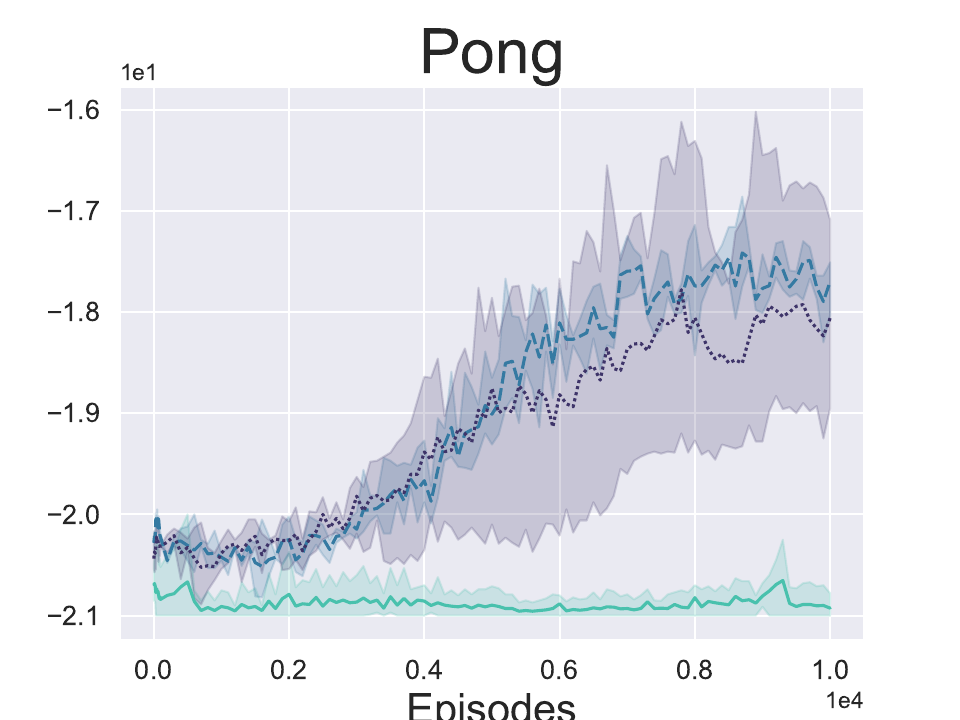}}
    \subfigure{ \label{subfig:qbert_rest3} 
        \includegraphics[ width=0.185\textwidth]{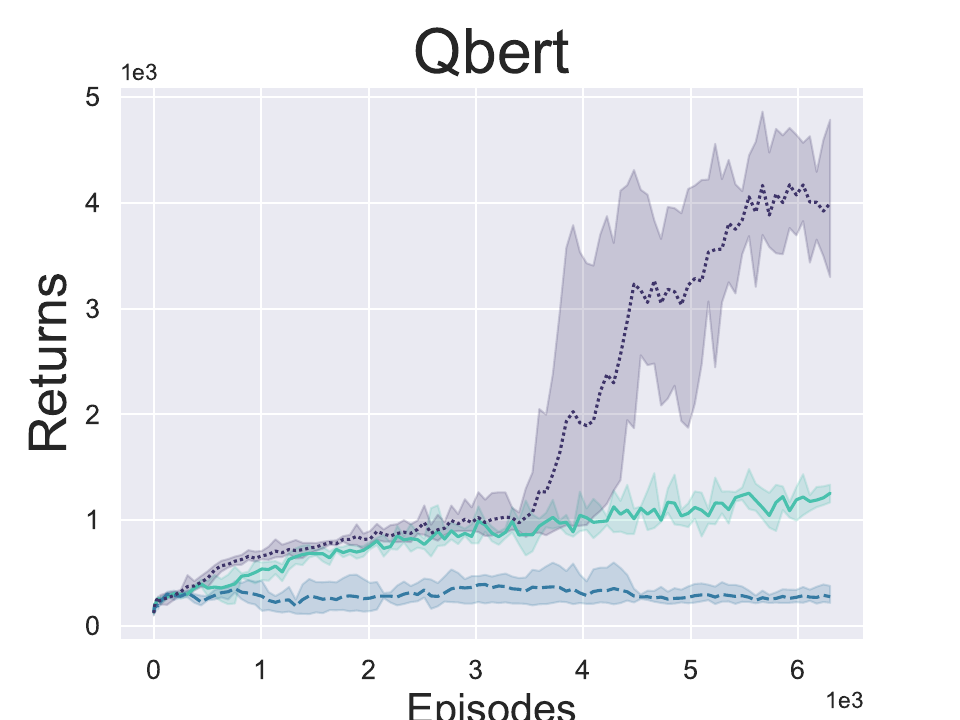}}
    \subfigure{ \label{subfig:sea_rest3} 
        \includegraphics[ width=0.185\textwidth]{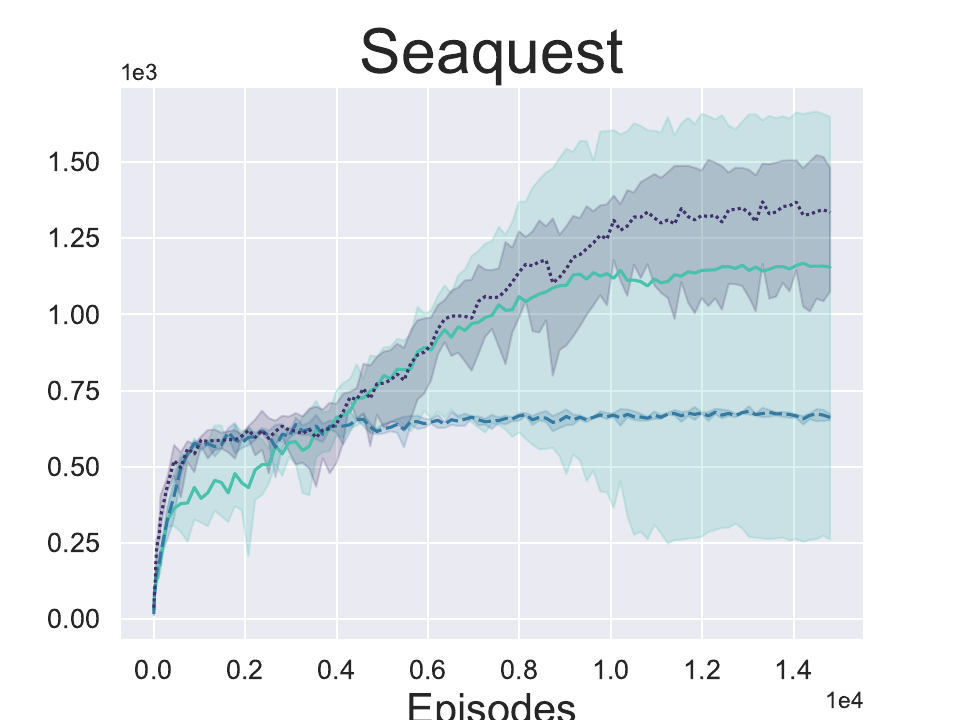}}
    \subfigure{ \label{subfig:space_rest3} 
        \includegraphics[ width=0.185\textwidth]{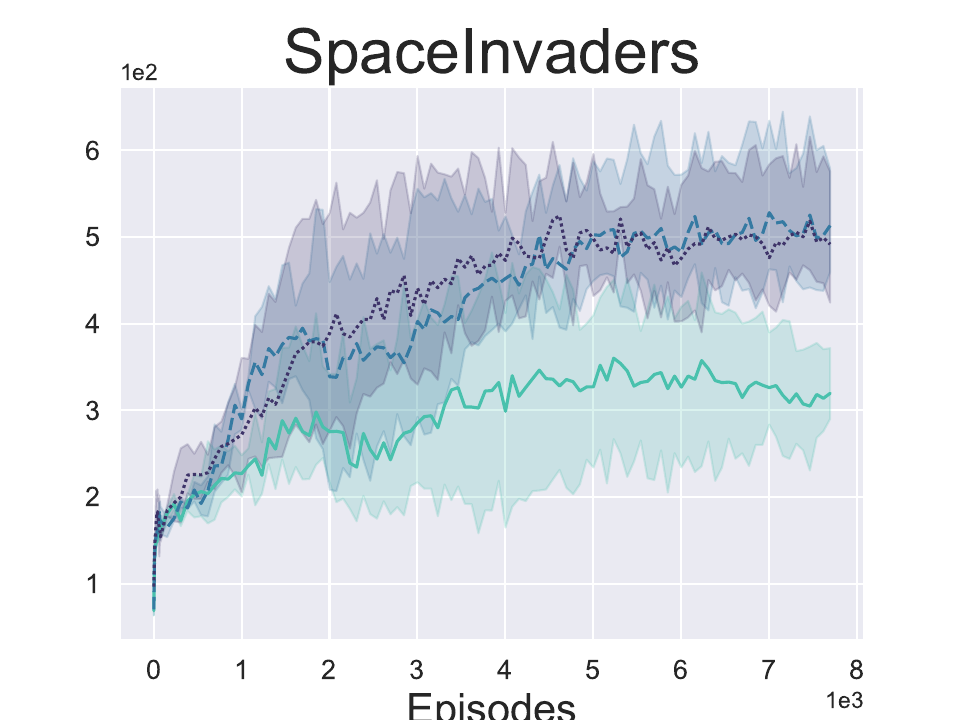}}
    \subfigure{ \label{subfig:up_rest3} 
        \includegraphics[ width=0.185\textwidth]{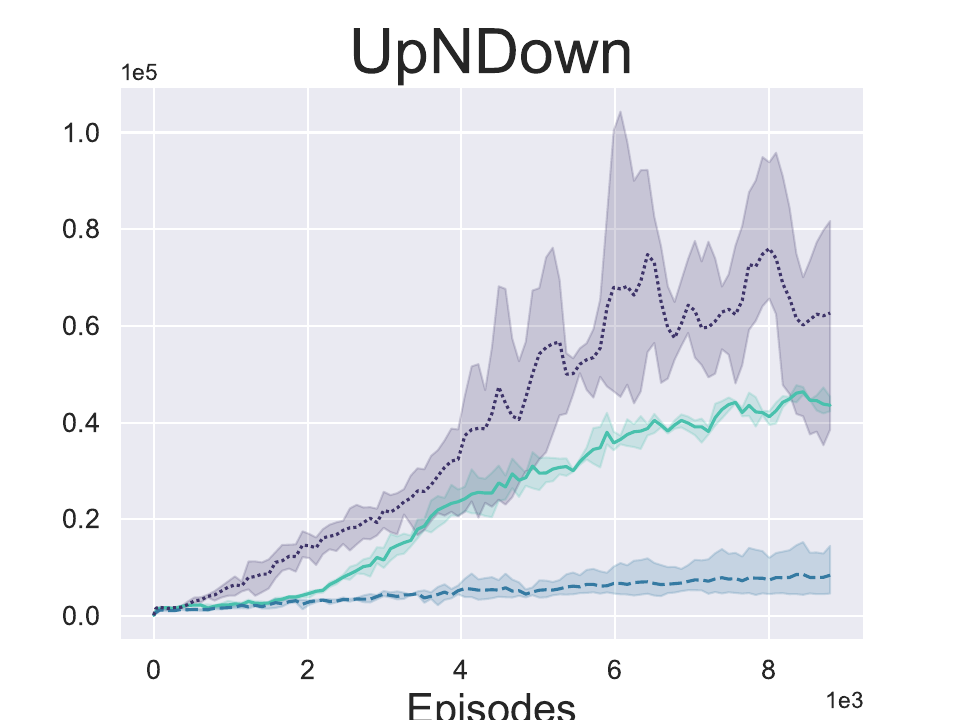}}
    \subfigure{ \label{subfig:za_rest3} 
        \includegraphics[ width=0.185\textwidth]{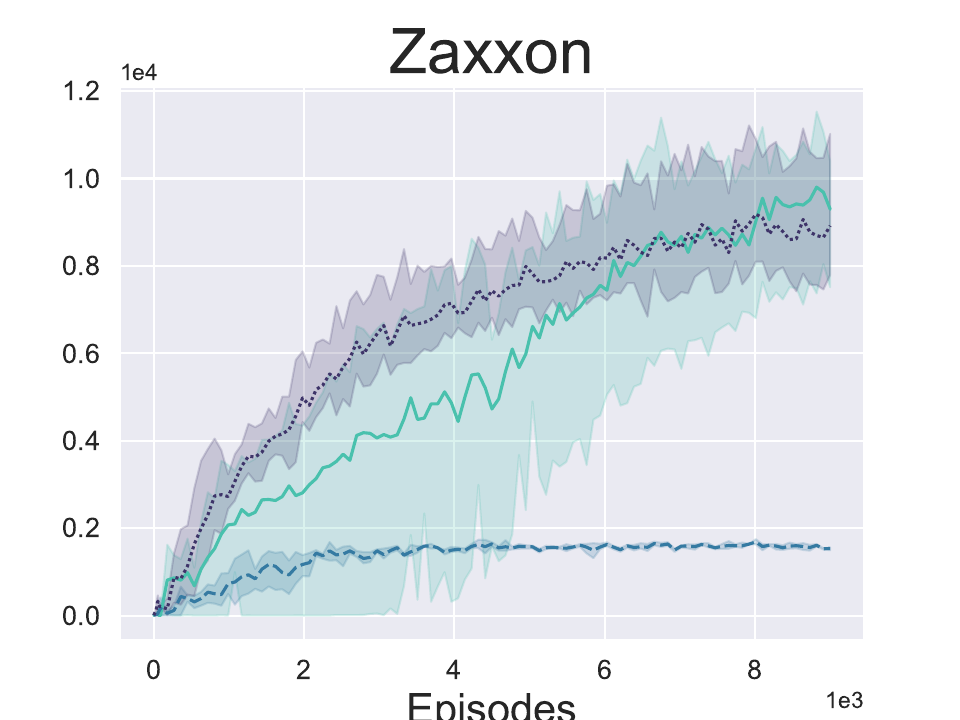}}
    \subfigure{ \label{subfig:human_rest3} 
        \includegraphics[ width=0.185\textwidth]{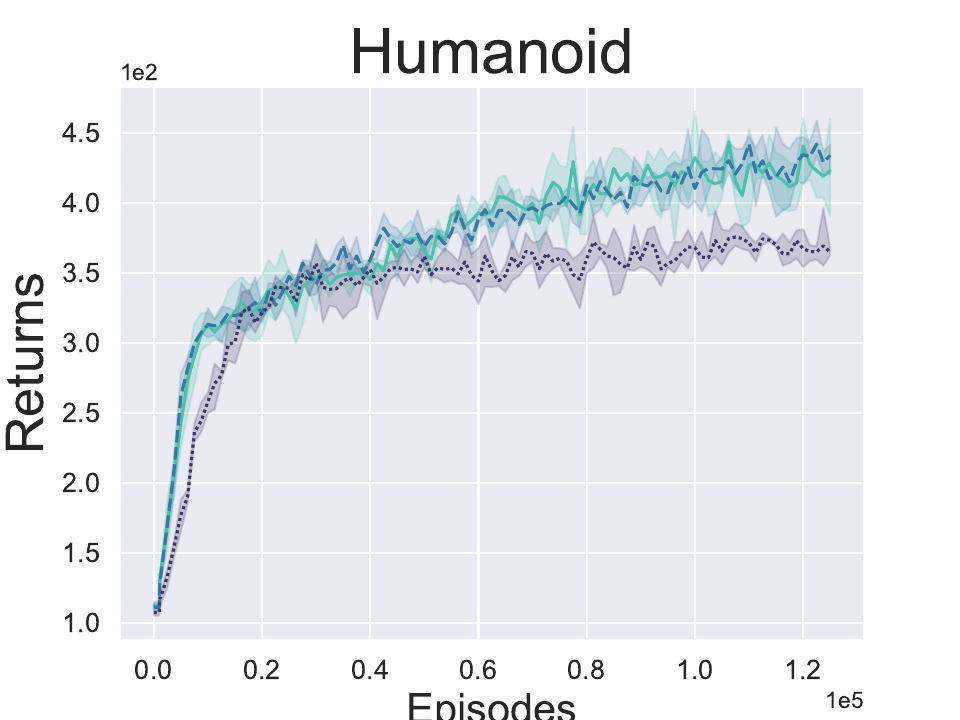}}
    \subfigure{ \label{subfig:half_pixel_rest3} 
        \includegraphics[ width=0.185\textwidth]{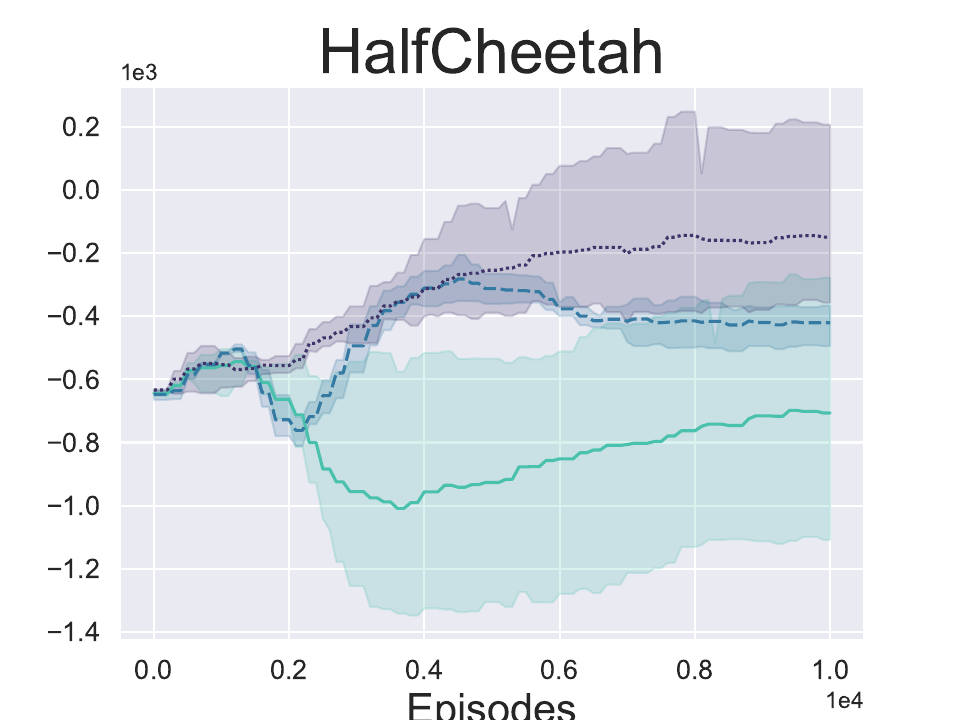}}
    \subfigure{ \label{subfig:hopper_rest3} 
        \includegraphics[ width=0.185\textwidth]{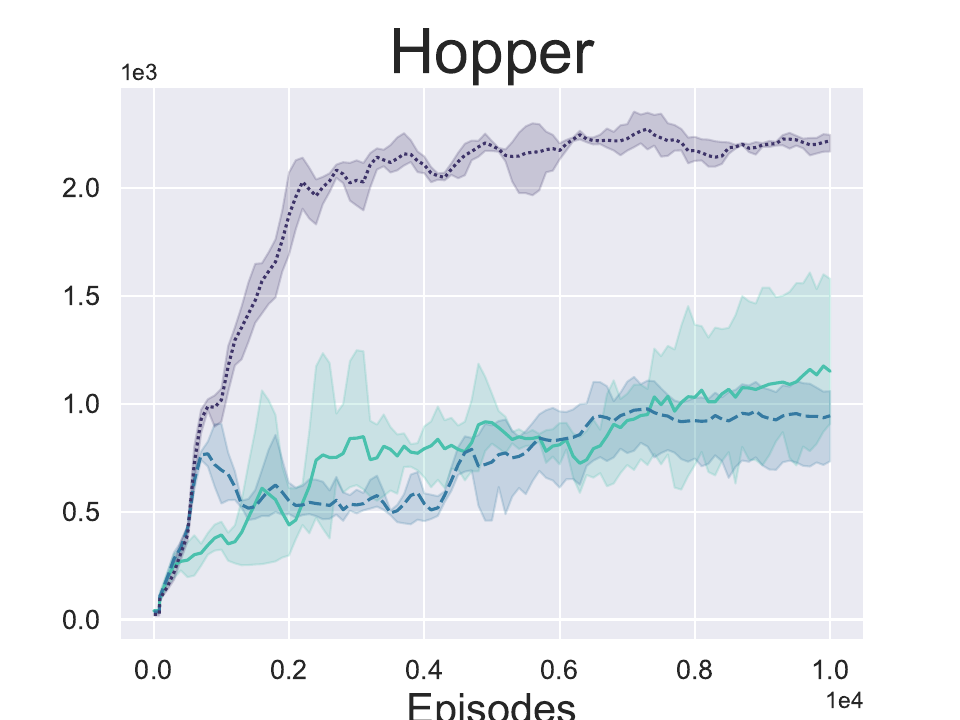}}
    \subfigure{ \label{subfig:humanstand_rest3} 
        \includegraphics[ width=0.185\textwidth]{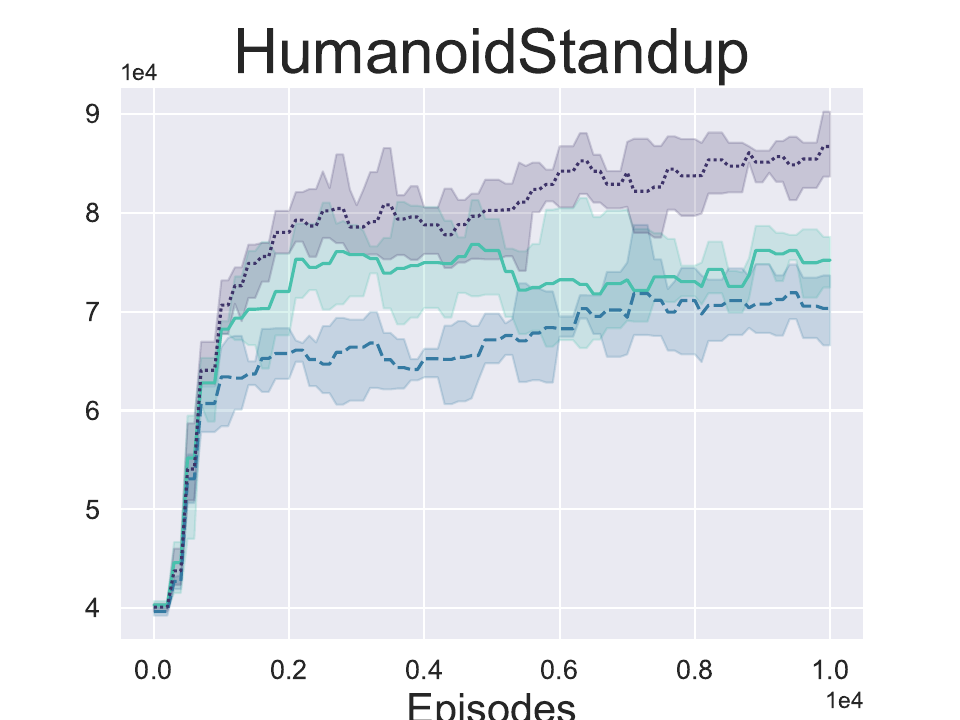}}
    \subfigure{ \label{subfig:reacher_rest3} 
        \includegraphics[ width=0.185\textwidth]{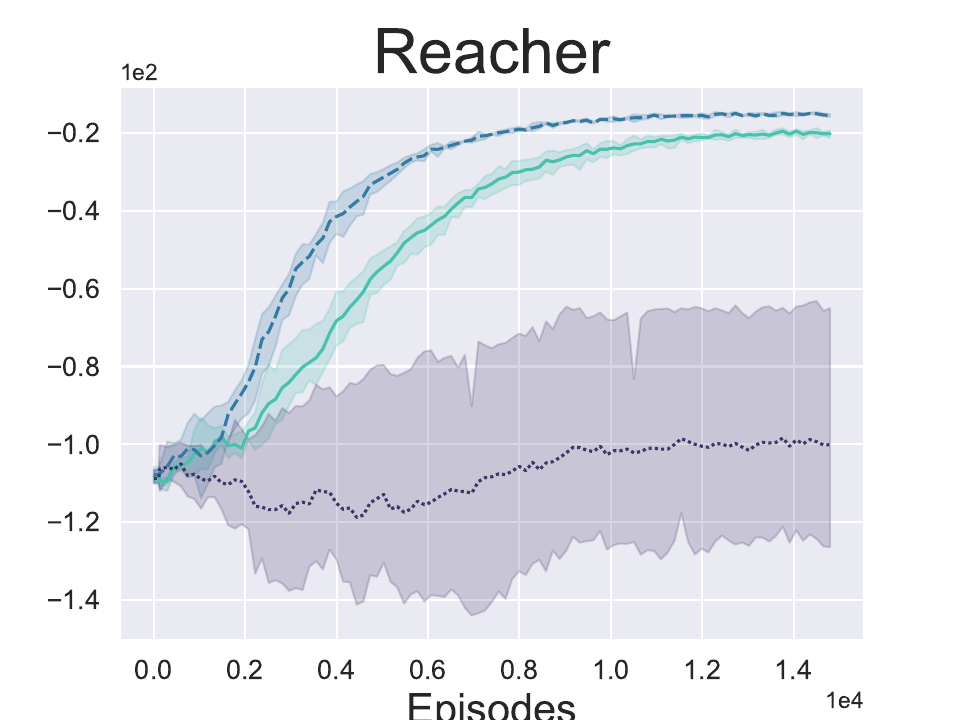}}
    \caption{Reward curves of other three approaches on all environments. 3 trials are conducted for each method on each game, and the shaded region indicates the standard deviation.}
    \label{fig:curve_three}
\end{figure*}
The neural network architecture of HashReward is shown in Figure~\ref{fig:net-architecture}. Note that the output of the hashing code layer is the hashing code $\lfloor b(s) \rceil \in \{-1, 1\}$.

The network architectures utilized for GAIL-AE, GAIL-UH and HashReward are the same. For GAIL and VAIL, the discriminator directly accepts input from original pixels, whose network architecture is the same as in Figure~\ref{fig:net-architecture} without the right hand side decoder part. 

\section{Results in Empirical Experiment}
\label{sec:result}
In this section, the empirical results on all environments as well as the pseudo-true reward curves are reported.
\subsection{Reward Curves}
The reward curves of other three methods (GAIL-AE-Up, GAIL-UH-Up and HashReward-AE) on all environments are reported in Figure~\ref{fig:curve_three}.

\subsection{Pseudo-True Reward Curves}
The comparisons of pseudo reward from all adversary-based approaches on all environments are shown in the figures in the last 4 pages. The blue curves denote pseudo reward provided by reward function. The red ones denote the true reward. The left y-axis indicates the pseudo reward, and the right one indicates the true reward. We can observe consistent phenomena as in the main paper. 


\begin{figure*}[!t]
    \centering
    \subfigure[Breakout]{
    \begin{minipage}[b]{0.185\linewidth}
        \includegraphics[ width=1\textwidth]{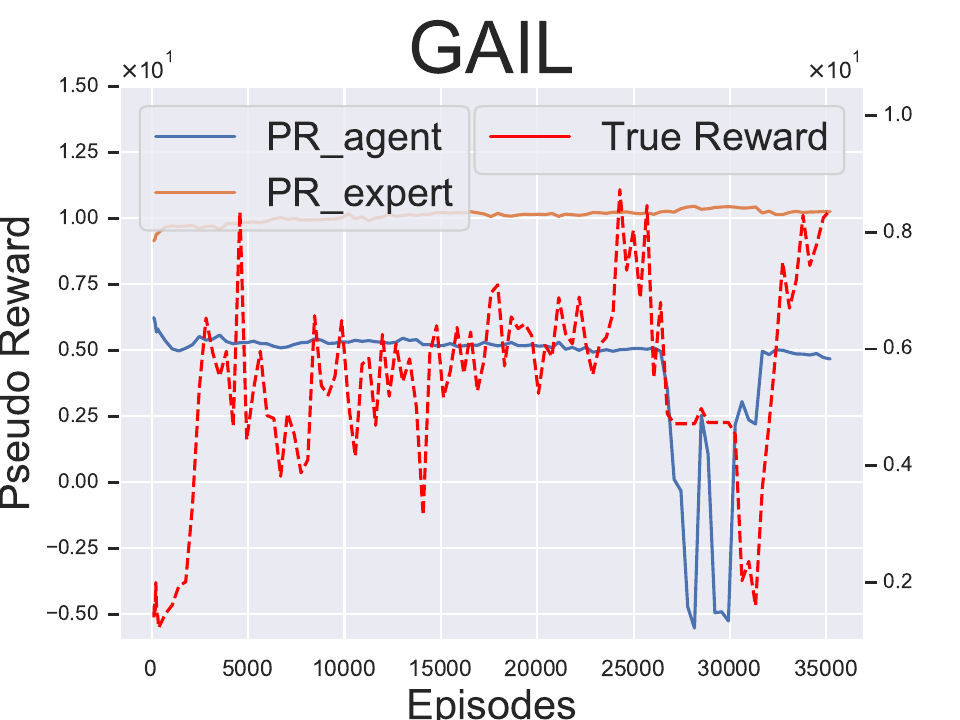}
        \includegraphics[ width=1\textwidth]{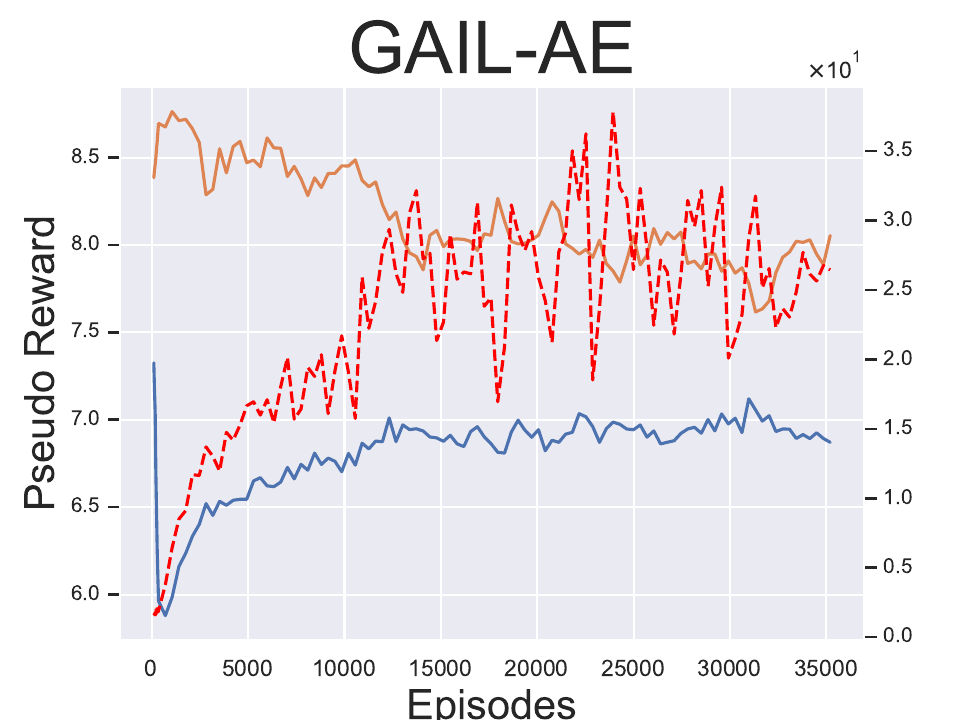}
        \includegraphics[ width=1\textwidth]{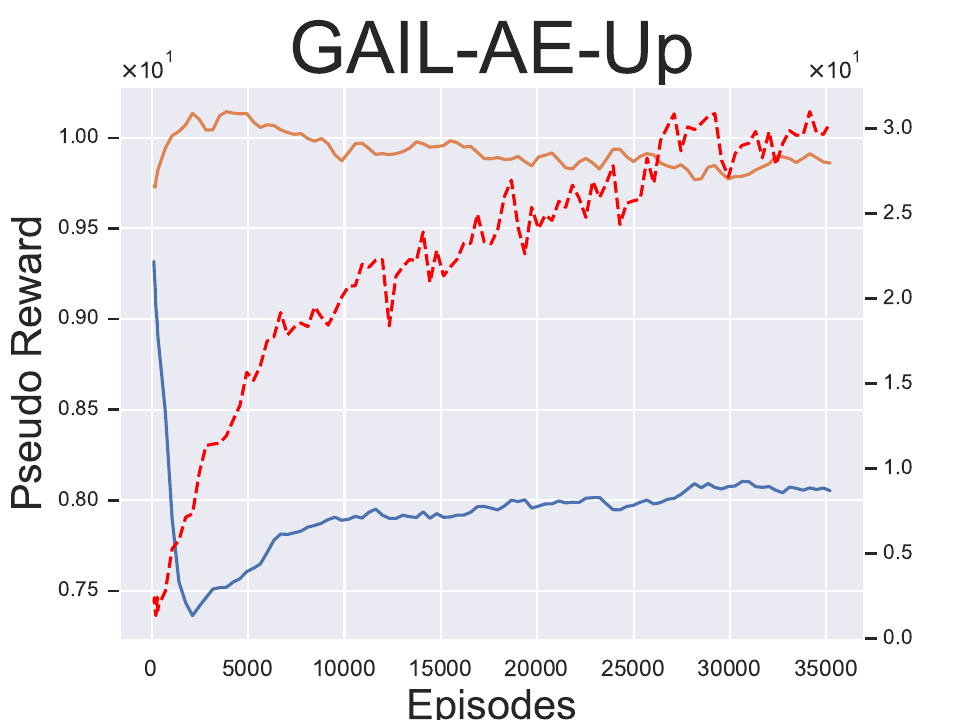}
        \includegraphics[ width=1\textwidth]{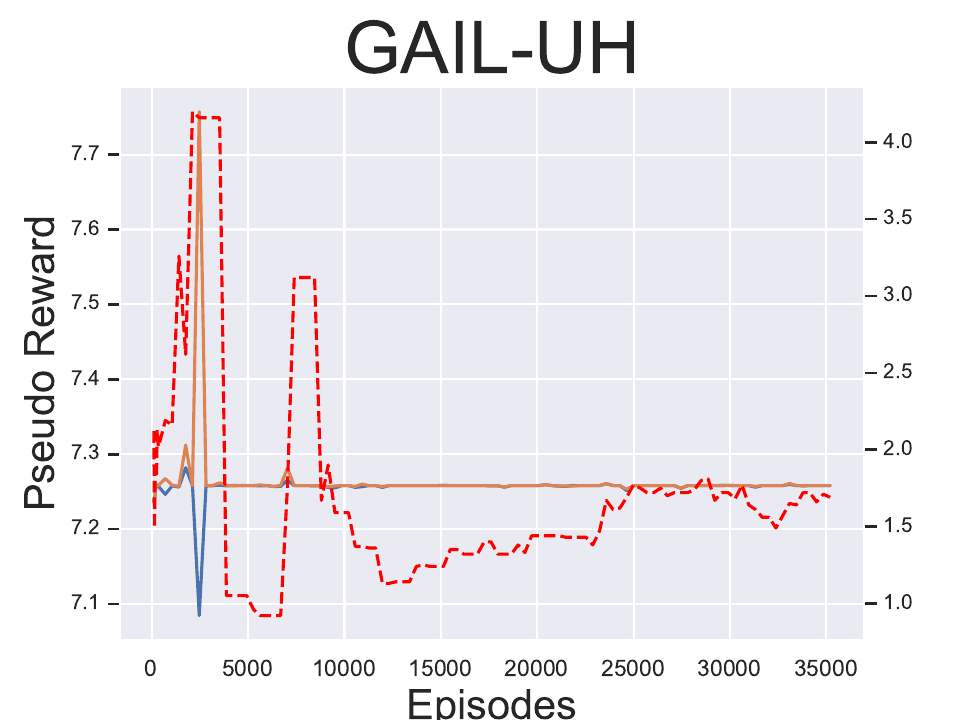}
        \includegraphics[ width=1\textwidth]{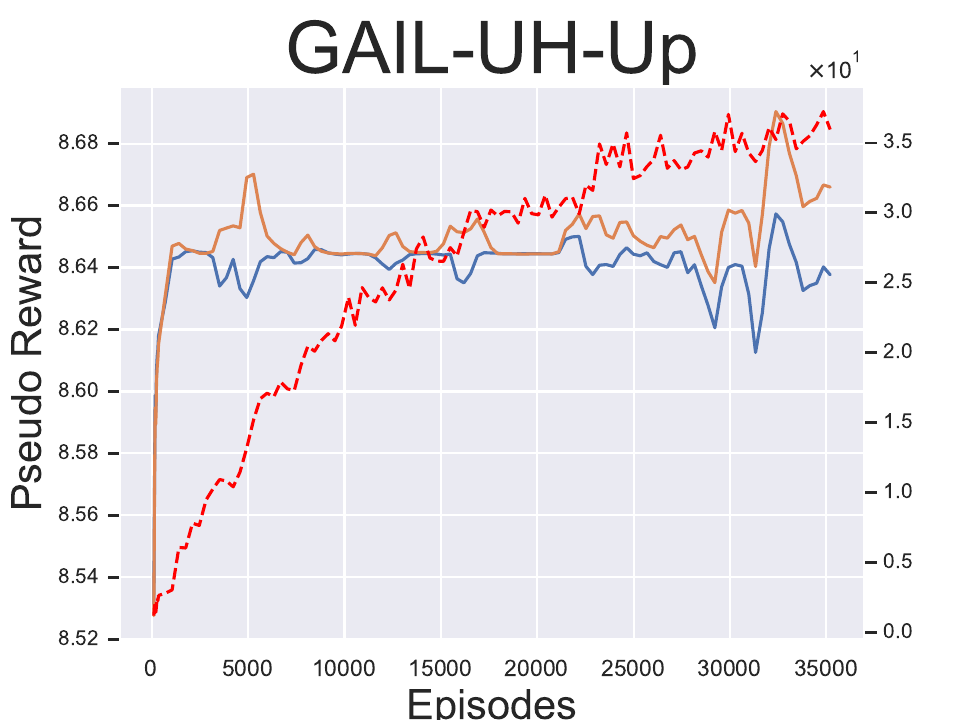}
        \includegraphics[ width=1\textwidth]{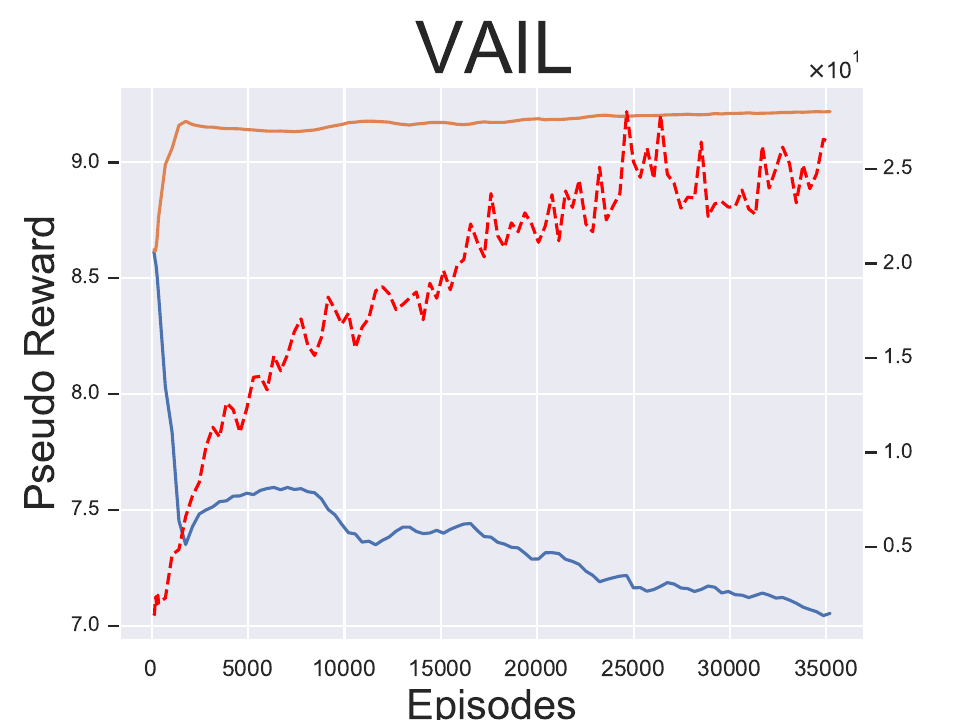}
        \includegraphics[ width=1\textwidth]{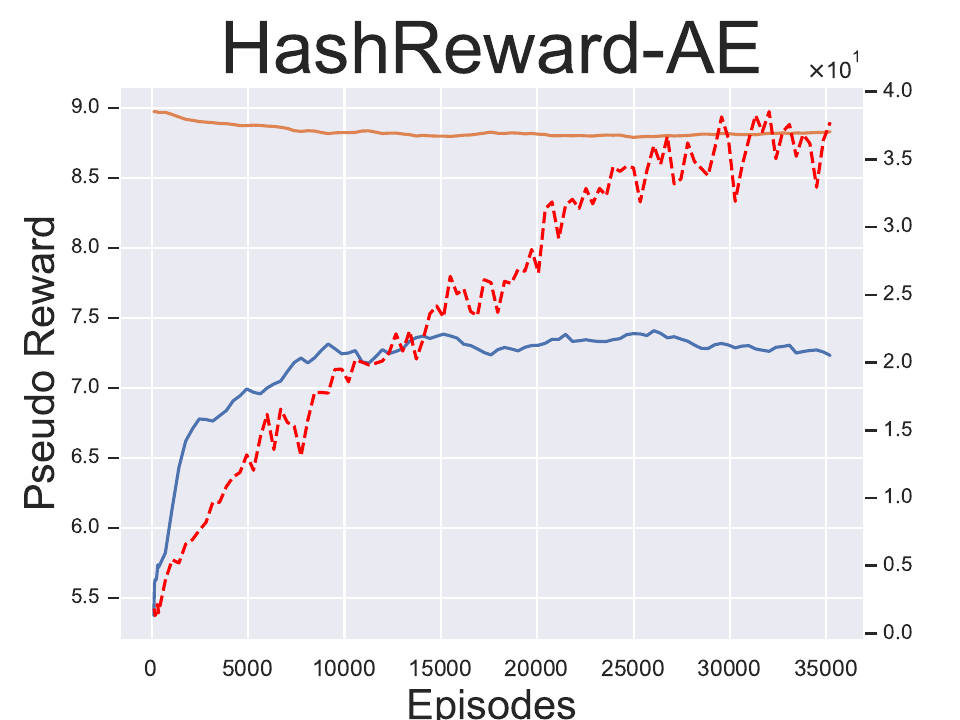}
        \includegraphics[ width=1\textwidth]{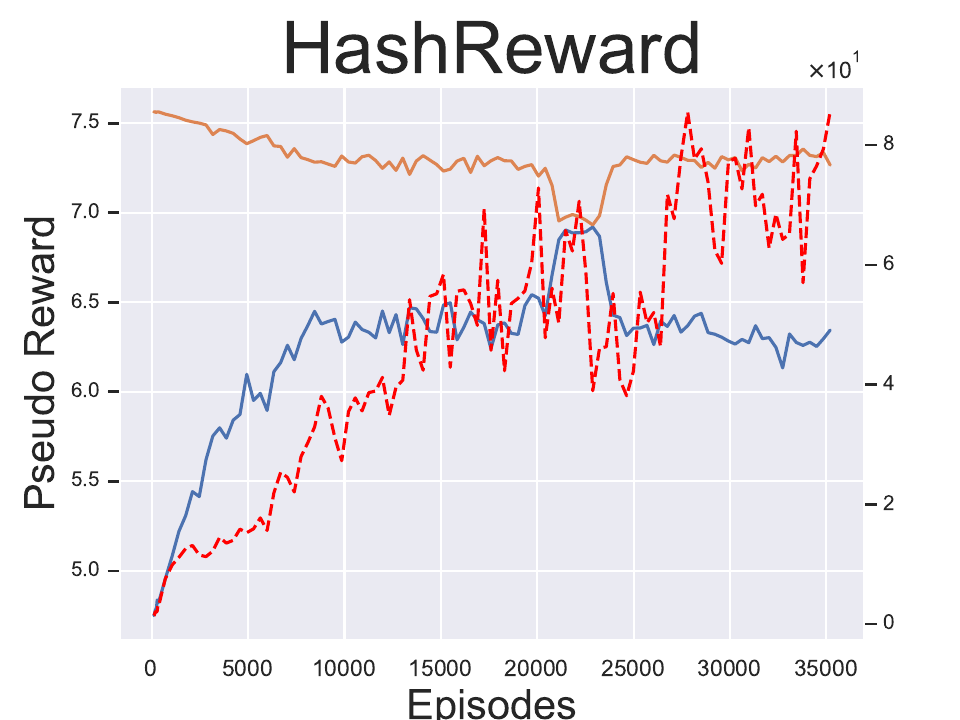}
    \end{minipage}}
    \subfigure[BeamRider]{
    \begin{minipage}[b]{0.185\linewidth}
        \includegraphics[ width=1\textwidth]{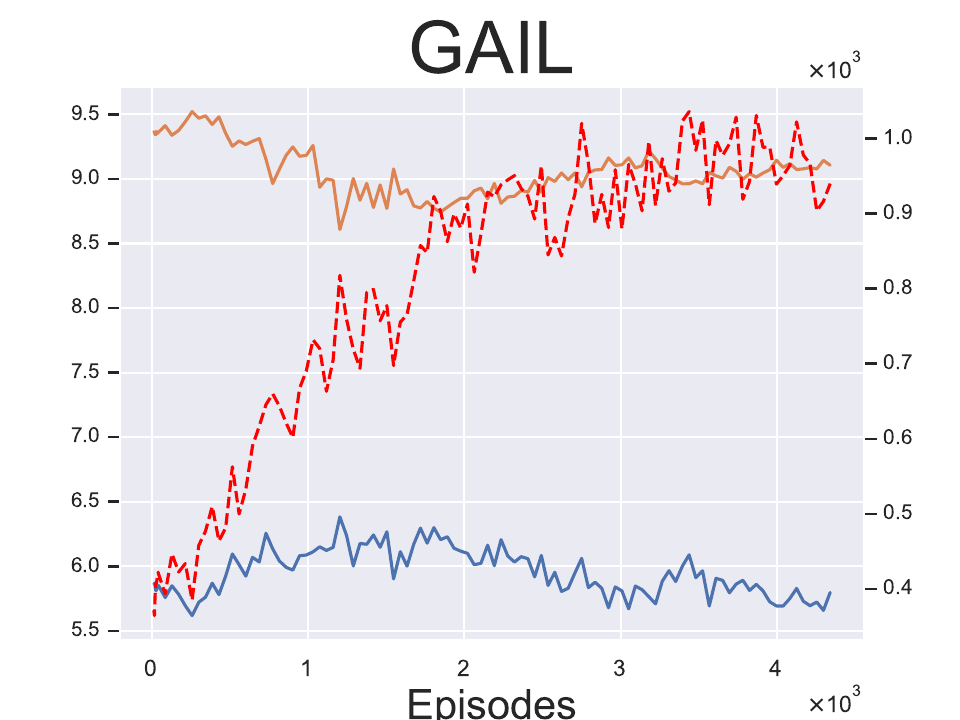}
        \includegraphics[ width=1\textwidth]{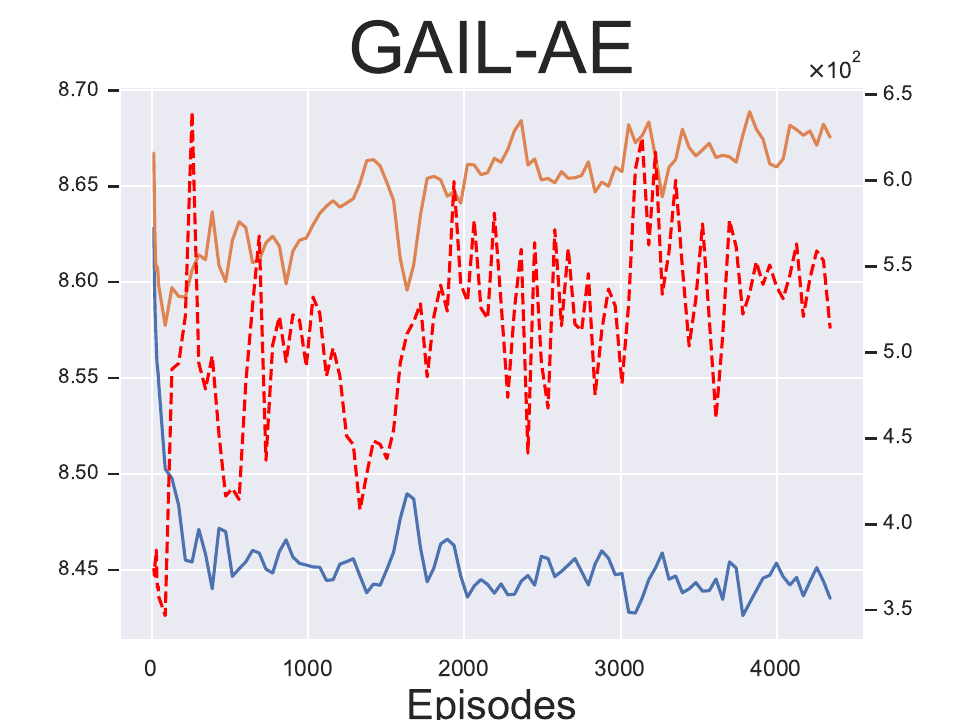}
        \includegraphics[ width=1\textwidth]{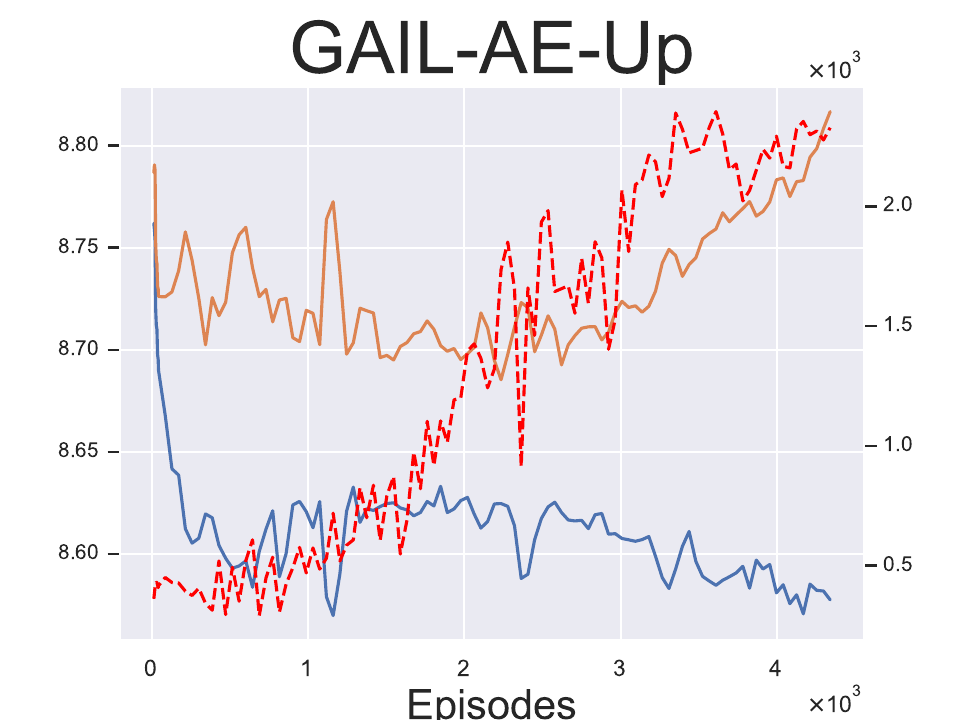}
        \includegraphics[ width=1\textwidth]{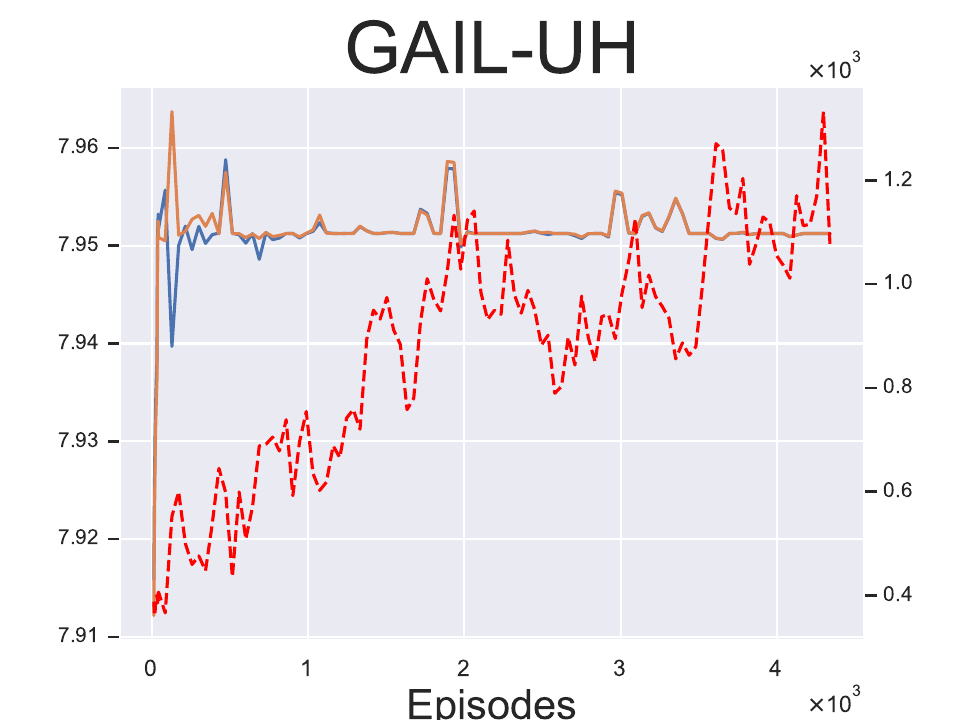}
        \includegraphics[ width=1\textwidth]{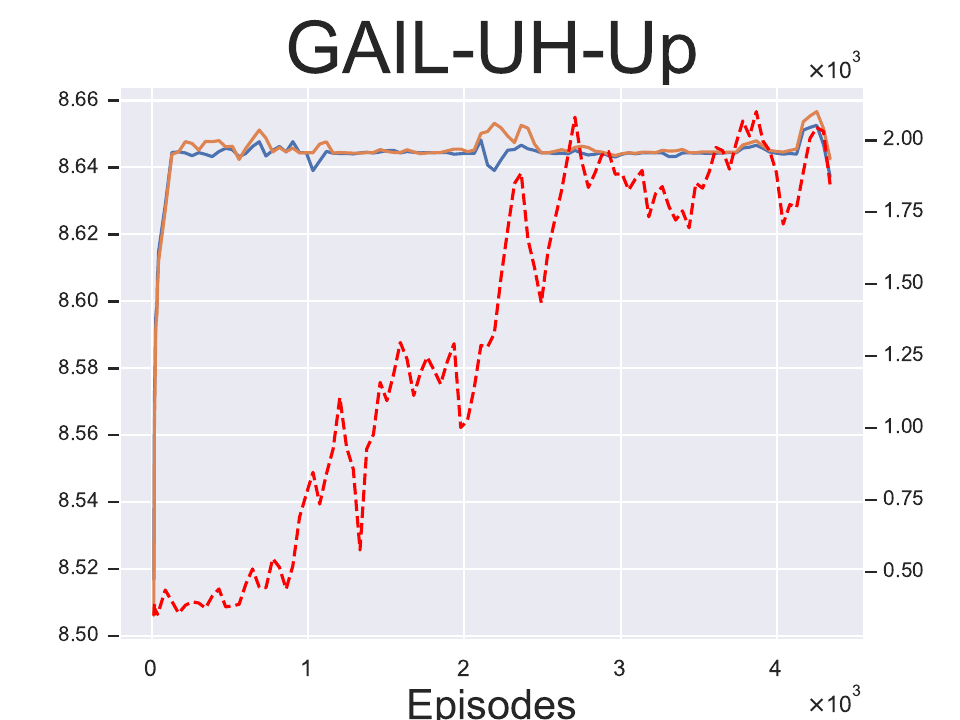}
        \includegraphics[ width=1\textwidth]{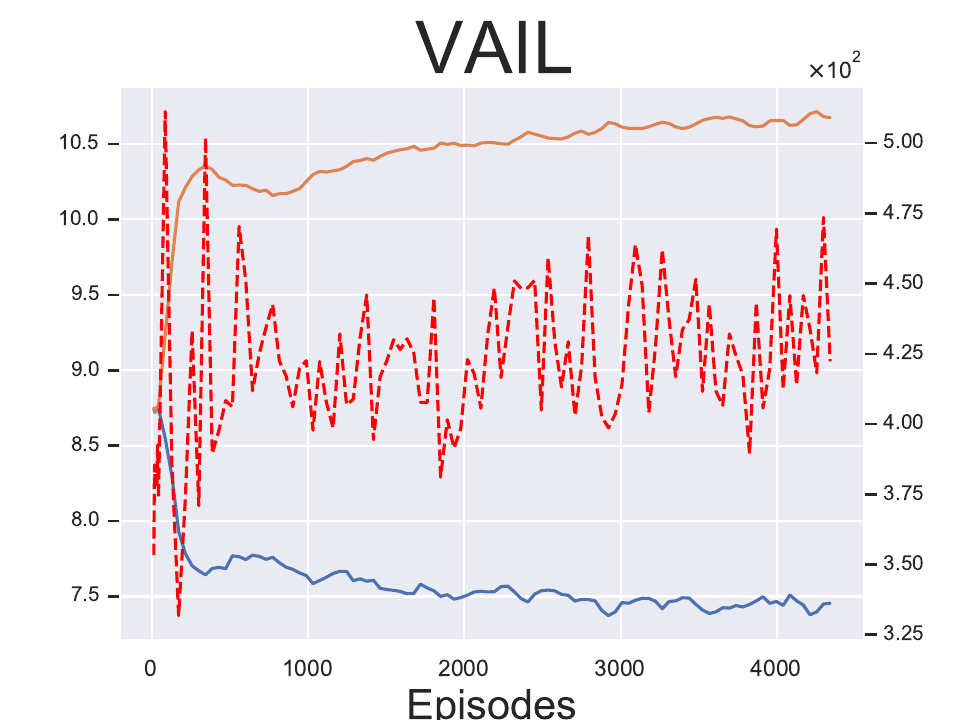}
        \includegraphics[ width=1\textwidth]{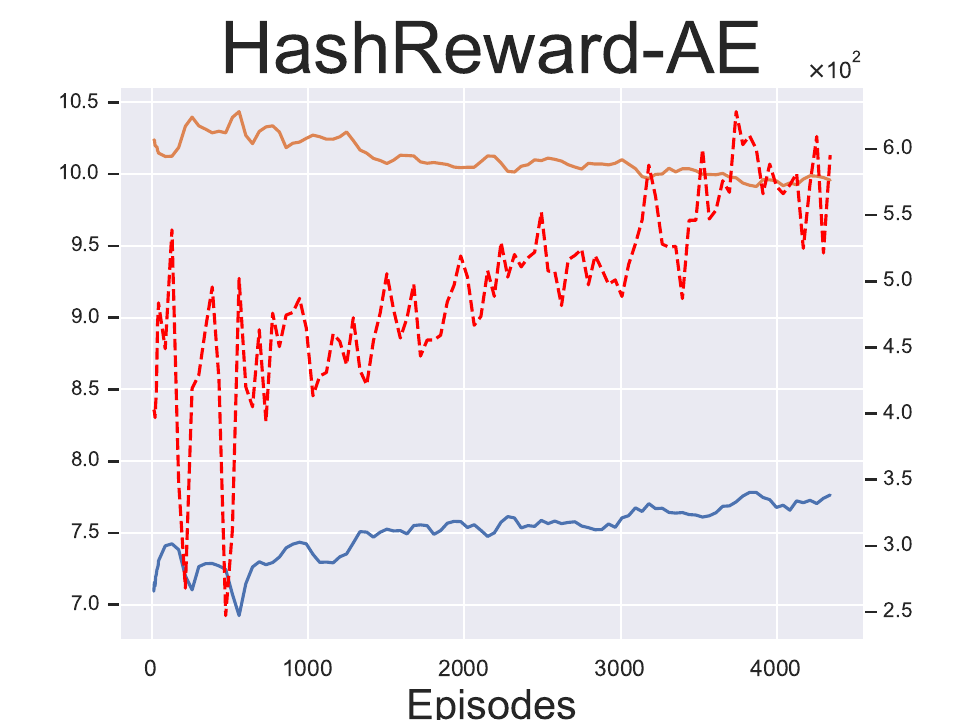}
        \includegraphics[ width=1\textwidth]{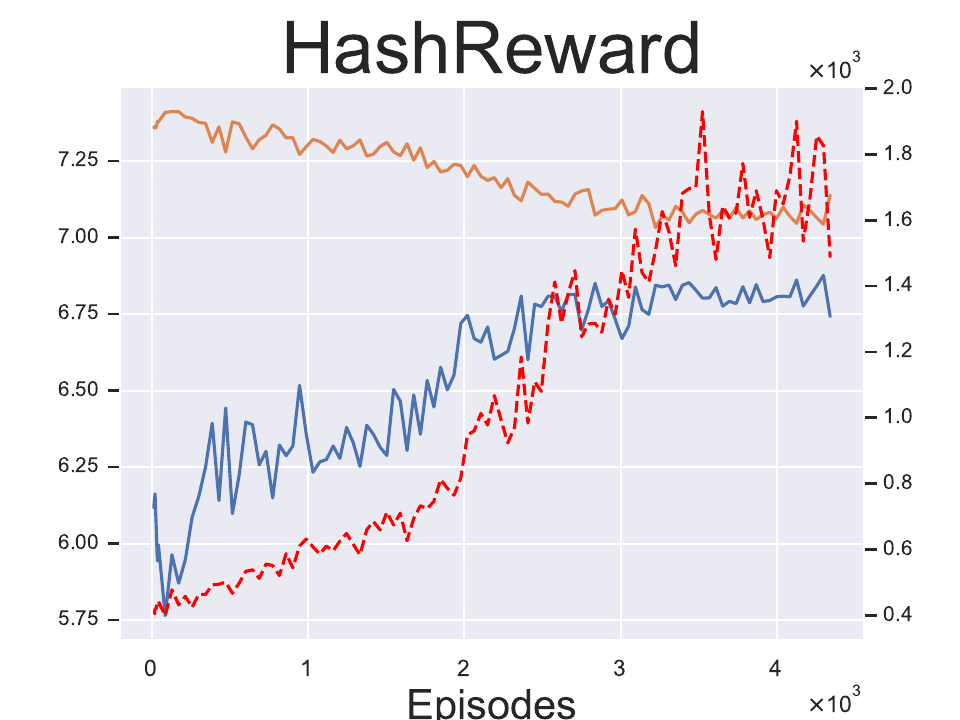}
    \end{minipage}}
    \subfigure[Boxing]{
    \begin{minipage}[b]{0.185\linewidth}
        \includegraphics[ width=1\textwidth]{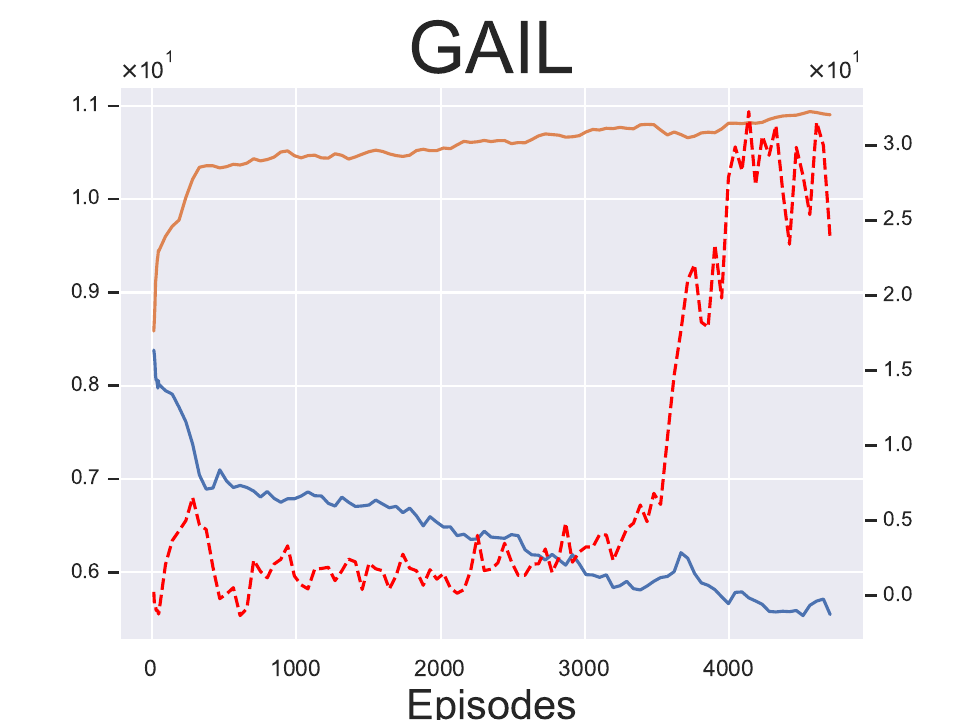}
        \includegraphics[ width=1\textwidth]{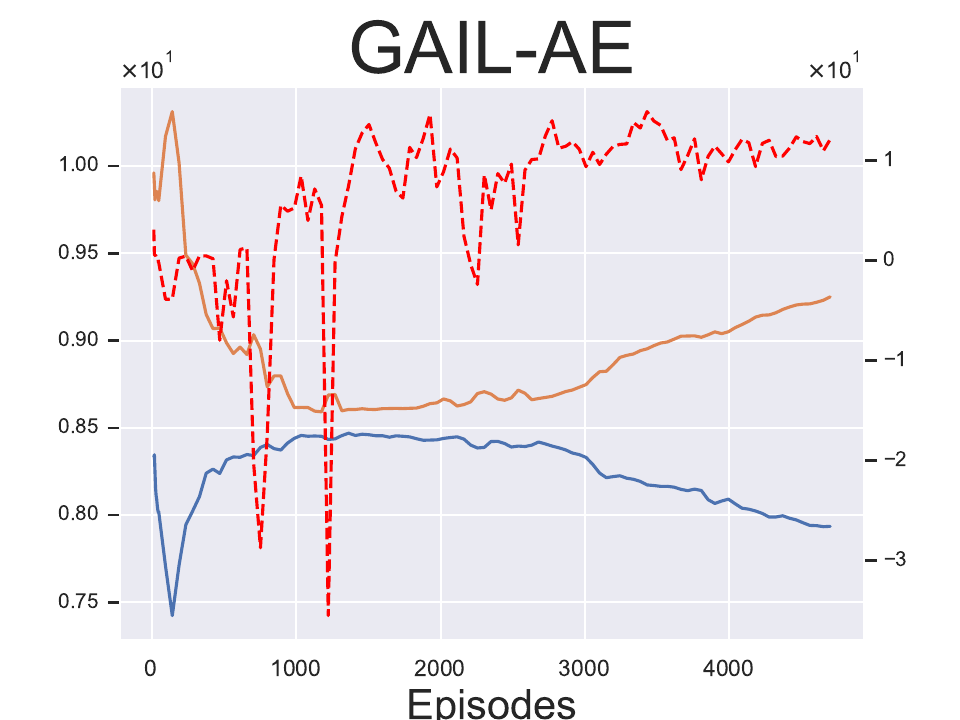}
        \includegraphics[ width=1\textwidth]{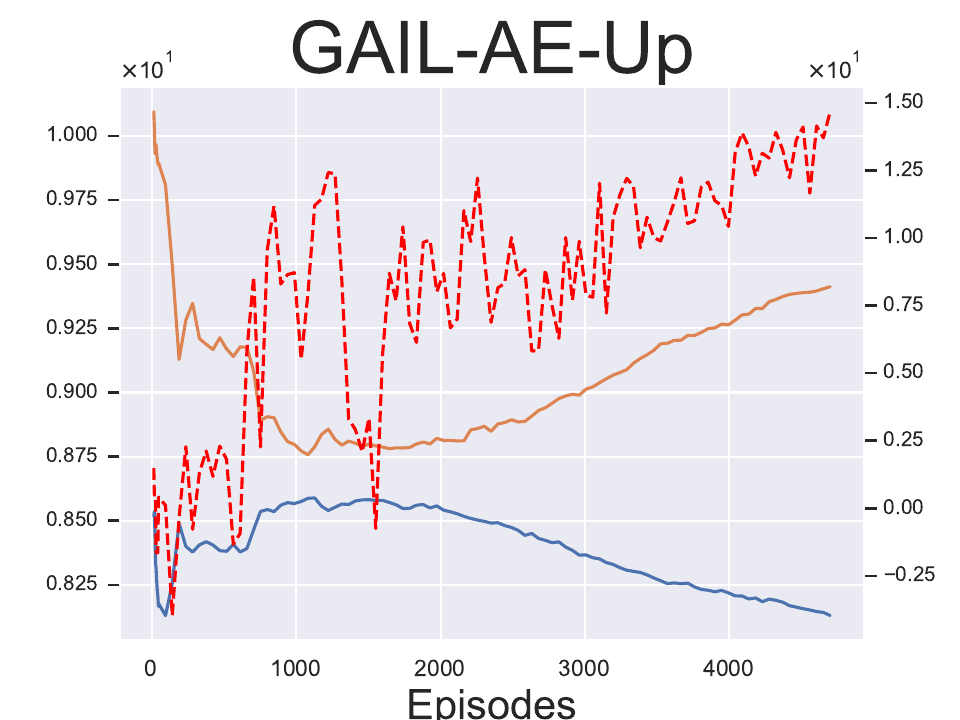}
        \includegraphics[ width=1\textwidth]{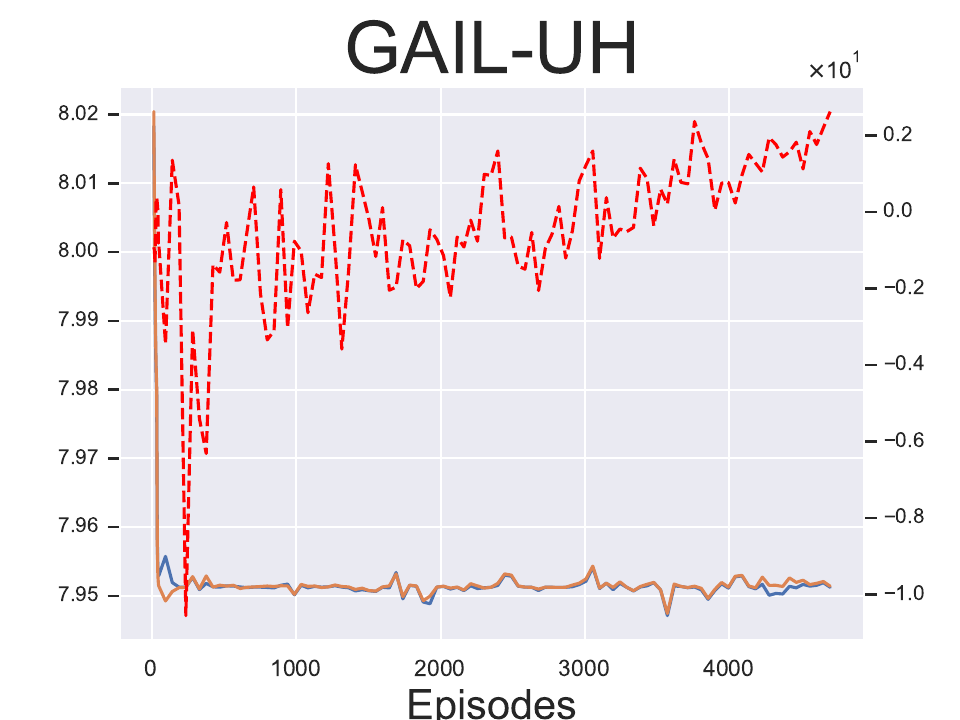}
        \includegraphics[ width=1\textwidth]{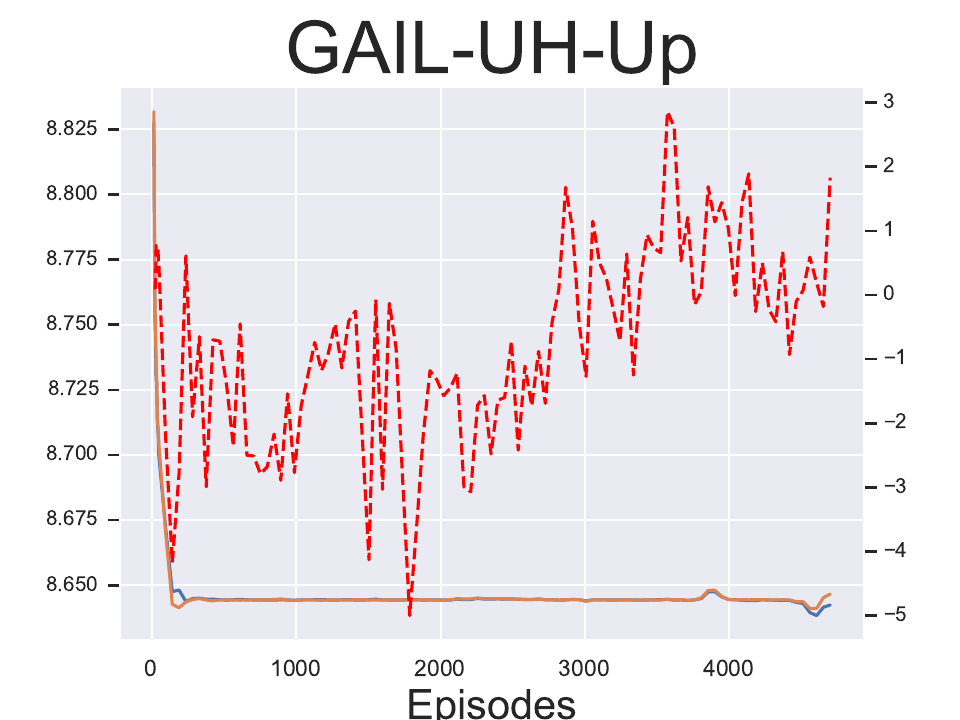}
        \includegraphics[ width=1\textwidth]{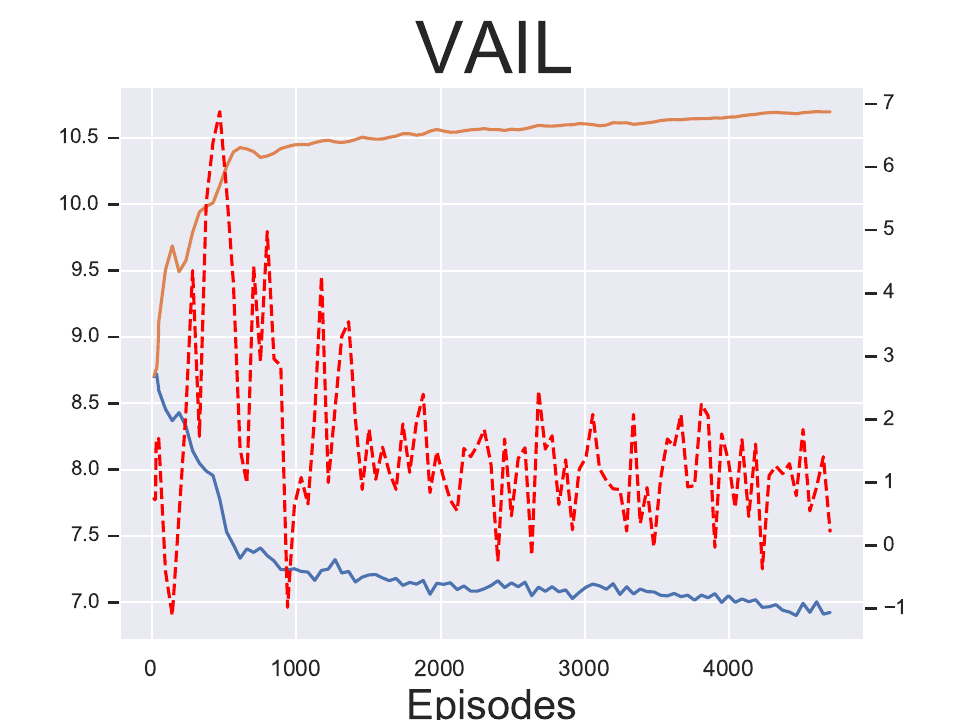}
        \includegraphics[ width=1\textwidth]{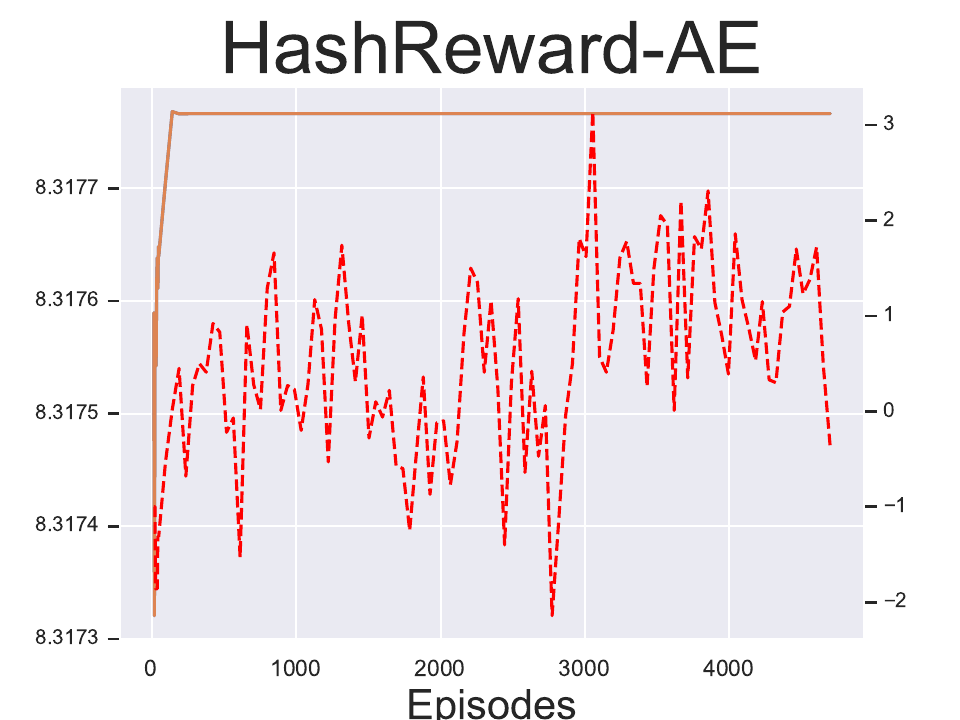}
        \includegraphics[ width=1\textwidth]{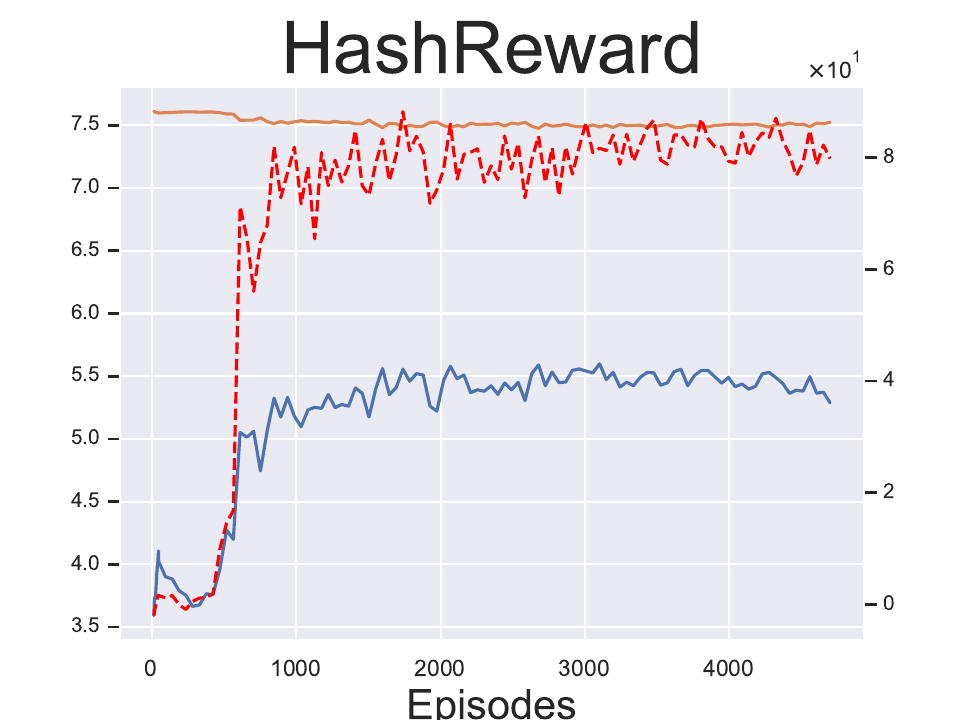}
    \end{minipage}}
    \subfigure[BattleZone]{
    \begin{minipage}[b]{0.185\linewidth}
        \includegraphics[ width=1\textwidth]{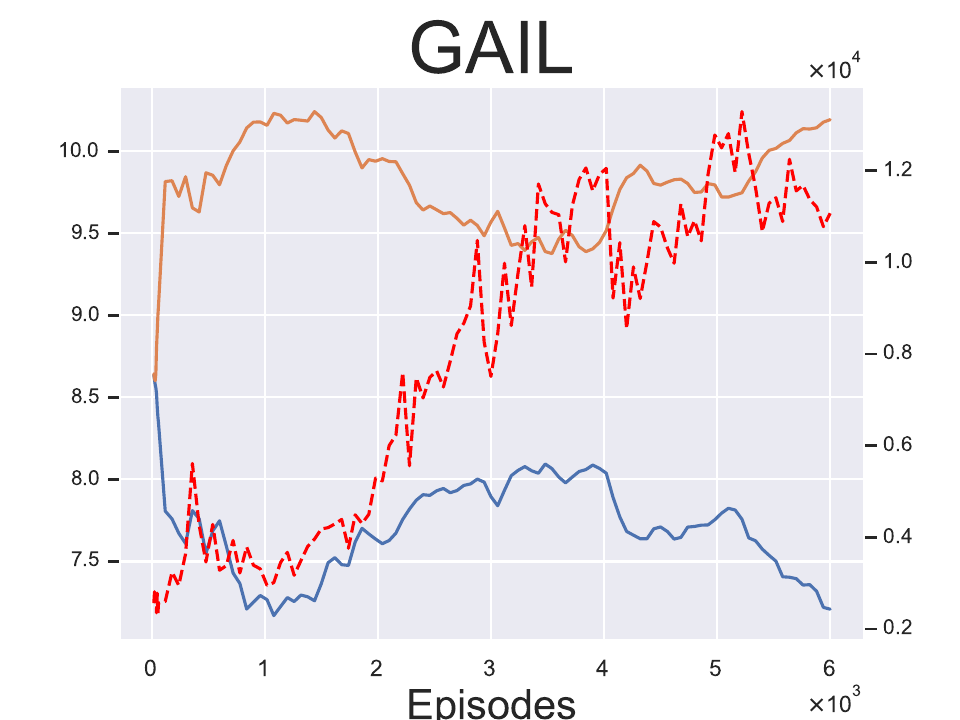}
        \includegraphics[ width=1\textwidth]{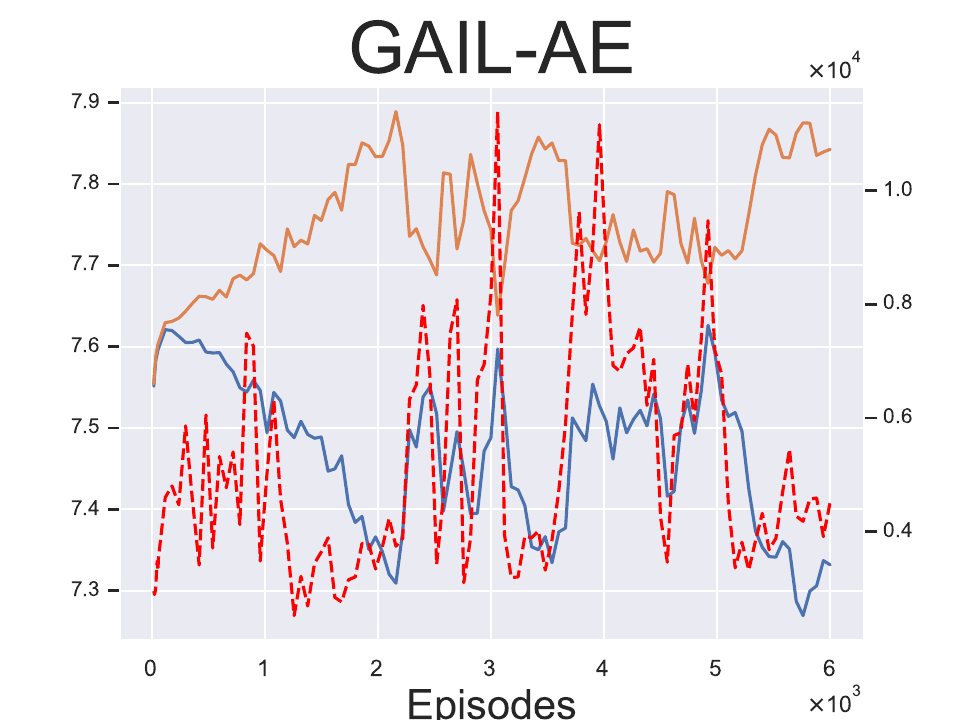}
        \includegraphics[ width=1\textwidth]{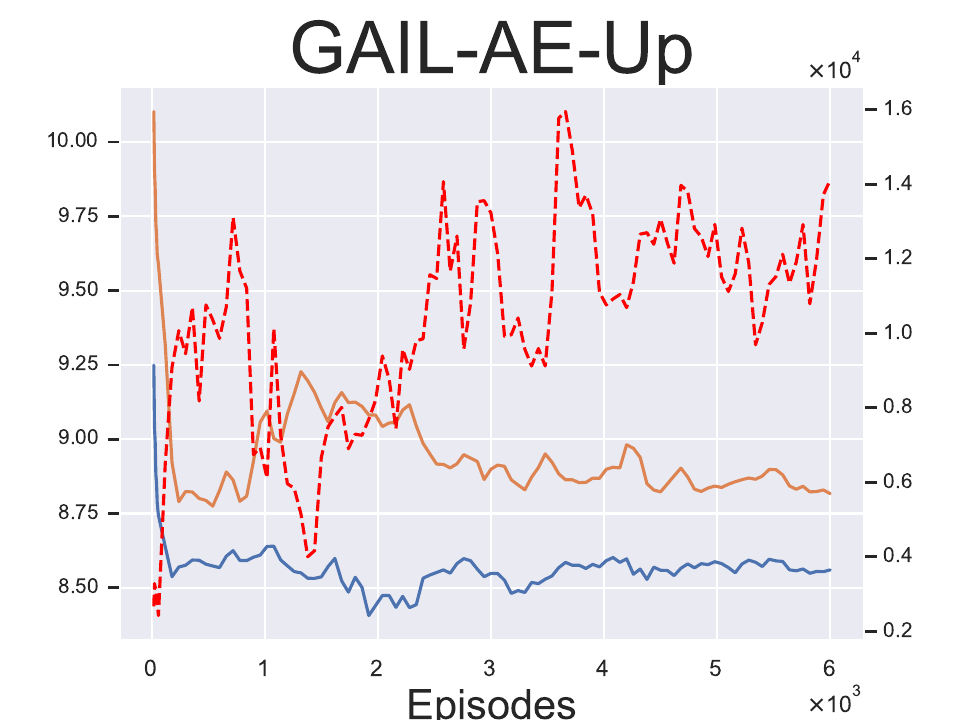}
        \includegraphics[ width=1\textwidth]{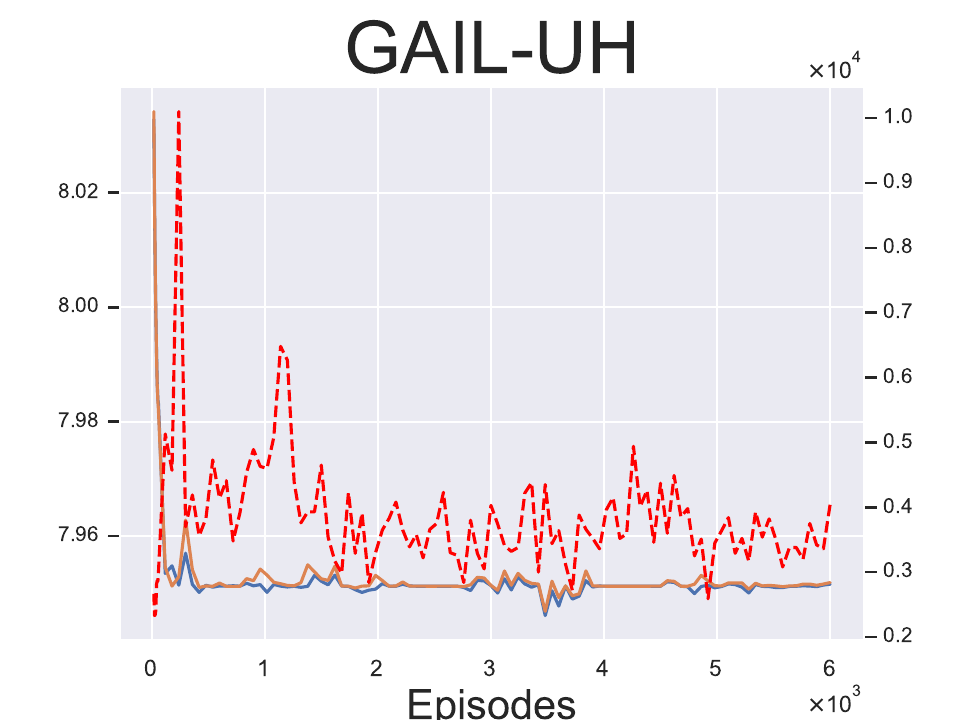}
        \includegraphics[ width=1\textwidth]{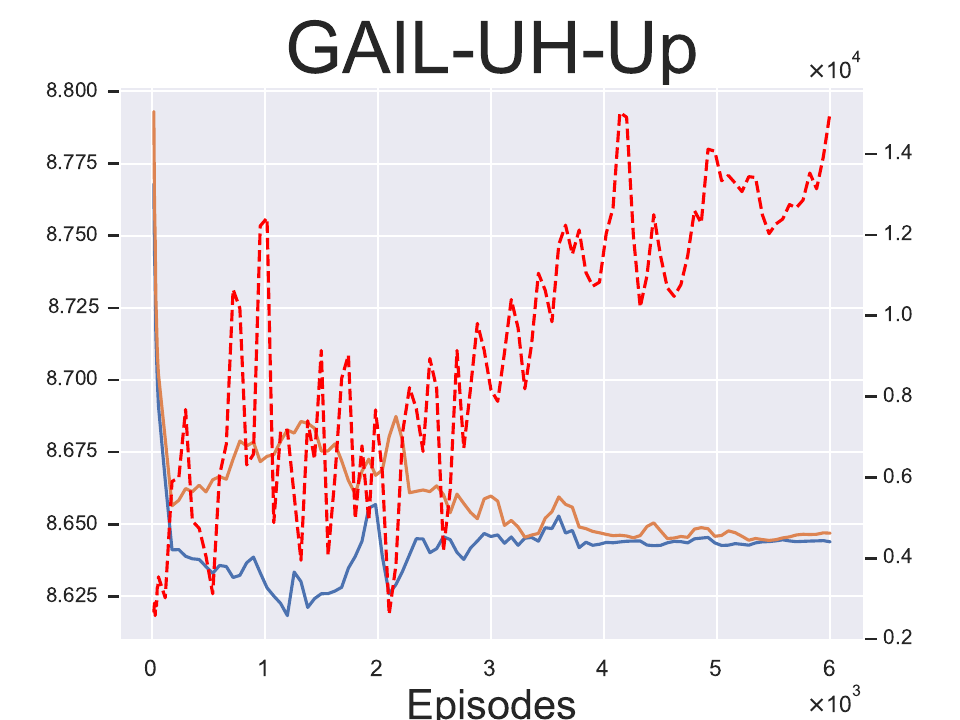}
        \includegraphics[ width=1\textwidth]{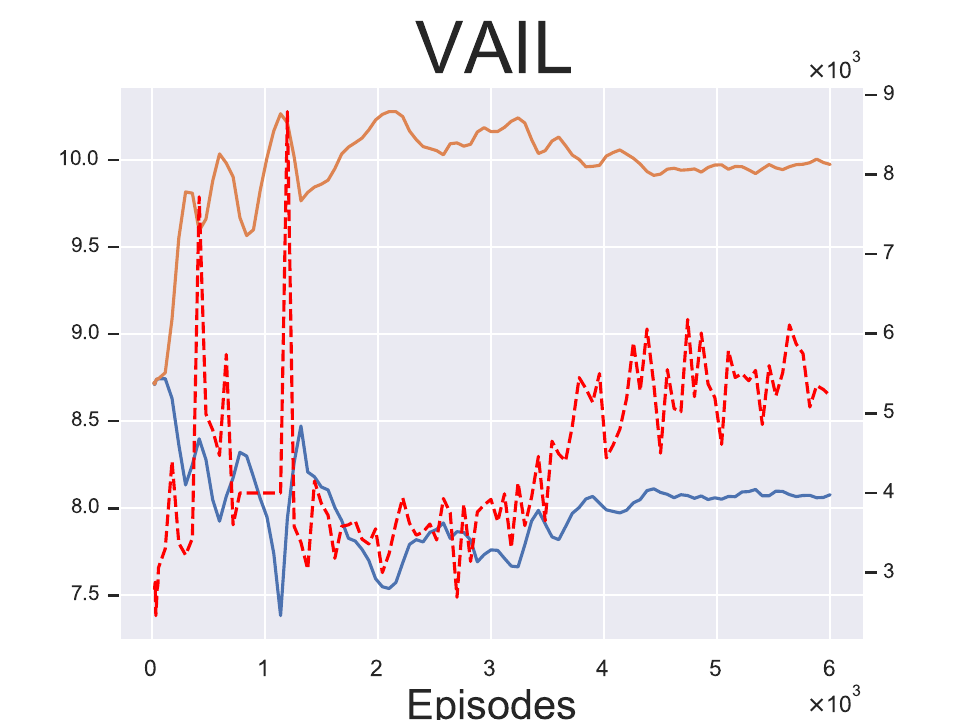}
        \includegraphics[ width=1\textwidth]{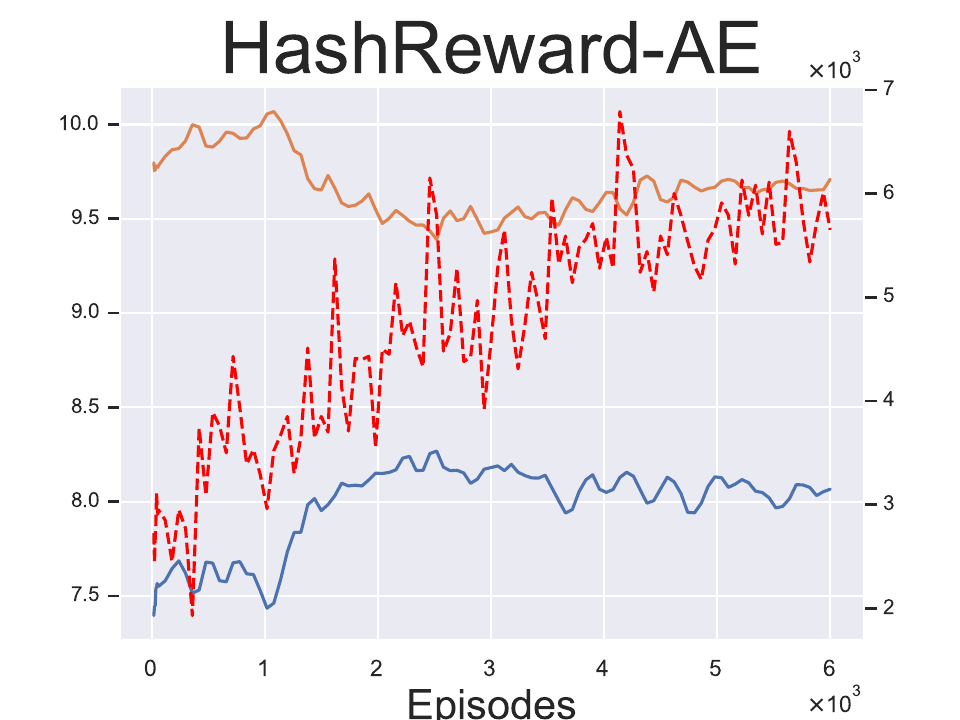}
        \includegraphics[ width=1\textwidth]{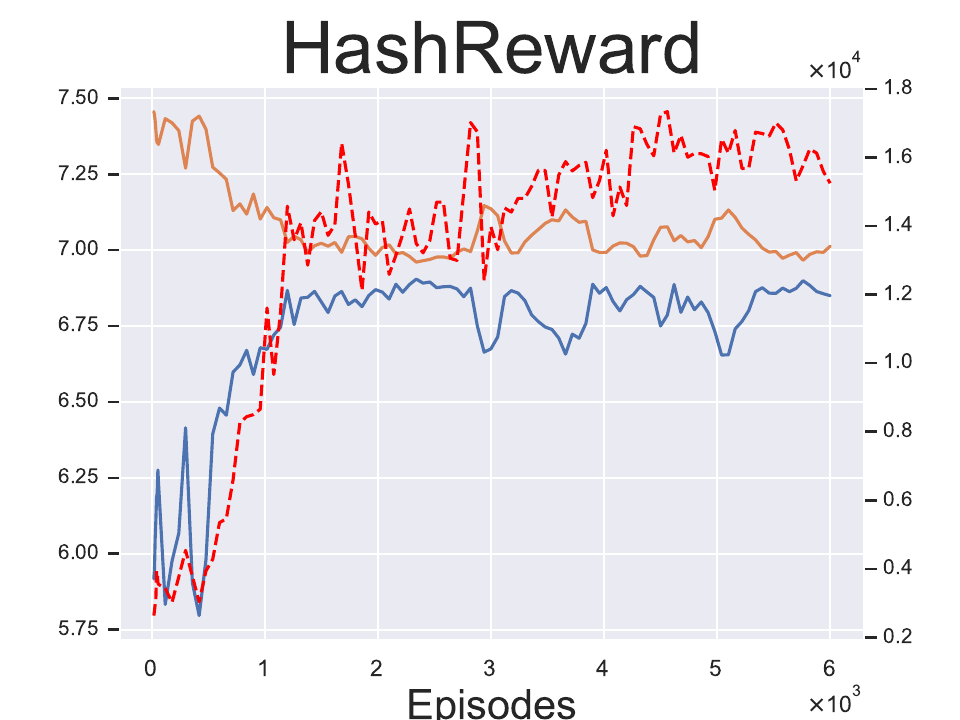}
    \end{minipage}}
    \subfigure[ChopperCommand]{
    \begin{minipage}[b]{0.185\linewidth}
        \includegraphics[ width=1\textwidth]{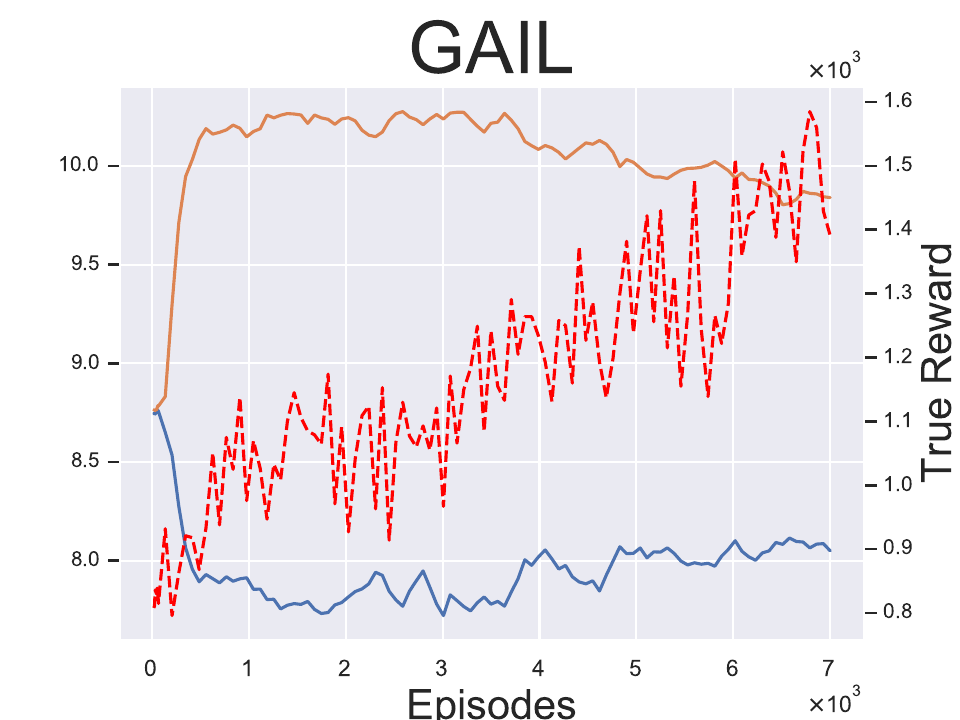}
        \includegraphics[ width=1\textwidth]{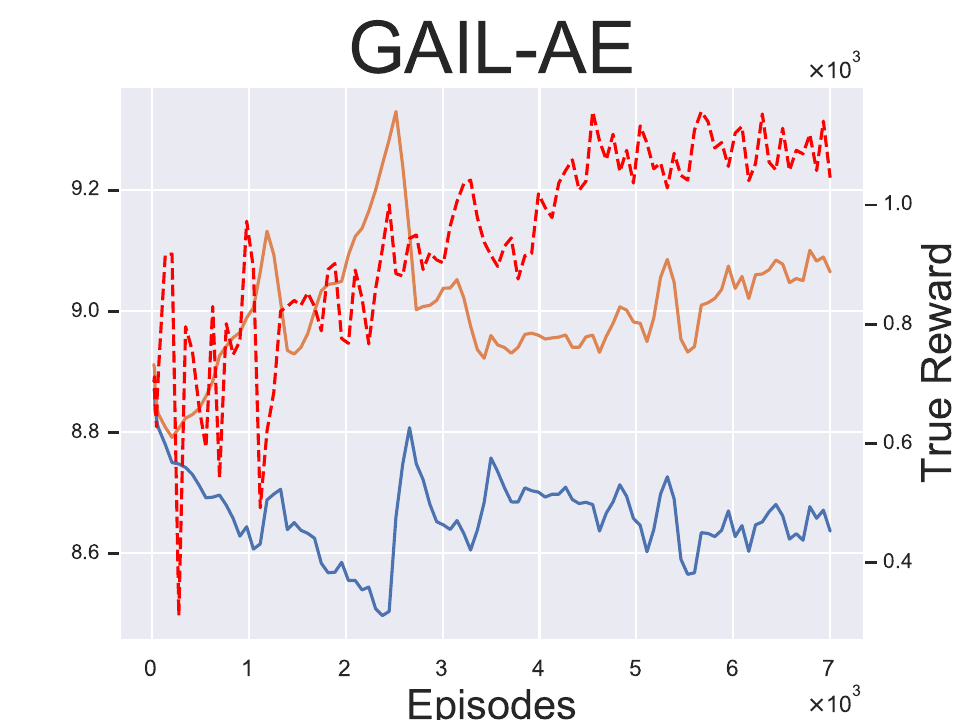}
        \includegraphics[ width=1\textwidth]{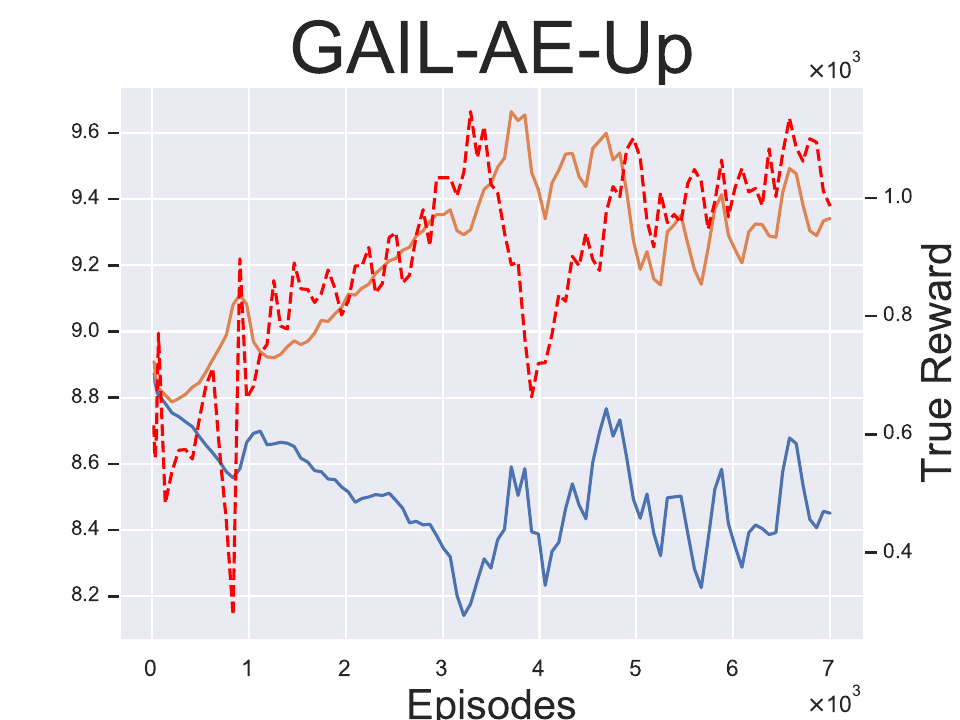}
        \includegraphics[ width=1\textwidth]{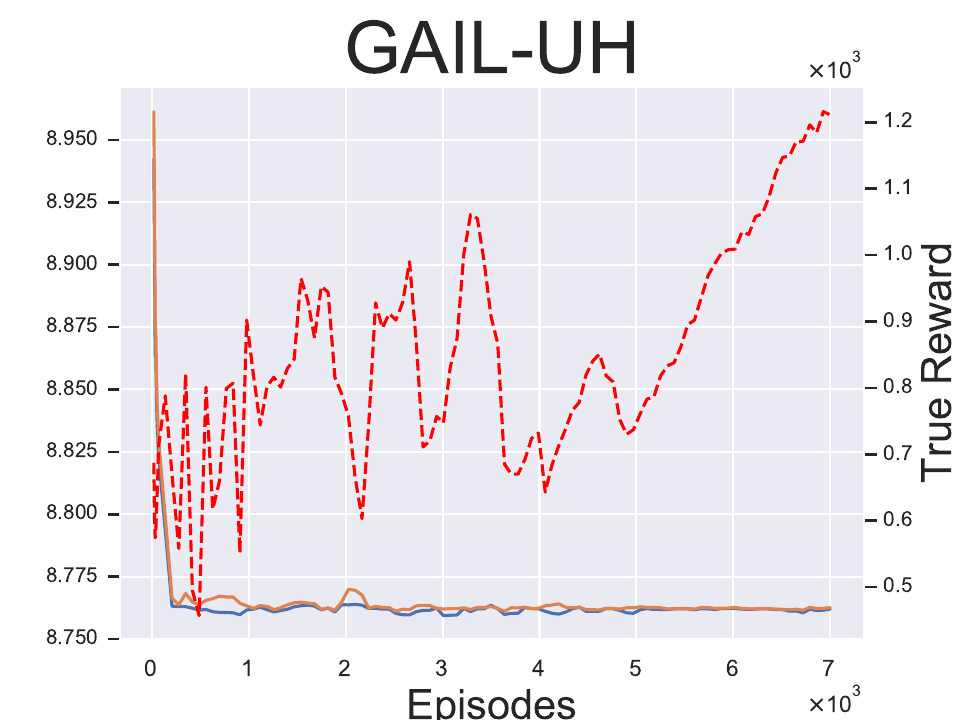}
        \includegraphics[ width=1\textwidth]{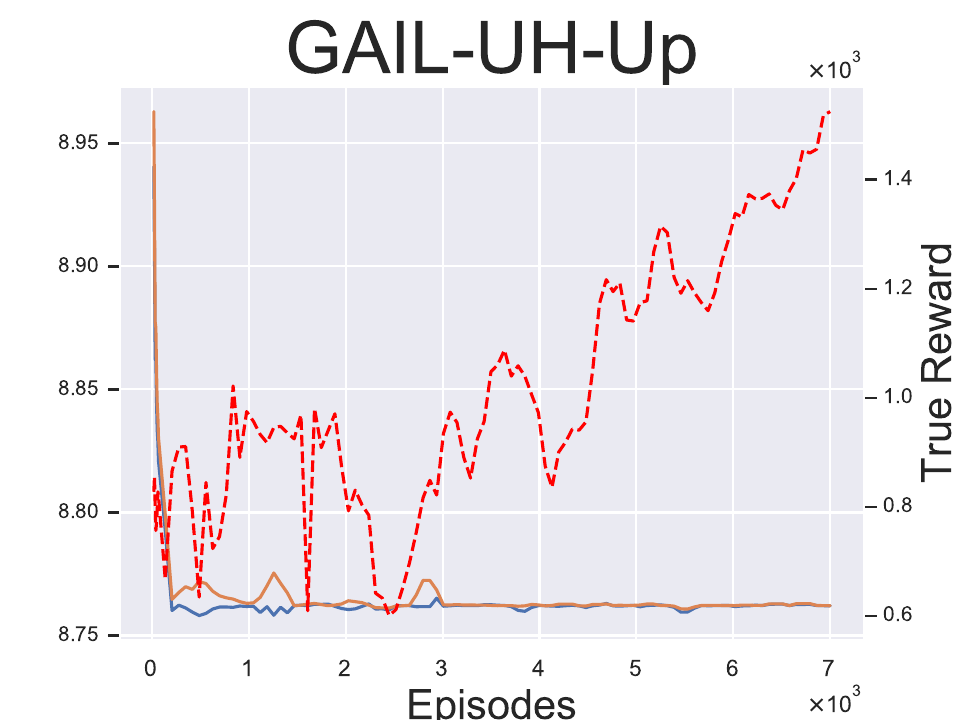}
        \includegraphics[ width=1\textwidth]{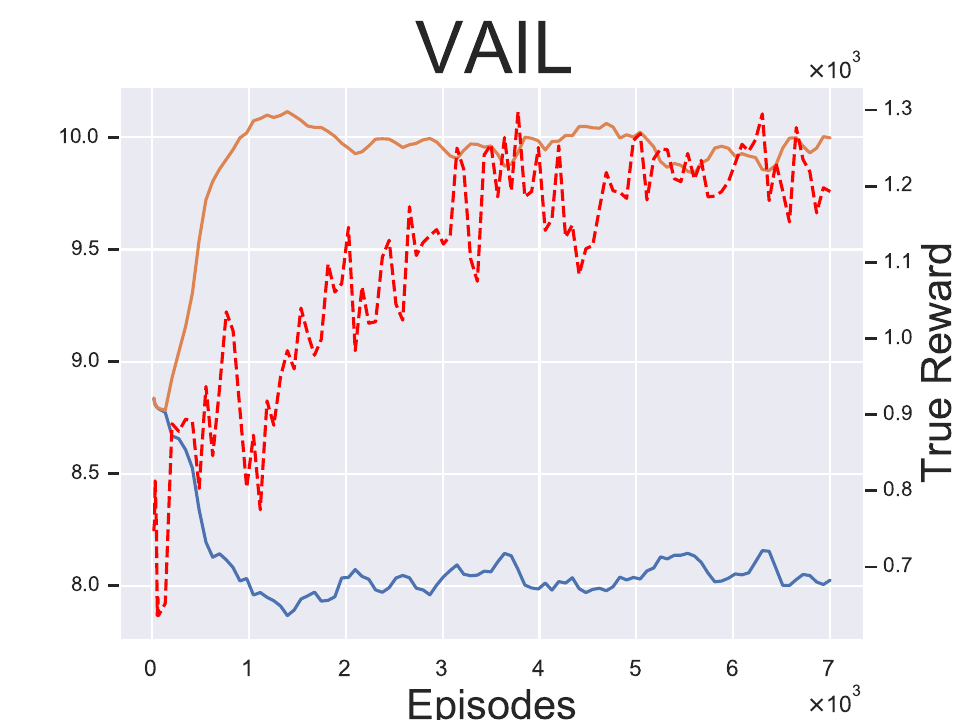}
        \includegraphics[ width=1\textwidth]{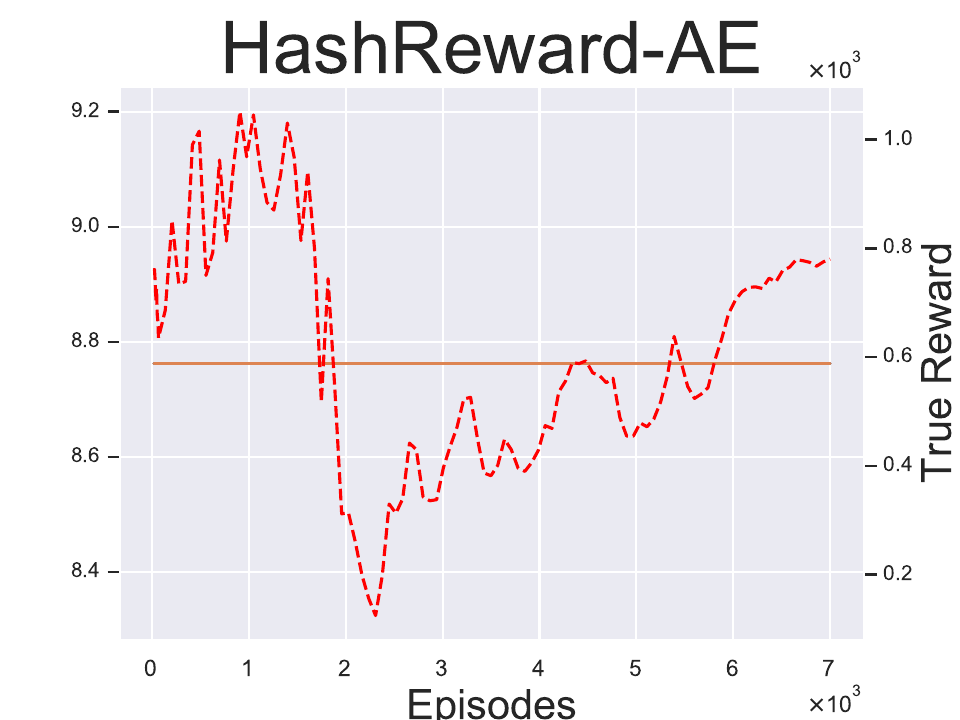}
        \includegraphics[ width=1\textwidth]{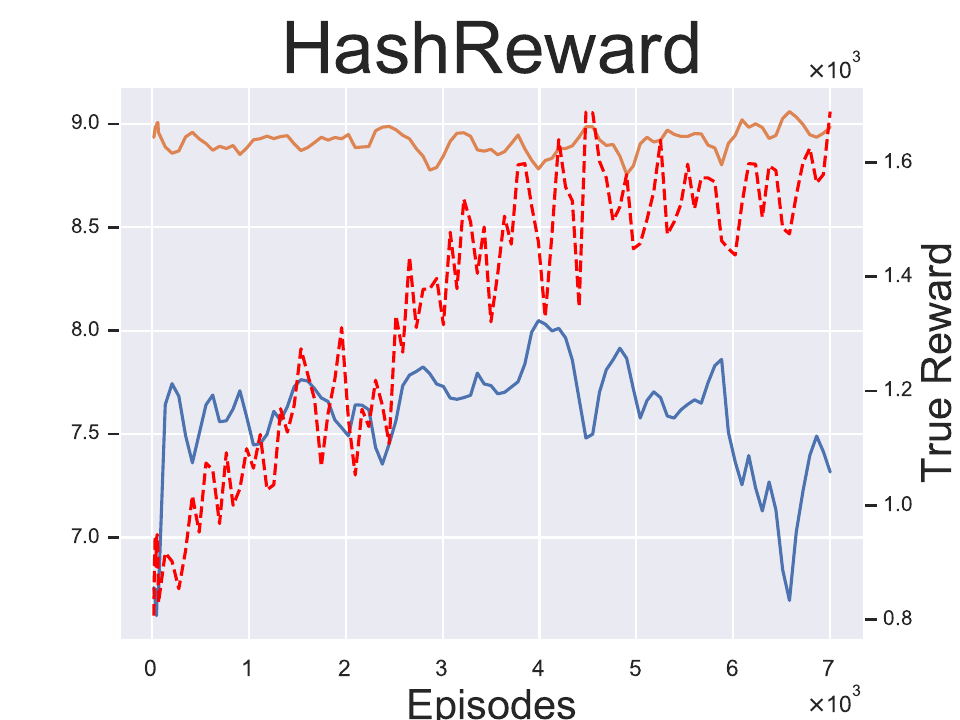}
    \end{minipage}}
\end{figure*}
\begin{figure*}[!t]
    \addtocounter{figure}{1}
    \centering
    \subfigure[CrazyClimber]{
    \begin{minipage}[b]{0.185\linewidth}
        \includegraphics[ width=1\textwidth]{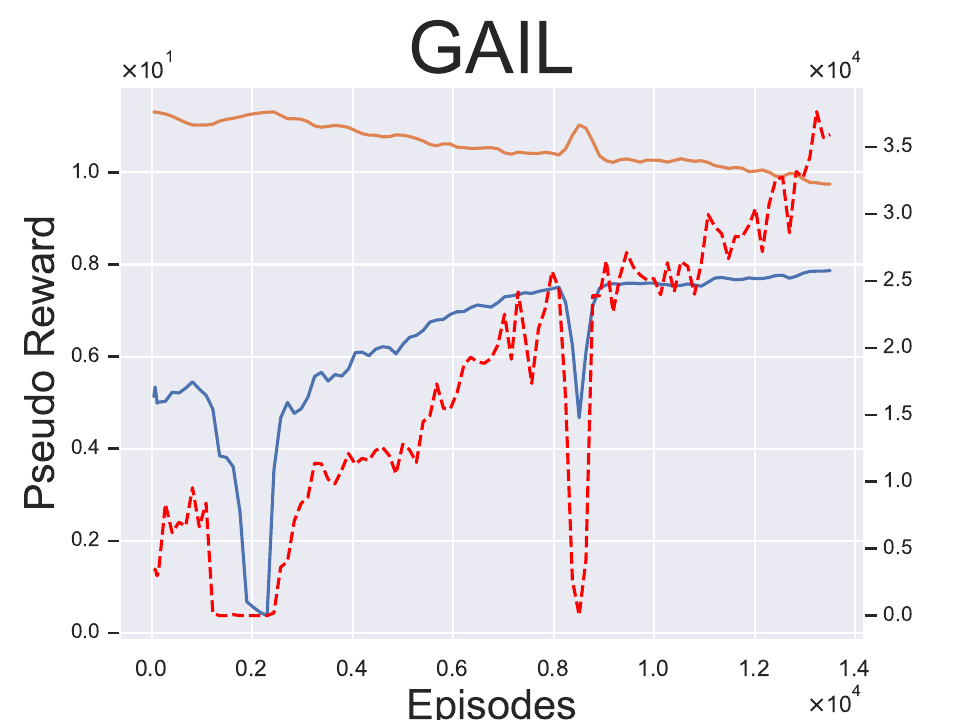}
        \includegraphics[ width=1\textwidth]{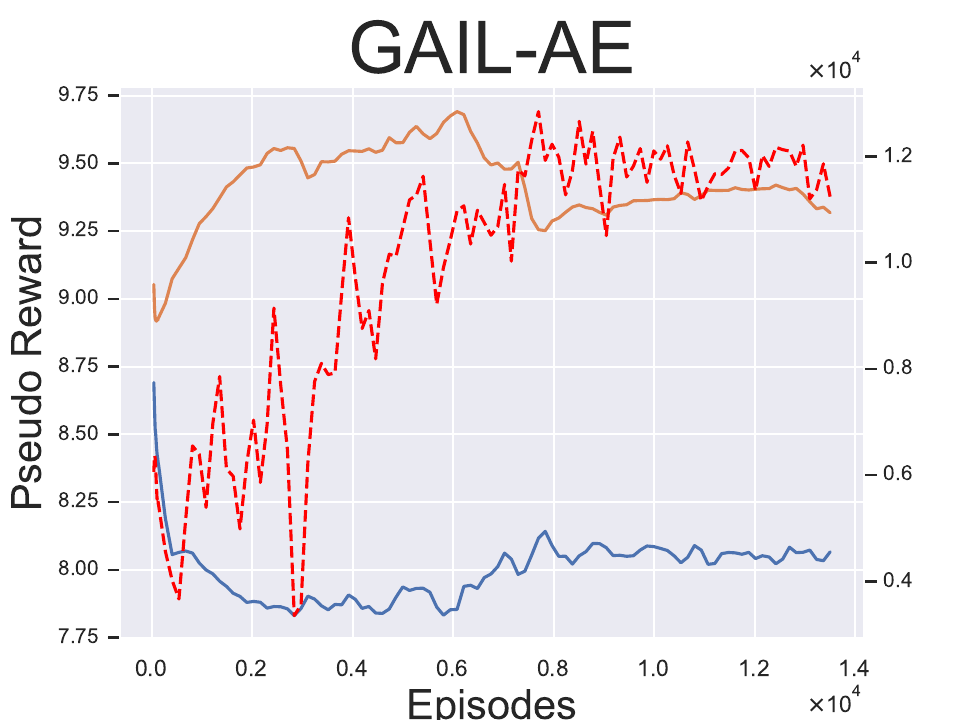}
        \includegraphics[ width=1\textwidth]{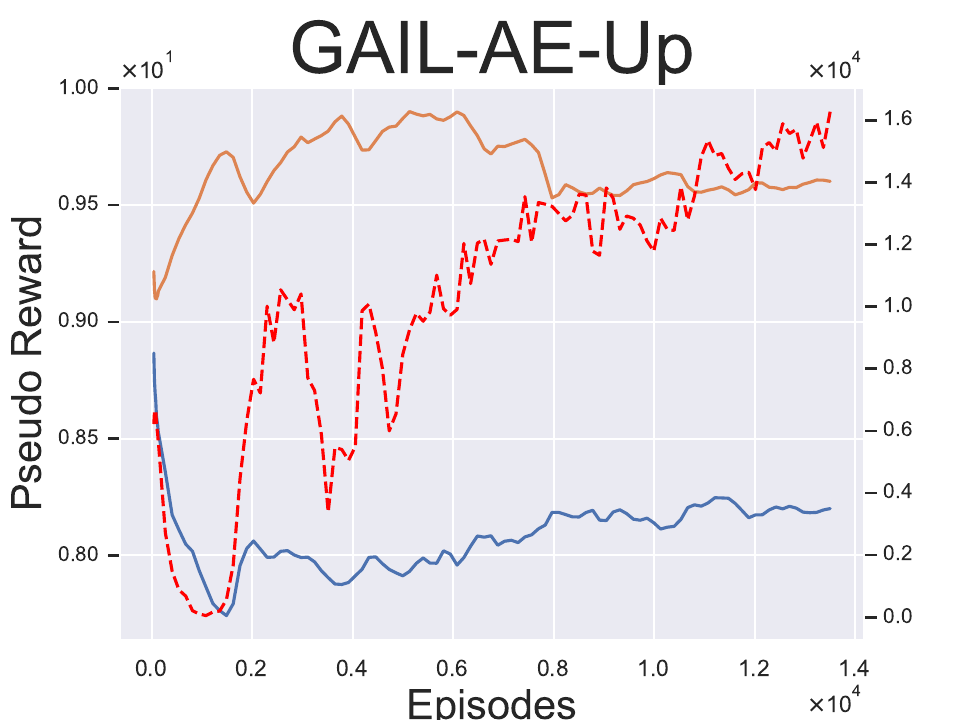}
        \includegraphics[ width=1\textwidth]{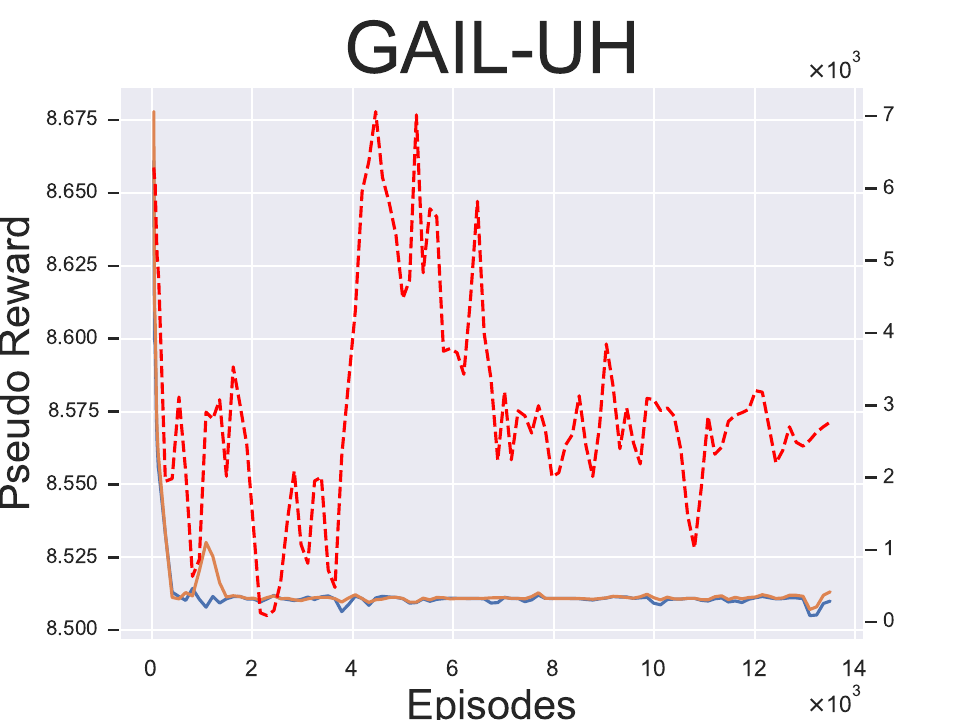}
        \includegraphics[ width=1\textwidth]{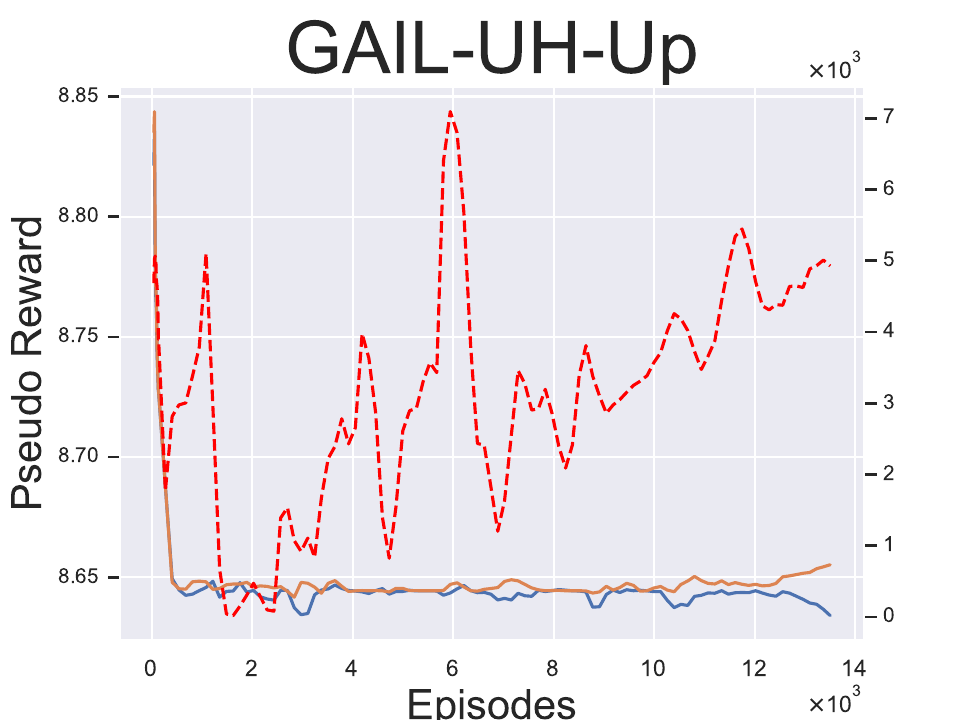}
        \includegraphics[ width=1\textwidth]{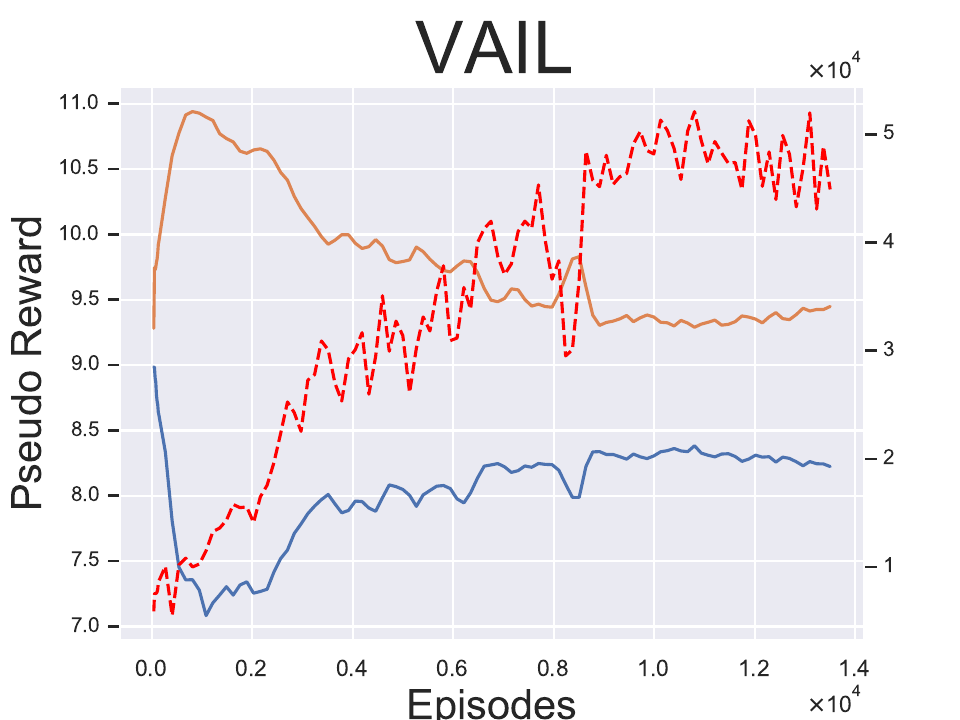}
        \includegraphics[ width=1\textwidth]{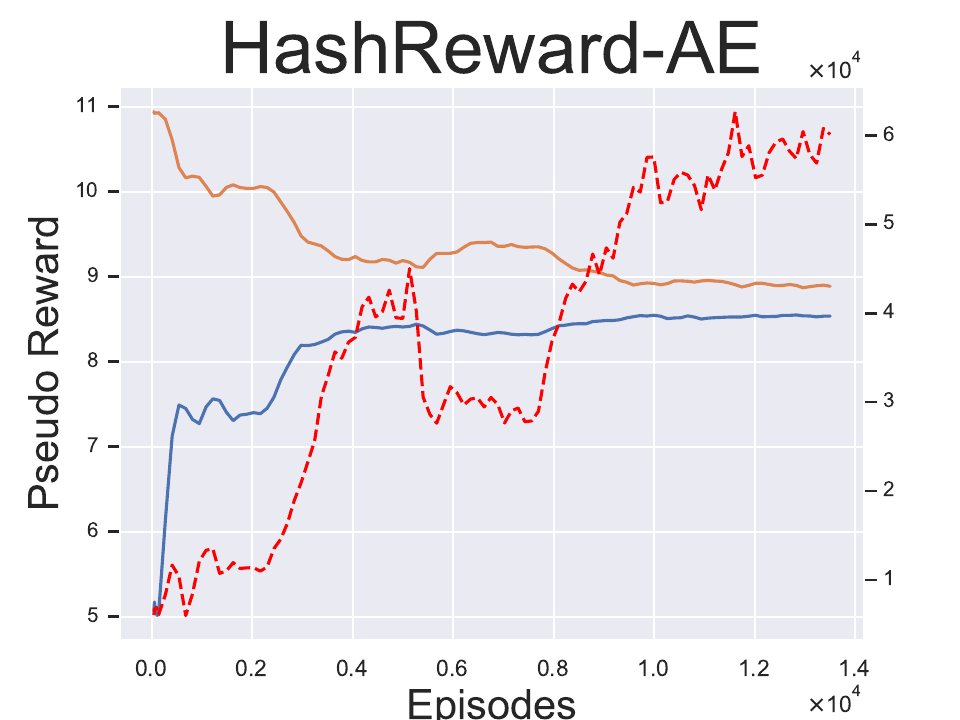}
        \includegraphics[ width=1\textwidth]{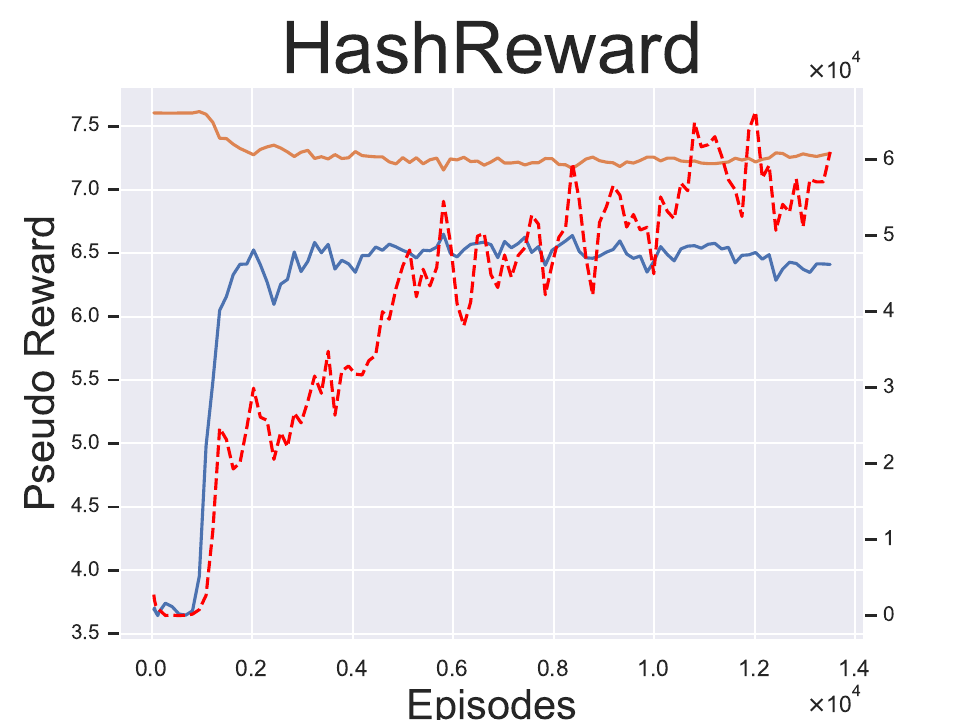}
    \end{minipage}}
    \subfigure[Enduro]{
    \begin{minipage}[b]{0.185\linewidth}
        \includegraphics[ width=1\textwidth]{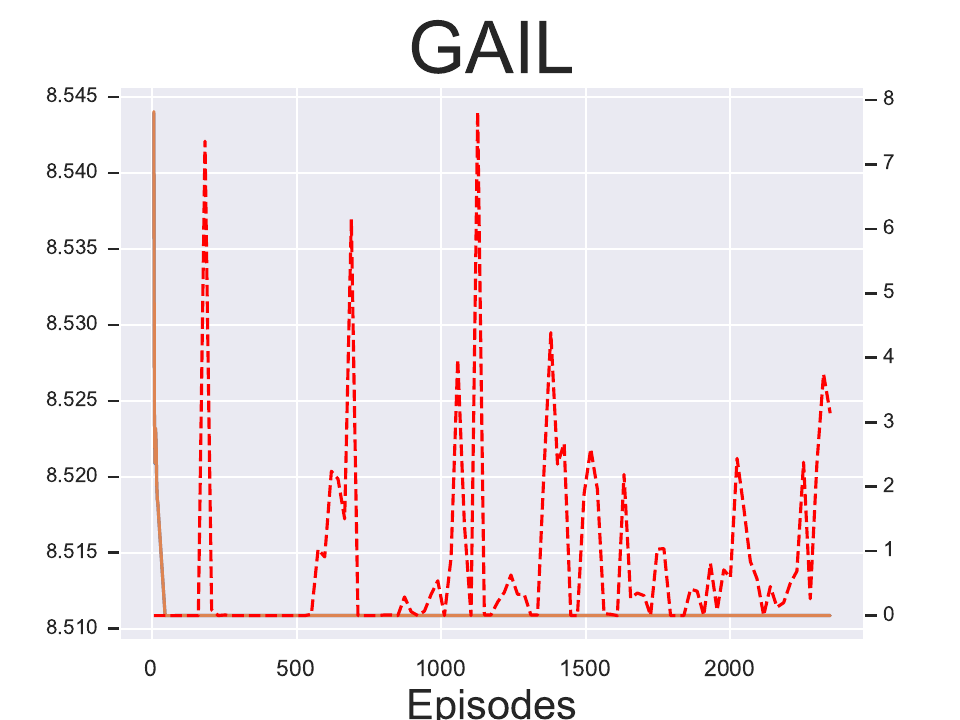}
        \includegraphics[ width=1\textwidth]{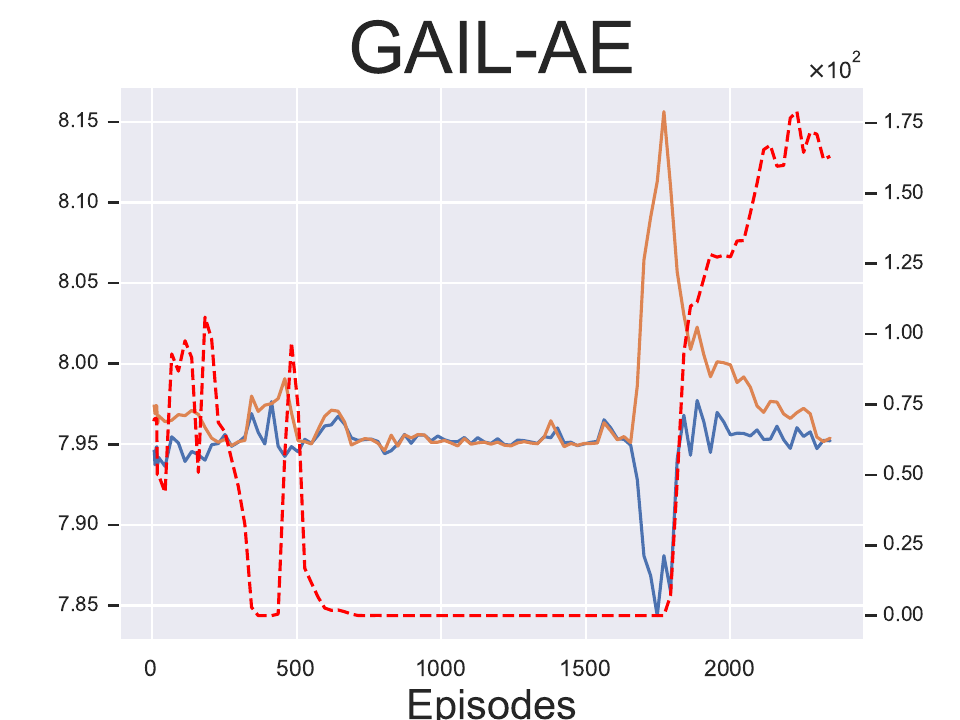}
        \includegraphics[ width=1\textwidth]{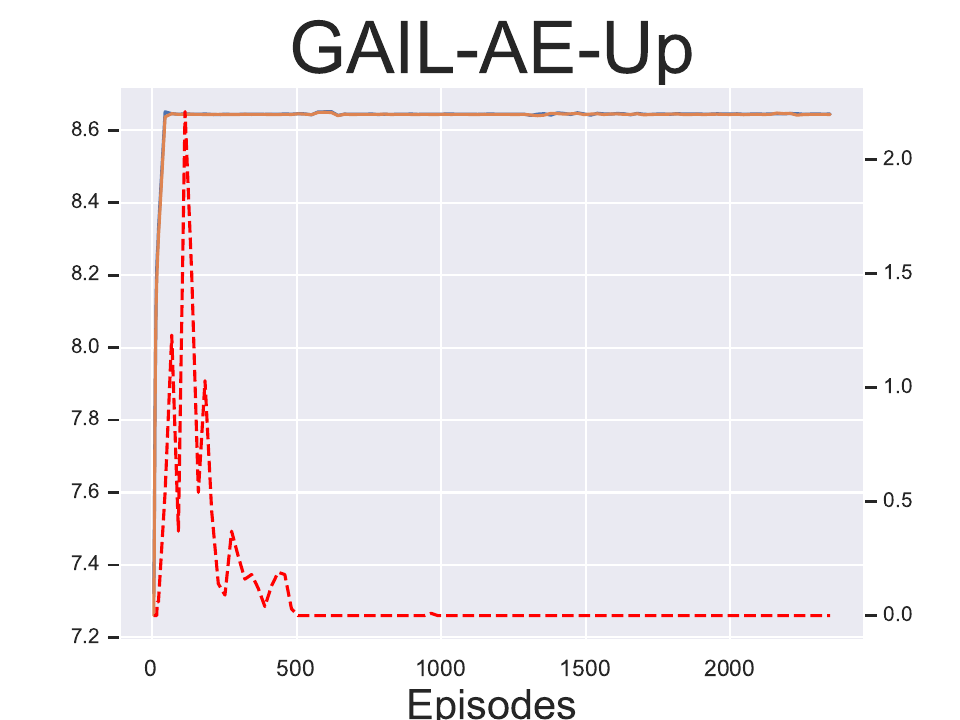}
        \includegraphics[ width=1\textwidth]{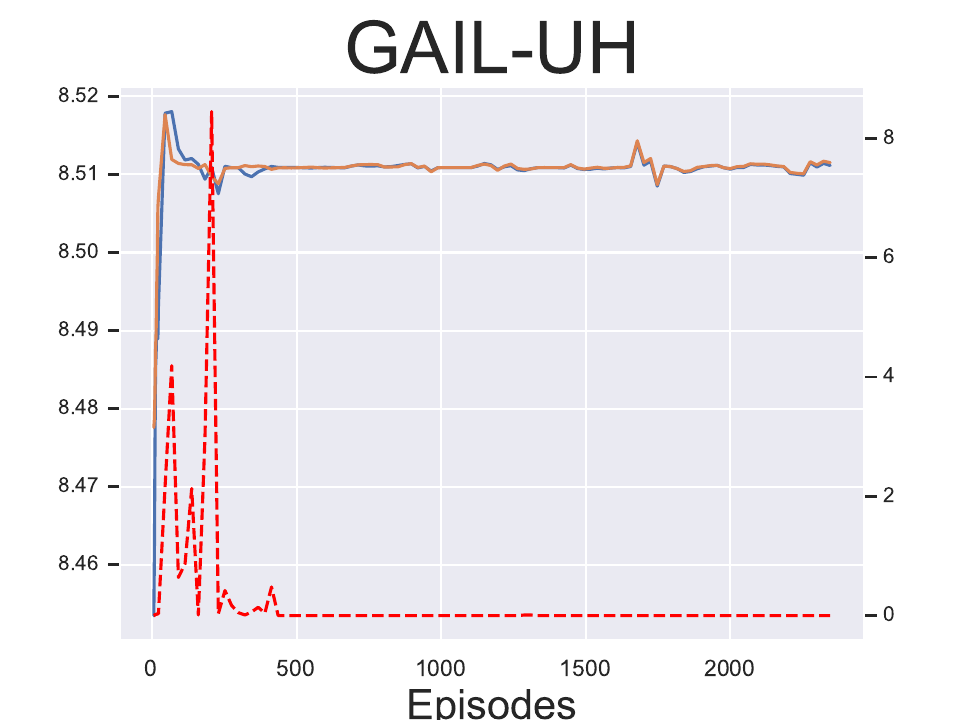}
        \includegraphics[ width=1\textwidth]{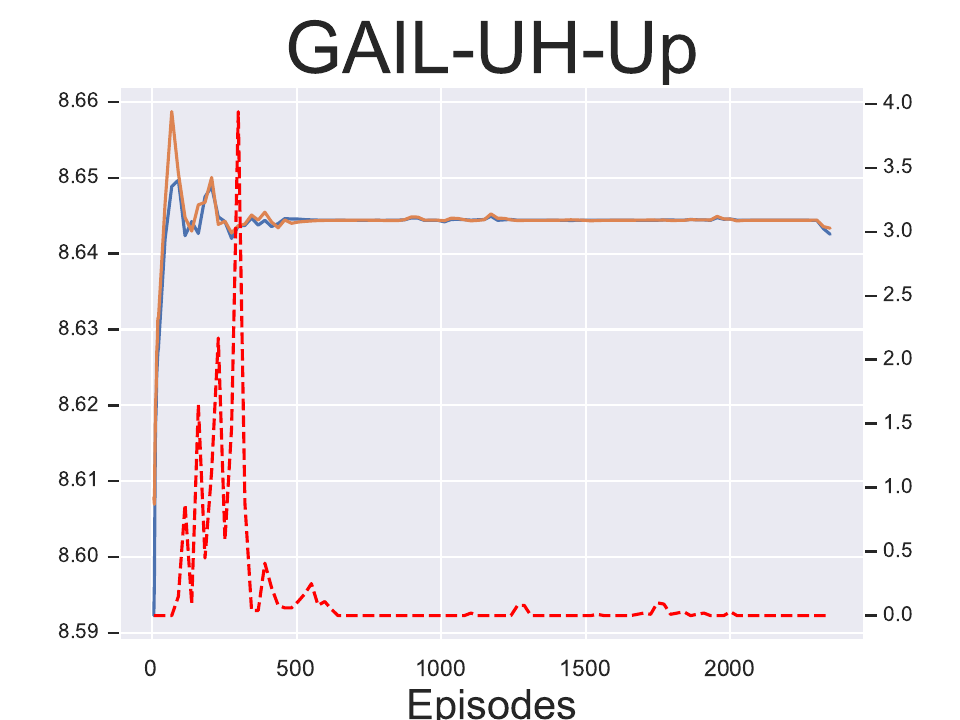}
        \includegraphics[ width=1\textwidth]{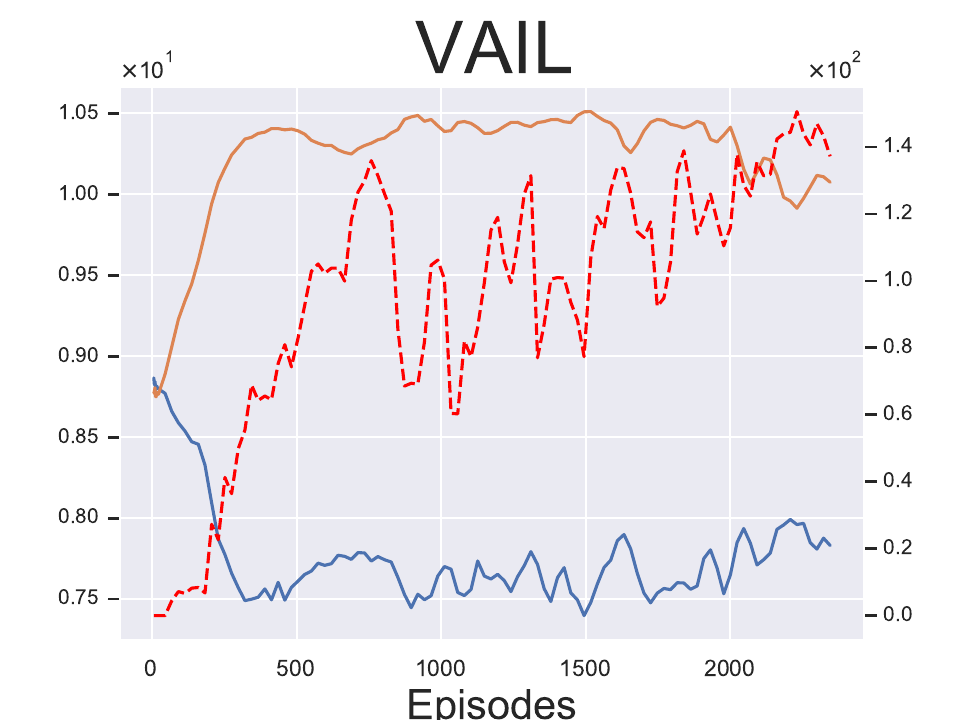}
        \includegraphics[ width=1\textwidth]{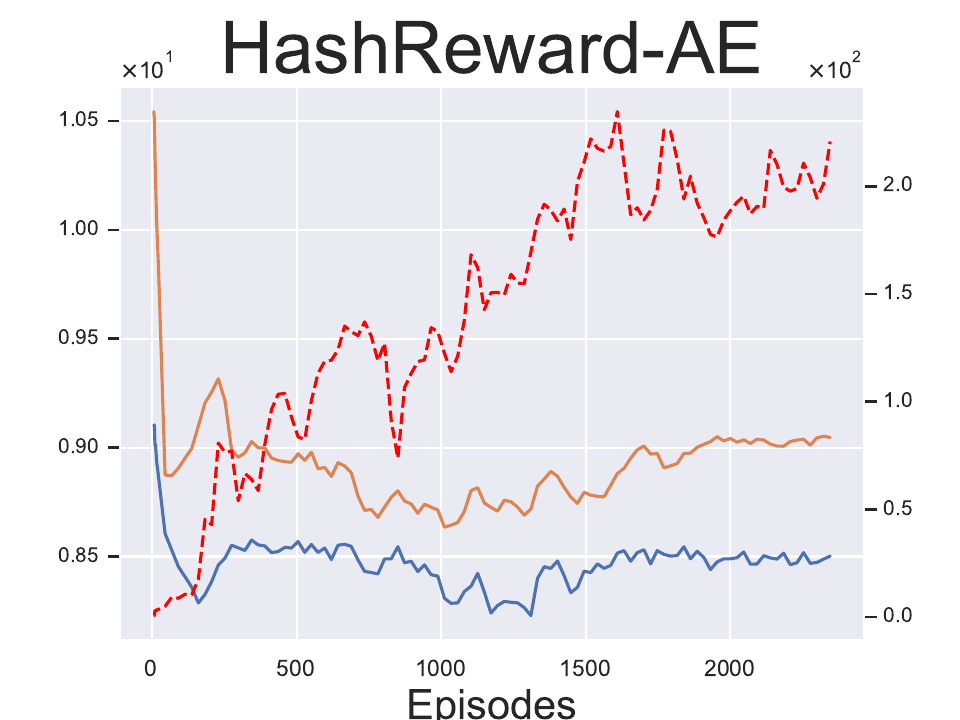}
        \includegraphics[ width=1\textwidth]{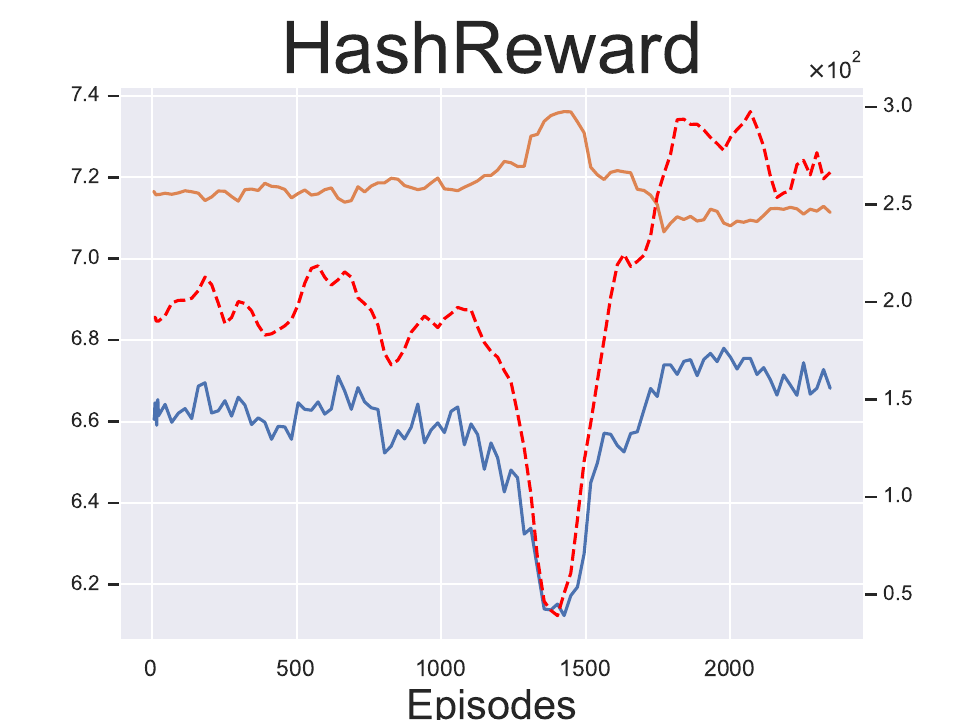}
    \end{minipage}}
    \subfigure[Kangaroo]{
    \begin{minipage}[b]{0.185\linewidth}
        \includegraphics[ width=1\textwidth]{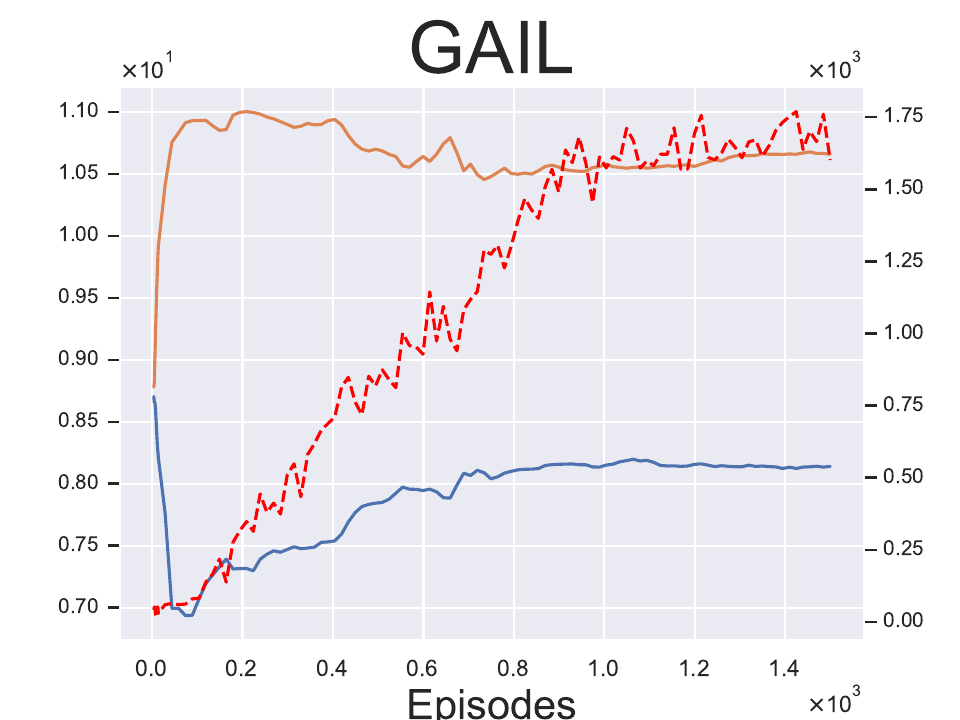}
        \includegraphics[ width=1\textwidth]{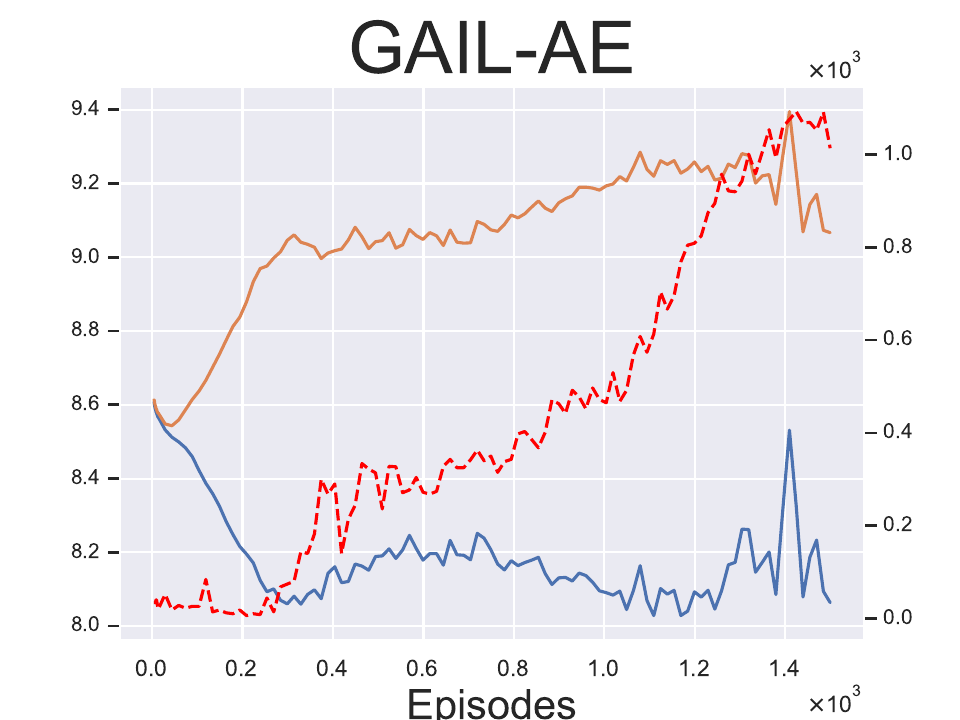}
        \includegraphics[ width=1\textwidth]{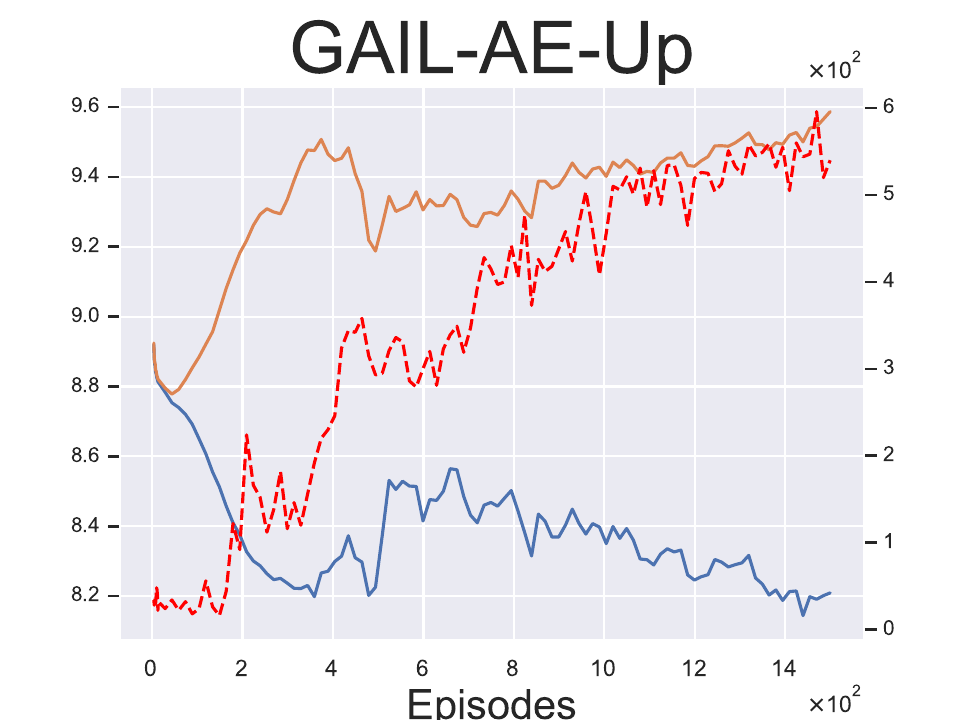}
        \includegraphics[ width=1\textwidth]{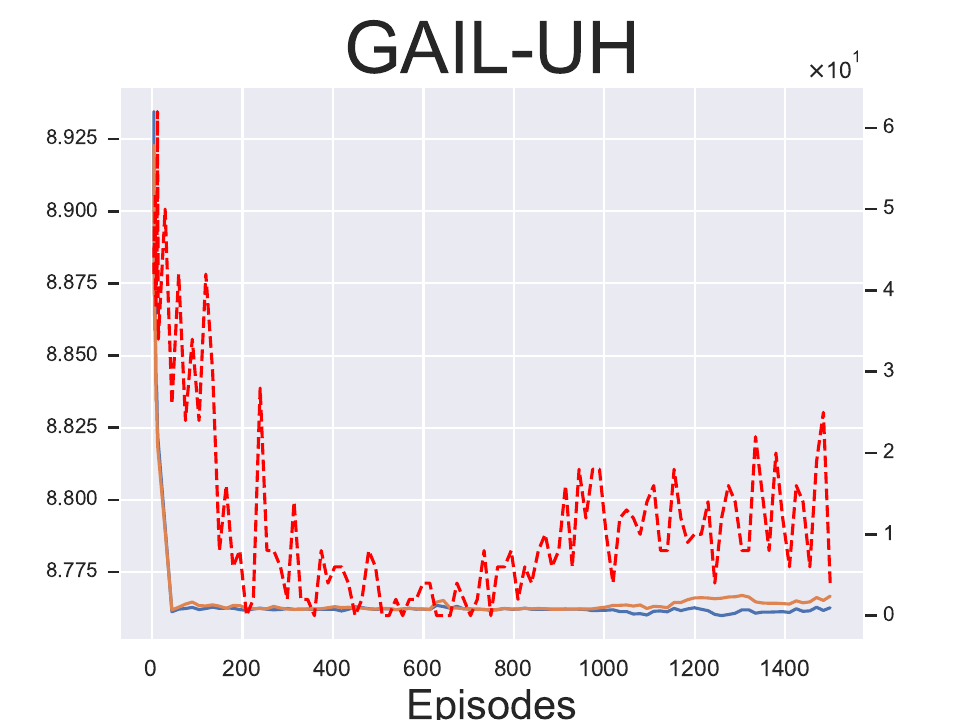}
        \includegraphics[ width=1\textwidth]{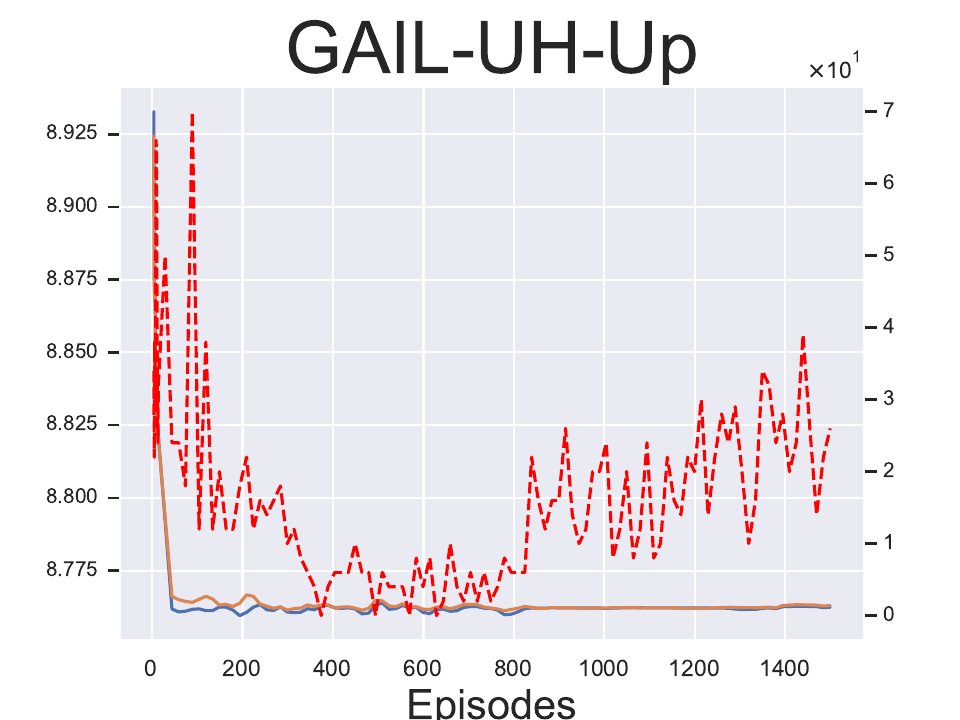}
        \includegraphics[ width=1\textwidth]{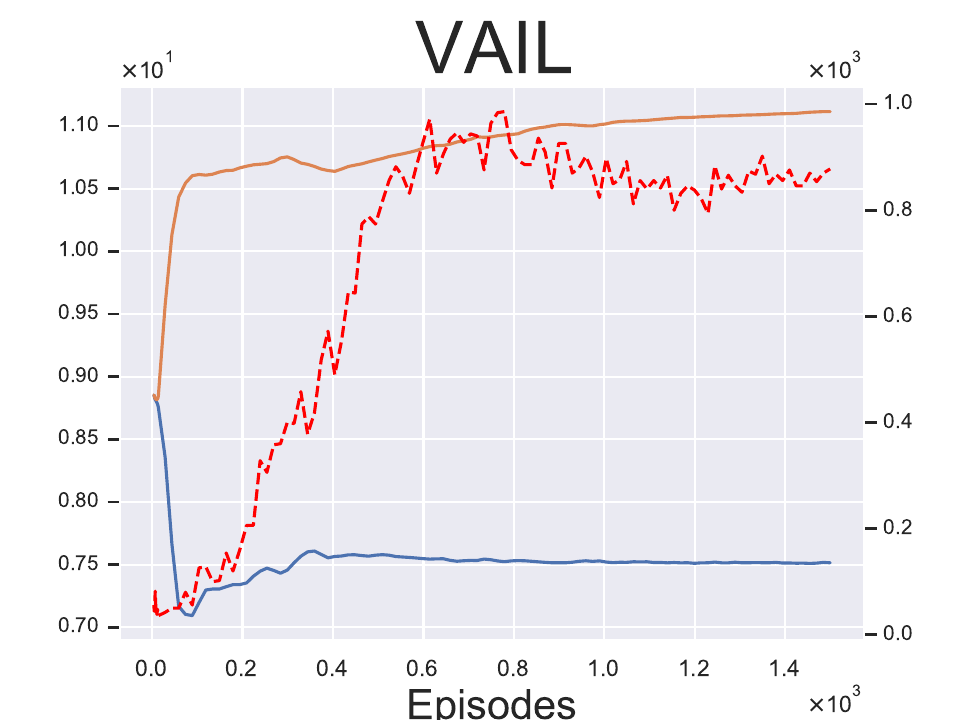}
        \includegraphics[ width=1\textwidth]{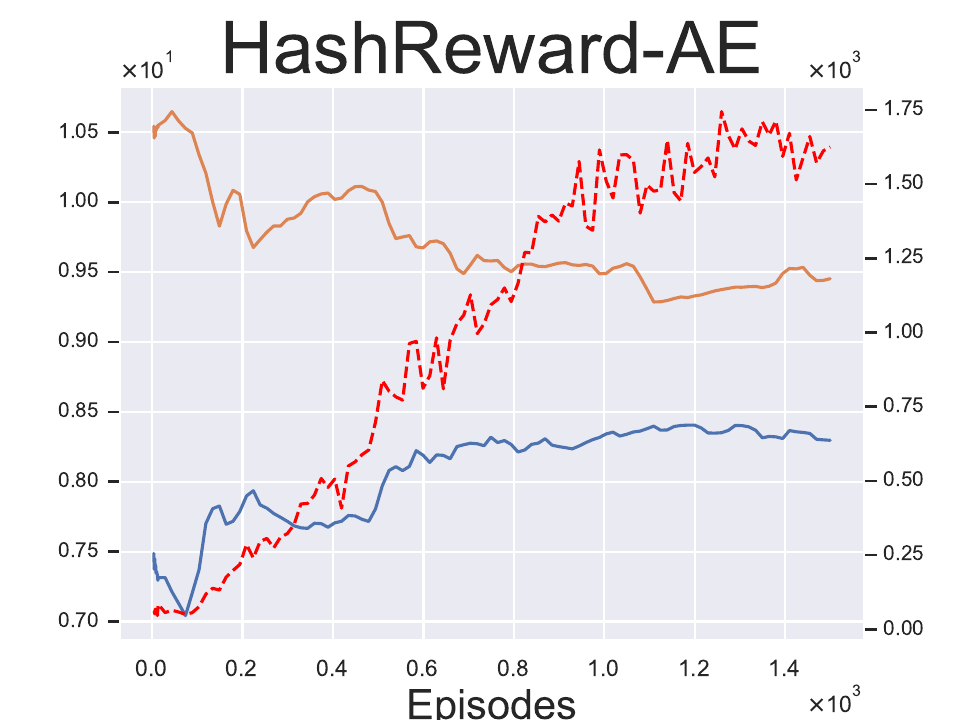}
        \includegraphics[ width=1\textwidth]{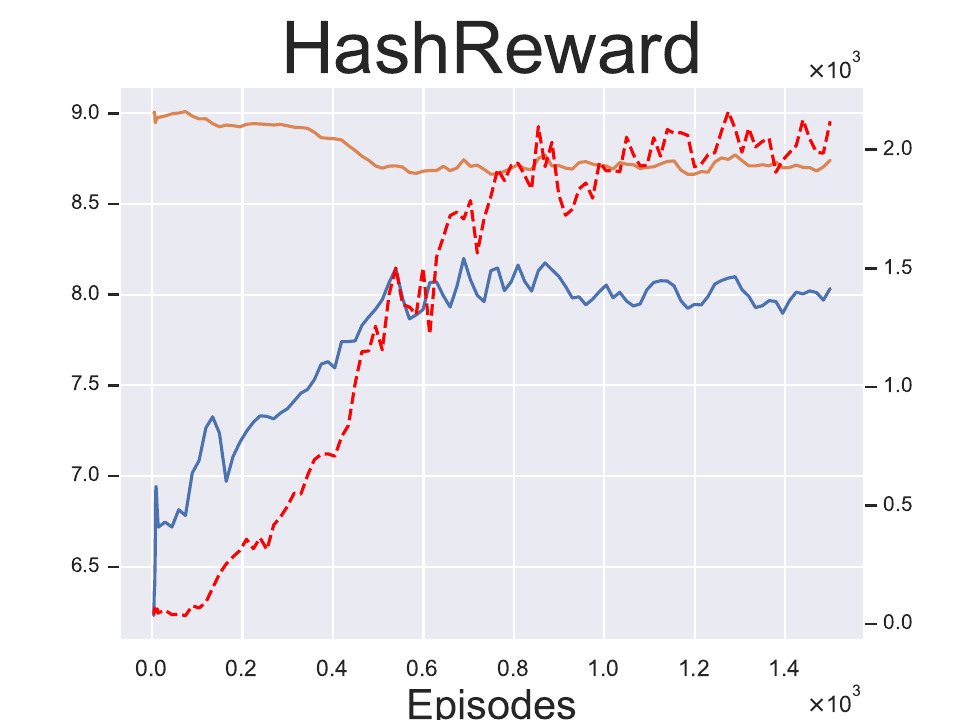}
    \end{minipage}}
    \subfigure[MsPacman]{
    \begin{minipage}[b]{0.185\linewidth}
        \includegraphics[ width=1\textwidth]{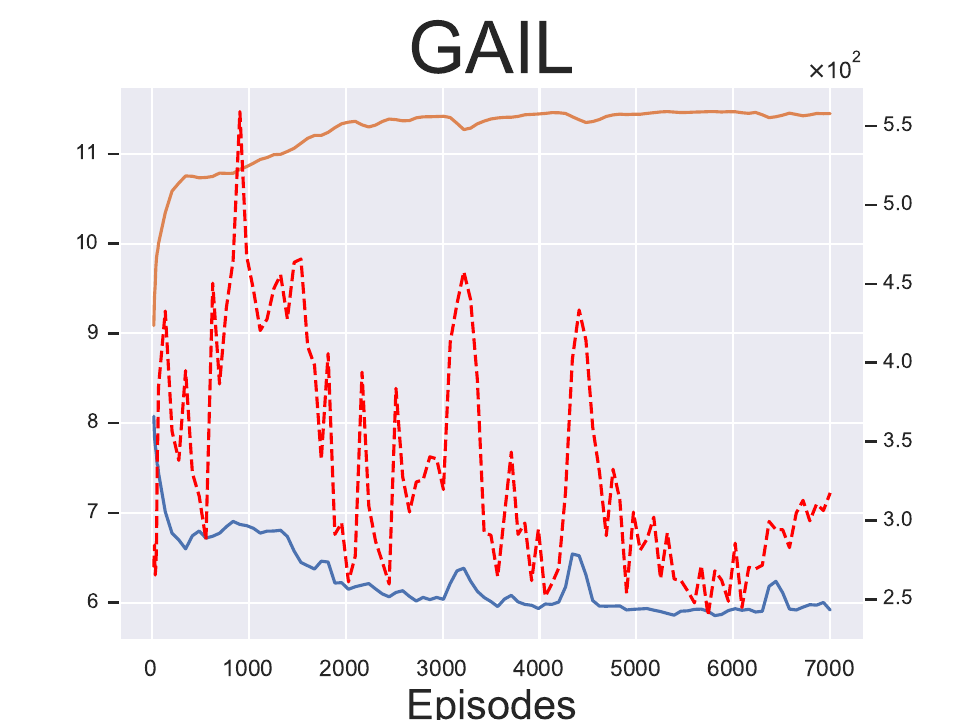}
        \includegraphics[ width=1\textwidth]{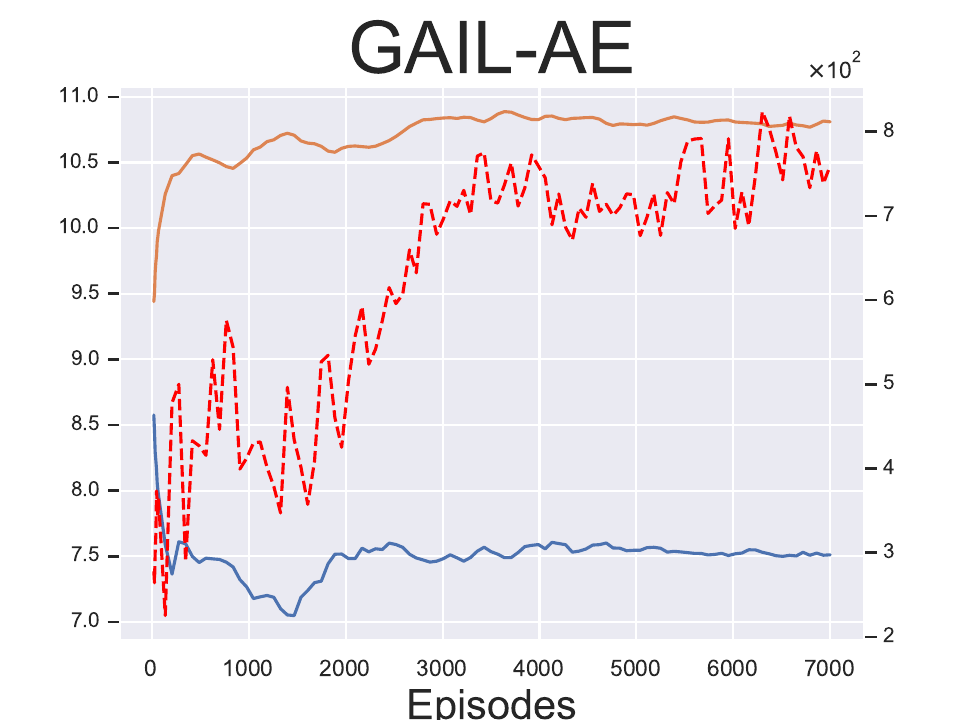}
        \includegraphics[ width=1\textwidth]{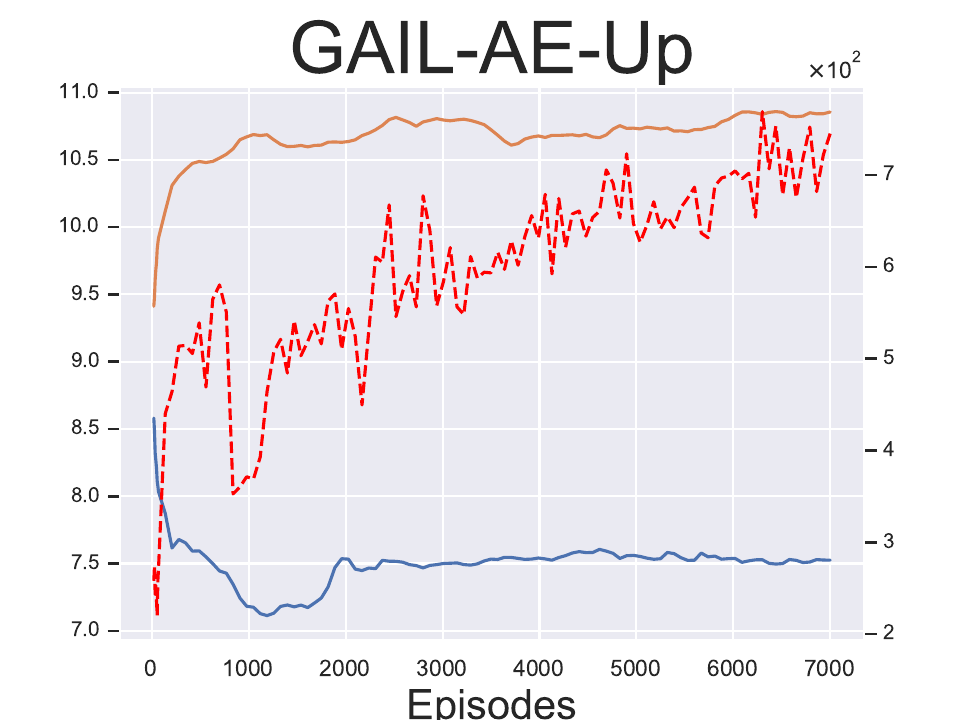}
        \includegraphics[ width=1\textwidth]{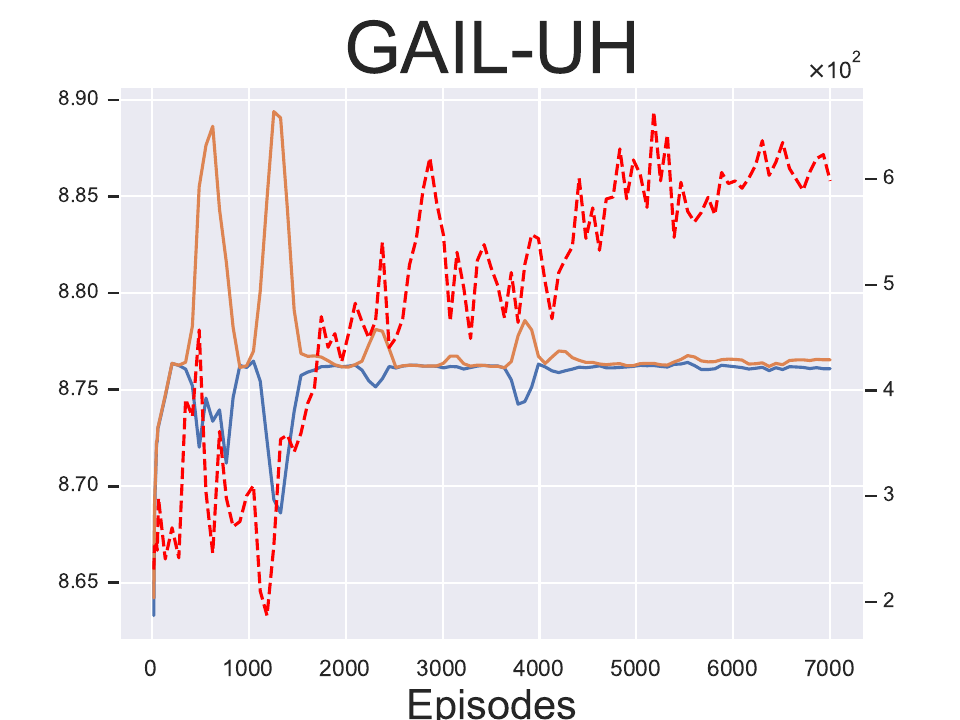}
        \includegraphics[ width=1\textwidth]{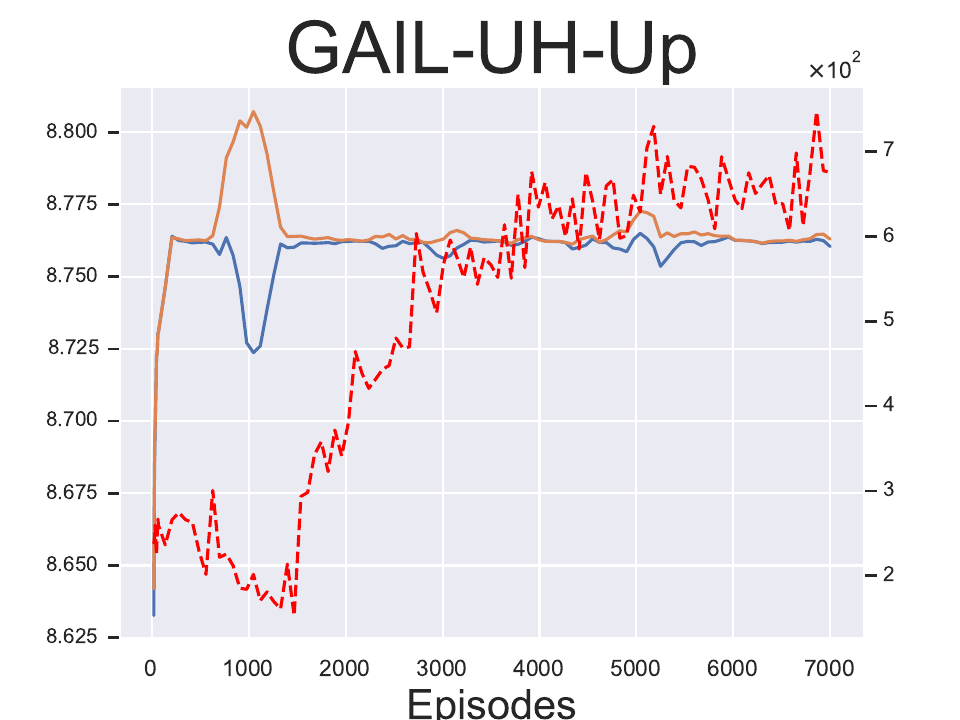}
        \includegraphics[ width=1\textwidth]{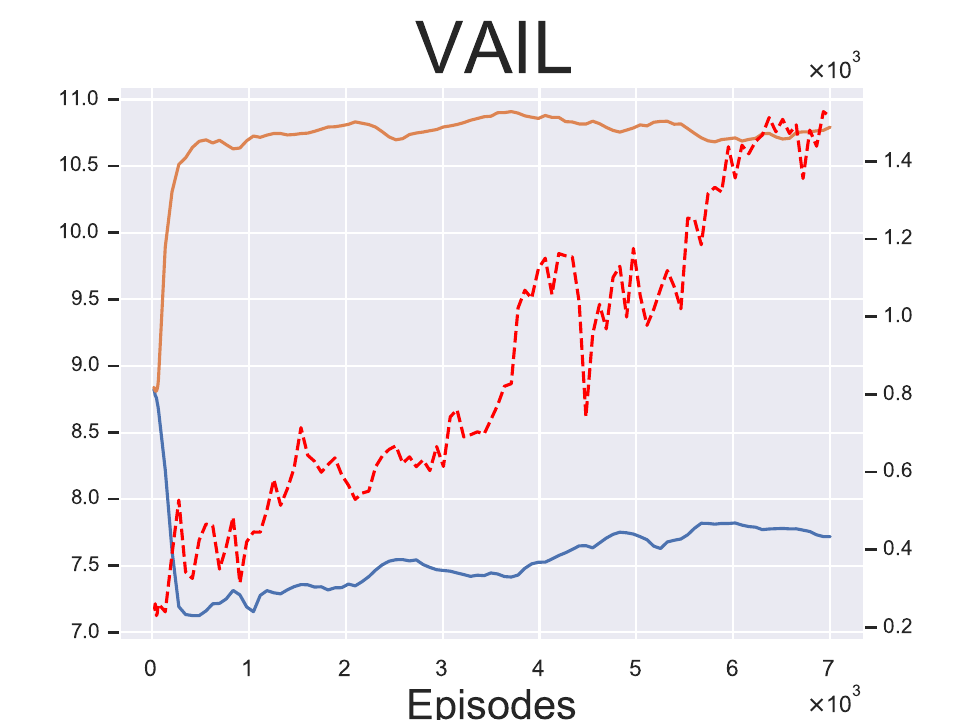}
        \includegraphics[ width=1\textwidth]{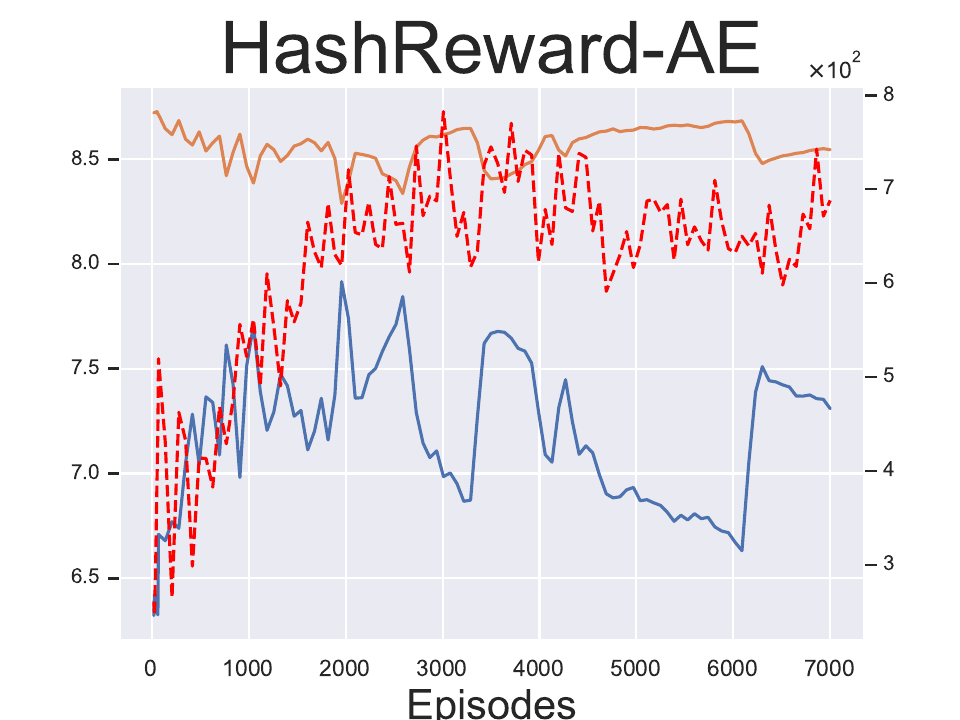}
        \includegraphics[ width=1\textwidth]{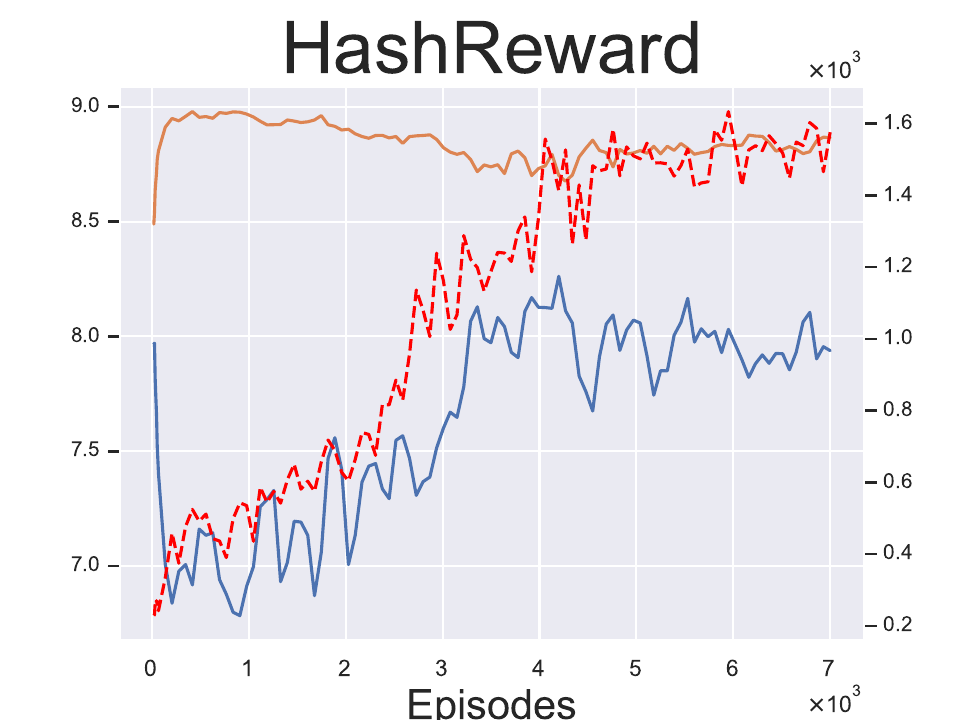}
    \end{minipage}}
    \subfigure[Pong]{
    \begin{minipage}[b]{0.185\linewidth}
        \includegraphics[ width=1\textwidth]{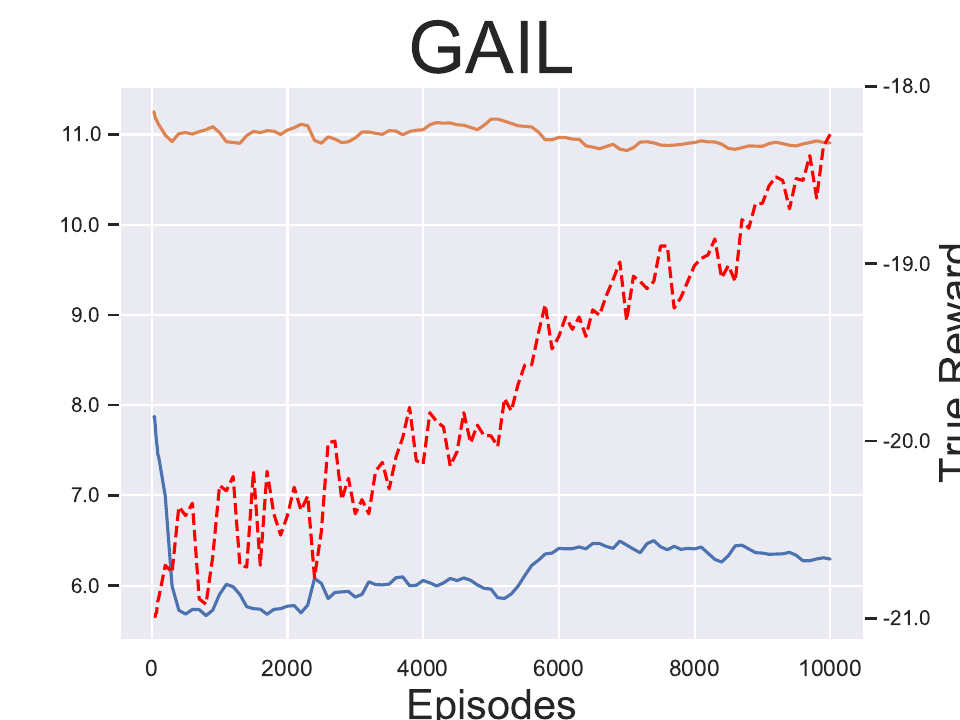}
        \includegraphics[ width=1\textwidth]{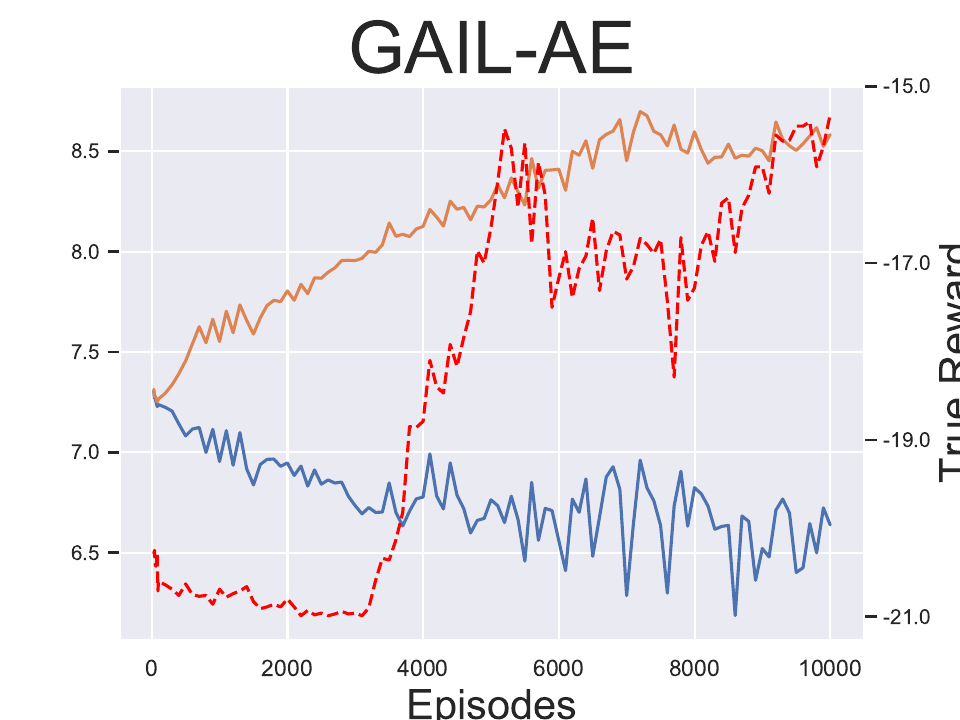}
        \includegraphics[ width=1\textwidth]{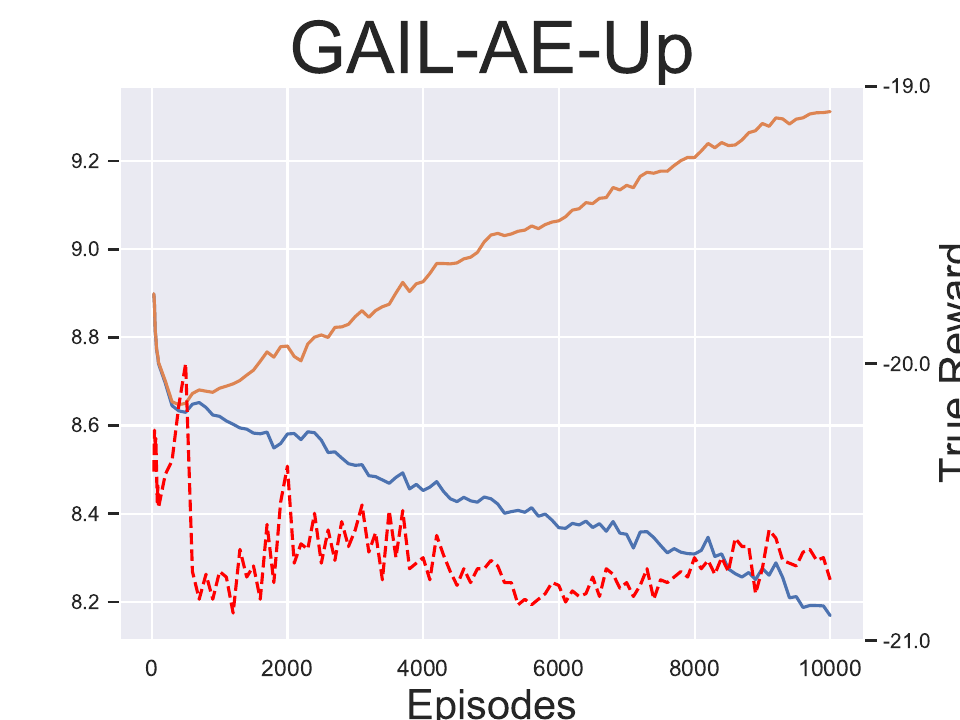}
        \includegraphics[ width=1\textwidth]{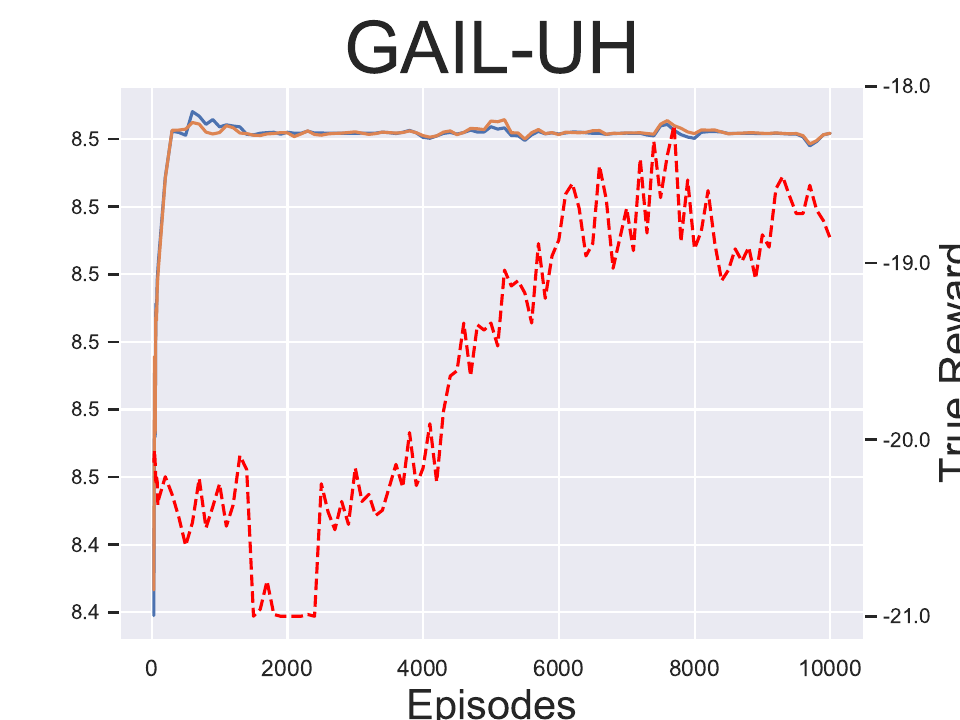}
        \includegraphics[ width=1\textwidth]{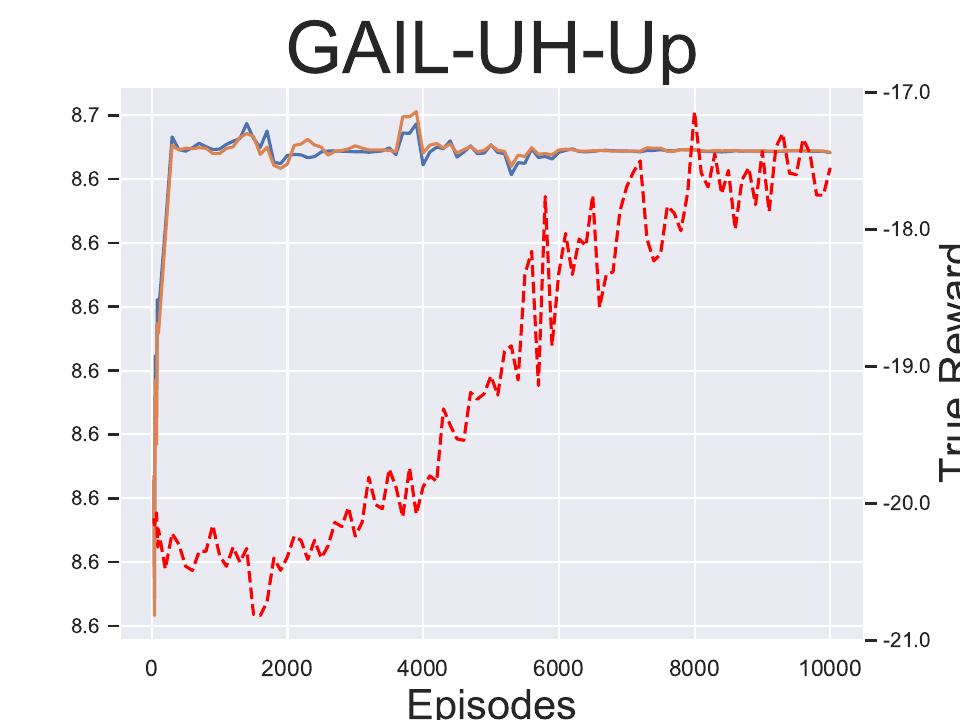}
        \includegraphics[ width=1\textwidth]{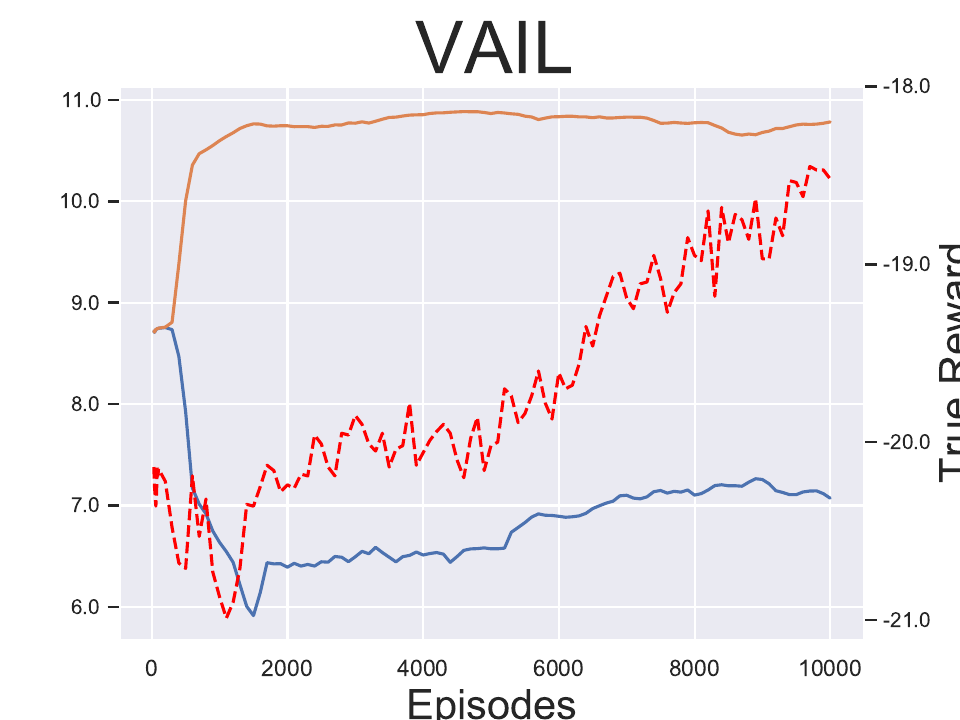}
        \includegraphics[ width=1\textwidth]{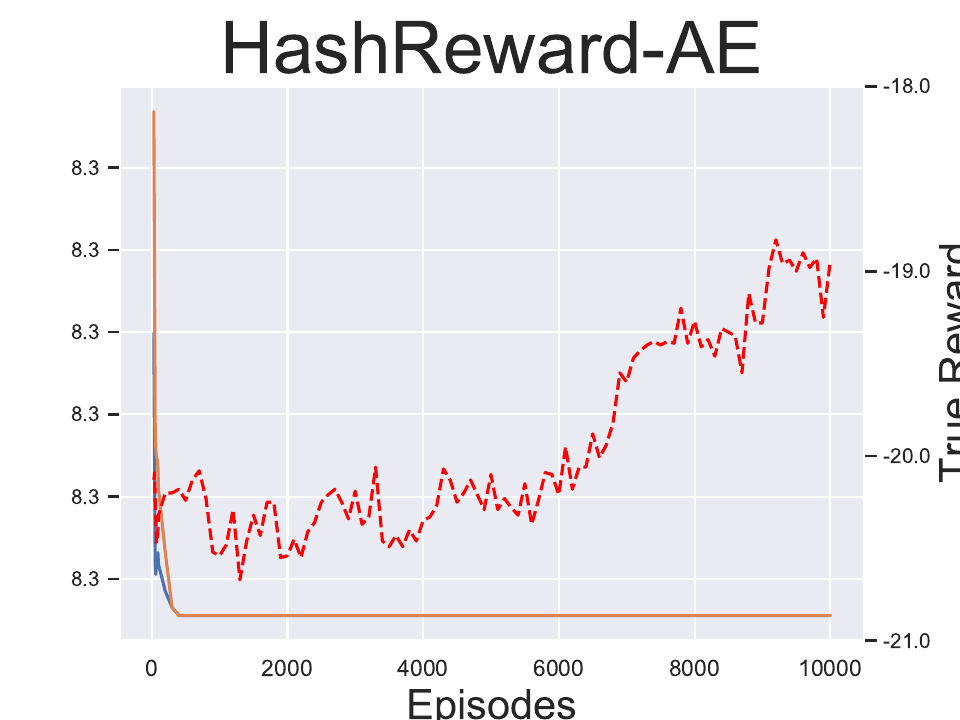}
        \includegraphics[ width=1\textwidth]{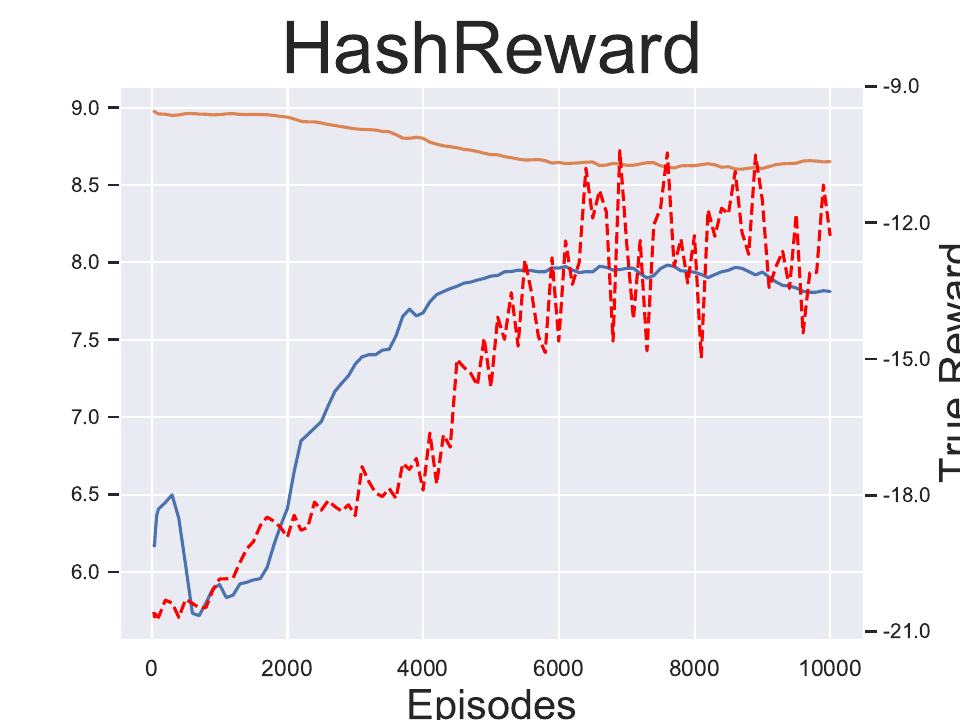}
    \end{minipage}}
\end{figure*}
\begin{figure*}[!t]
    \addtocounter{figure}{1}
    \centering
    \subfigure[Qbert]{
    \begin{minipage}[b]{0.185\linewidth}
        \includegraphics[ width=1\textwidth]{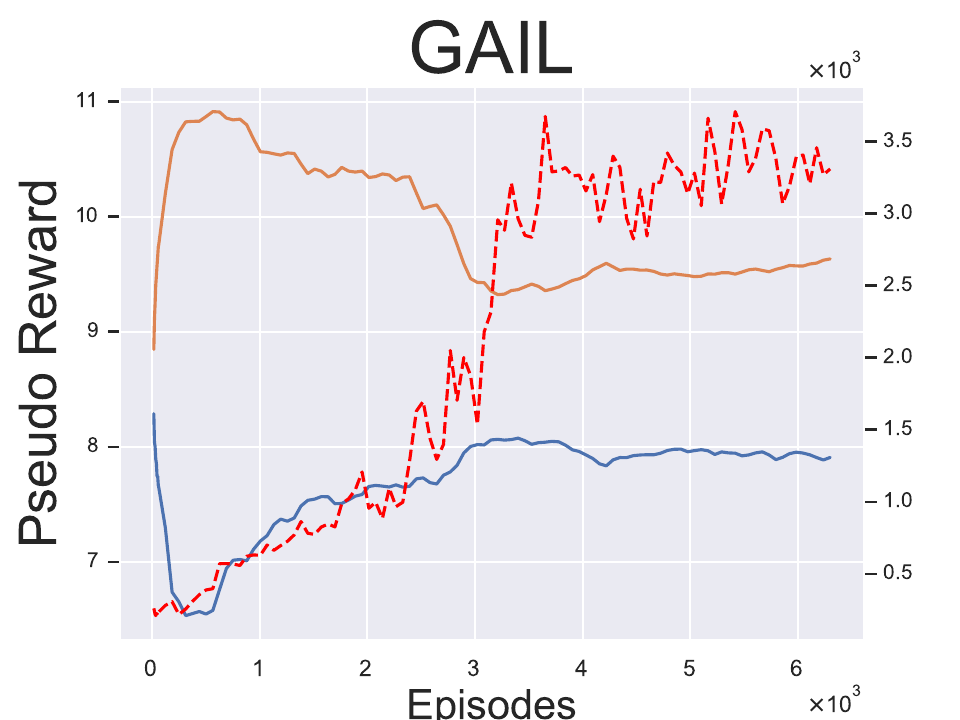}
        \includegraphics[ width=1\textwidth]{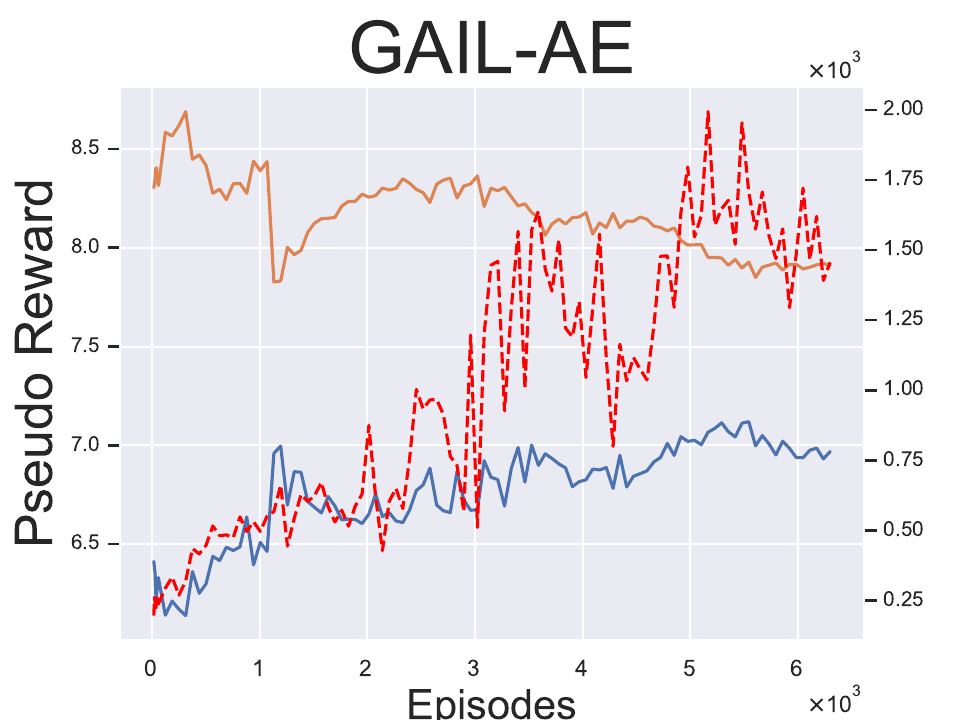}
        \includegraphics[ width=1\textwidth]{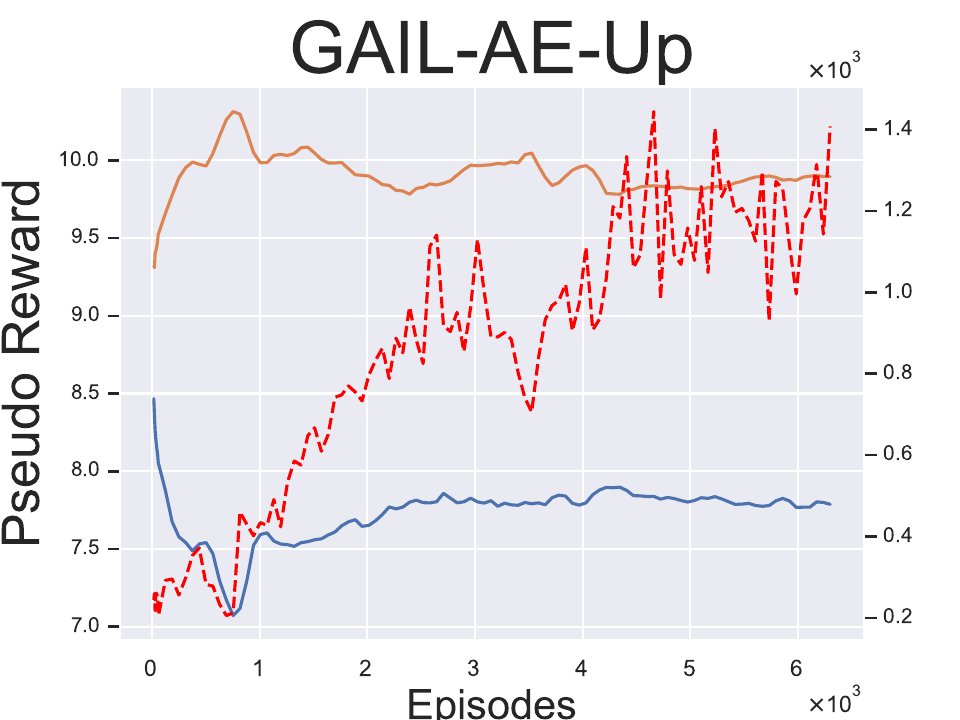}
        \includegraphics[ width=1\textwidth]{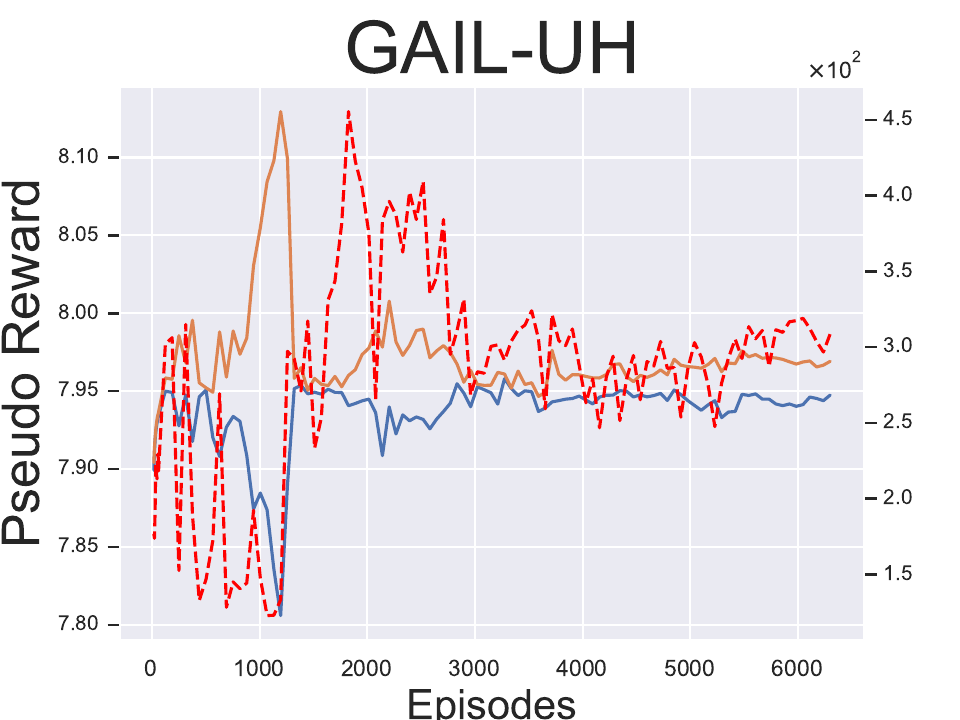}
        \includegraphics[ width=1\textwidth]{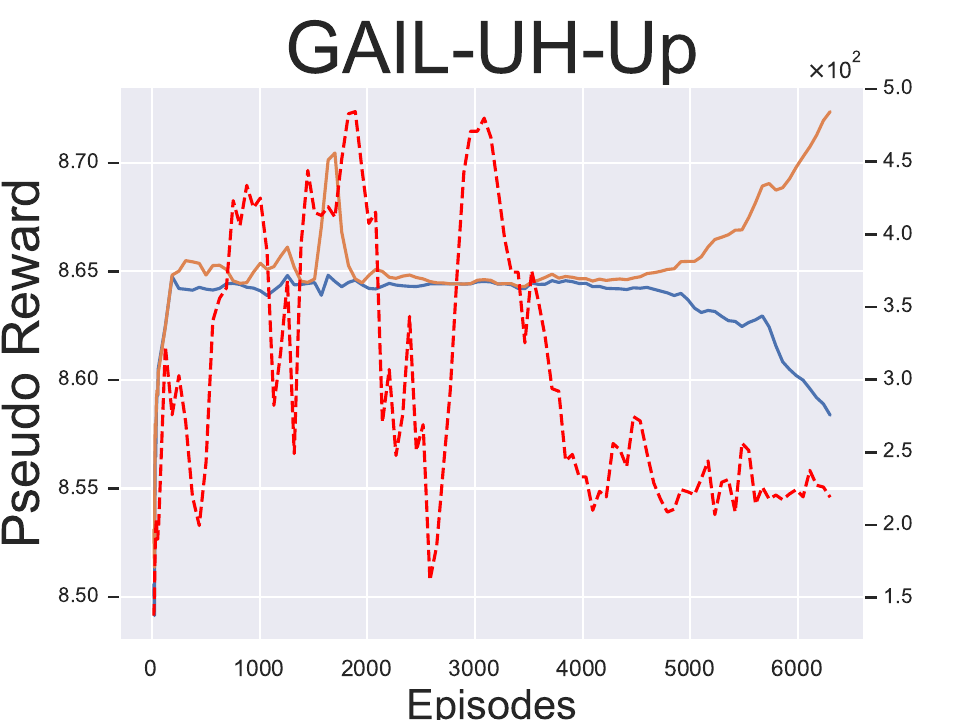}
        \includegraphics[ width=1\textwidth]{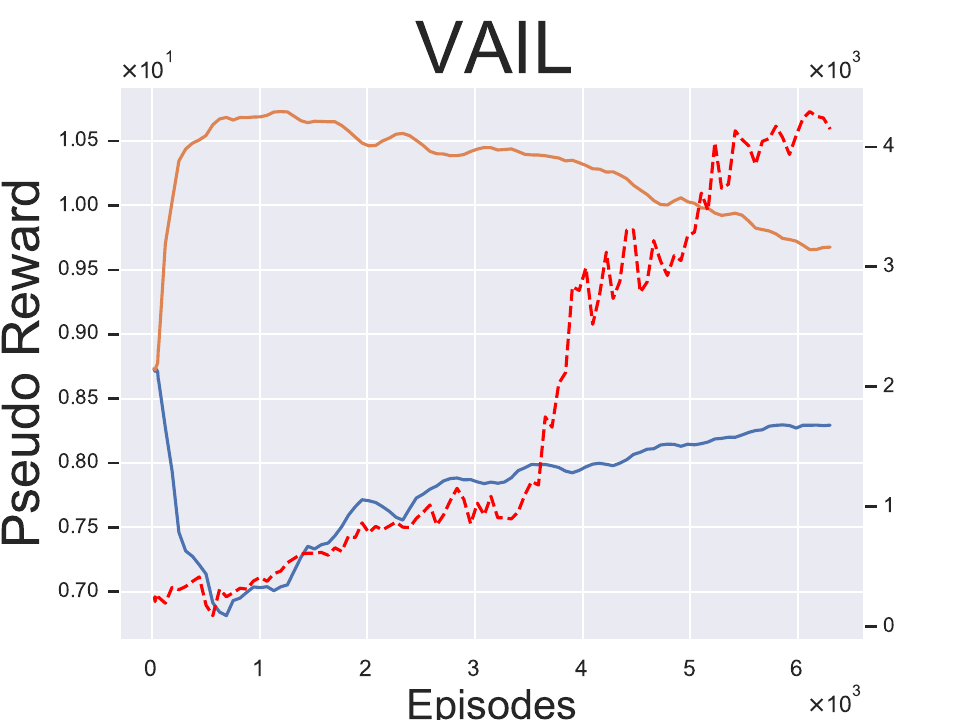}
        \includegraphics[ width=1\textwidth]{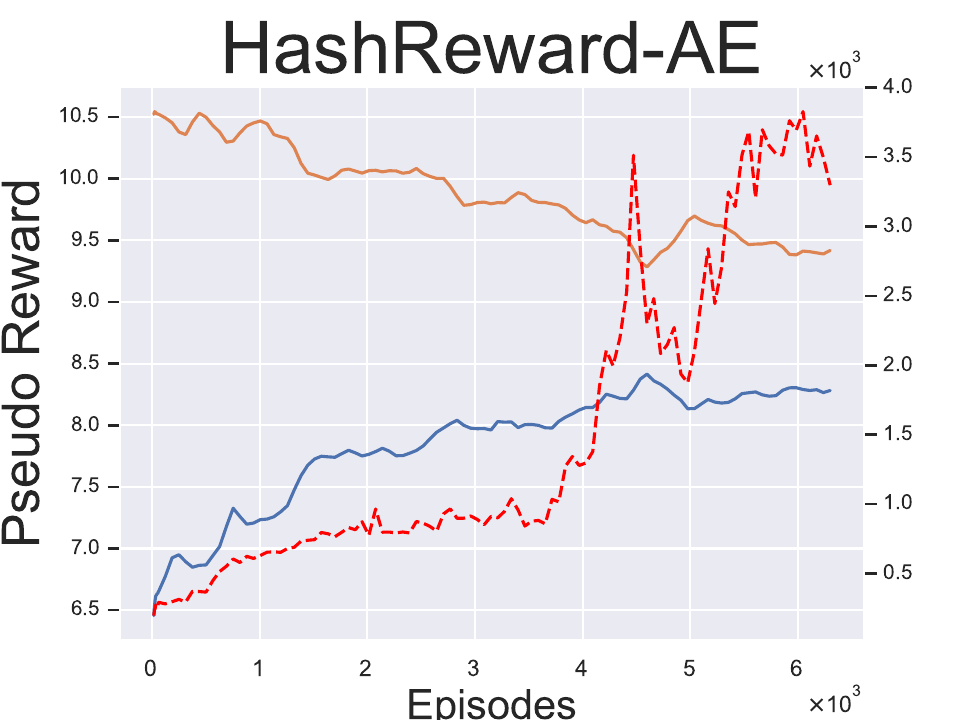}
        \includegraphics[ width=1\textwidth]{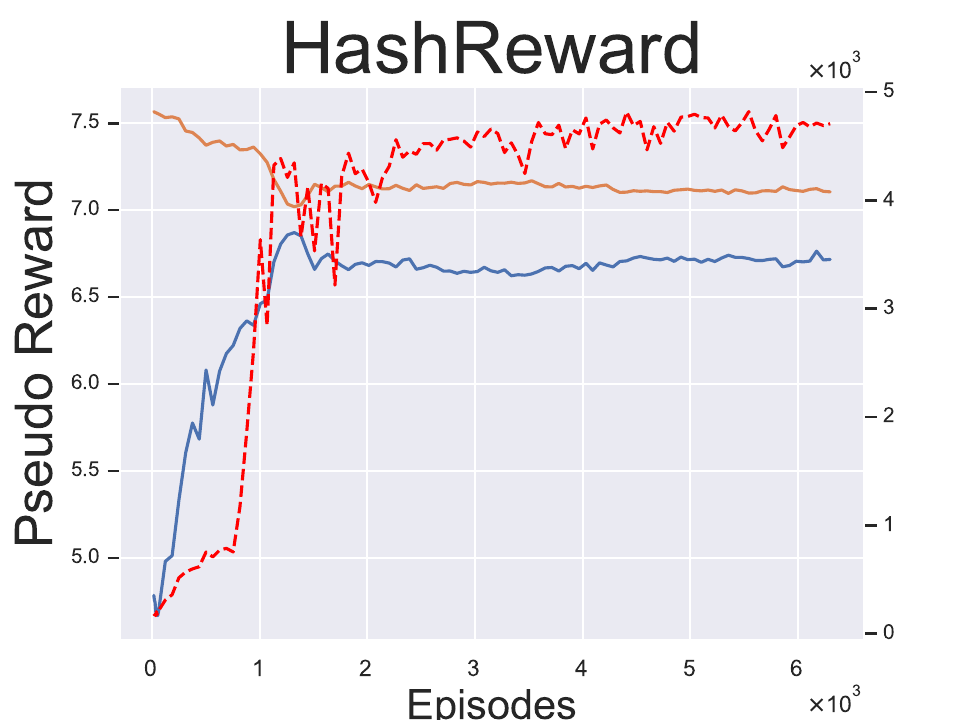}
    \end{minipage}}
    \subfigure[Seaquest]{
    \begin{minipage}[b]{0.185\linewidth}
        \includegraphics[ width=1\textwidth]{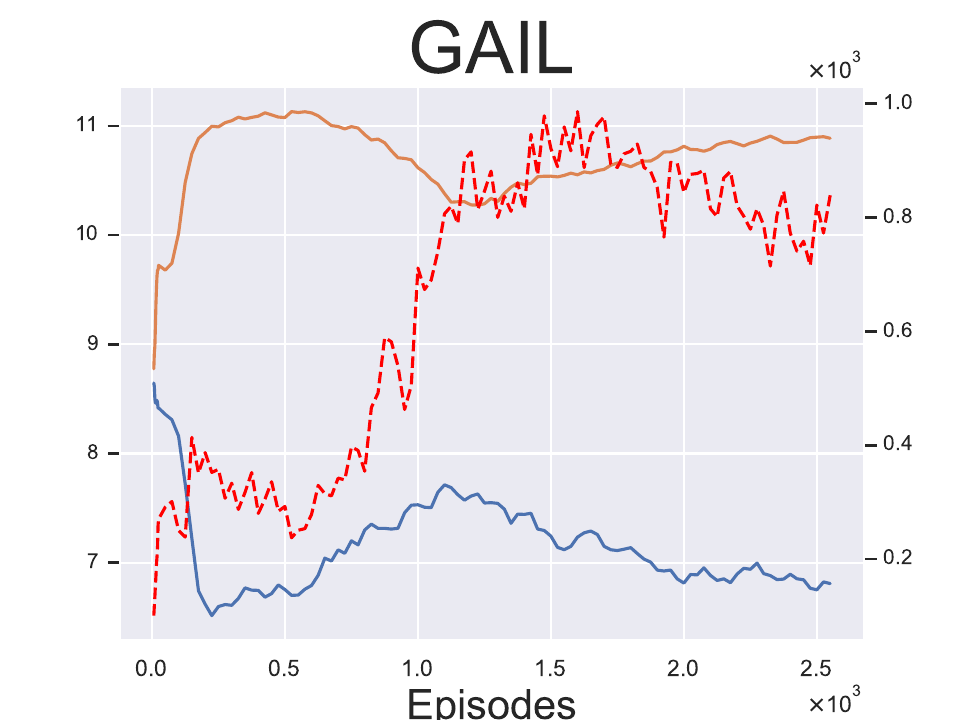}
        \includegraphics[ width=1\textwidth]{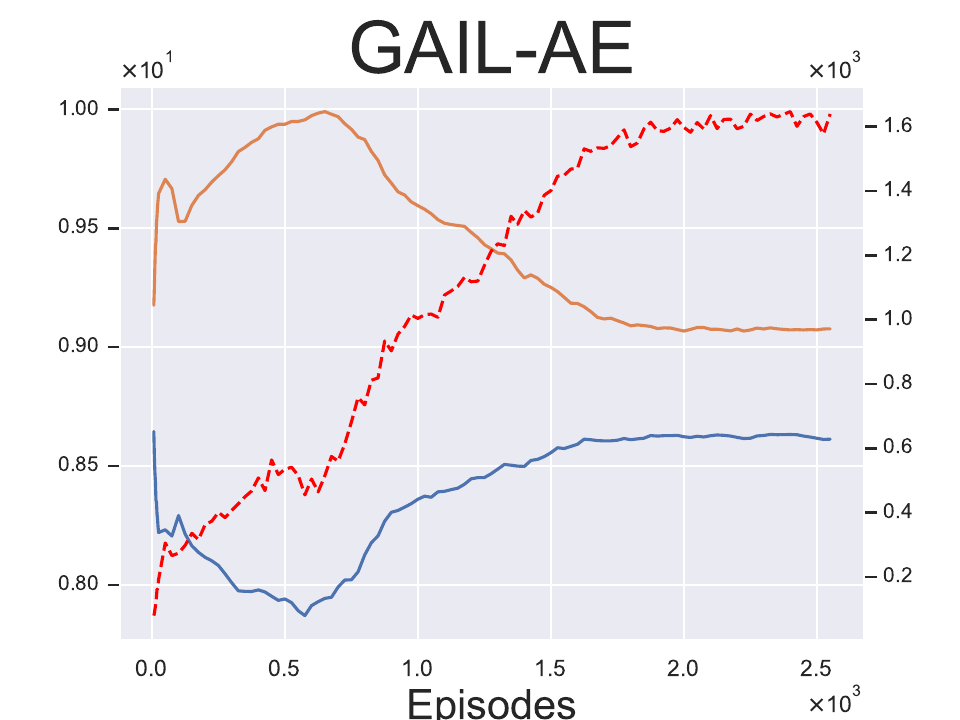}
        \includegraphics[ width=1\textwidth]{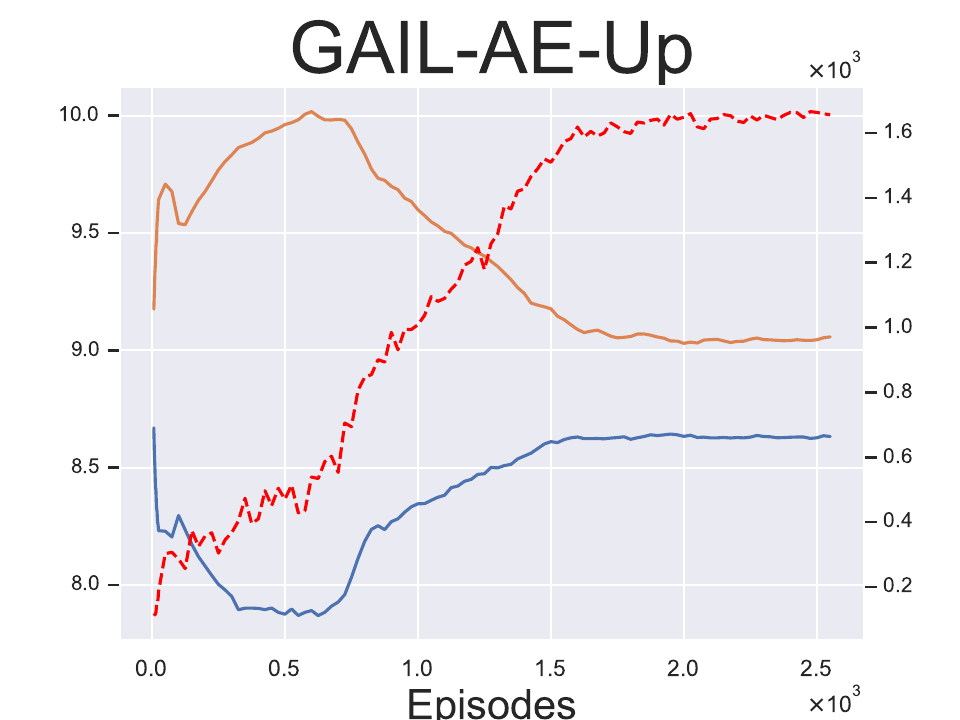}
        \includegraphics[ width=1\textwidth]{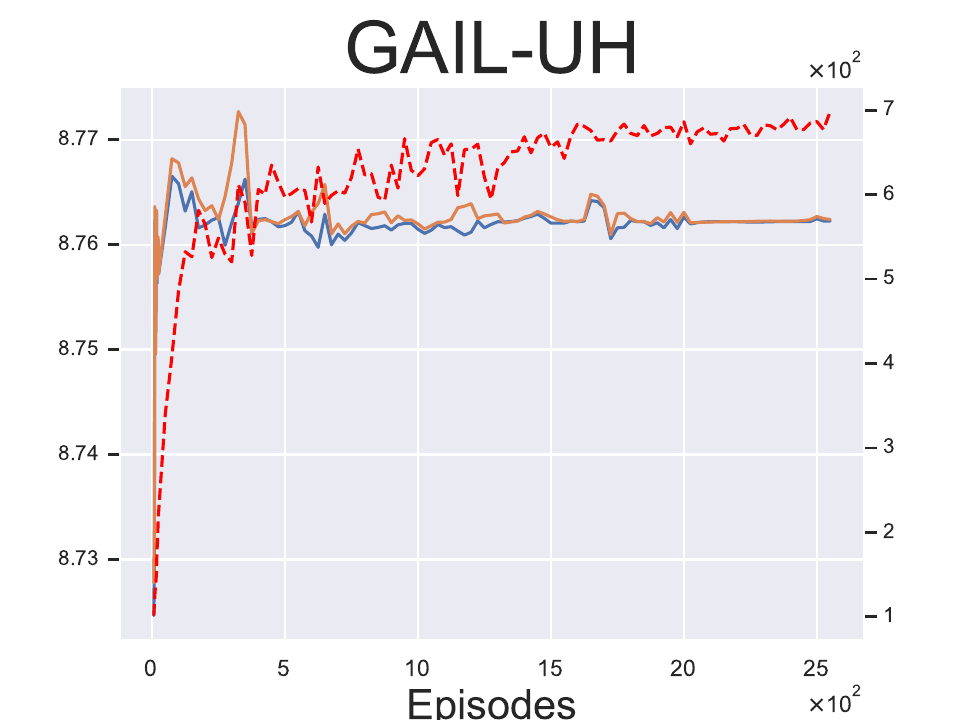}
        \includegraphics[ width=1\textwidth]{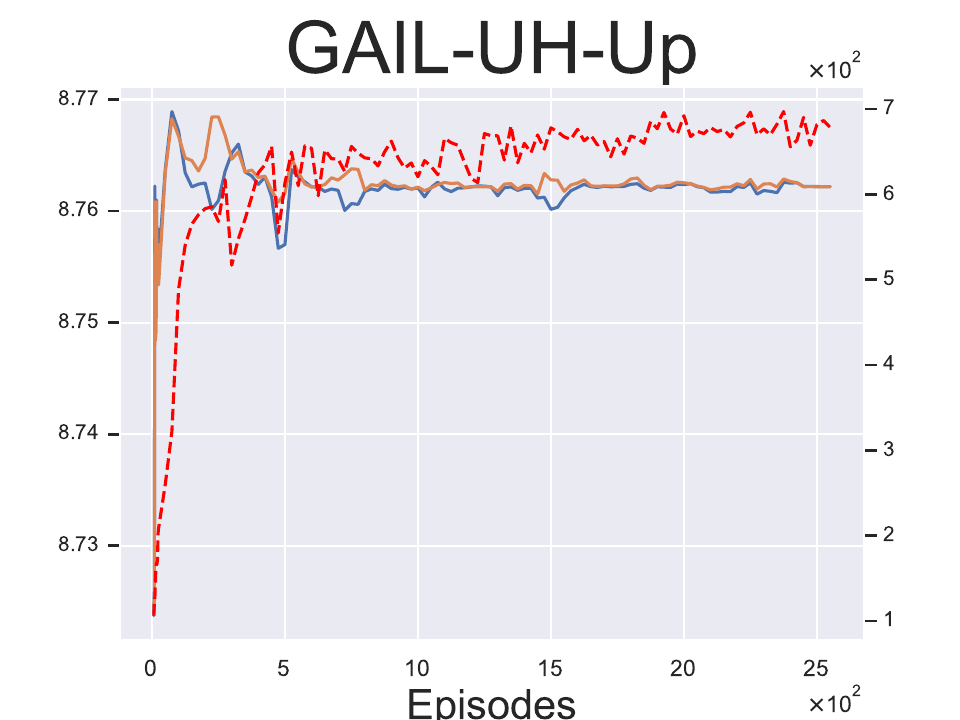}
        \includegraphics[ width=1\textwidth]{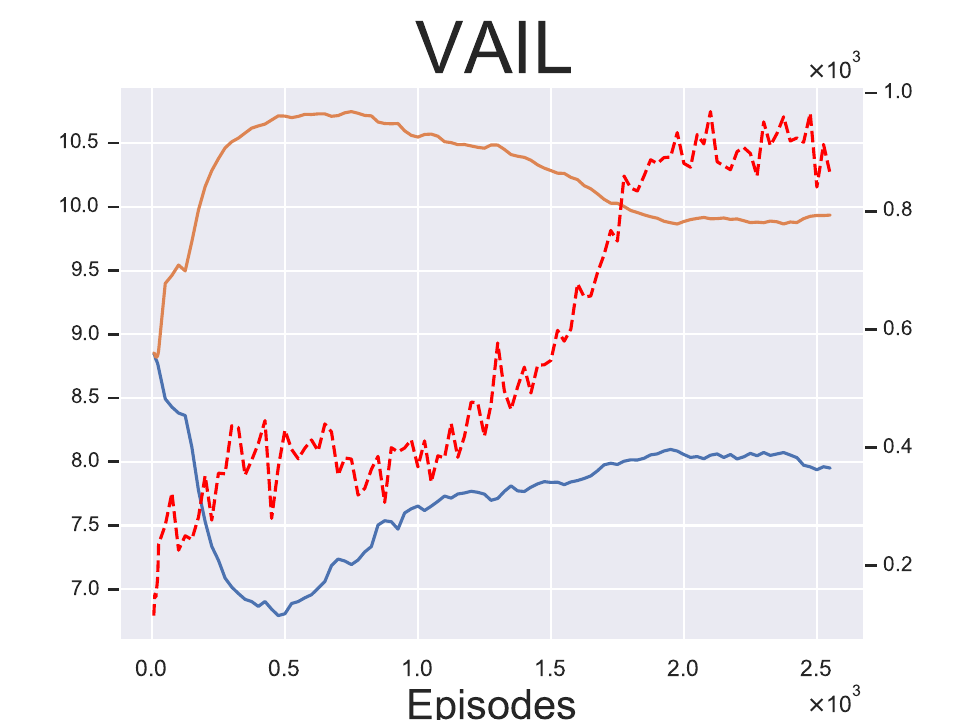}
        \includegraphics[ width=1\textwidth]{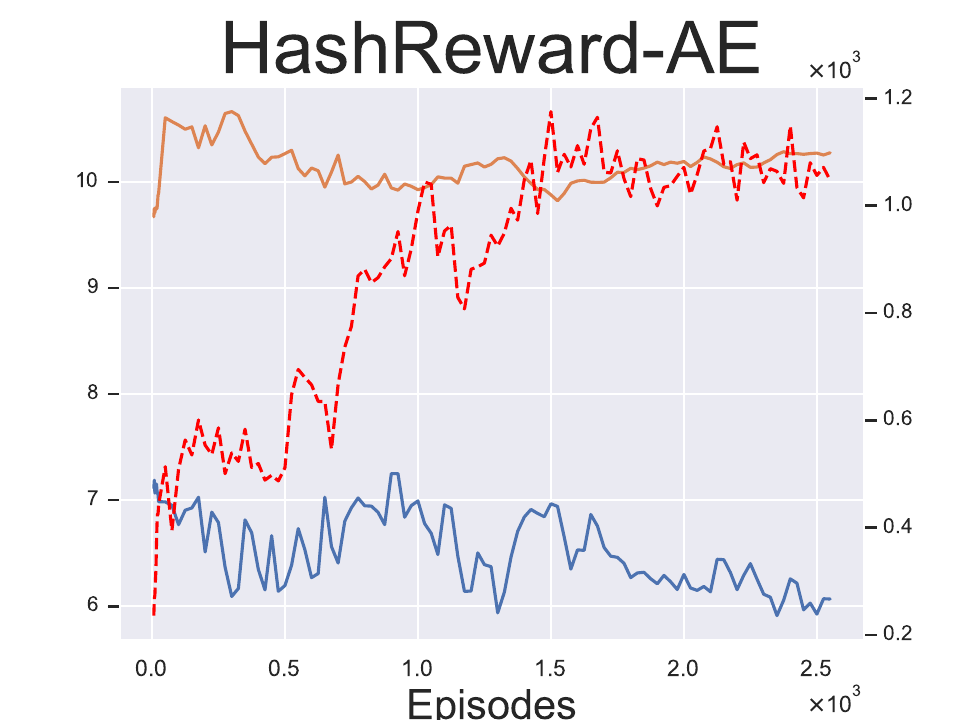}
        \includegraphics[ width=1\textwidth]{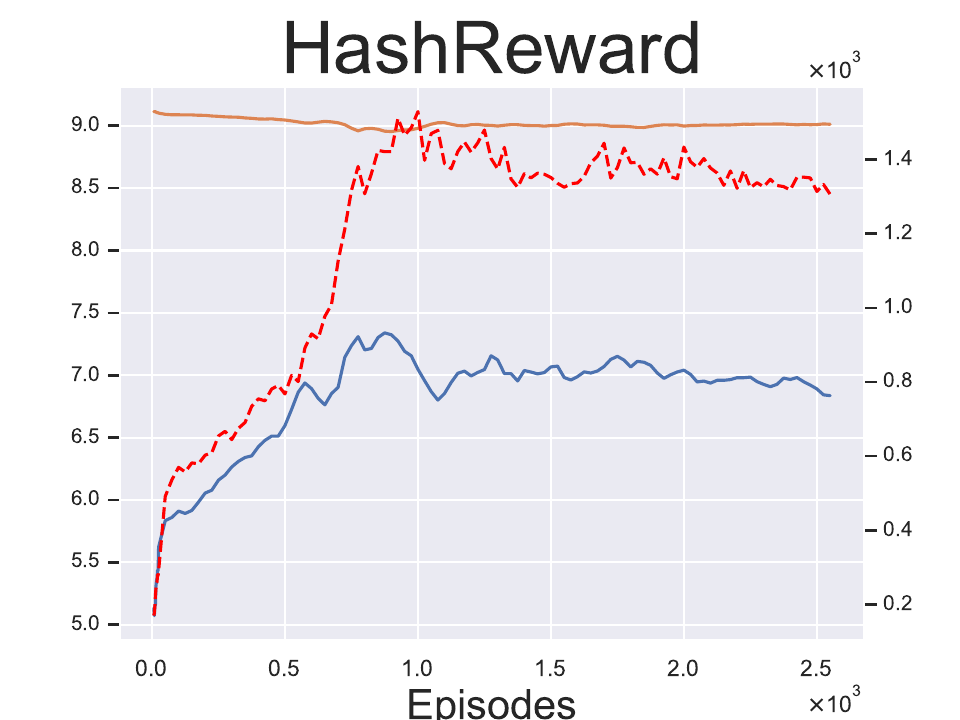}
    \end{minipage}}
    \subfigure[SpaceInvaders]{
    \begin{minipage}[b]{0.185\linewidth}
        \includegraphics[ width=1\textwidth]{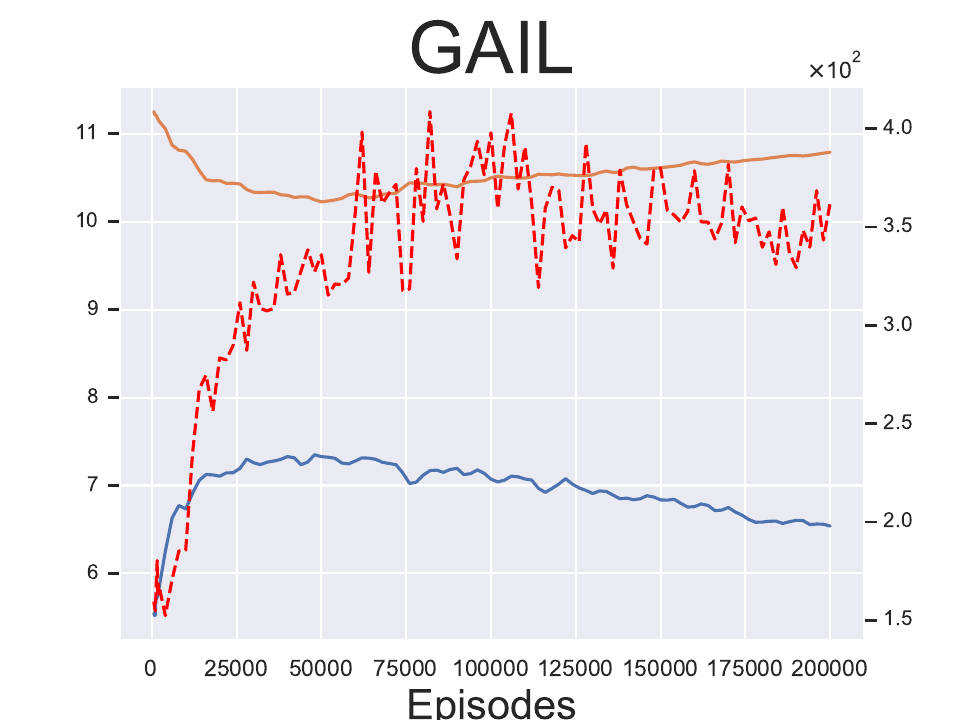}
        \includegraphics[ width=1\textwidth]{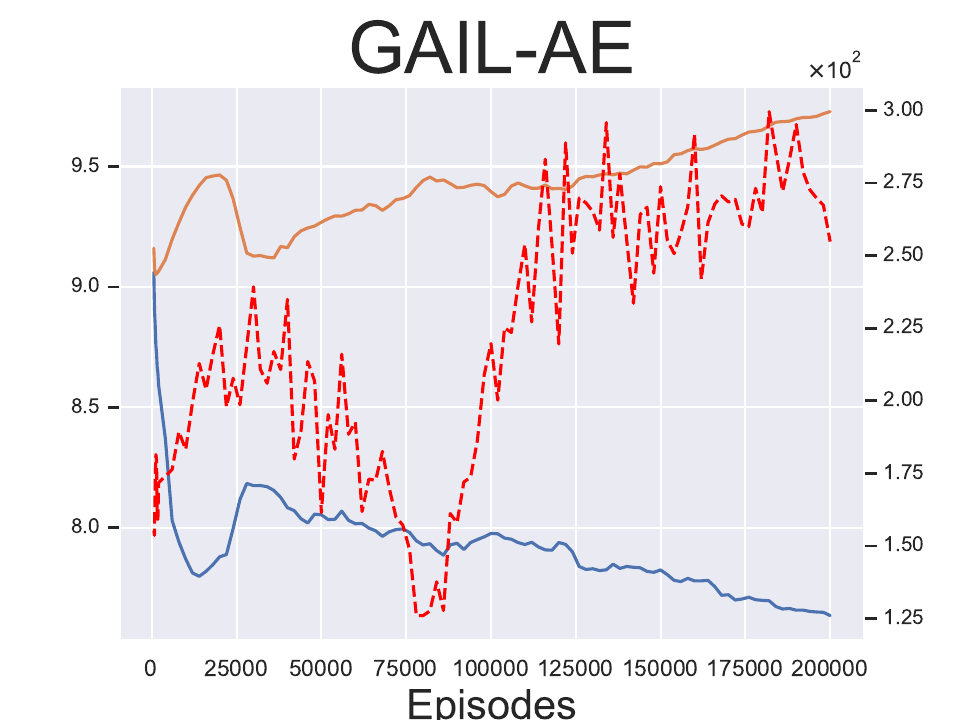}
        \includegraphics[ width=1\textwidth]{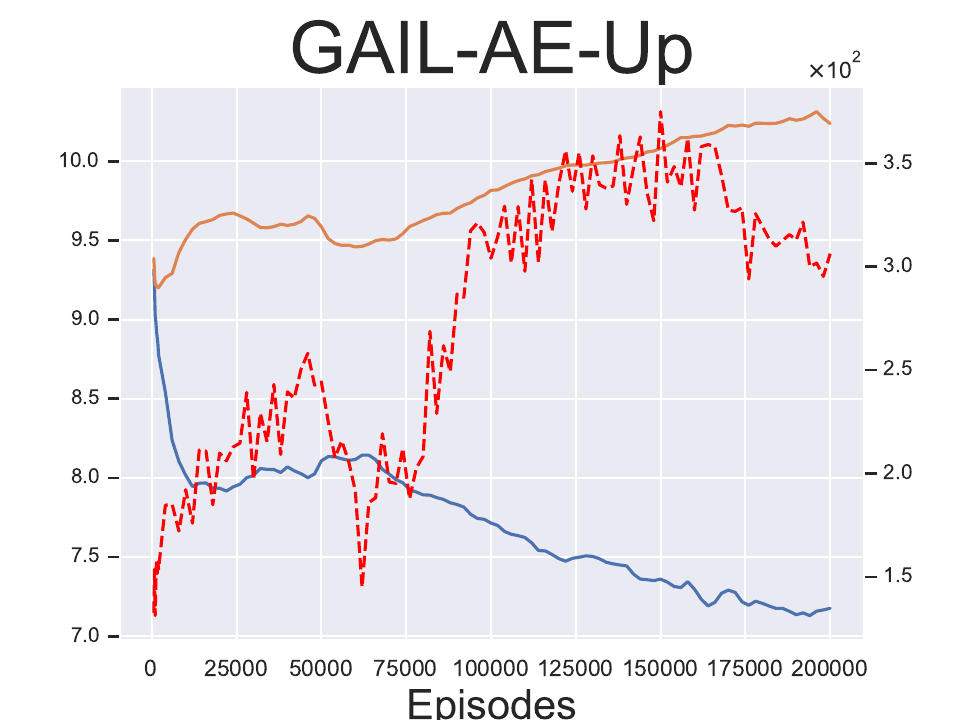}
        \includegraphics[ width=1\textwidth]{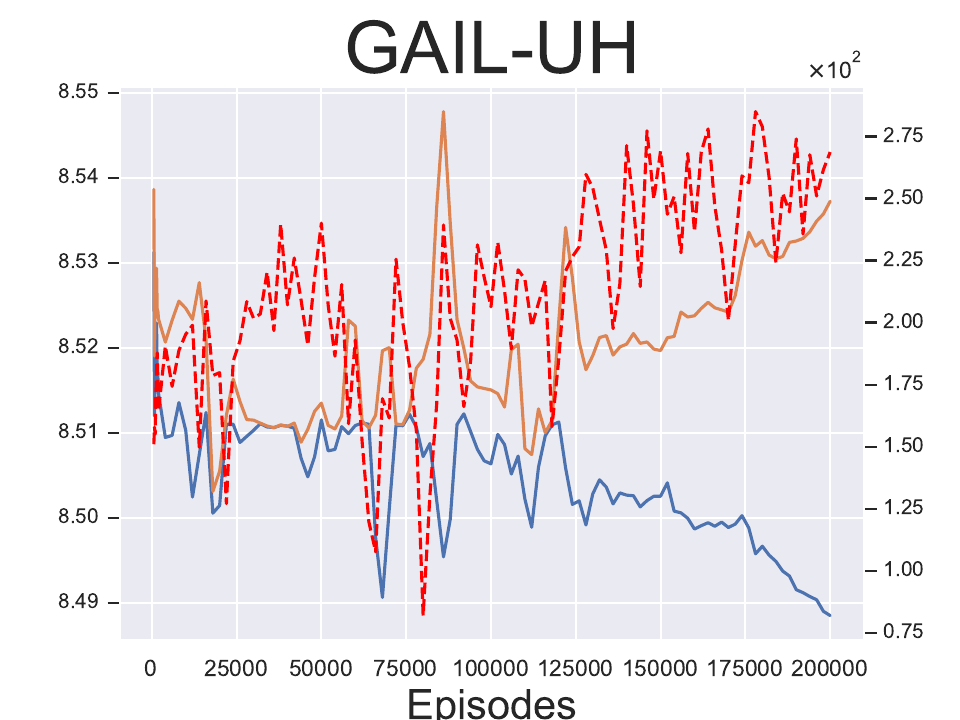}
        \includegraphics[ width=1\textwidth]{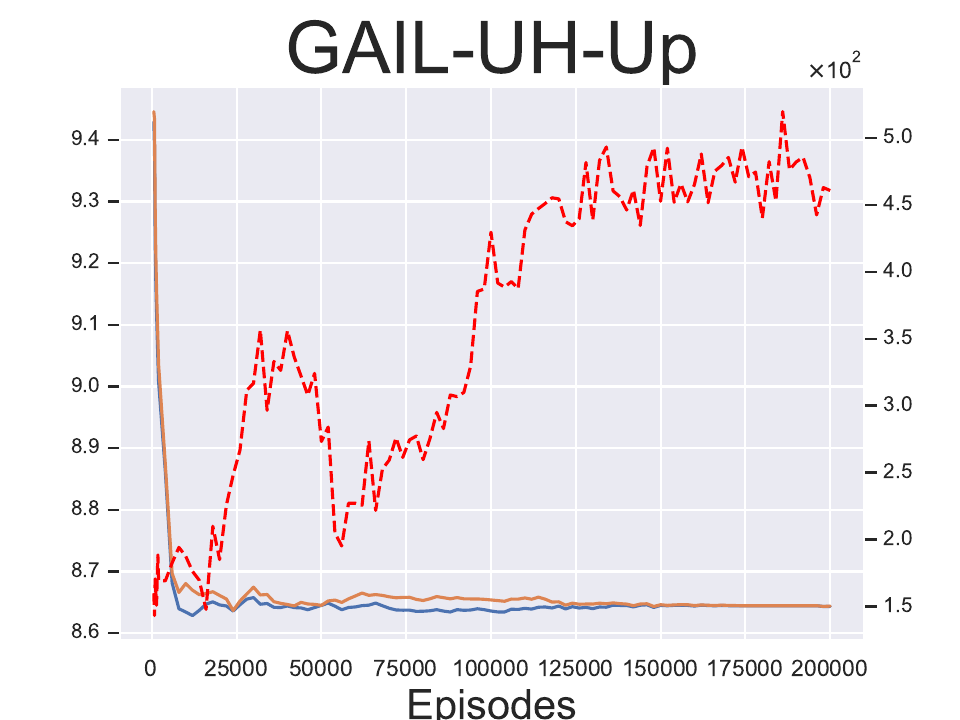}
        \includegraphics[ width=1\textwidth]{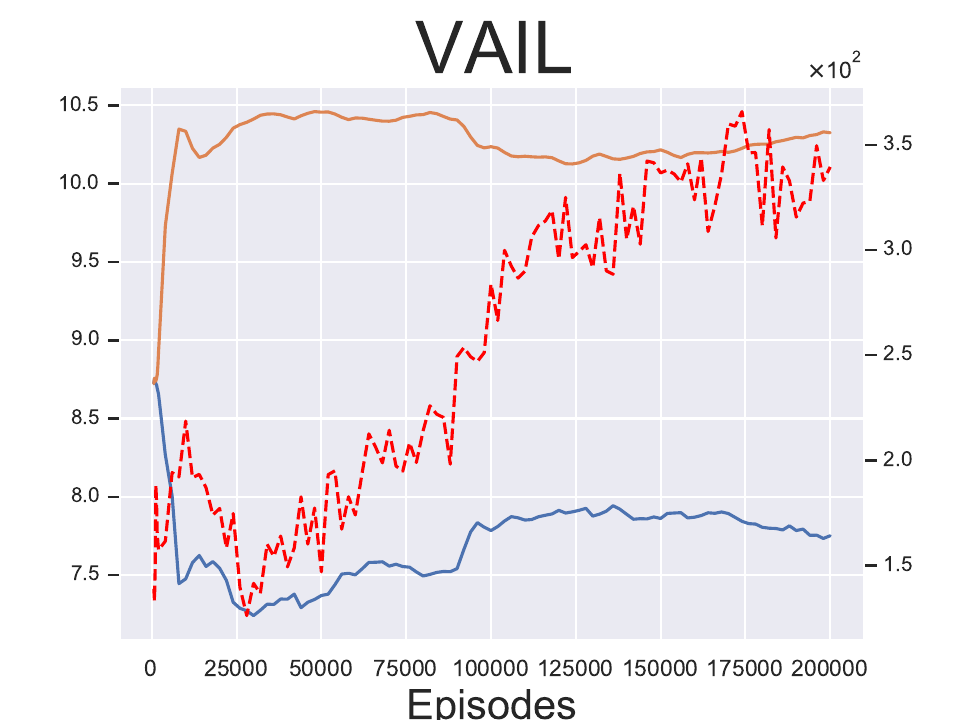}
        \includegraphics[ width=1\textwidth]{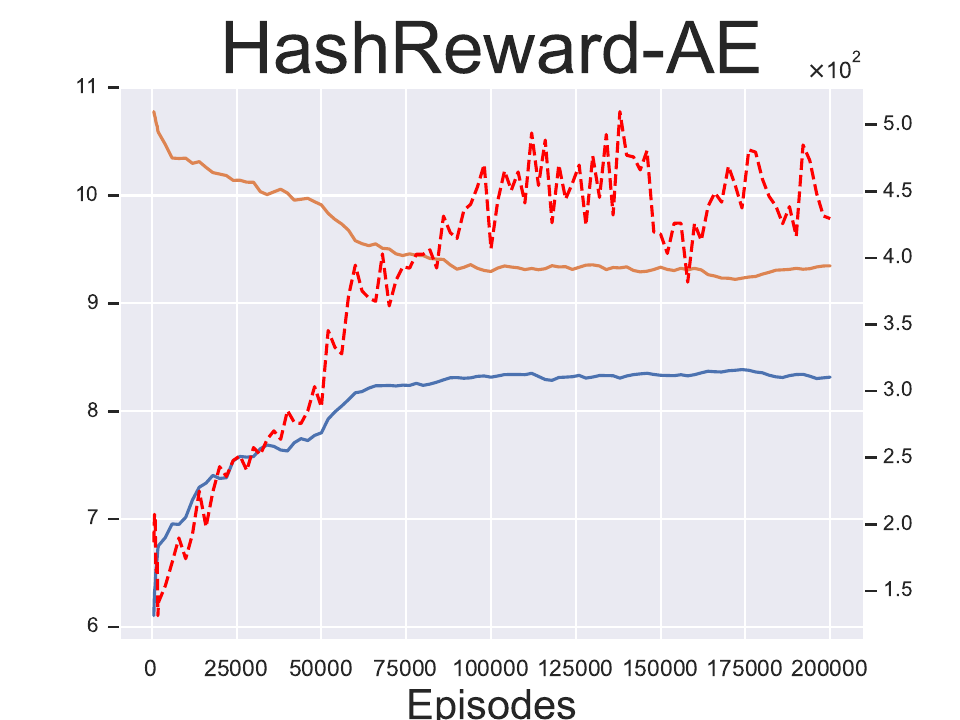}
        \includegraphics[ width=1\textwidth]{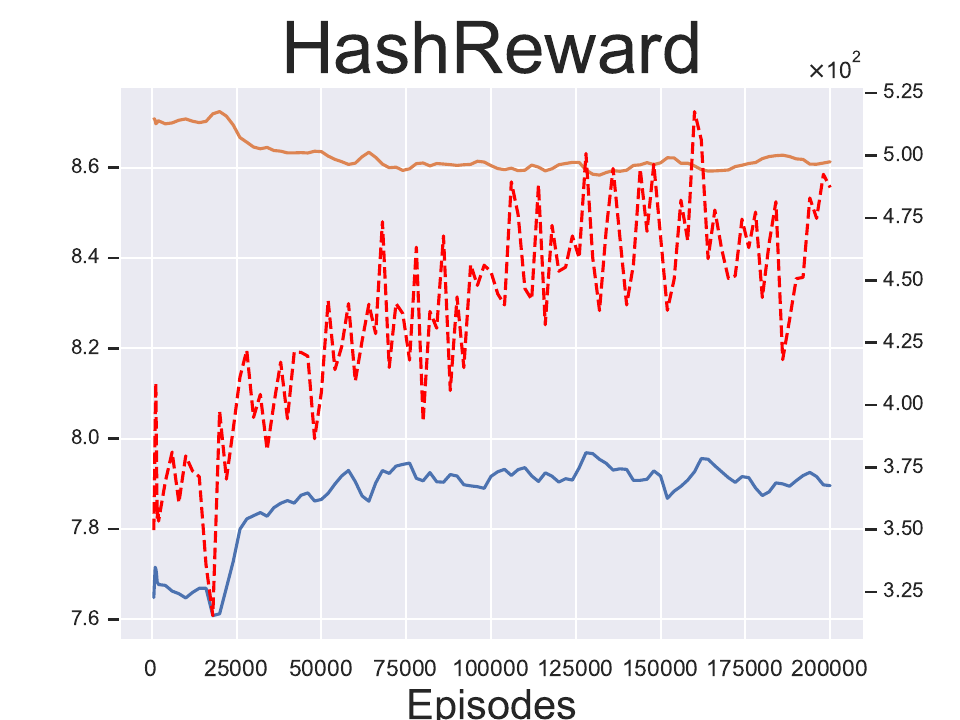}
    \end{minipage}}
    \subfigure[UpNDown]{
    \begin{minipage}[b]{0.185\linewidth}
        \includegraphics[ width=1\textwidth]{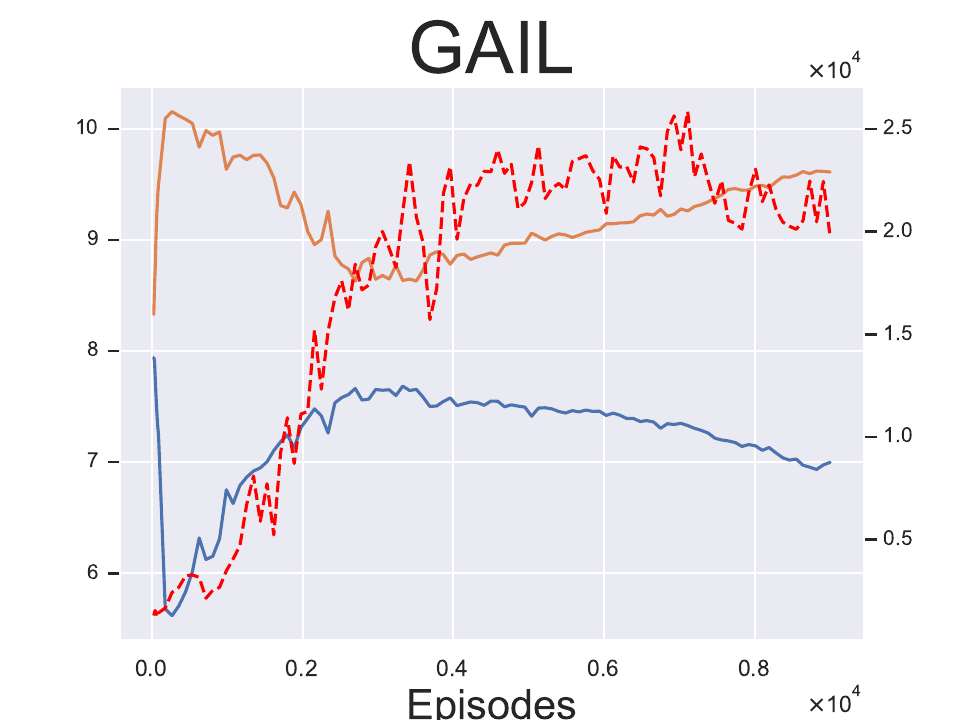}
        \includegraphics[ width=1\textwidth]{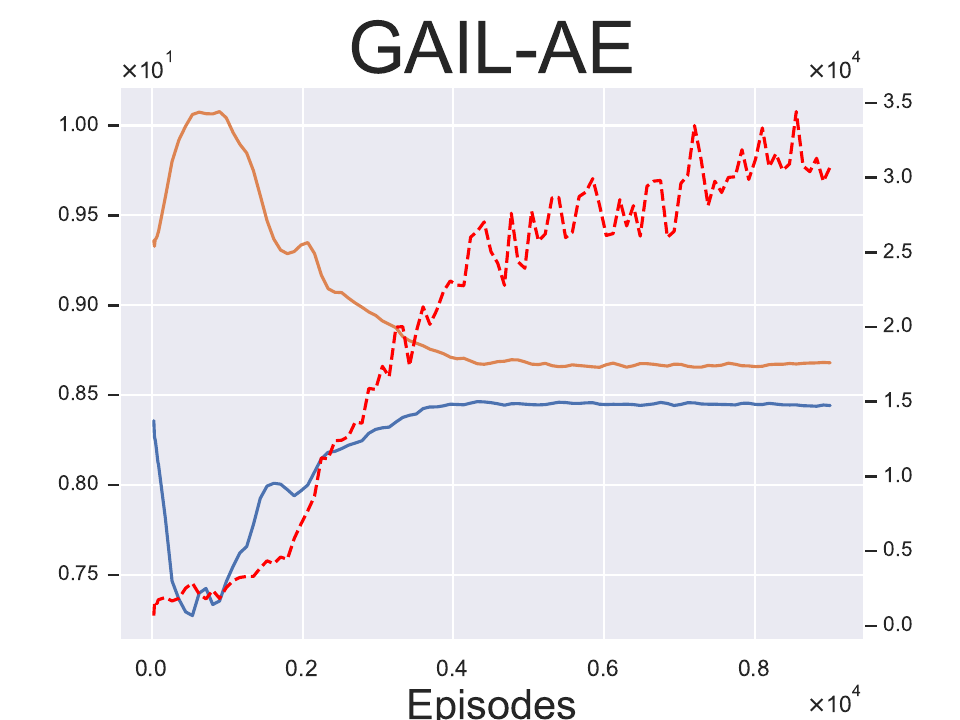}
        \includegraphics[ width=1\textwidth]{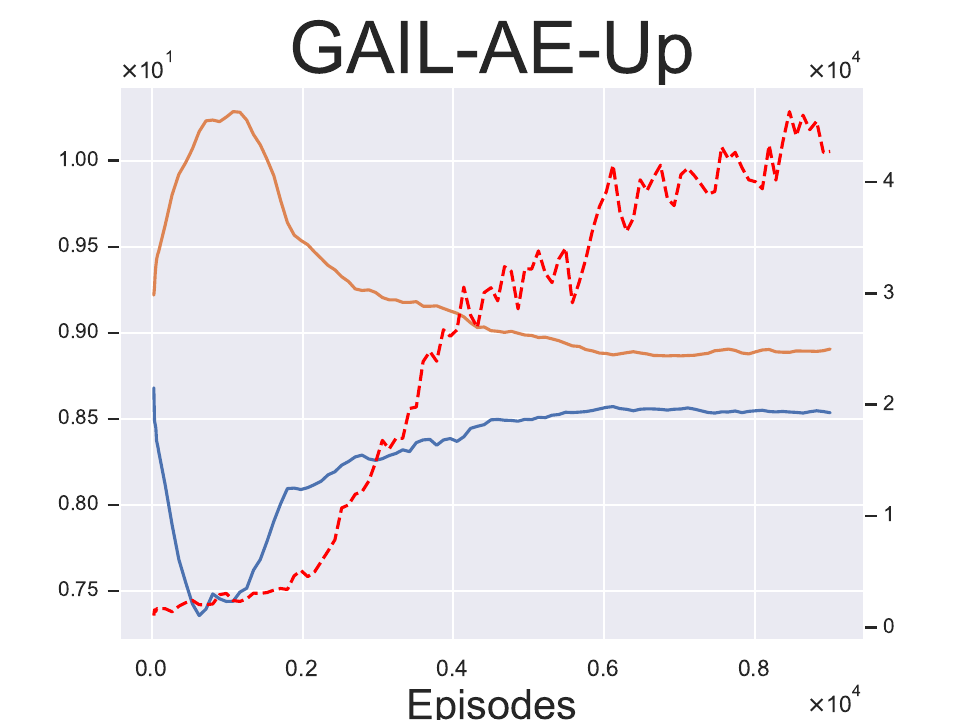}
        \includegraphics[ width=1\textwidth]{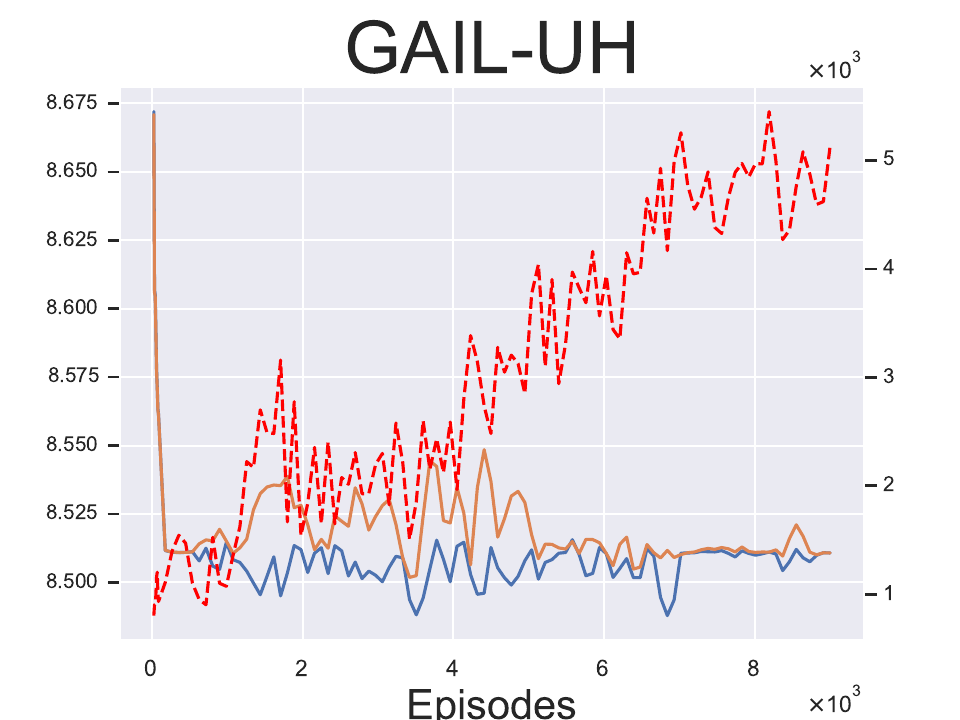}
        \includegraphics[ width=1\textwidth]{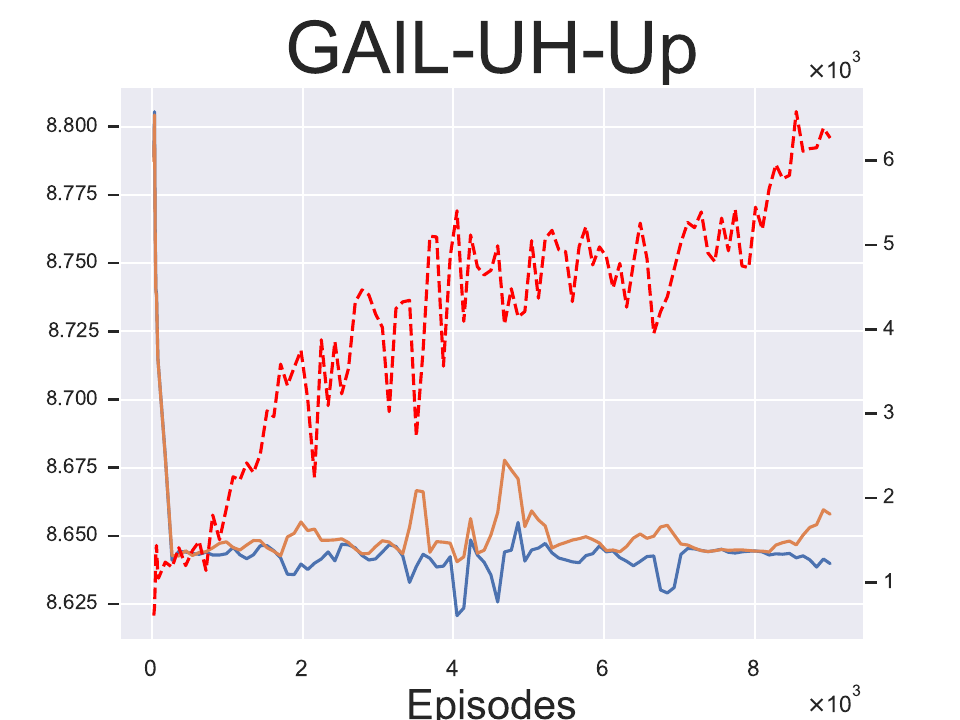}
        \includegraphics[ width=1\textwidth]{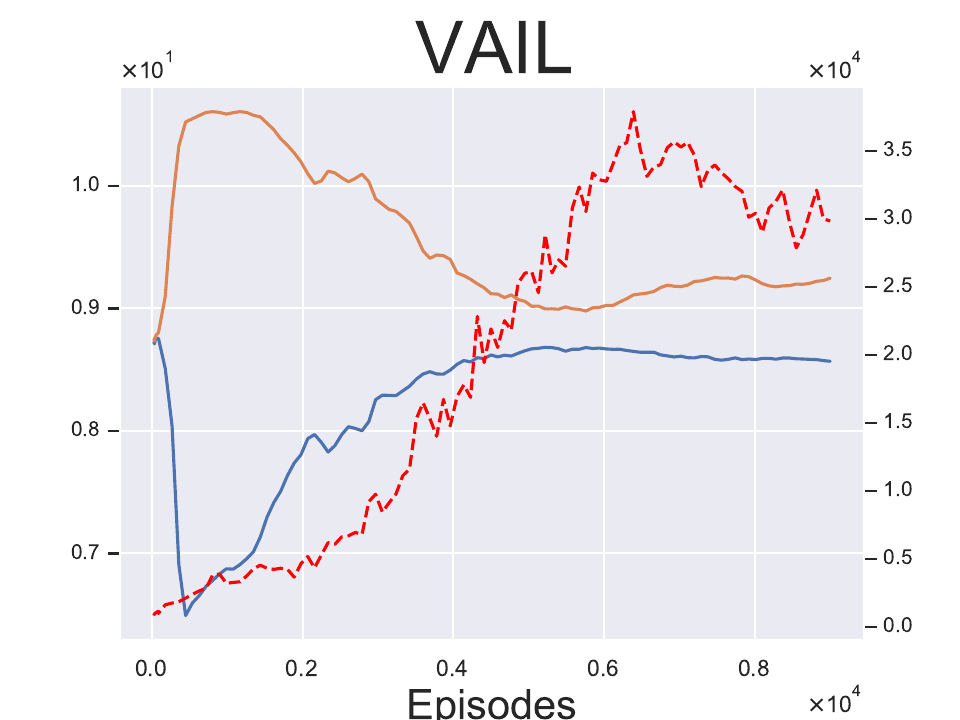}
        \includegraphics[ width=1\textwidth]{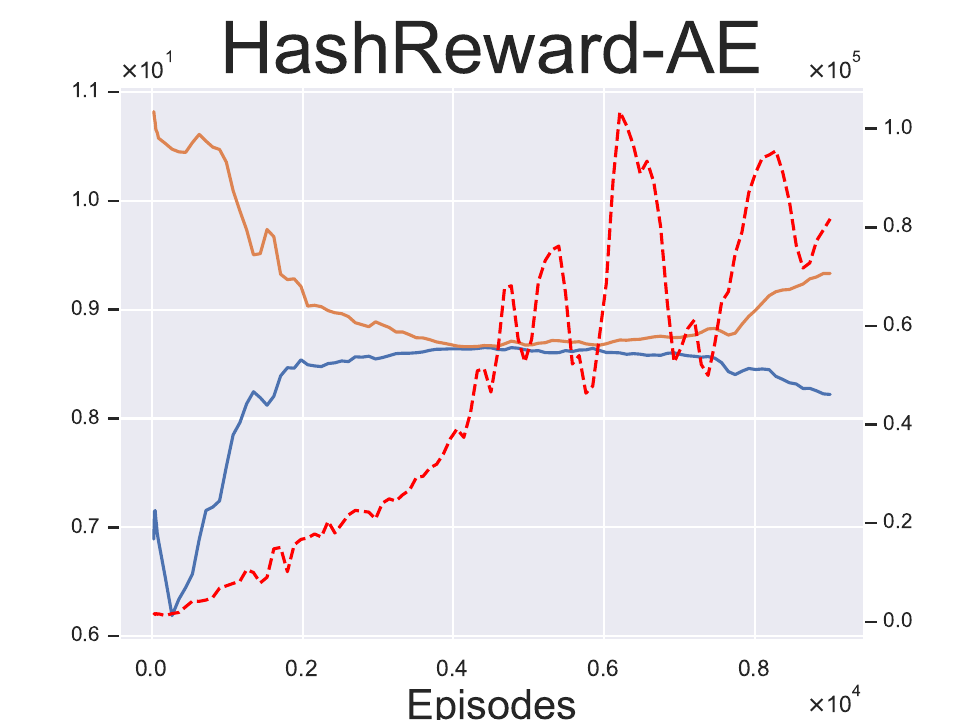}
        \includegraphics[ width=1\textwidth]{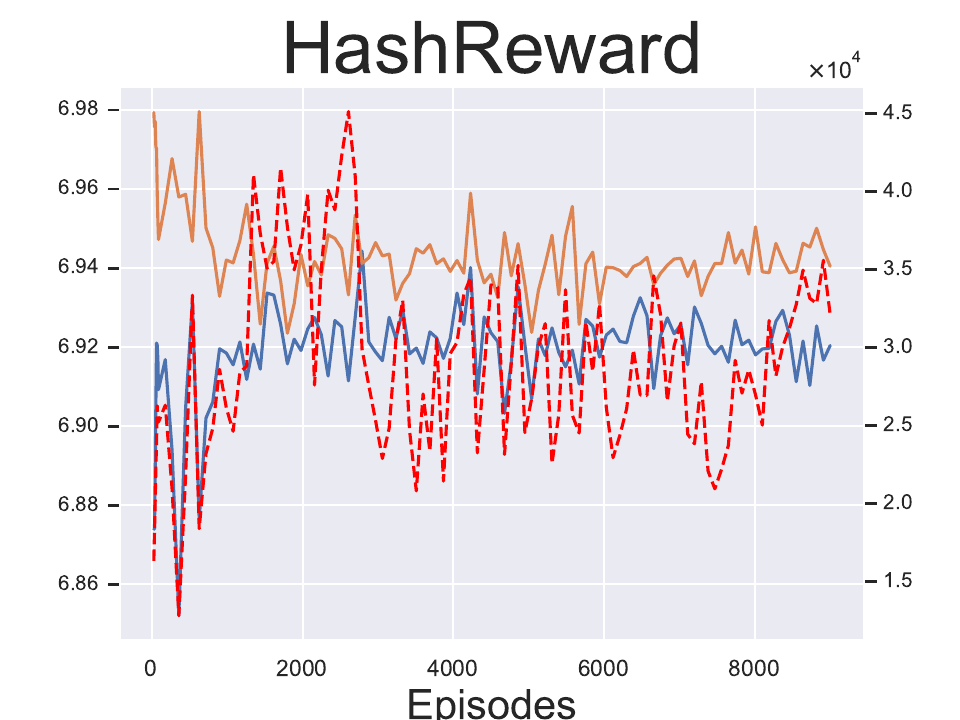}
    \end{minipage}}
    \subfigure[Zaxxon]{
    \begin{minipage}[b]{0.185\linewidth}
        \includegraphics[ width=1\textwidth]{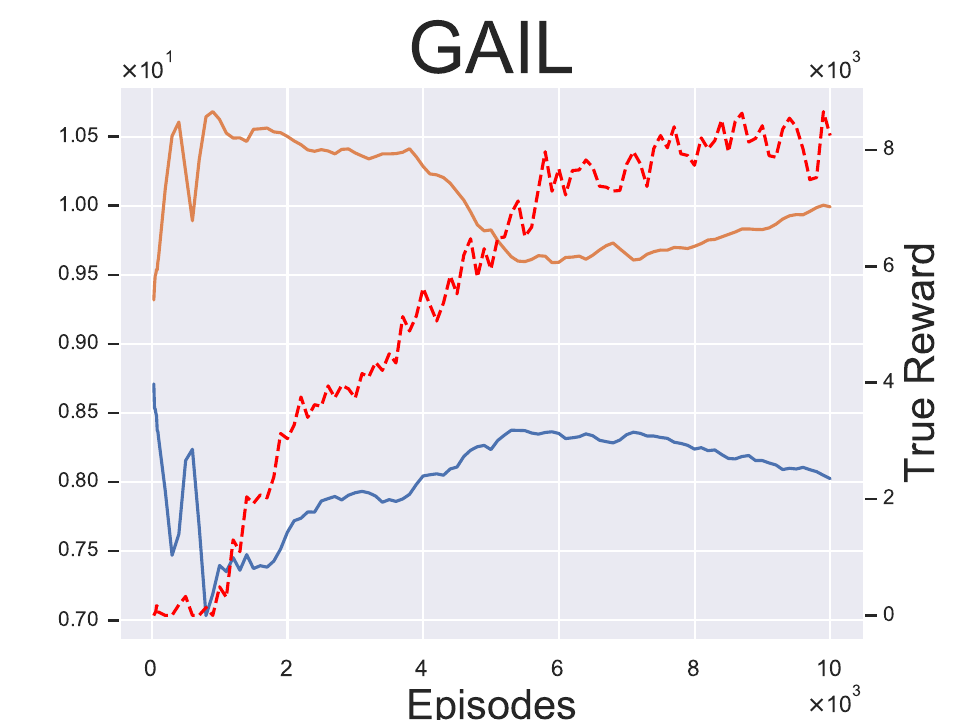}
        \includegraphics[ width=1\textwidth]{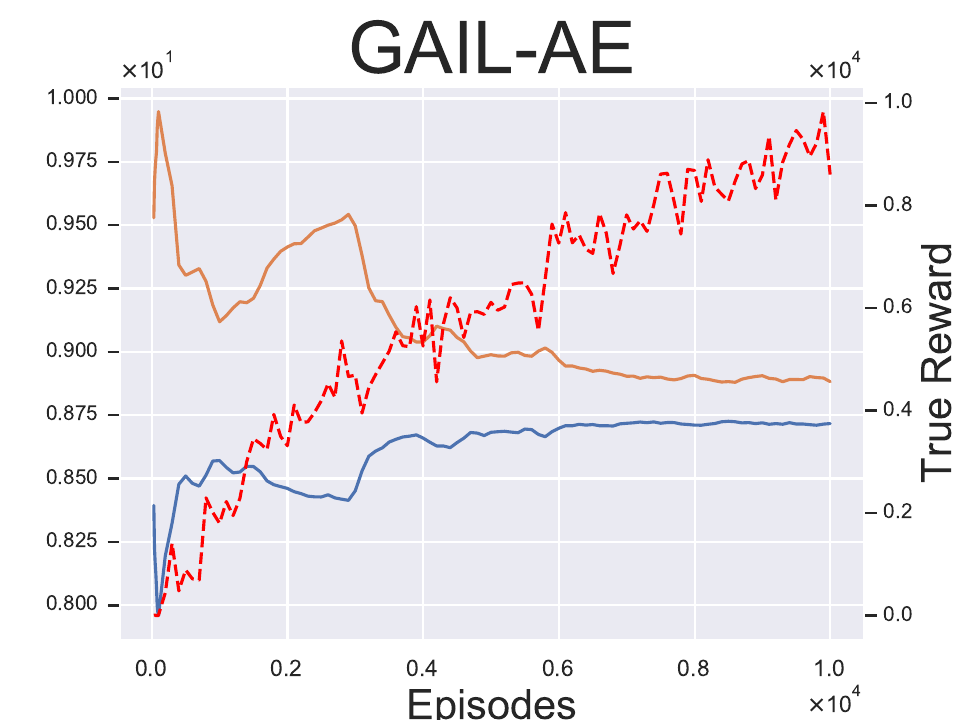}
        \includegraphics[ width=1\textwidth]{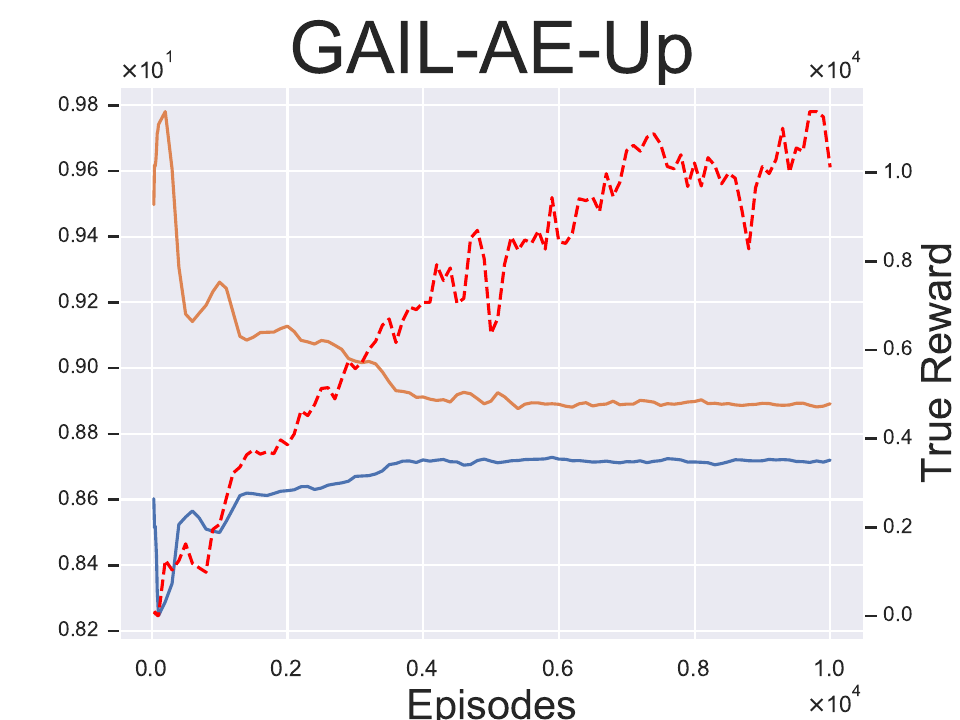}
        \includegraphics[ width=1\textwidth]{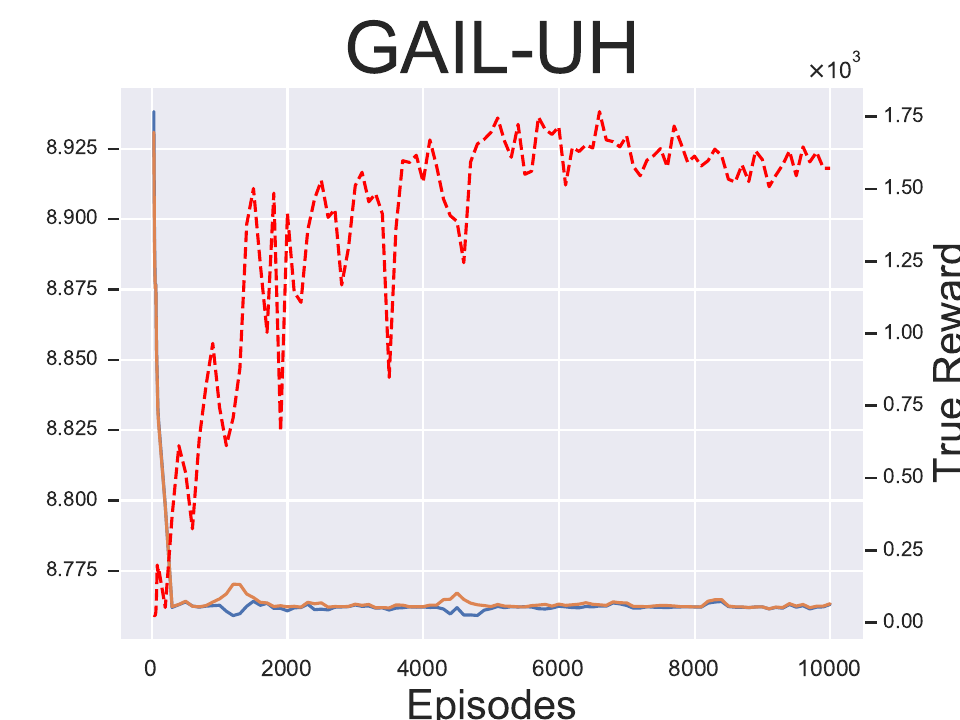}
        \includegraphics[ width=1\textwidth]{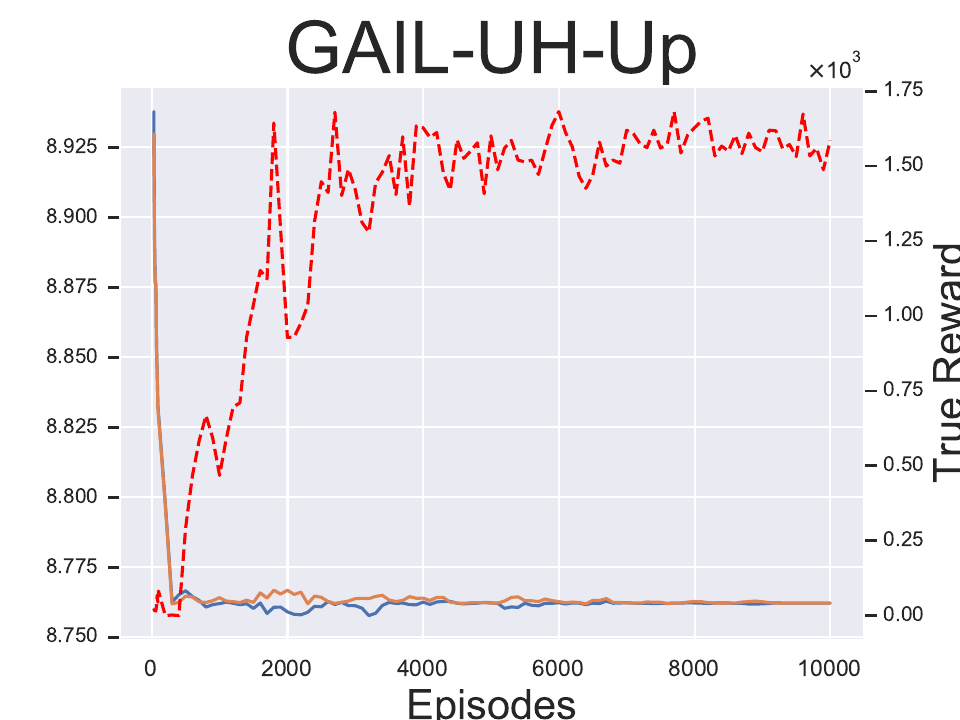}
        \includegraphics[ width=1\textwidth]{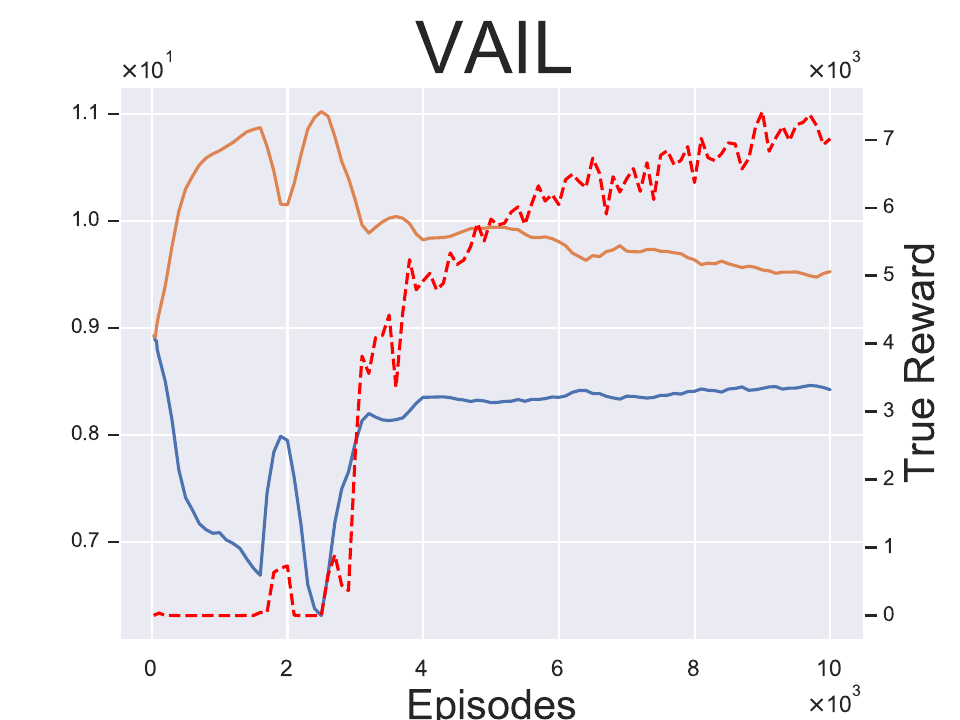}
        \includegraphics[ width=1\textwidth]{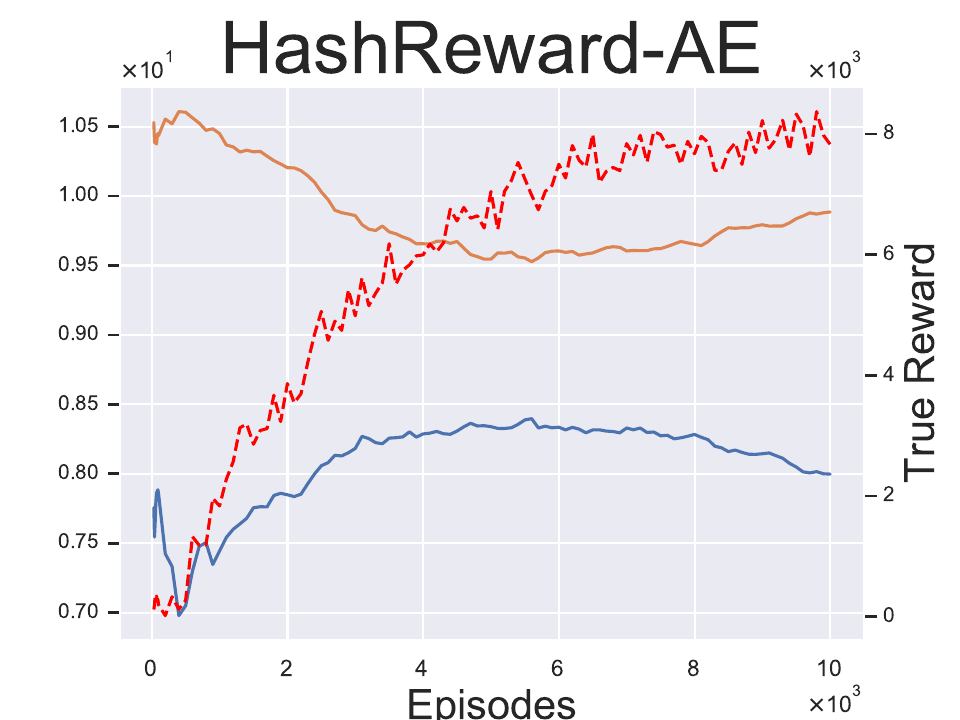}
        \includegraphics[ width=1\textwidth]{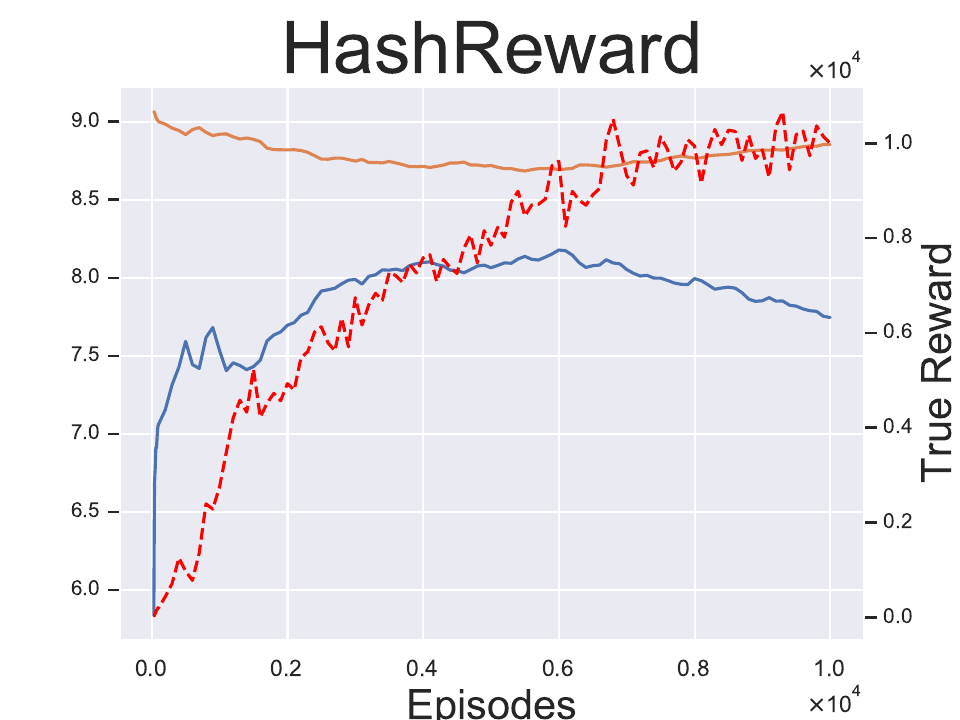}
    \end{minipage}}
\end{figure*}
\begin{figure*}[!t]
    \addtocounter{figure}{1}
    \centering
    \subfigure[Humanoid]{
    \begin{minipage}[b]{0.185\linewidth}
        \includegraphics[ width=1\textwidth]{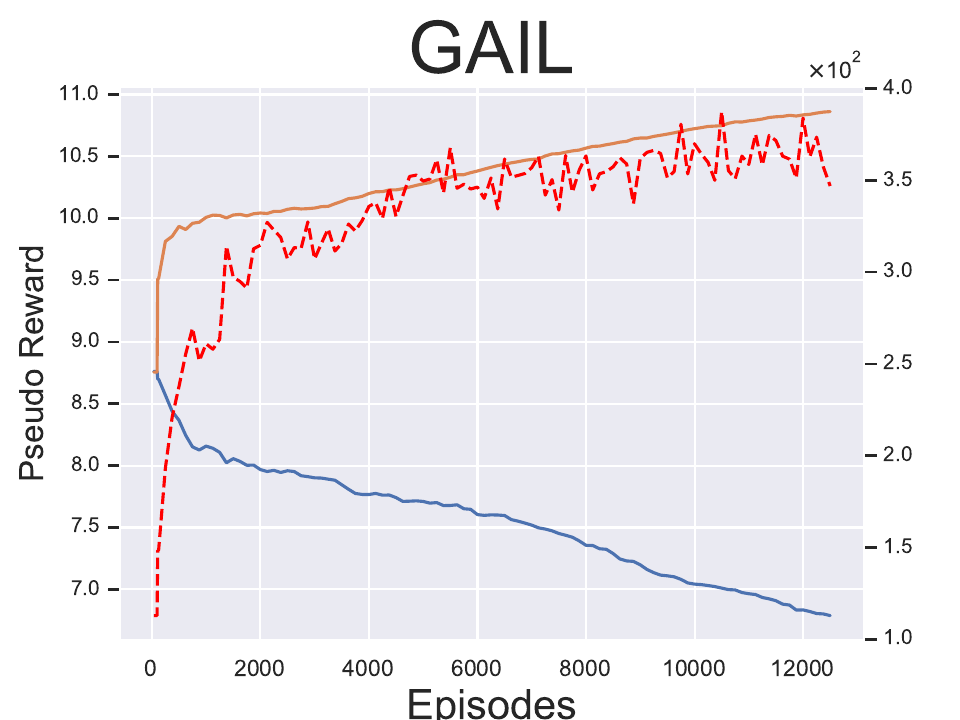}
        \includegraphics[ width=1\textwidth]{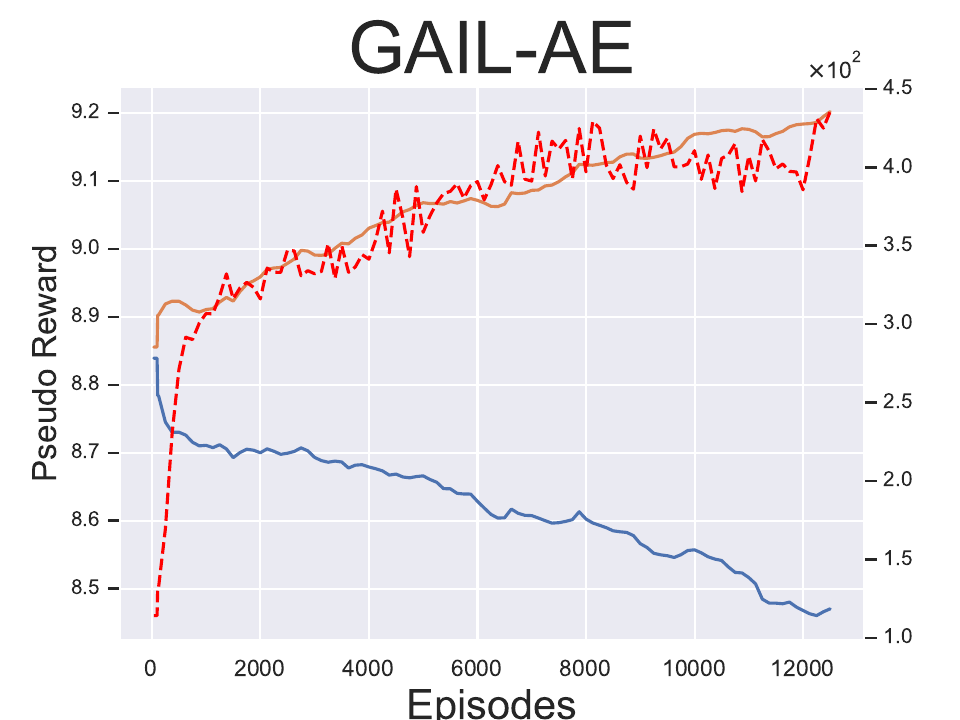}
        \includegraphics[ width=1\textwidth]{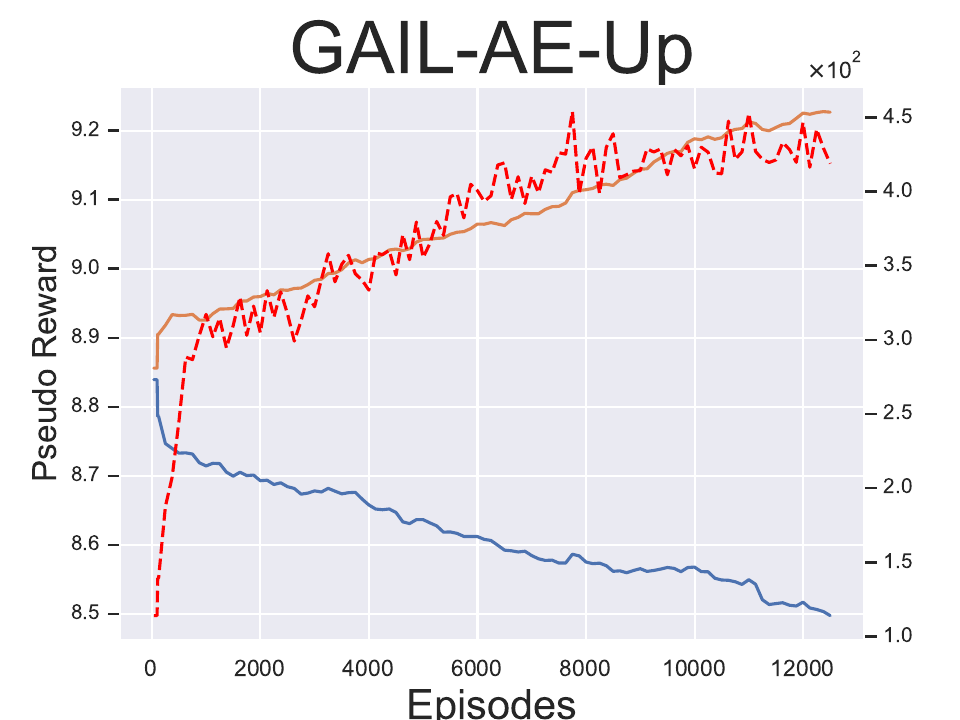}
        \includegraphics[ width=1\textwidth]{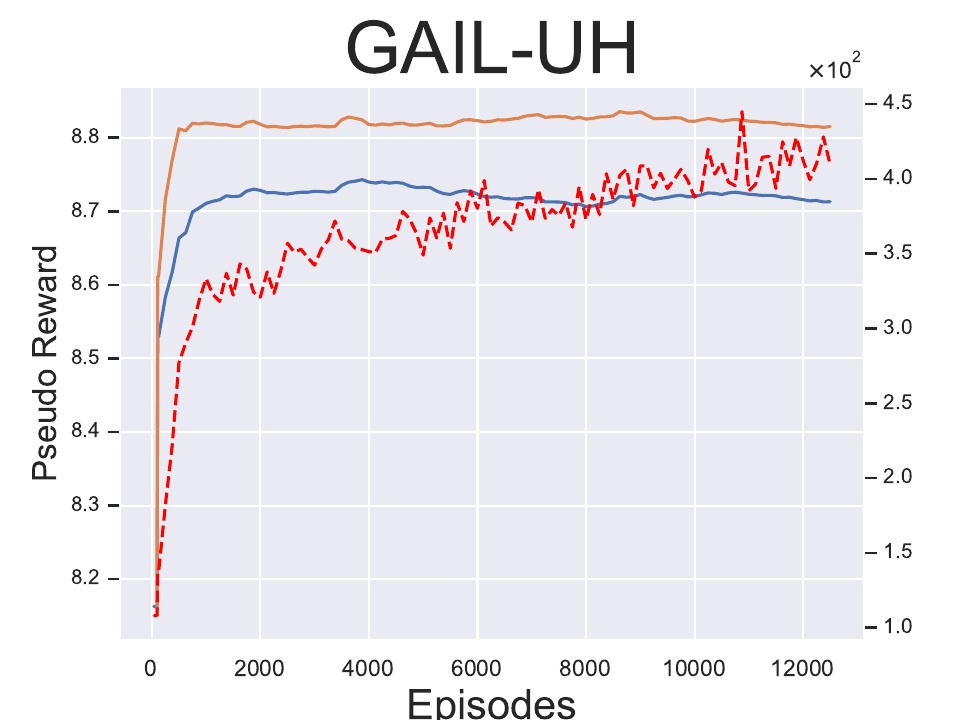}
        \includegraphics[ width=1\textwidth]{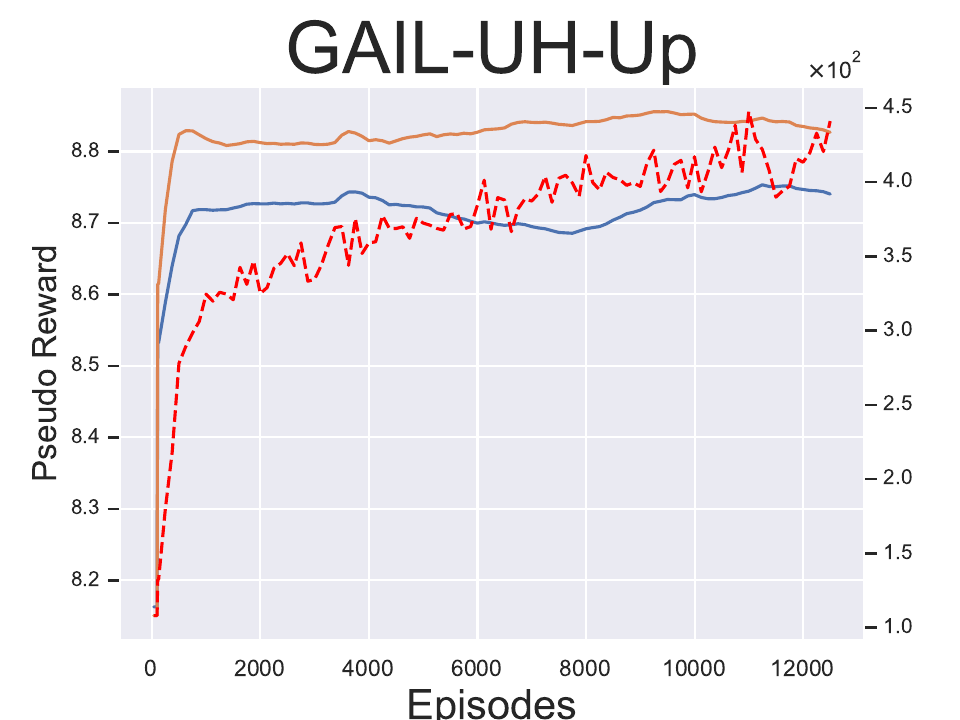}
        \includegraphics[ width=1\textwidth]{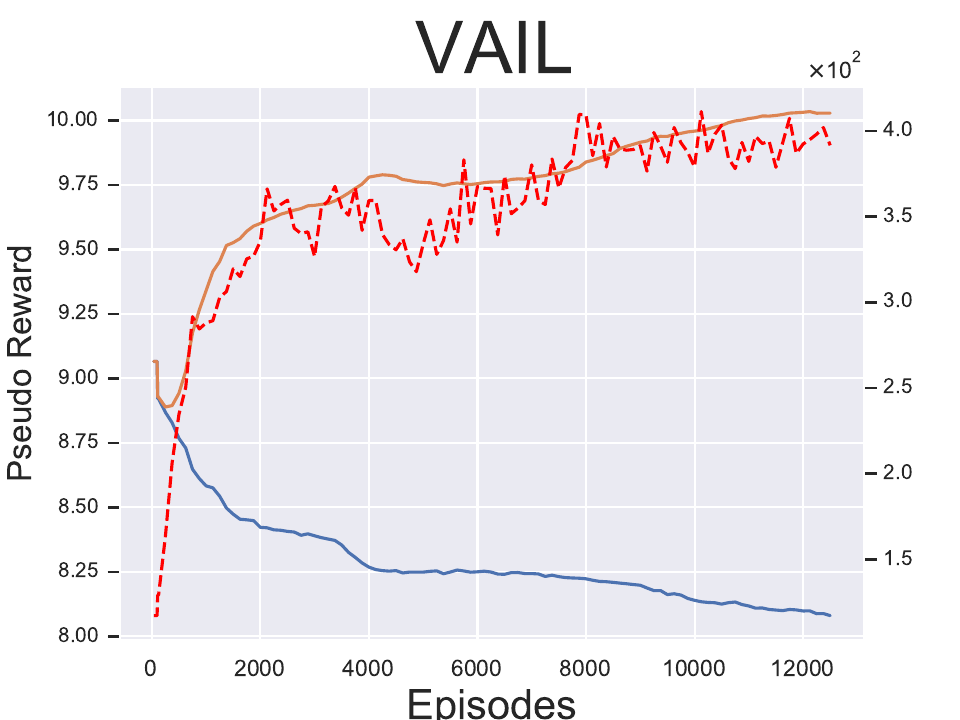}
        \includegraphics[ width=1\textwidth]{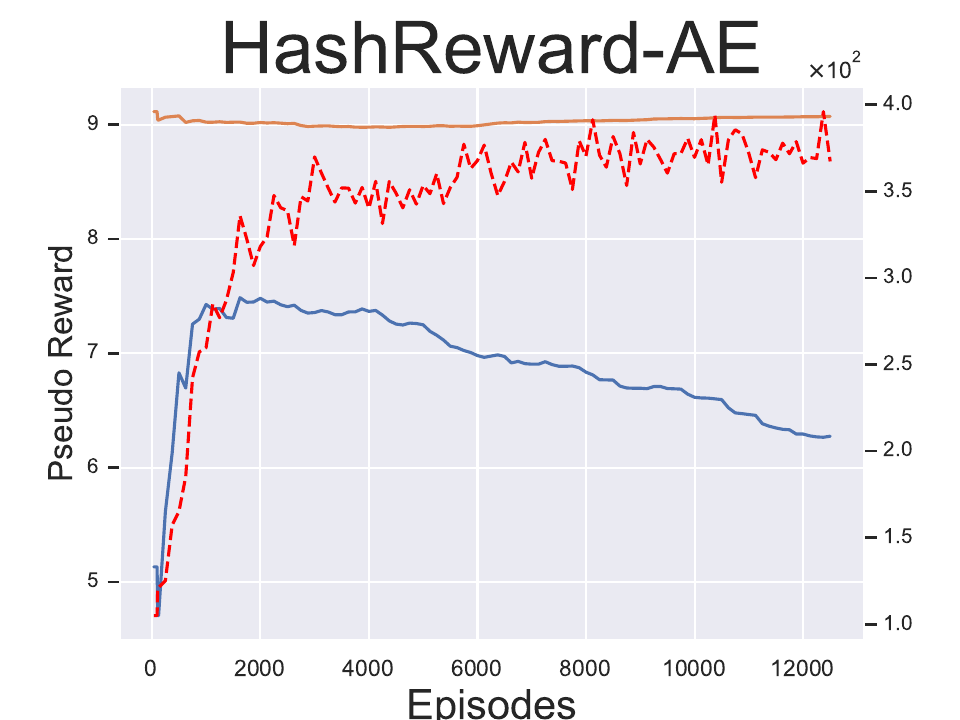}
        \includegraphics[ width=1\textwidth]{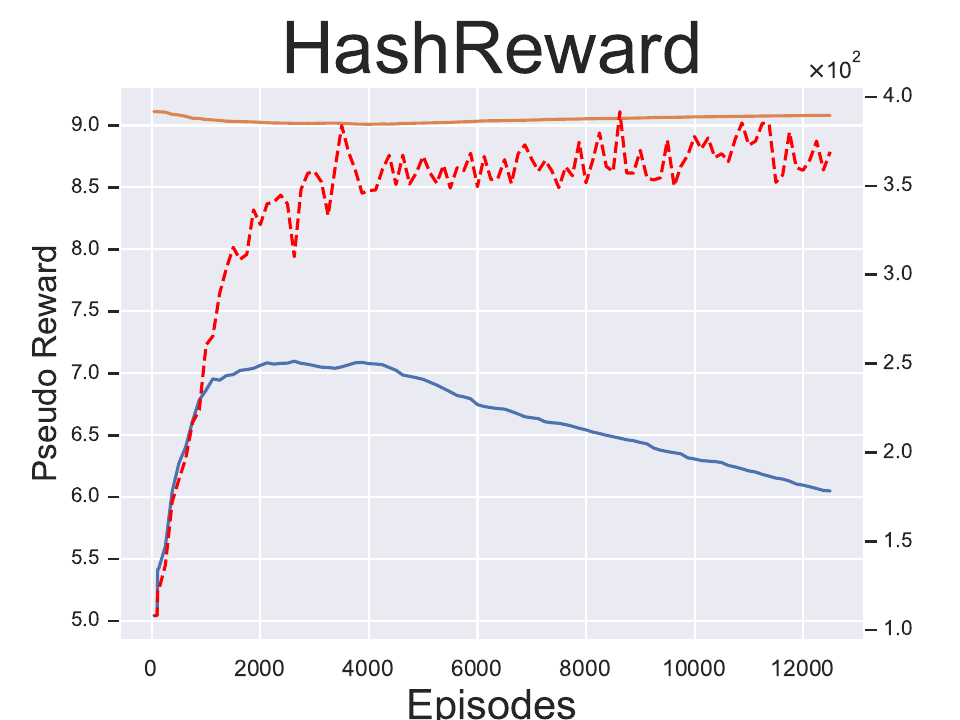}
    \end{minipage}}
    \subfigure[HalfCheetah]{
    \begin{minipage}[b]{0.185\linewidth}
        \includegraphics[ width=1\textwidth]{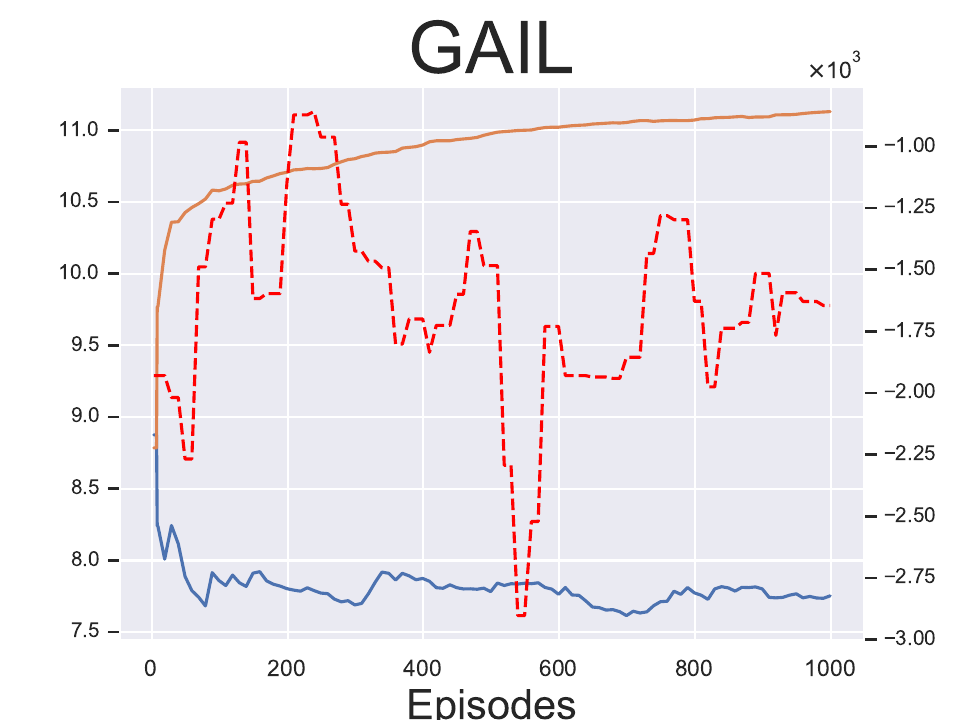}
        \includegraphics[ width=1\textwidth]{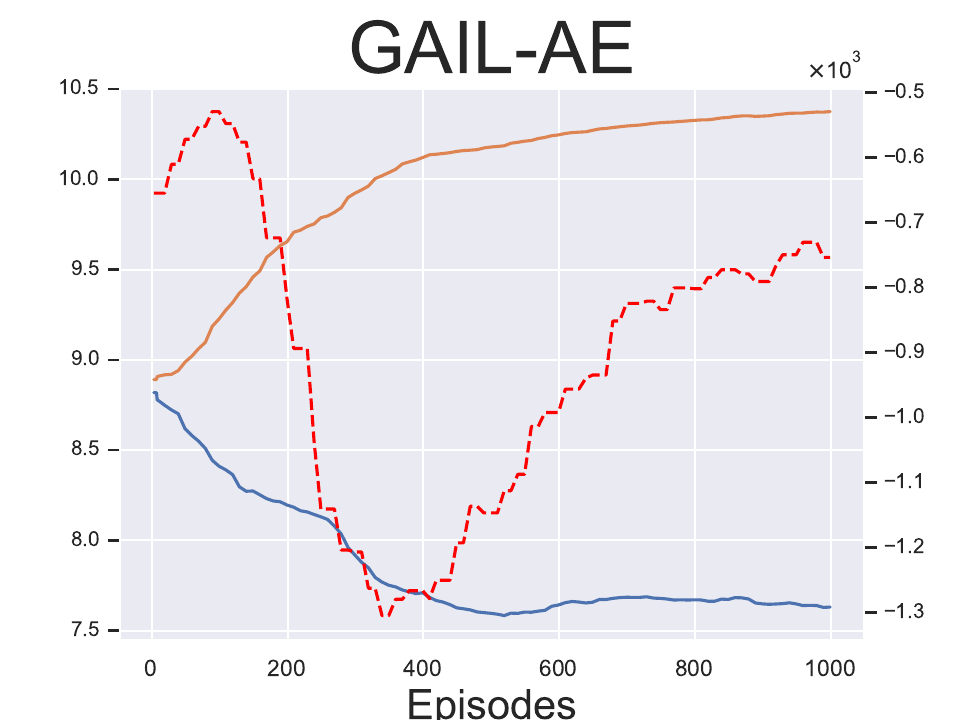}
        \includegraphics[ width=1\textwidth]{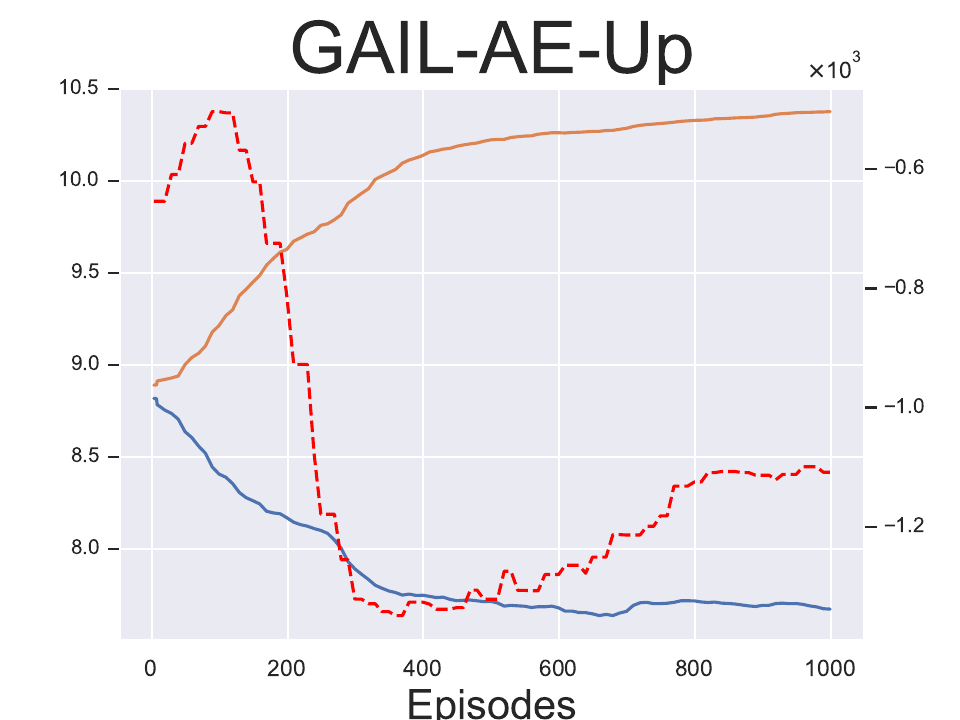}
        \includegraphics[ width=1\textwidth]{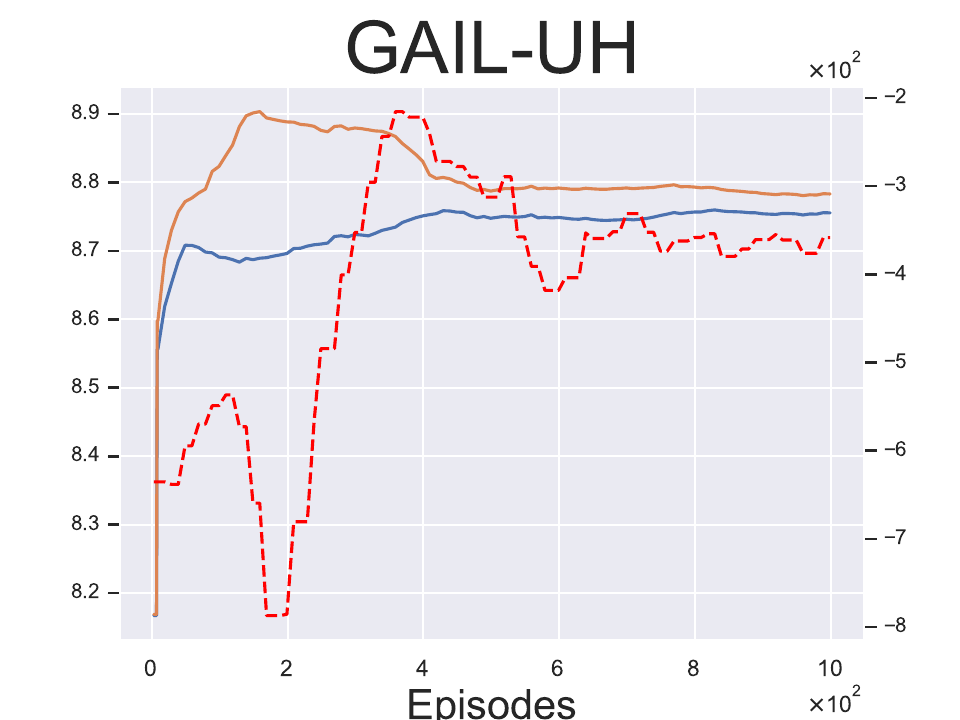}
        \includegraphics[ width=1\textwidth]{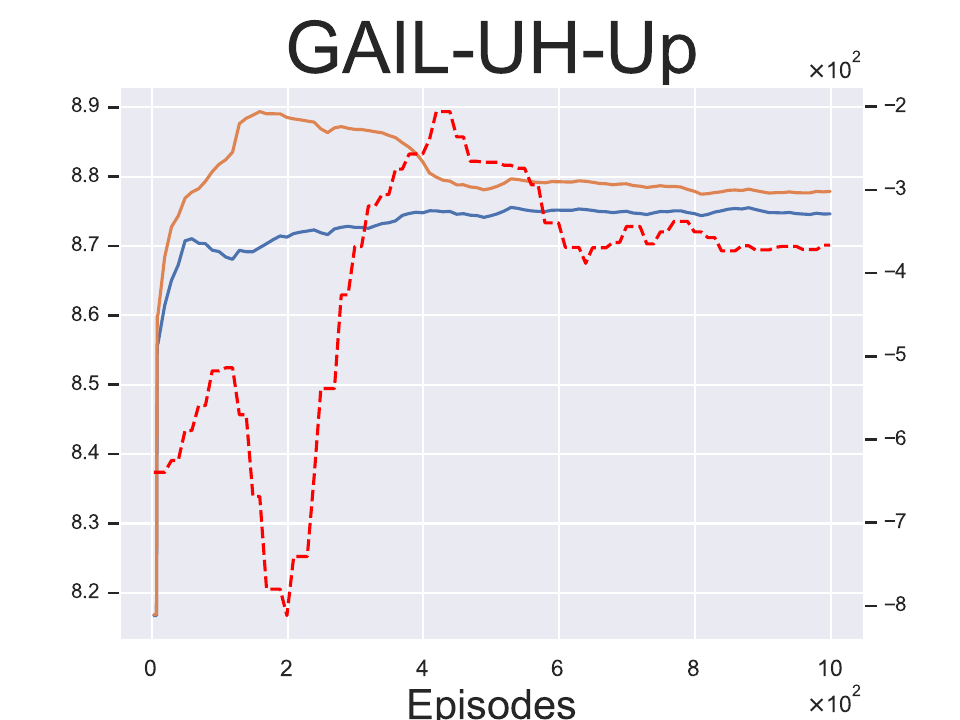}
        \includegraphics[ width=1\textwidth]{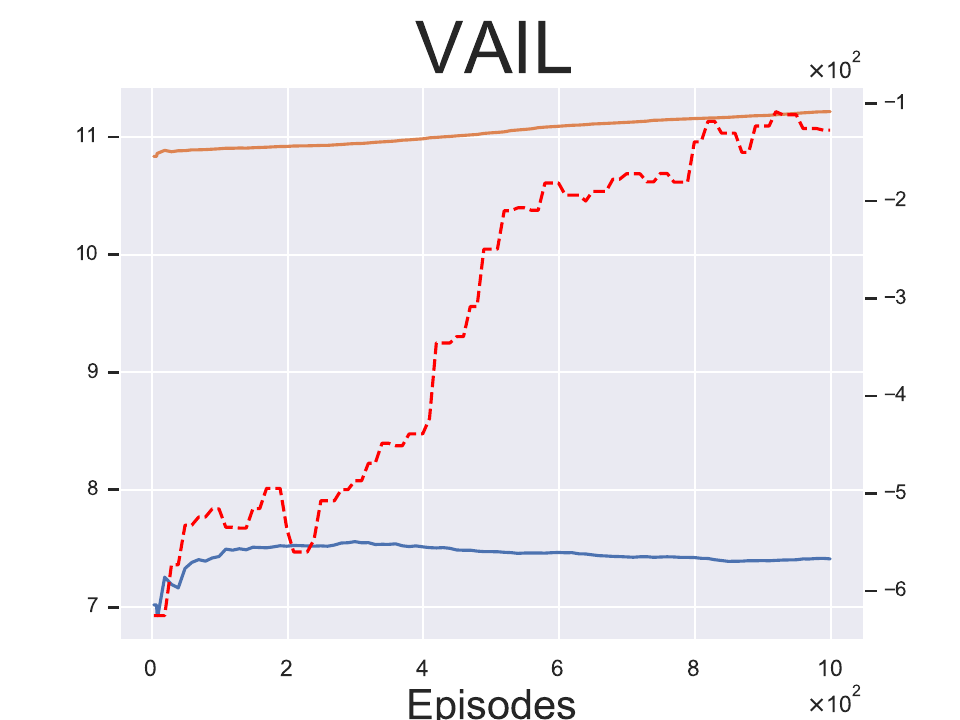}
        \includegraphics[ width=1\textwidth]{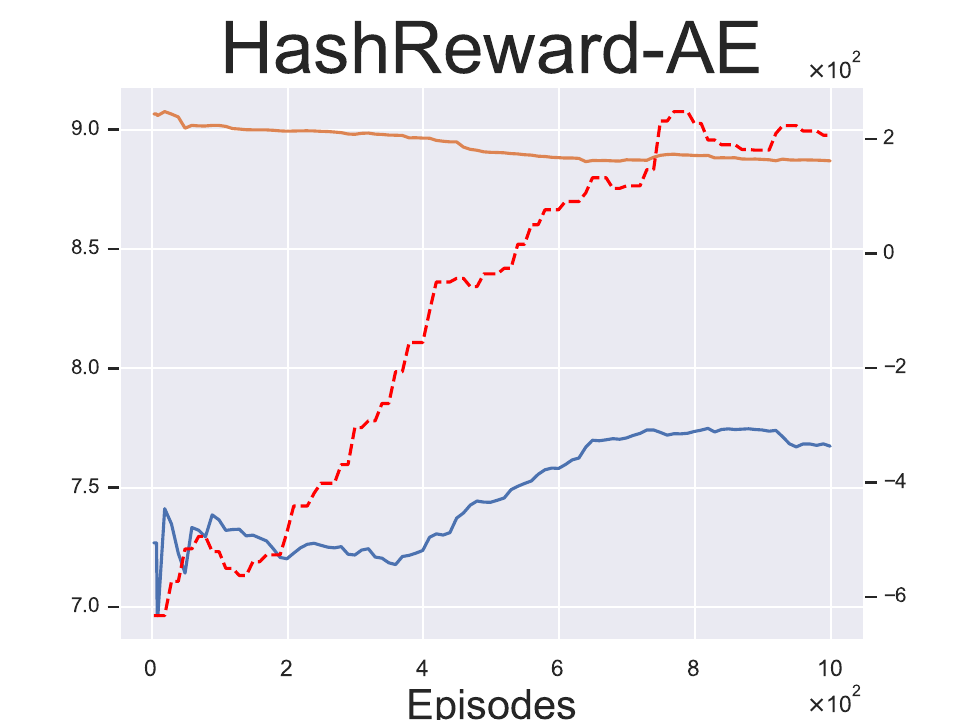}
        \includegraphics[ width=1\textwidth]{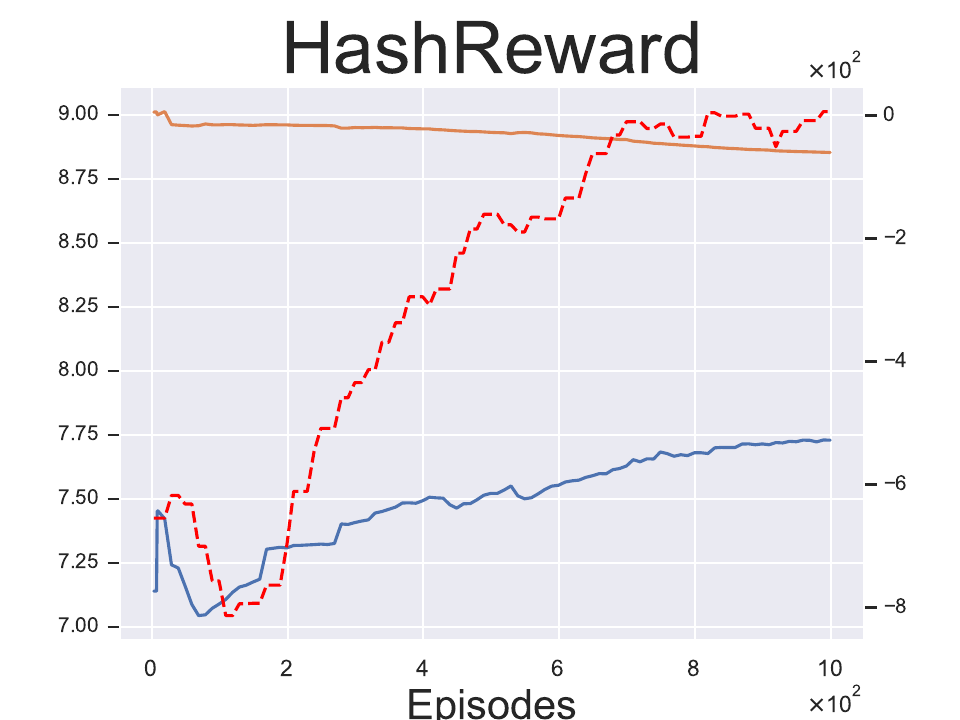}
    \end{minipage}}
    \subfigure[Hopper]{
    \begin{minipage}[b]{0.185\linewidth}
        \includegraphics[ width=1\textwidth]{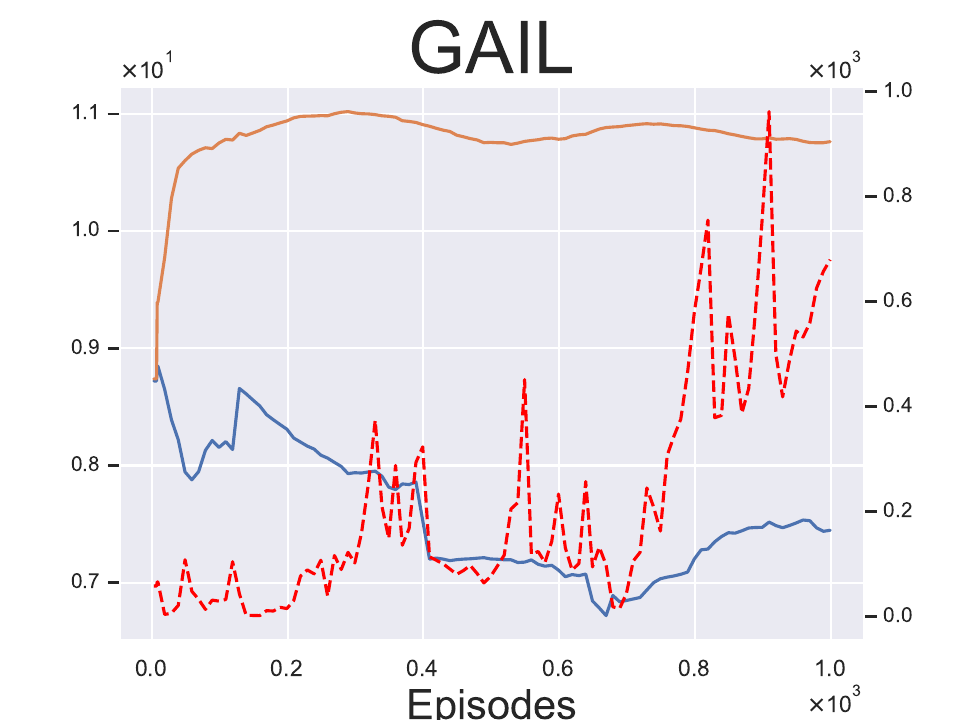}
        \includegraphics[ width=1\textwidth]{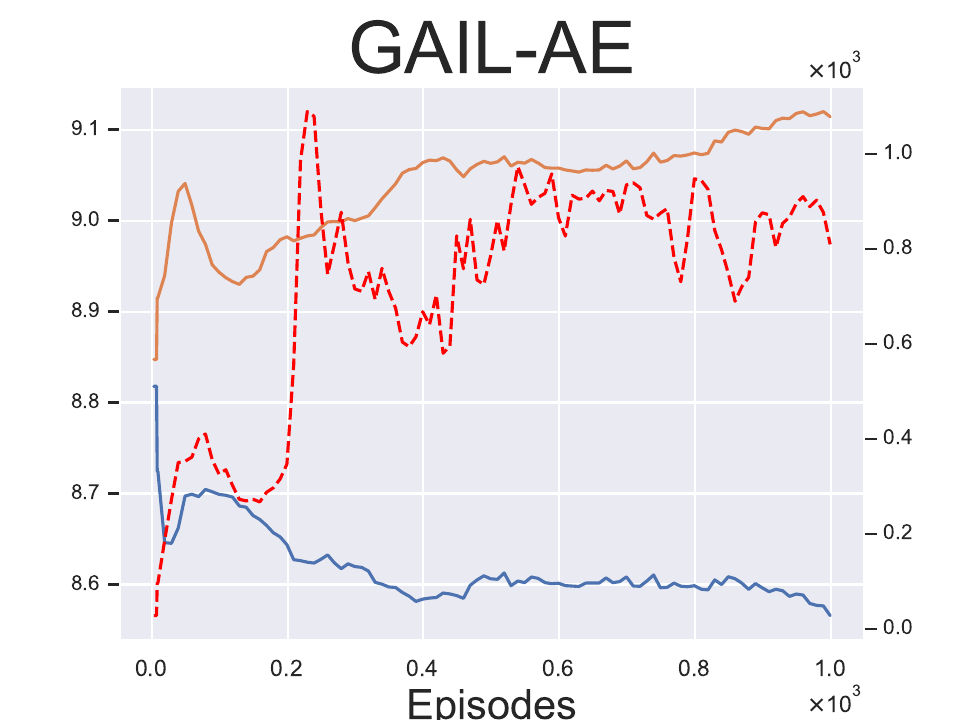}
        \includegraphics[ width=1\textwidth]{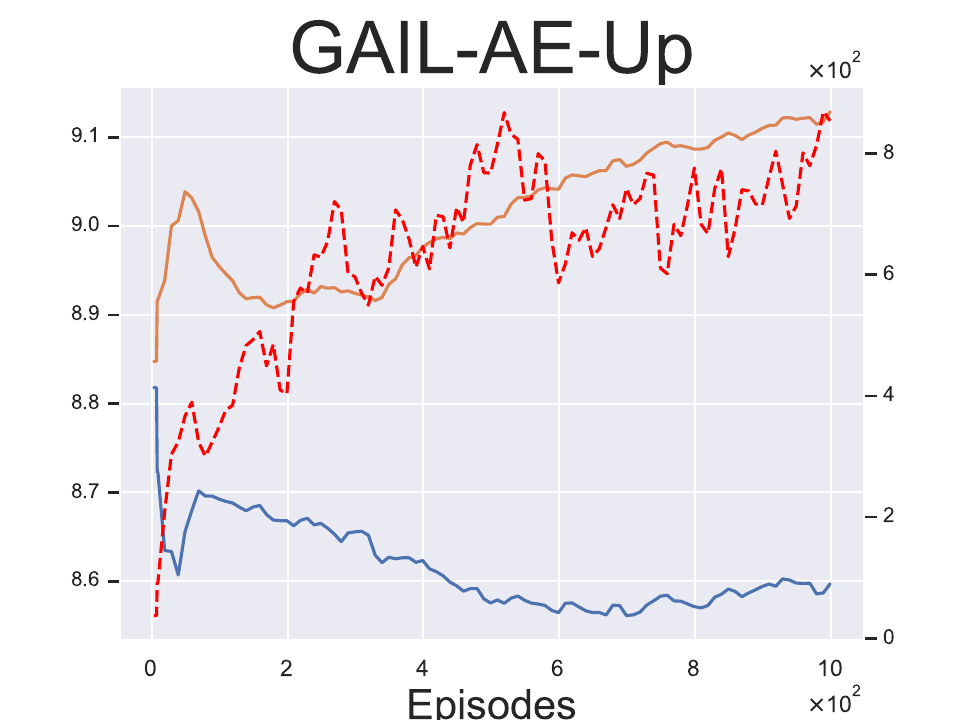}
        \includegraphics[ width=1\textwidth]{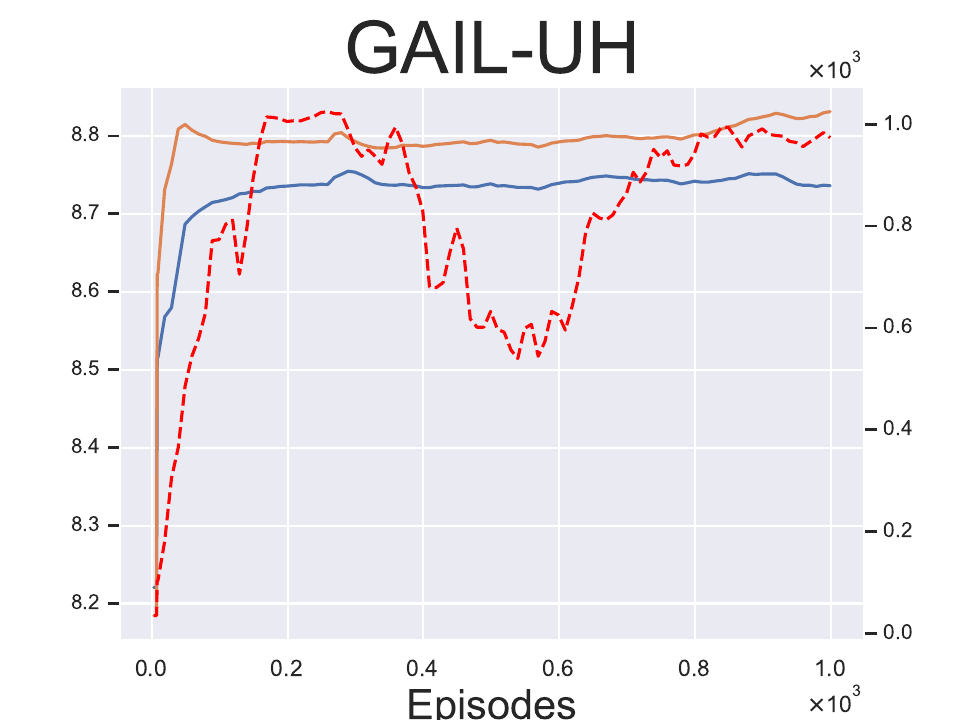}
        \includegraphics[ width=1\textwidth]{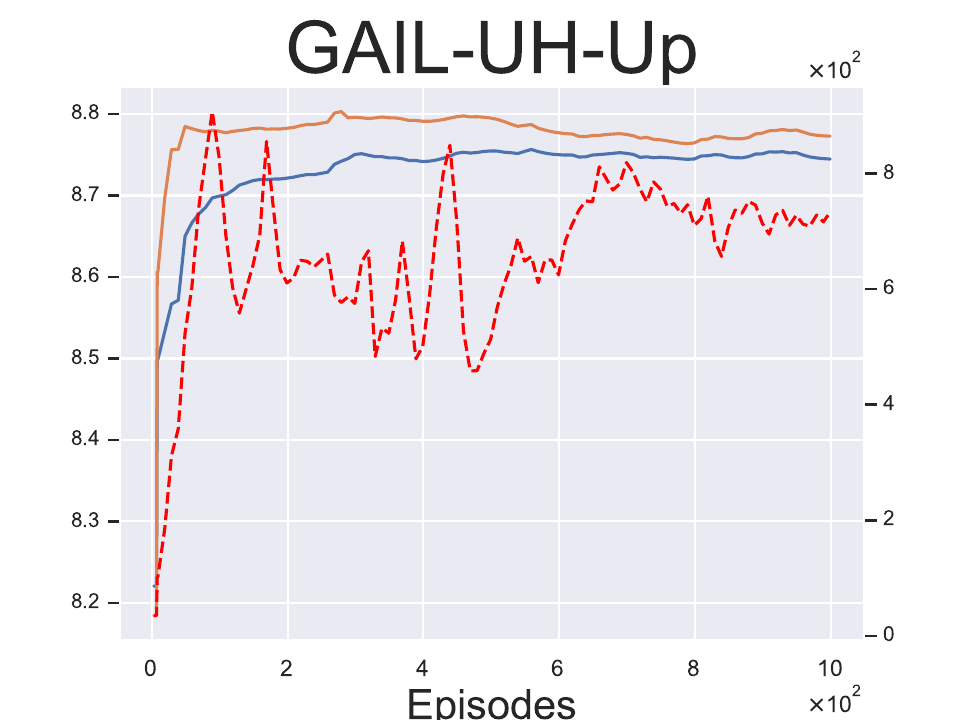}
        \includegraphics[ width=1\textwidth]{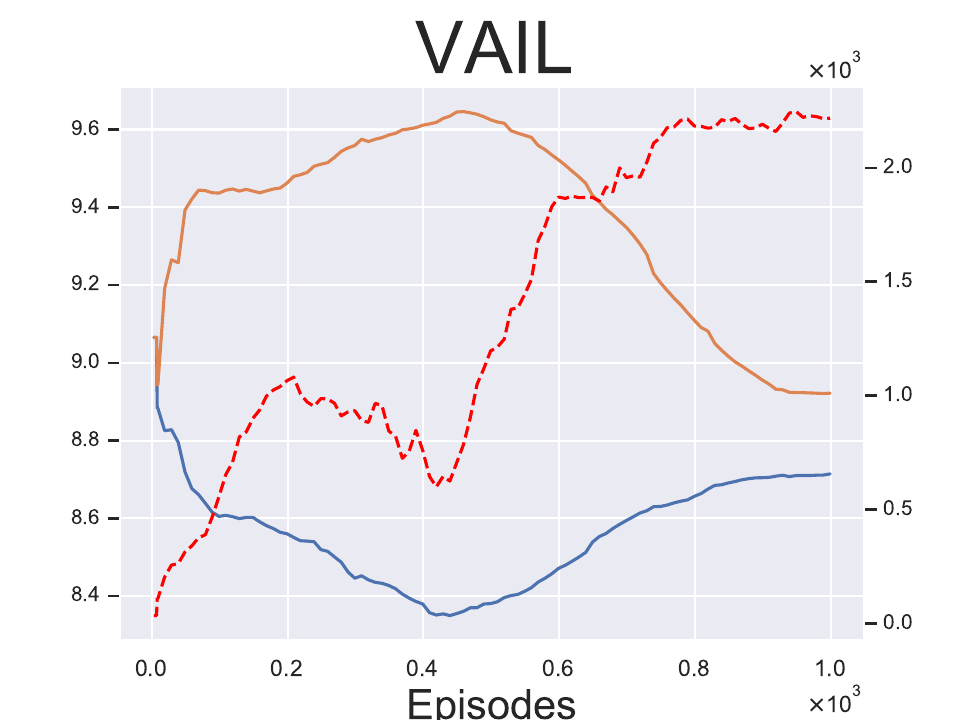}
        \includegraphics[ width=1\textwidth]{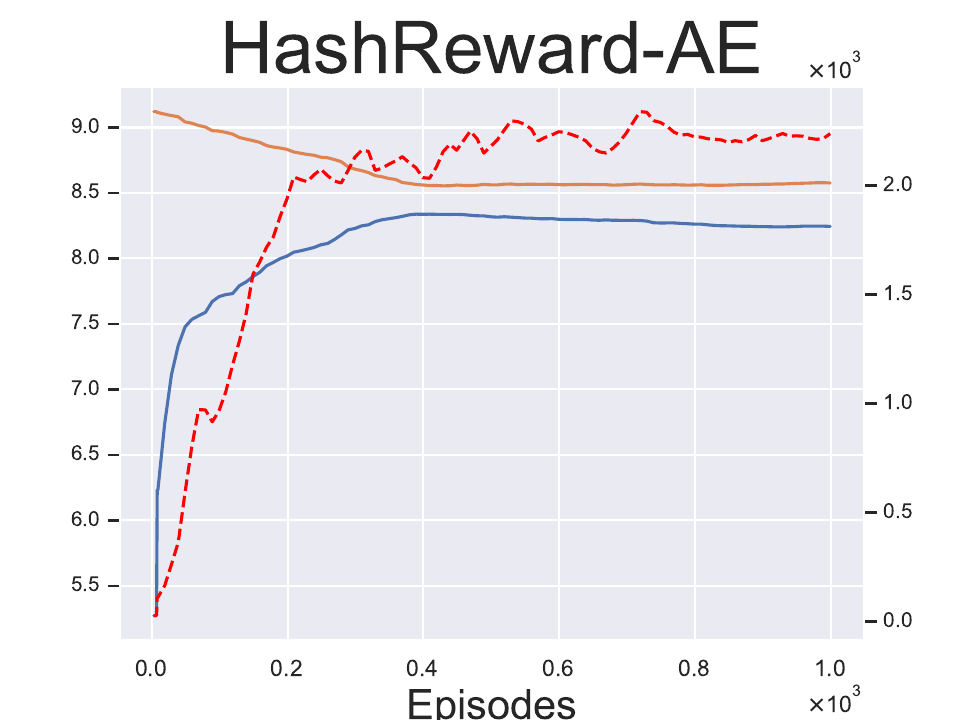}
        \includegraphics[ width=1\textwidth]{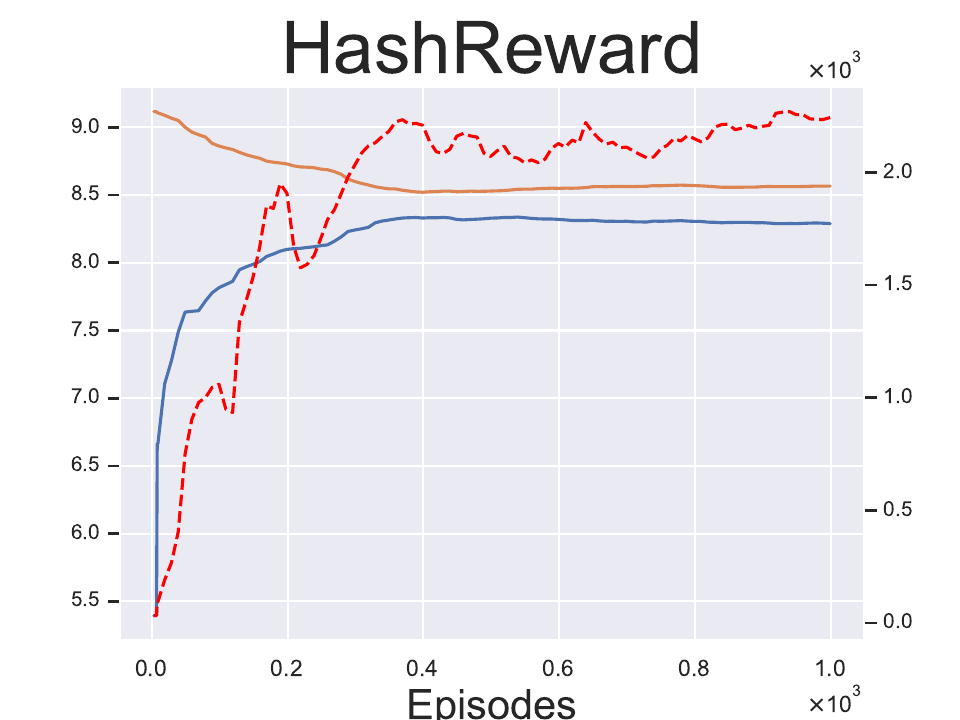}
    \end{minipage}}
    \subfigure[HumanoidStandup]{
    \begin{minipage}[b]{0.185\linewidth}
        \includegraphics[ width=1\textwidth]{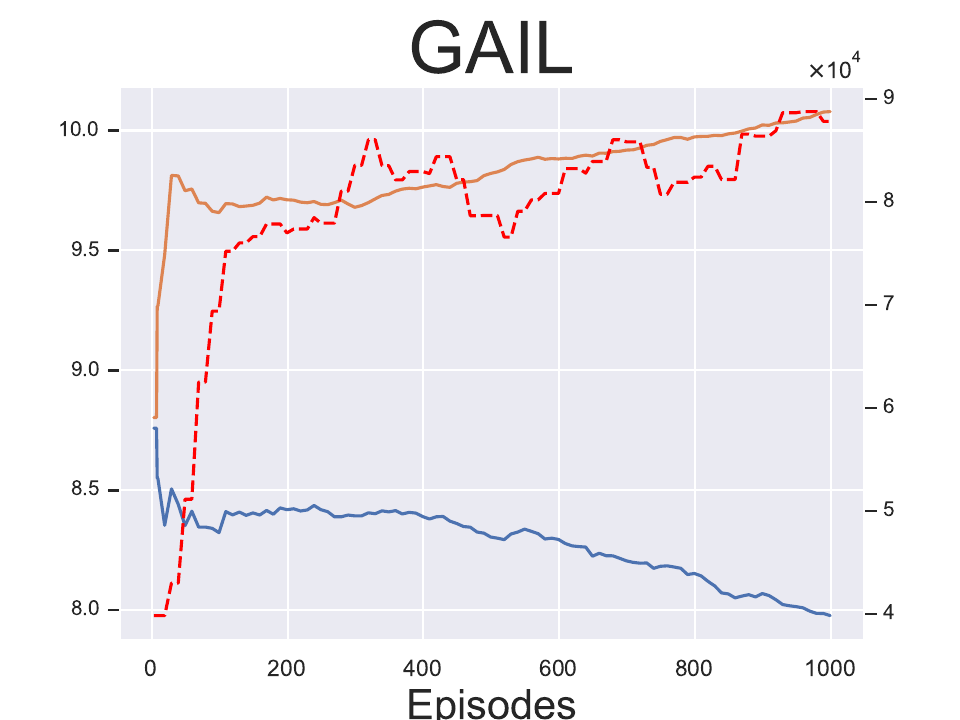}
        \includegraphics[ width=1\textwidth]{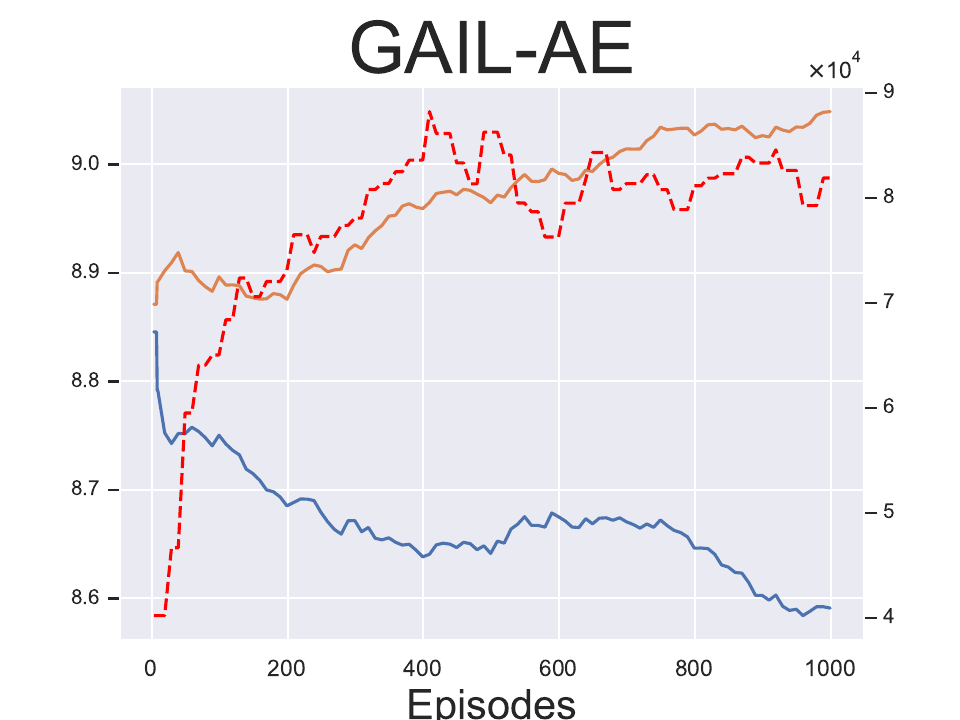}
        \includegraphics[ width=1\textwidth]{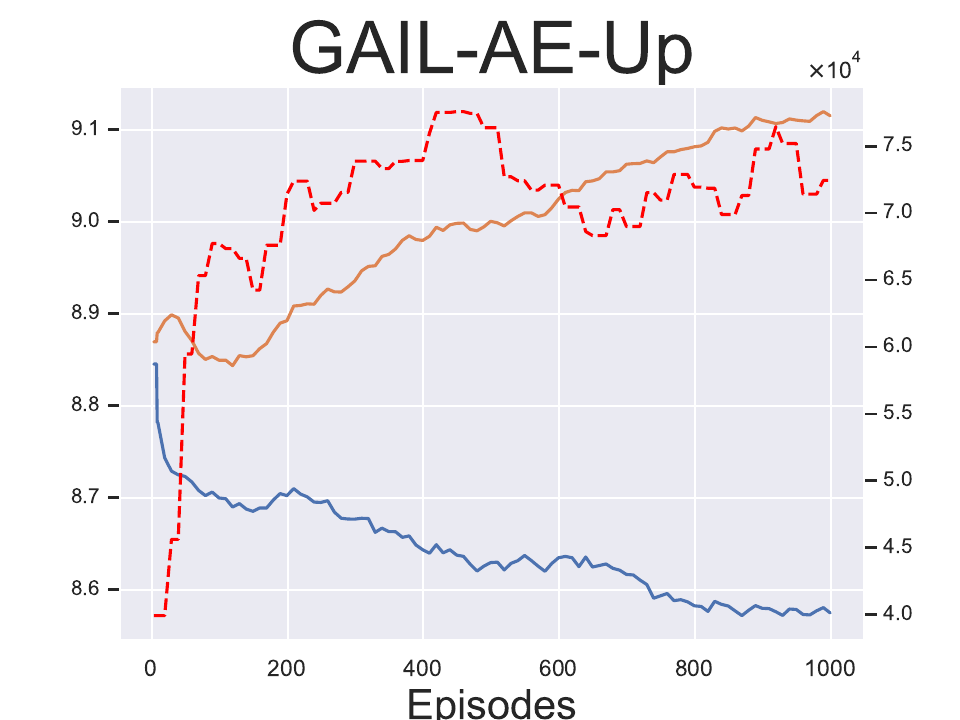}
        \includegraphics[ width=1\textwidth]{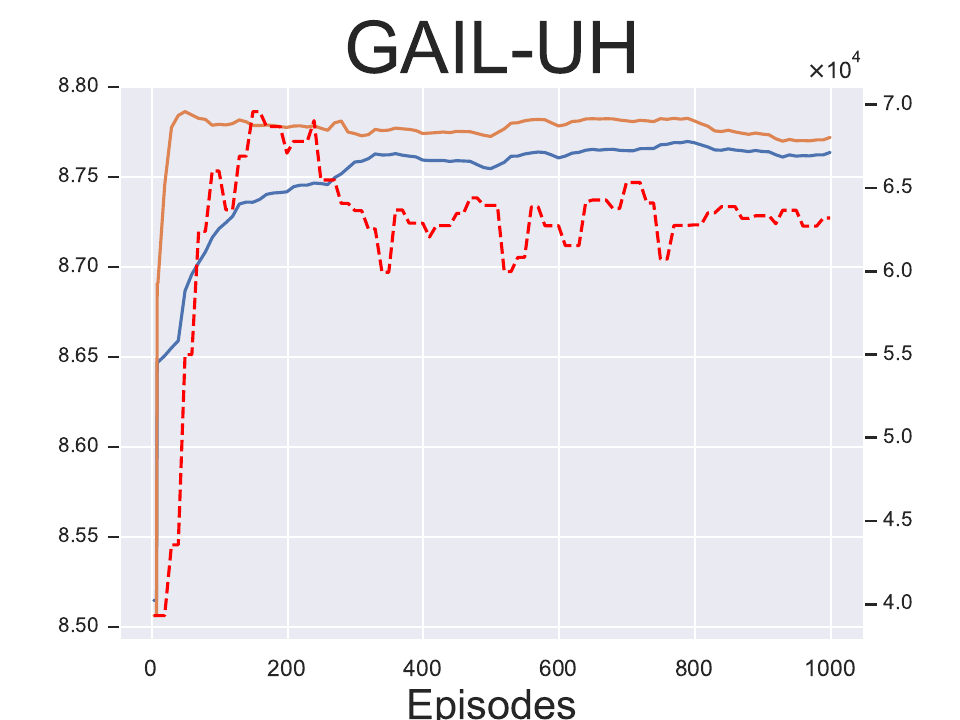}
        \includegraphics[ width=1\textwidth]{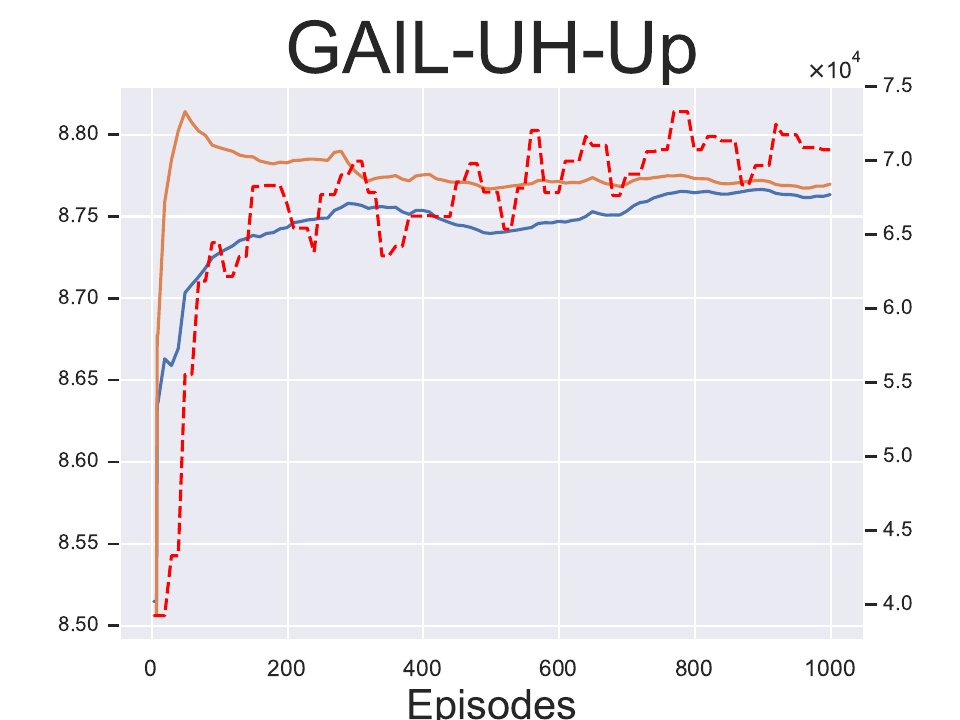}
        \includegraphics[ width=1\textwidth]{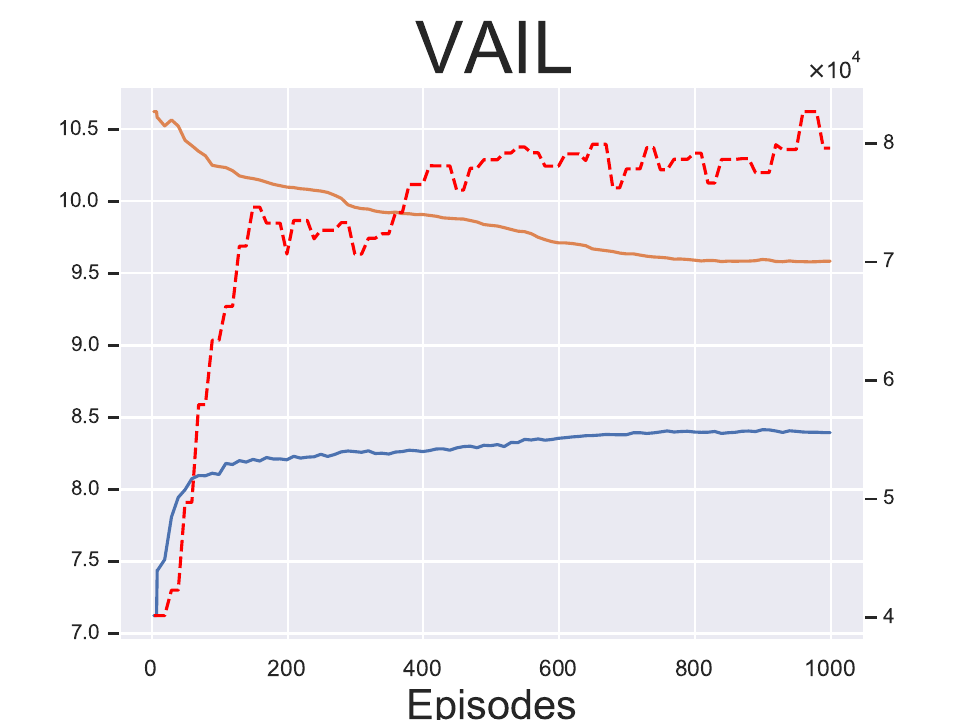}
        \includegraphics[ width=1\textwidth]{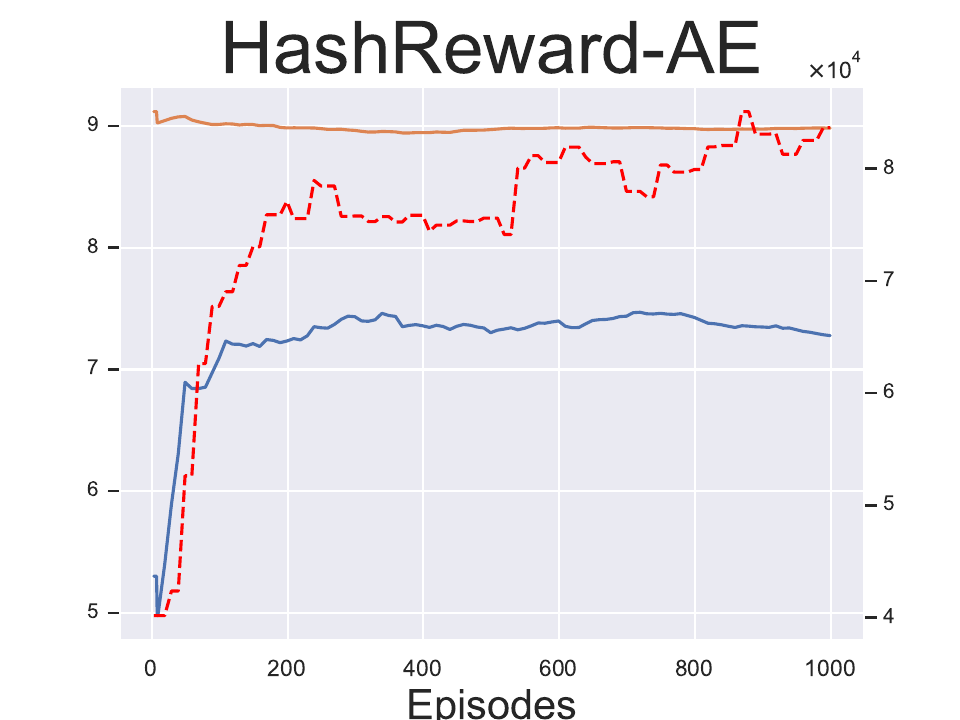}
        \includegraphics[ width=1\textwidth]{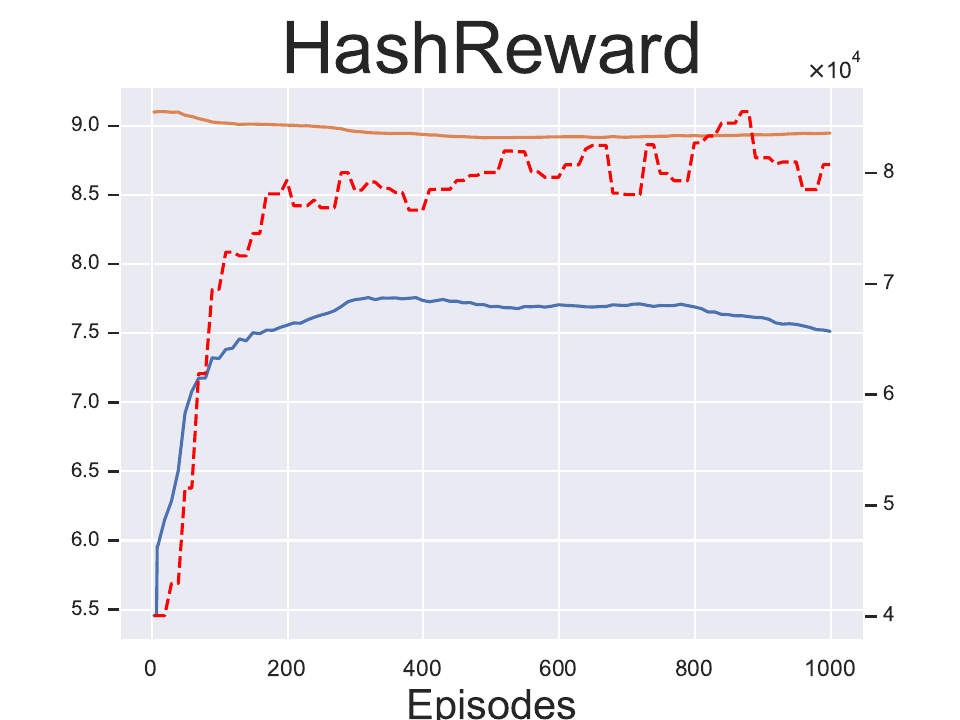}
    \end{minipage}}
    \subfigure[Reacher]{
    \begin{minipage}[b]{0.185\linewidth}
        \includegraphics[ width=1\textwidth]{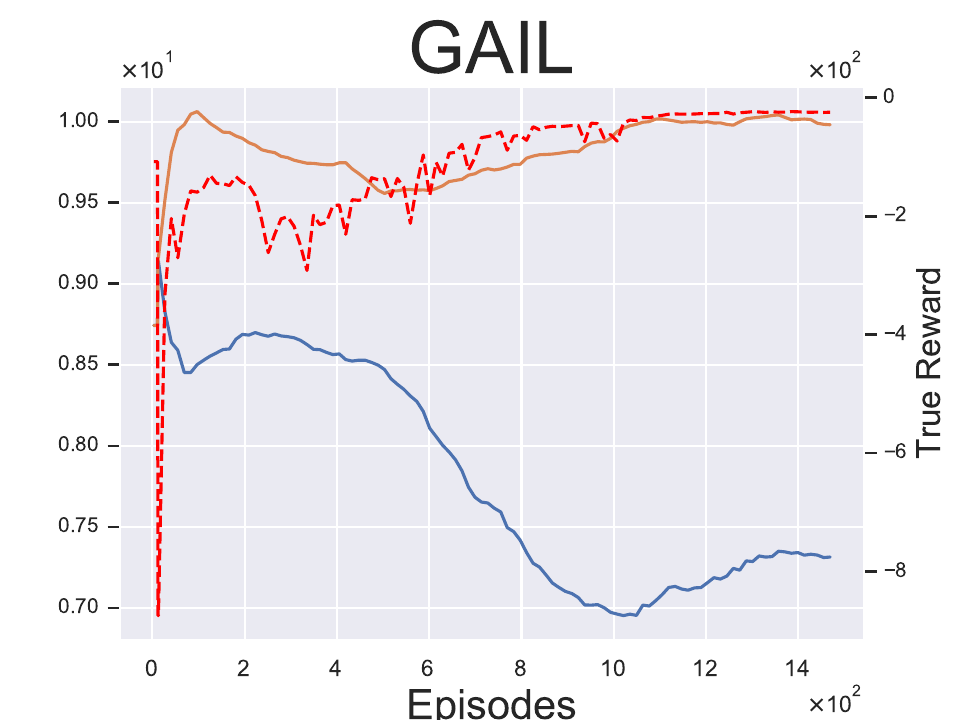}
        \includegraphics[ width=1\textwidth]{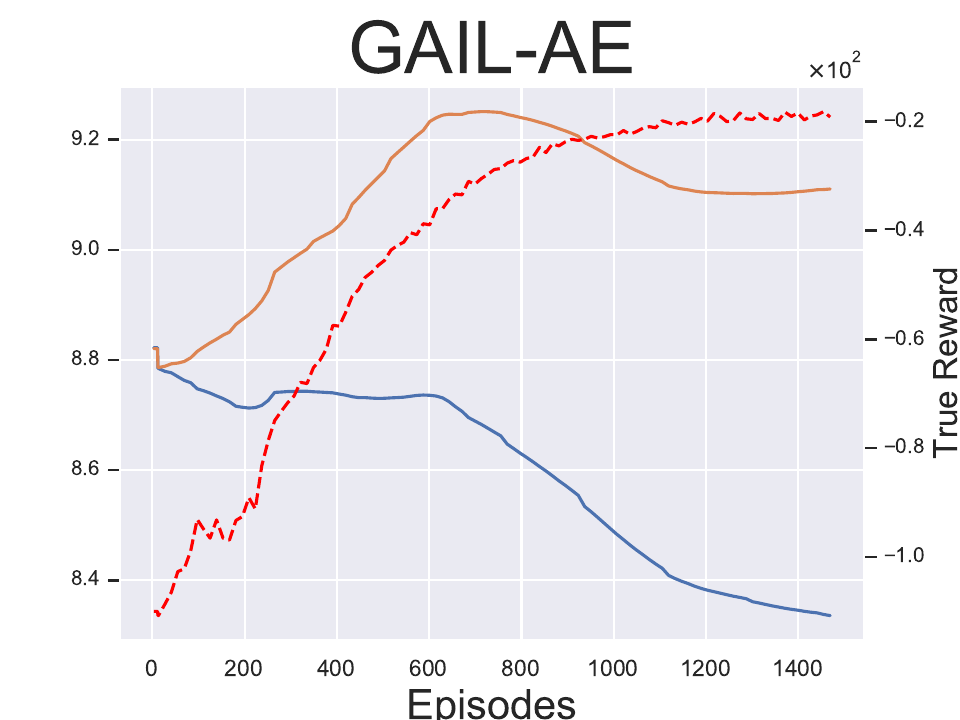}
        \includegraphics[ width=1\textwidth]{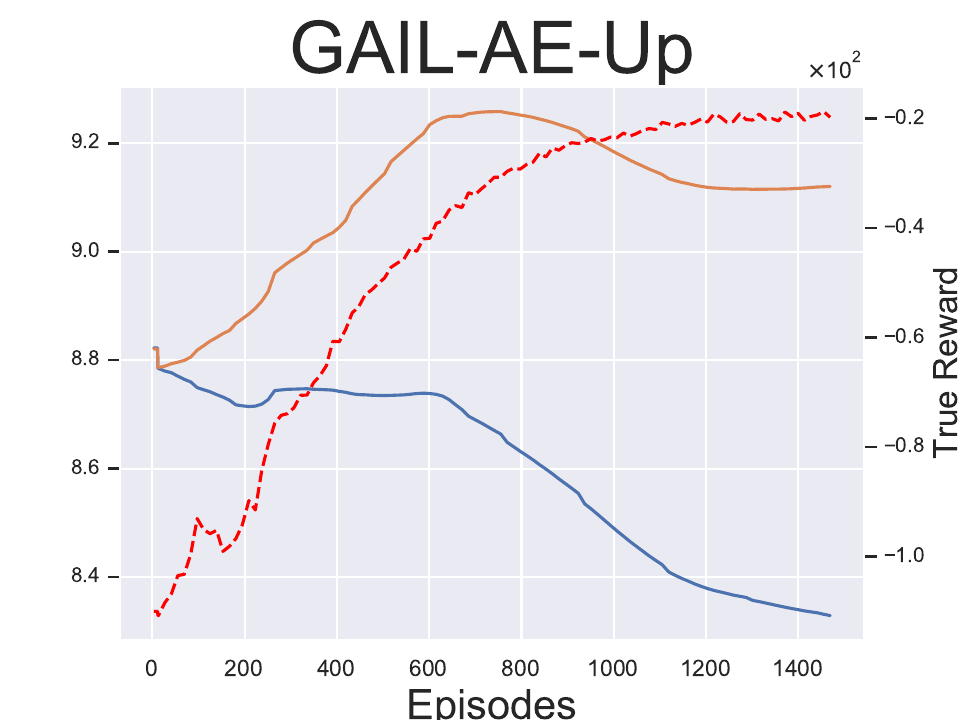}
        \includegraphics[ width=1\textwidth]{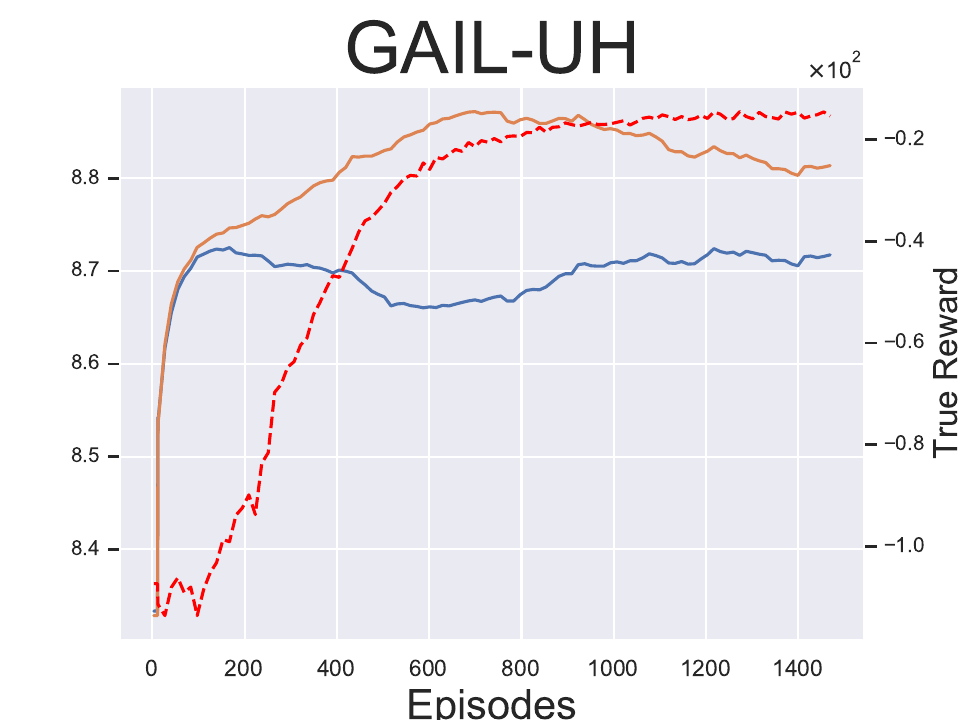}
        \includegraphics[ width=1\textwidth]{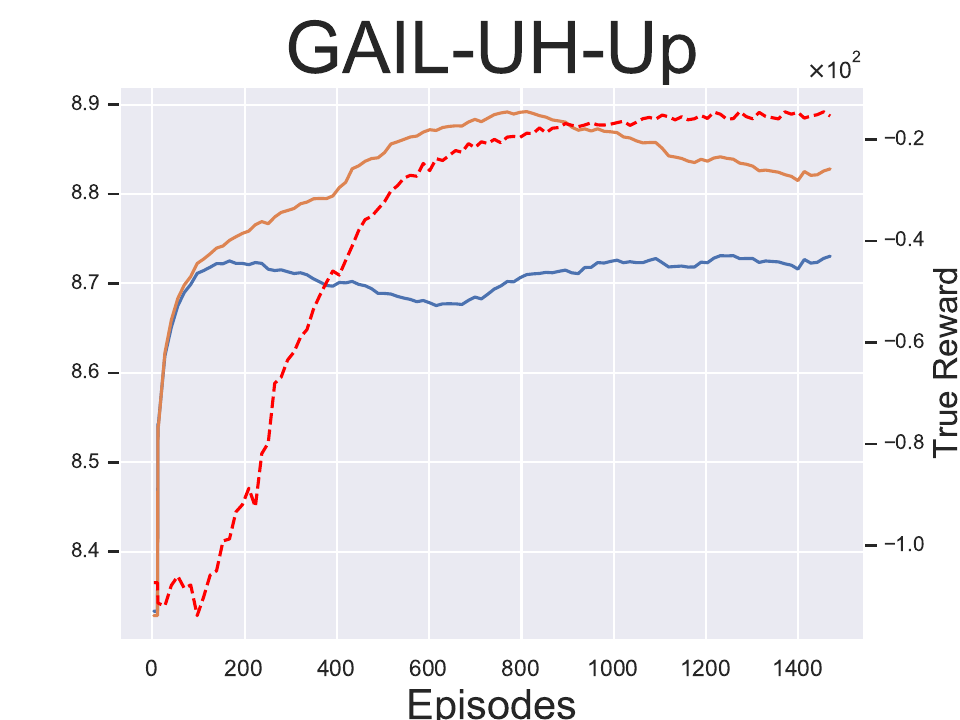}
        \includegraphics[ width=1\textwidth]{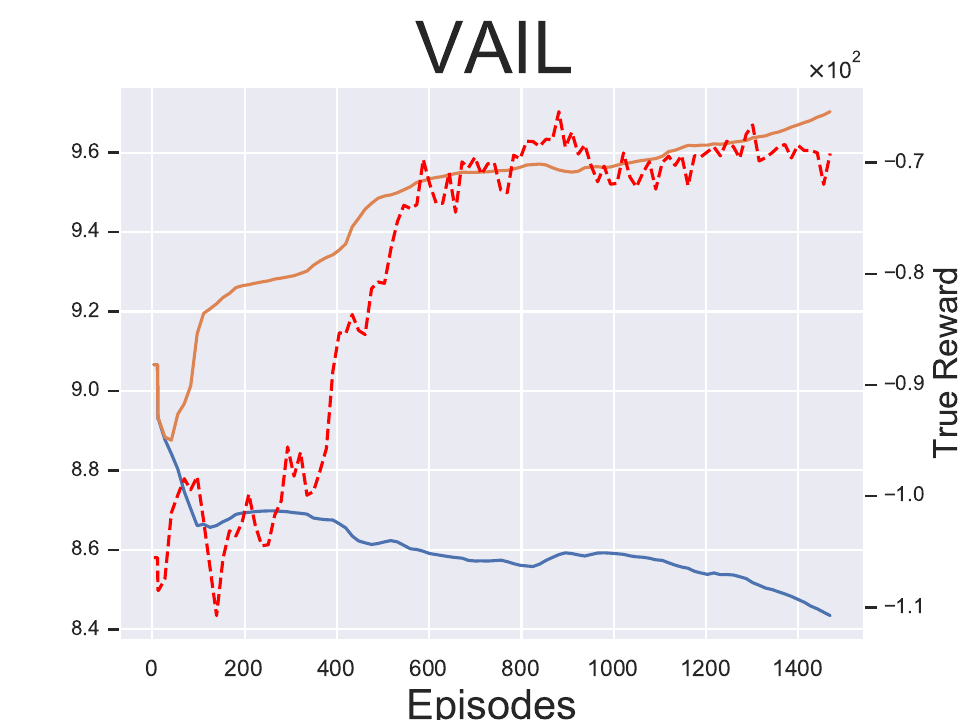}
        \includegraphics[ width=1\textwidth]{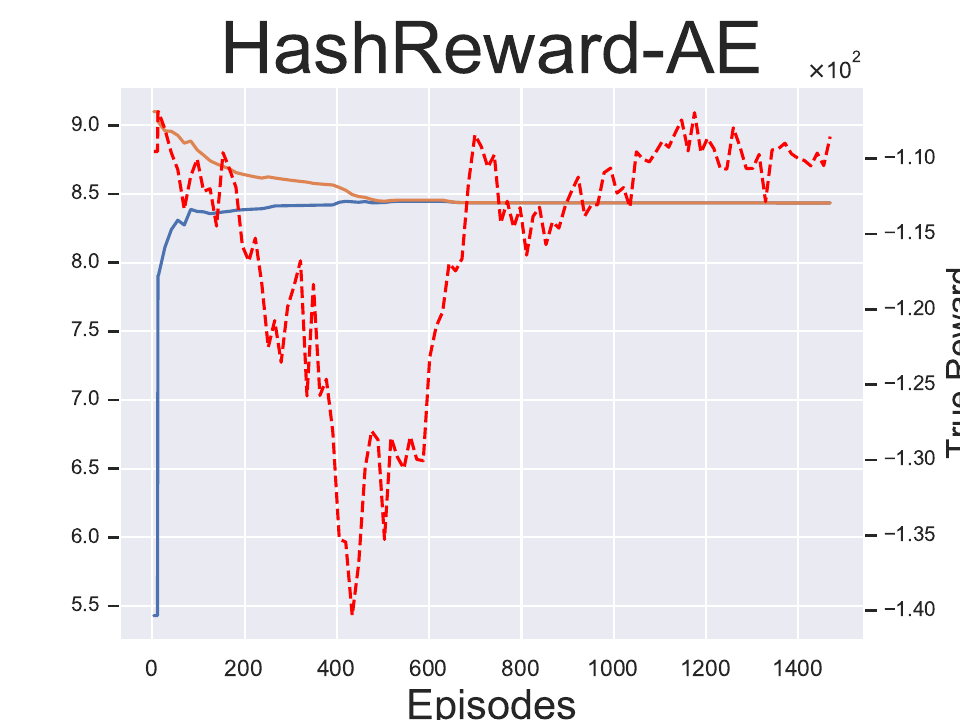}
        \includegraphics[ width=1\textwidth]{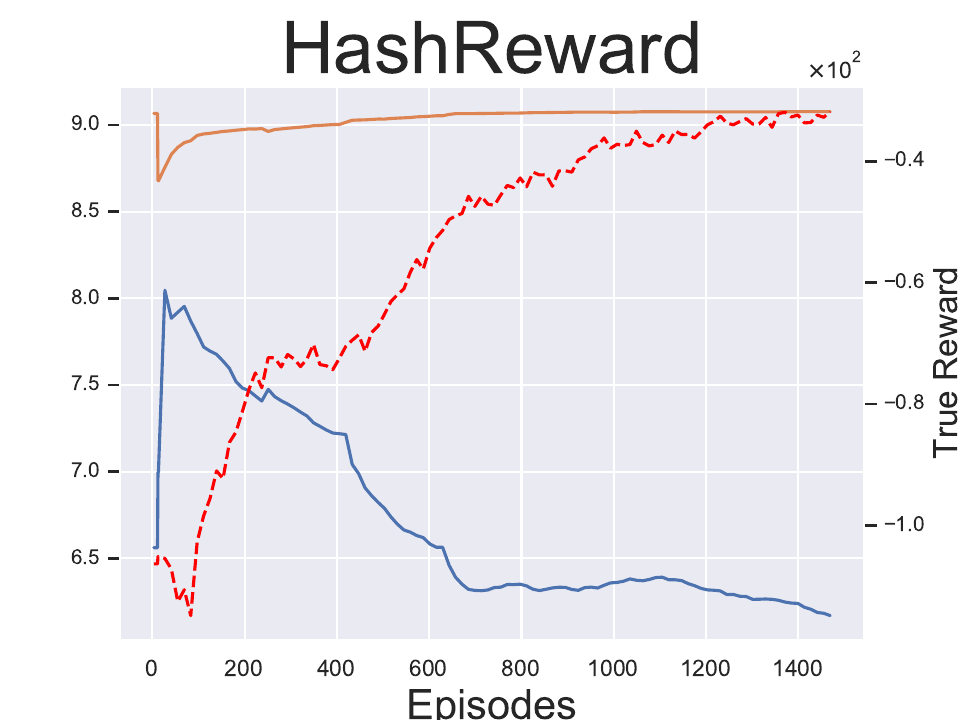}
    \end{minipage}}
\end{figure*}

\end{document}